\documentclass[11pt]{article}

\usepackage[margin=1in]{geometry}
\usepackage{amsthm,amsmath,amsfonts,amssymb, mathtools, thm-restate}
\usepackage[numbers,sort&compress]{natbib}

\usepackage[colorlinks,citecolor=blue,urlcolor=blue,linkcolor=blue]{hyperref}
\usepackage{graphicx}

\usepackage{amssymb}
\usepackage{amsmath}
\usepackage[scr]{rsfso}
\usepackage{mathtools}
\usepackage{mathcommand}
\usepackage{stmaryrd}
\usepackage{csquotes}
\usepackage{parskip}
\usepackage{enumitem}
\usepackage{cleveref}
\usepackage{bbm}
\usepackage{authblk}

\makeatletter
\let\originalref\ref
\let\originaleqref\eqref
\let\originalcref\cref % Included since you are using cleveref!

\renewcommand{\ref}{\@ifstar\originalref\originalref}
\renewcommand{\eqref}{\@ifstar\originaleqref\originaleqref}
\renewcommand{\cref}{\@ifstar\originalcref\originalcref}
\makeatother
\newcommand{\RegisterPairedDelimiter}[3][1]{
\ifnum#1=1 \newcommand{#2}[2][-1]{%
\ifnum##1=-1 #3*{##2}\relax\fi%
\ifnum##1=0 #3{##2}\relax\fi%
\ifnum##1=1 #3[\big]{##2}\relax\fi%
\ifnum##1=2 #3[\Big]{##2}\relax\fi%
\ifnum##1=3 #3[\bigg]{##2}\relax\fi%
\ifnum##1=4 #3[\Bigg]{##2}\relax\fi%
}\fi%
\ifnum#1=2 \newcommand{#2}[3][-1]{%
\ifnum##1=-1 #3*{##2}{##3}\relax\fi%
\ifnum##1=0 #3{##2}{##3}\relax\fi%
\ifnum##1=1 #3[\big]{##2}{##3}\relax\fi%
\ifnum##1=2 #3[\Big]{##2}{##3}\relax\fi%
\ifnum##1=3 #3[\bigg]{##2}{##3}\relax\fi%
\ifnum##1=4 #3[\Bigg]{##2}{##3}\relax\fi%
}\fi%
}

\DeclarePairedDelimiter{\deldelim}{(}{)}
\RegisterPairedDelimiter{\del}{\deldelim}

\makeatletter
\def\£#1\£{%
  \in@{&}{#1}%
  \ifin@%
    \begin{align}#1\end{align}%
  \else%
    \in@{\\}{#1}%
    \ifin@%
      \begin{multline}#1\end{multline}%
    \else%
      \begin{equation}#1\end{equation}%
    \fi%
  \fi%
}
\makeatother

\makeatletter
\def\[#1\]{%
  \in@{&}{#1}%
  \ifin@%
    \begin{align*}#1\end{align*}%
  \else%
    \in@{\\}{#1}%
    \ifin@%
      \begin{multline*}#1\end{multline*}%
    \else%
      \begin{equation*}#1\end{equation*}%
    \fi%
  \fi%
}
\makeatother

\DeclareMathOperator*{\argmin}{\arg\min} % argmin
\DeclareMathOperator{\supp}{supp} % support

\DeclareMathOperator{\Var}{Var}

\DeclareMathOperator{\spn}{span}
\DeclareMathOperator{\polylog}{\mathrm{polylog}}

\DeclareMathOperator{\KL}{\mathsf{KL}}
\DeclareMathOperator{\TV}{\mathsf{TV}}
\DeclareMathOperator{\rP}{\mathbb{P}}
\DeclareMathOperator{\Id}{\mathrm{Id}}

\DeclareMathOperator{\Vol}{\mathrm{Vol}}
\DeclareMathOperator{\pr}{\operatorname{pr}}
\DeclareMathOperator{\clip}{\operatorname{clip}}
\DeclareMathOperator{\dist}{\operatorname{dist}}
\DeclareMathOperator{\diam}{\operatorname{diam}}
\DeclareMathOperator{\relu}{\operatorname{ReLU}}
\DeclareMathOperator{\Ker}{\operatorname{Ker}}

\def\abs#1{\left|#1\right|}
\def\norm#1{\left\lVert#1\right\rVert}
\def\curly#1{\{#1\}}
\def\square#1{\left[#1\right]}

\newcommand{\grad}{\nabla}
\newcommand{\R}{\mathbb{R}} % real numbers
\hypersetup{colorlinks,citecolor=blue}

\newcommand{\expectation}{\mathbb{E}}
\newcommand{\cA}{\mathcal{A}}

\newcommand{\cE}{\mathcal{E}}
\newcommand{\cG}{\mathcal{G}}
\newcommand{\cF}{\mathcal{F}}
\newcommand{\cH}{\mathcal{H}}
\newcommand{\cL}{\mathcal{L}}

\newcommand{\cN}{\mathcal{N}}
\newcommand{\cP}{\mathcal{P}}
\newcommand{\cR}{\mathcal{R}}
\newcommand{\cS}{\mathcal{S}}

\newcommand{\cV}{\mathcal{V}}
\newcommand{\cX}{\mathcal{X}}
\newcommand{\cY}{\mathcal{Y}}
\newcommand{\ind}{\mathbbm{1}}
\newcommand{\innerproduct}[2]{\left\langle #1, #2 \right\rangle}

\renewcommand{\Im}{\operatorname{Im}}

\numberwithin{equation}{section}

\theoremstyle{plain}

\newtheorem{theorem}{Theorem}[section]
\newtheorem{lemma}[theorem]{Lemma}
\newtheorem{corollary}[theorem]{Corollary}
\newtheorem{proposition}[theorem]{Proposition}
\theoremstyle{remark}
\newtheorem{definition}{Definition}
\newtheorem{example}{Example}
\newtheorem*{remark}{Remark}

\declaretheorem[style=remark]{Assumption}

\begin{document}

\title{Convergence of Diffusion Models Under the Manifold Hypothesis in High-Dimensions}

\author[1]{Iskander Azangulov\thanks{\texttt{iskander.azangulov@spc.ox.ac.uk}}}
\author[2]{Judith Rousseau\thanks{\texttt{rousseau@ceremade.dauphine.fr}}}
\author[1]{George Deligiannidis\thanks{\texttt{george.deligiannidis@stats.ox.ac.uk}}}
\affil[1]{Department of Statistics, University of Oxford}
\affil[2]{CEREMADE, CNRS, Universit\'e Paris-Dauphine, PSL University}
\date{}

\maketitle

\begin{abstract}
Denoising Diffusion Probabilistic Models (DDPM) are powerful state-of-the-art methods used to generate synthetic data from high-dimensional data distributions and are widely used for image, audio, and video generation as well as many more applications in science and beyond. The \textit{manifold hypothesis} states that high-dimensional data often lie on lower-dimensional manifolds within the ambient space, and is widely believed to hold in provided examples. While recent results have provided invaluable insight into how diffusion models adapt to the manifold hypothesis, they do not capture the great empirical success of these models, making this a very fruitful research direction. 

In this work, we study DDPMs under the manifold hypothesis and prove that they achieve rates independent of the ambient dimension in terms of score learning. In terms of sampling complexity, we obtain rates independent of the ambient dimension w.r.t.\ 
the Wasserstein distance. We do this by developing a new framework connecting diffusion models to the well-studied theory of extrema of Gaussian Processes. 

\end{abstract}


%\noindent\textbf{MSC2020 subject classifications:} Primary 62G05, 68T07; secondary 60H35, 62C20.

\section{Introduction}

Generative models~\cite{tomczak2022deep} represent a cornerstone of modern machine learning, tasked with the synthesis of new samples from an unknown distribution having only access to samples from it. Once trained, they are employed in various applications such as image and audio generation, natural language processing, and scientific simulations. Prominent examples include Generative Adversarial Networks (GANs), Variational Autoencoders (VAEs), and, more recently, Diffusion Models.

Diffusion Models~\cite{ho2020denoising,song2021scorebased}, also known as Denoising Diffusion Probabilistic Models~(DDPM) or Score-based generative models, have drawn significant attention due to their robust theoretical foundation and exceptional performance in generating high-quality data~\cite{dhariwal2021diffusionmodelsbeatgans, watson2023molecule, ho2022videodiffusionmodels,evans2024fasttimingconditionedlatentaudio, yang2023diffusion}. They operate by simulating a stochastic process where data is first gradually corrupted with Gaussian noise and then reconstructed back, allowing the model to generate complex data distributions from white noise. 

One of the most intriguing capabilities of diffusion models, supported with a lot of empirical and theoretical evidence~\cite{debortoli2023convergencedenoisingdiffusionmodels, NEURIPS2022_e8fb575e, stanczuk2023diffusionmodelsecretlyknows, pmlr-v238-tang24a, Chen20235327}, is their efficiency in learning distributions supported on low-dimensional manifold structures. Even more surprising is that this happens even though diffusion models are solely defined in terms of the ambient space. 

The manifold hypothesis, see e.g.~\cite{ma2012manifold}, postulates 
that high-dimensional data often lie on lower-dimensional manifolds within the ambient
space and offers a framework to explain why complex data can often be represented or approximated using fewer degrees of freedom than their raw dimensionality suggests. Different approaches to incorporate the manifold hypothesis have been fruitfully studied in the literature~\cite{aamari2018, divol2022measure, JMLR:v13:genovese12a, berenfeld2024estimatingdensitynearunknown, connor2021variationalautoencoderlearnedlatent}.

\subsection{Related Works}
Formally, we assume that we observe $n$ samples from an unknown $\alpha$-smooth distribution $\mu$ that is supported on an unknown $d$-dimensional, $\beta$-smooth manifold $M$, isometrically embedded into $\R^D$ (see~Section~\ref{sec:Statistical_Model_for_Manifolds} for definitions). The concept of \textit{minimax estimation} refers to finding an estimator that minimizes the worst-case error and plays an important role in statistical learning theory. The minimax rate of estimation of the unknown manifold $M$ was studied by \cite{aamari2018}, who proved that a local polynomial estimator achieves the minimax optimal rate $n^{-{\beta}/{d}}$ up to a $\log$-factor. 
This result was then adapted by \cite{divol2022measure} to construct a kernel-based minimax estimator $\hat{\mu}_n$ of $ \mu$ which converges at the optimal rate $n^{-\del{\alpha+1}/\del{2\alpha+d}}$ in the $W_1$ metric. However, sampling from $\hat{\mu}_n$ requires running a costly MCMC algorithm, thereby making it hard to use in practice. 

In contrast, sampling from diffusion models is relatively fast, and the resulting samples are of high quality even in high-dimensional settings~\cite{dhariwal2021diffusionmodelsbeatgans}. On a high level~\cite{yang2023diffusion}, diffusion models are trained to reverse the \textit{forward process} that gradually adds noise to the unknown distribution $\mu$. The resulting \textit{reverse process} can be described as a Stochastic Differential Equation~(SDE) given in terms of the \textit{score function} $s(t,x)$. More precisely, they first learn an approximation $\hat{s}(t,x)$ to the true score, and then simulate the discretized SDE to get samples from an approximation $\hat{\mu}$ of the true distribution $\mu$. Thereby two key quantities determining the performance of DDPMs are the accuracy of $\hat{s}(t,x)$, and the design of the discretization scheme.

In their seminal paper \cite{oko2023} obtain (near) minimax convergence rates $n^{-\del{\alpha+1}/\del{2\alpha+D}+\delta}$ of diffusion models for any $\delta >0$ under the Wasserstein-1 metric, assuming $\mu$ has $\alpha$-smooth compactly supported density on $\R^D$ and perfect SDE simulation. Simulateneously \cite{Chen20235327} obtained suboptimal convergence rates in total variation under a linear manifold hypothesis. 

 Recently, \cite{pmlr-v238-tang24a} generalized this result to the case when $\mu$ is $\alpha$-smooth and supported on unknown compact $d$-dimensional manifold $M$.  They showed that diffusion models achieve the Wasserstein convergence rate of order $D^{\alpha+d/2}n^{-(\alpha+1)/(2\alpha+d)}$, which is optimal in $n$ but has a strong dependence on $D$ via the term $D^{\alpha+d/2}$. 
It is worth mentioning that the rate obtained in \cite{divol2022measure} is independent of the ambient dimension, raising the question of whether similar results are possible for DDPMs. 

Since the first version of this work, several related results have appeared. 
\cite{kwon2026,fan2025optimalestimationfactorizabledensity} study diffusion models under low-dimensional factorization assumptions, showing that they can achieve minimax-optimal rates beyond the full ambient-dimensional setting. 
Closer to the manifold framework, \cite{yakovlev2025generalizationerrorbounddenoising} relaxes the exact manifold assumption while retaining intrinsic-dimensional rates, and \cite{stephanovitch2025generalizationboundsscorebasedgenerative} gives an alternative generalization proof based on stability of the reverse dynamics and induced score regularity.

\subsection{Our Contribution}
The strong dependence on $D$ of the bound $D^{\alpha+d/2} n^{-(\alpha+1)/(2\alpha+d)}$ obtained by~\cite{pmlr-v238-tang24a} does not allow a satisfactory explanation of the excellent empirical behavior of diffusion models in typical scenarios, when $D$ is very large, possibly much larger than $n$. For instance if $D \gtrsim n$ then the error becomes $n^{\alpha+d/2}n^{-(\alpha+1)/(2\alpha+d)}$ which goes to infinity with $n$ as soon as $d\ge 2$.
Therefore, in order to understand if the manifold hypothesis provides a reasonable explanation for the success of diffusion models, we need to derive sharper bounds in terms of $D$.

In this paper, we fill this gap by showing that the normalized score function $\sigma_ts(t,x)$ can be learned by a neural network estimator with the (near) optimal convergence rate $n^{-\del{\alpha+1}/\del{2\alpha+d}}$ up to $\polylog$ term w.r.t.\ the score matching loss, implying an optimal (up to $\polylog$) rate $n^{-{(\alpha+1)}/\del{2\alpha+d}}$ in the Wasserstein metric as long as $\log D = O\del{\log n}$. In other words the ambient dimension $D$ only impacts the rate via  a logarithmic term.

The backbone of both results is a novel high-probability bound on the score function depending only on the intrinsic dimension $d$. To obtain this bound, we leverage classic concentration results on the maximum of Gaussian processes and carefully study the interaction of the $D$-dimensional Gaussian noise added during the forward process with vectors lying on the $d$-dimensional manifold.

We show that the score function is sensitive only to the points around the denoised preimage, and as a result, the precision with which the score function points in the direction of the denoised point does not depend on the ambient dimension, demonstrating that diffusion models adapt well to the geometry of the manifold. 

We slightly modify the manifold estimator given in \cite{aamari2018} and construct a dimension-reduction scheme allowing us to use $n$ samples to build an efficient approximation of a $\beta$-smooth manifold $M$, with error $n^{-\beta/d}$, by $n$ polynomial surfaces each contained in easy-to-find sub-spaces of dimension $O(\log n)$.

Combining our high-probability bounds with this dimension reduction scheme, we improve the neural network architecture presented in \cite{pmlr-v238-tang24a}, thus obtaining the aforementioned ambient-dimension-free convergence bound $n^{-\del{\alpha+1}/\del{2\alpha+d}}\polylog(n)$ on the score estimation.

Finally, by combining the discretization schemes presented in~\cite{oko2023} and~\cite{potaptchik2024linearconvergencediffusionmodels}, we prove that simulating the diffusion model with the presented score estimator for $n^{\frac{2\alpha+d}{\alpha+1}}\polylog n$ steps achieves an error of $n^{-\frac{\alpha+1}{2\alpha+d}}\polylog n$. This allows us to derive an inequality on the Wasserstein distance between the simulations and the true generating process independent of the ambient dimension $D$. 

In Section~\ref{sec:preliminaries}, we introduce the notation and definitions related to diffusion models and manifold learning. In~Section~\ref{sec:main_results} we present our main results. In Sections~\ref{sec:high_probability_bounds}--\ref{sec:manifold_approximation} we present key steps in the control of the score matching loss to avoid the curse of dimensionality. In~Section~\ref{sec:alternative:Divol} we propose an alternative estimator of the score function based on an estimator of $\mu$ proposed by \cite{divol2022measure}. In~Section~\ref{sec:W_main}, we present the discretization scheme achieving the minimax optimal $W_1$ rate. Discussion is provided in~Section~\ref{sec:conclusion}. Finally, the detailed proofs are provided in the appendices.

\section{Preliminaries}
\label{sec:preliminaries}
\subsection{Score-Matching Generative Models}
\label{sec:score_matching_introduction}
Throughout the paper, we consider a data set $\cY = \curly{y_1, \cdots, y_n}$ consisting of independent, identically distributed samples from $\mu$. Generative models aim at simulating new data from an approximation of $\mu$ trained on $\cY$. 
    We follow the notation used in \cite{oko2023} to describe the methodology.
    
Let $\curly{B_t}_{t\in [0, \overline{T}]}$  denote a $D$-dimensional Brownian motion on $[0, \overline{T}]$. For a measure $\mu$ supported on $\cX \subset \R^D$, we consider the forward process $\curly{X_t}_{t\in [0, \overline{T}]}$ defined as a standard $D$-dimensional Ornstein–Ulhenbeck (OU) process with initial condition $X_0\sim \mu$, which is given as the solution of the SDE
    \begin{equation}    
    \label{eq:forward_process}
    \begin{cases}    
    dX_t = -X_tdt + \sqrt{2}dB_t,\\
    X_0 \sim \mu.
    \end{cases}
    \end{equation}
    
    Let $Z_D \sim \cN\del{0, \Id_D}$ denote a standard $D$-dimensional random normal vector independent from $X_0$. It is well known that 
    \£
    \label{eq:OU_as_mixture_with_normal}
    X_t\mid X_0 \stackrel{dist.}{=} c_tX_0 + \sigma_t Z_D \sim q_t(\cdot | X_0),
    \£ 
    where $c_t := e^{-t}$ and $\sigma_t := \sqrt{1-e^{-2t}}$ and $q_t$ is the transition density of the OU process. 
    
    We use $\mu_t$ to denote the law of $X_t$ and $p(t,\cdot)$ to denote the (unnormalized) marginal density of $X_t$, i.e. 
    \£
    \label{eq:p(t,x)_intro_diffusion}
    p(t,x) := \int_{\cX} 
    e^{-\|x-c_t y\|^2/2\sigma^2_t}\mu(dy).
    \£

    Under  mild conditions on $\mu$, 
    \cite{ANDERSON1982313} has shown that 
    the backward process $\curly{Y_t}_{t\in [0, \overline{T}]}$ defined as $Y_t = X_{\overline{T}-t}$ solves the SDE
    \begin{equation}    
    \label{eq:exact_backward_dynamic}
    \begin{cases}
        dY_t = \square{Y_t + 2\grad \log p\del{\overline{T} - t, Y_t}}dt + \sqrt{2}dB_t,\\
        Y_0 \stackrel{dist.}{=} X_{\overline{T}}.
    \end{cases}
    \end{equation}
    Notice that using~\eqref{eq:p(t,x)_intro_diffusion}, the score function $\nabla \log p(t,x)$ can be expressed in terms of $\mu$ as
    \£
    \label{eq:intro_diffusion_score_function}
    s(t,x) = \grad\log p(t, x) = \frac{\grad p(t,x)}{p(t,x)} = \frac{1}{\sigma^2_t}\frac{\int_{\cX} 
    (c_ty-x)e^{-\|x-c_t y\|^2/2\sigma^2_t}\mu(dy)}{\int_{\cX} 
    e^{-\|x-c_t y\|^2/2\sigma^2_t}\mu(dy)}.
    \£

    By construction, $Y_{\overline{T}} = X_0\sim \mu$; in other words, one can obtain a sample from $\mu$ by first sampling from $Y_0 = X_{\overline{T}}$, and then simulating the backward dynamics given in \eqref{eq:exact_backward_dynamic}. This scheme, however, is not feasible in practice as it requires access to samples from $X_{\overline{T}}$, to the true score function $s(t,x) = \grad\log p(t, x)$ and to be able to solve \eqref{eq:exact_backward_dynamic}. 
    
    The main idea behind generative score-matching models is to simulate the sampling procedure just described approximately. Firstly, since $X_{\overline{T}}\stackrel{dist.}{\rightarrow} \cN\del{0,\Id_D}$ as $\overline T$ goes to infinity, at a dimension-free rate,  we can approximate $Y_0 = X_{\overline{T}}$ by $\hat{Y}_0 \sim \cN(0,\Id_D)$. Secondly, the true score $s(t,x)$ is replaced with an approximation $\hat{s}(t,x)$ learned using data $\cY$, and the backward dynamics \eqref{eq:exact_backward_dynamic} is replaced by an appropriate discretization of the process
    \begin{equation}    
        \label{eq:approx_backward_process}
        \begin{cases}
            d\hat{Y}_t = \square{\hat{Y}_t + 2\hat{s}\del{\overline{T}- t, \hat{Y}_t}}dt + \sqrt{2}B_t,\\
            \hat{Y}_0 \sim \cN\del{0, \Id_D}.
        \end{cases}
    \end{equation}
    
    Let $\hat{\mu}_t := \textsf{Law}(\hat{Y}_{\overline{T}-t})$ denote the law of the obtained samples. A common practice is to use early stopping to avoid the numerical problems with the score function $s(t,x)$ that emerge when $t=0$. Early stopping entails  stopping the simulation early at a time $\overline{T}-\underline{T}$ for a small $\underline{T} > 0$ and uses $\hat{\mu}_{\underline{T}}$ to approximate $\mu$.

    The approximation $\hat{s}(t,x)$ is learned by minimization of the empirical version of the score matching loss which depends on $\underline{T} < \overline{T}$ and  is defined as
    \£
    \label{eq:the_score_matching_loss}
    \int_{\underline{T}}^{\overline{T}} \int_{\R^D} \|\hat{s}(t,x)-s(t,x)\|^2p(t,x)dtdx = \int_{\underline{T}}^{\overline{T}} \expectation\|\hat{s}\del{t,X_t}-s\del{t,X_t}\|^2 dt.
    \£
    
    The score matching loss controls~\cite[Eq. (90)]{oko2023} the Wasserstein loss between the target $\mu$ and the approximation $\hat{\mu}_{\underline{T}}$. More precisely, for $T_0 = \underline{T} < T_1 < T_2 <\ldots T_K = \overline{T}$, and any $\delta > 0$
        \£
        \label{thm:from_score_matching_to_wasserstein}
        W_1\del{\hat{\mu}, \mu} \lesssim \sqrt{D}\del{\sqrt{\underline{T}} +  \sum_{k=0}^{K-1} \sqrt{\log \delta^{-1}\cdot\sigma^2_{T_{k+1}}  \int_{T_k}^{T_{k+1}} \expectation\|\hat{s}\del{t,X_t}-s\del{t,X_t}\|^2 dt} + \delta + e^{-\overline{T}}}.
        \£
    \begin{remark}
    Eq. (90) in \cite{oko2023} doesn't address the dependence on the ambient dimension $D$, however, it can be easily deduced from the proof of Lemma~D.7.
    \end{remark}
    Another metric that is often used to assess the performance of Score-matching Generative models is the $\KL$ distance between $\hat{Y}_{\overline{T}-\underline{T}}$ and $X_{\underline{T}}$. More precisely~\cite[Section 5.2]{chen2023samplingeasylearningscore}, 
    if $\mu$ has finite second moments
    \£
    \label{thm:SML_to_KL}
    D_{\KL}\del{\hat{\mu}_{\underline{T}}\|\mu_{\underline{T}}} \lesssim De^{-2\overline{T}} + \int_{\underline{T}}^{\overline{T}} \expectation\|\hat{s}\del{t,X_t}-s\del{t,X_t}\|^2 dt,
    \£
    where we recall that $\hat{\mu}_{\underline{T}}, \mu_{\underline{T}}$ are distributions of $\hat{Y}_{\overline{T}-\underline{T}}$ and $ X_{\underline{T}}$.
    
    Since the computation of the loss function \eqref{eq:the_score_matching_loss} requires access to the true score $s(t,x)$, the equivalent denoising score matching loss is used instead during the optimization step. 
    \cite{vincent2011} showed that there is a constant $C_{\mu}$ that does not depend on $\hat{s}$ such that
    \£
    \label{eq:SML_to_DSML}
    \int_{\underline{T}}^{\overline{T}} \expectation\|\hat{s}\del{t,X_t}-s\del{t,X_t}\|^2 dt = \int_{\underline{T}}^{\overline{T}} \expectation_{X_0\sim \mu}\expectation_{Z_D}\|\hat{s}\del{t,c_tX_0+\sigma_t Z_D}+Z_D/\sigma_t\|^2 dt + C_{\mu},
    \£
    
    Hence defining for $y\in M$ the loss function $\ell_y$ as
    \£
    \label{eq:intro_diffusion_ell_y}
    \ell_y(\hat{s}, a, b) := \int_{a}^{b} \expectation_{Z_D}\|\hat{s}\del{t,c_ty+\sigma_t Z_D}+Z_D/\sigma_t\|^2 dt,
    \£
    it is enough to control
    \£
    \label{eq:intro_diffusion_risk}
    \cR(\hat{s}, a, b) = \expectation_{y\sim \mu} \ell_y(\hat{s}, a, b) = \int_{a}^{b} \expectation\|\hat{s}\del{t,X_t}-s\del{t,X_t}\|^2 dt -C_\mu,
    \£
    for all $a=T_k, b = T_{k+1}$ and $k\le K$.
    Finally, we define the empirical risk $\cR_{\cY}(\hat{s},a , b)$ w.r.t.\ samples $\cY = \curly{y_1,\ldots, y_n}$ as
    \£
    \label{eq:empirical_score_matching_loss}
    \cR_{\cY}(\hat{s}, a, b) = \frac{1}{n} \sum_{i=1}^n \ell_{y_i}(\hat{s}, a, b). 
    \£

    Tweedie's formula~\cite{robbins1956} offers an alternative, very useful, interpretation of the score $s(t,x)$ as the expected value of the noise added to get $x$.
    \£
    \label{eq:score_as_expectation}
    s(t,x) = \frac{1}{\sigma^2_t}\frac{\int_{M} \del{c_t y-x}e^{-\norm{x-c_t y}^2/2\sigma^2_t}\mu(dy)}{\int_{M} e^{-\norm{x-c_t y}^2/2\sigma^2_t}\mu(dy)} = \int_{M} \frac{c_t y-x}{\sigma^2_t}\mu(dy|t, x)dy = -\sigma^{-1}_t\expectation \del{Z_D\mid X_t = x},
    \£
    where $\mu(dy|t, x)$ is defined as the probability measure proportional to $e^{-\norm{x-c_t y}^2/2\sigma^2_t}\mu(dy)$. 
    
    By Bayes rule $\mu(dy|t, x)$ is the conditional law of $\del{X_0\big|X_t=x}$. We will denote the corresponding conditional expectation as   
    \£
    \label{eq:intro_diffusion_e(t,x)}
    e(t,x) := \expectation \square{X_0\big|X_t=x} =  \frac{\int_{M} y\cdot e^{-\norm{x-c_t y}^2/2\sigma^2_t}\mu(dy)}{\int_{M} e^{-\norm{x-c_t y}^2/2\sigma^2_t}\mu(dy)}, 
    \£
    substituting into~\eqref{eq:intro_diffusion_score_function}, we represent the score $s(t,x)$ as
    \begin{equation}
    \label{eq:s_e_relationship}
    s(t,x) = \frac{c_t}{\sigma^2_t}e(t,x) - \frac{x}{\sigma^2_t}.
    \end{equation}

\subsection{Manifold Hypothesis}
    \label{sec:Statistical_Model_for_Manifolds}

    In this paper,
    we study the behavior of score-matching generative models
    under the manifold assumption, i.e.\ we assume that the support $M$ of $\mu$ is a low dimensional manifold. The manifold hypothesis is particularly relevant as a way to understand the behavior of statistical or learning algorithms in high dimensions; see for instance \cite{divol2022measure} who showed that under the manifold hypothesis, it is possible to construct estimators of $\mu$ whose Wasserstein distance to $\mu$ is independent of the ambient dimension $D$. Whether this is feasible in the context of score-matching diffusion models has been an open question so far.
    
    We mainly follow \cite{divol2022measure} to define the class of regular manifolds that we will study. Throughout the paper we denote the distance between a point $x$ and a set $M$ as $\dist(x,M) = \inf_{y \in M} \|x-y\|$, the ball in $\mathbb R^k$ centered at $x$ with radius $r$, for $k\geq 1, x\in \mathbb R^k$ and $r>0$, as $B_k(x, r)$ and more generally we write $B_{\mathcal U}(x, r)$ for the ball in a space $\mathcal U$. Moreover for $f:\Omega\subset \R^{d_1} \mapsto \R^{d_2}$,  we write  $d^i f(x)$ for the $i$-th differential of $f$ at $x$ and for $k\in \mathbb N$, the $k$-th H\"older norm of $f$, when it exists, is denoted by
    $
    \norm{f}_{C^\alpha(\Omega)} = \max_{0\le i\le k} \sup_{x\in \Omega}\norm{d^i f(x)}_{op}
    $, 
    where $\norm{\cdot}_{op}$ denotes the operator norm.
    
    The key quantity, first introduced in \cite{federer}, determining the manifold's regularity is called the \emph{reach} $\tau$ and is defined as 
 \begin{align*}
        \tau &:= \sup\curly{\varepsilon\big| \forall x\in M^\varepsilon\,\,\,\, \exists!\, y\in M, \text{ s.t. } \dist(x,M) =\|x-y\|}, \\
    \text{where }     M^\varepsilon &= \curly{x\in \R^D: \dist(x,M) < \varepsilon}.
\end{align*}
    Throughout the paper we assume that the reach is bounded from below by $\tau > \tau_{\min}$; we refer to \cite{aamari2018} for a detailed discussion about the necessity of this assumption.

    The smoothness of the manifold is determined in terms of the smoothness of local charts~\cite{lee2013introduction}; when the manifold is embedded into $\R^D$ this means that for all $y\in M$ the manifold can be locally represented as a graph of a $\beta$-smooth one-to-one function $\Phi_y: U_y\subset\R^d\mapsto \R^D$, where $0\in U_y$ is an open set and $\Phi_y(0) = y$. If this holds we say that the manifold $M$ is $\beta$-smooth. 
    
    A natural way to define  $\Phi_y$ is in terms of the orthogonal projection $\pi_y:= \pi_{T_yM}$ onto the tangent space $T_yM$ at $y$. When $\tau >0$, the map $\pi_y $ restricted to $M\cap B_D(y,\tau/4)$ is one-to-one and $B_{T_yM}(0,\tau/8) \subset \pi_y (M\cap B_D(y,\tau/4))$, see for instance \cite{aamari2018stability}. We can then define $\Phi_y$ as the inverse $\pi_y\big|_{M\cap B_D(y,\tau/4)}$ and  $U_y = B_{T_yM}(0,\tau/8)$.

    Similarly we say that a function $f:M\mapsto \R$ is $\alpha$-smooth, if for any $y\in M$ the function $f\circ \Phi_y:\R^d\mapsto \R$ is $\alpha$-smooth in the regular sense. The embedding of $M$ into $\R^D$ generates a natural Riemannian metric induced from $\R^D$ together with a volume element $dy$ known as the Hausdorff measure. Finally, we say that a measure $\mu$ is $\alpha$-smooth if the measure $\mu \circ \Phi_y^{-1}$ is $\alpha$-smooth on $\R^d$, i.e.\ admits an $\alpha$-smooth density. If a measure $\mu$ has a density $p$ w.r.t.\ $dy$, on a chart $U_y$, by the change of variables formula, the corresponding density has the form $p(\Phi(z))\abs{\grad \Phi_y(z) \grad^T \Phi_y(z)}^{-1}$, and since the Jacobian is only guaranteed to be $\beta-1$ smooth, we will assume that $\alpha + 1 \le \beta$.

    To control the  smoothness we recall the definition of the $C^k$ norm for vector-valued functions. Let $f:\Omega\subset \R^{d_1} \mapsto \R^{d_2}$, then  
    $
    \norm{f}_{C^k(\Omega)} = \max_{0 \le i\le k} \sup_{x\in \Omega}\norm{d^i f(x)}_{op},
    $
    where $d^i f(x):\del{\R^{d_1}}^{\otimes i}\mapsto \R^{d_2}$ is the $i$th differential -- a multilinear operator defined as 
    $
    d^i f(x)\del{v_1\otimes\ldots \otimes v_n} = d_{v_1}\ldots d_{v_i} f(x),
    $
    and $d_{v} f$ is a directional derivative of $f$ along the vector $v\in \R^{d_1}$. Note that the $0$-th differential is the function itself, so $\sup_{x\in \Omega}\norm{d^0 f(x)}_{op} = \norm{f}_{\infty}$ is the usual $\sup$-norm.
    
    With all these definitions in mind, we are finally ready to present our assumptions on a manifold $M$.
    \begin{Assumption}
    \label{asmp:smooth_manifold}
        The support $M$ of $\mu$ is a compact $\beta$-smooth manifold embedded in $\R^D$ of dimension $d\ge 1$ and reach $\tau > \tau_{\min}$. Furthermore, there is a constant $L_M > 0$ s.t.\ for all $y\in M$ the inverse projection function $\Phi_y$ satisfies $\norm{\Phi_y}_{C^\beta\del{B_{T_yM}(0,\tau/8)}} \le L_M$.
    \end{Assumption}
\noindent    Similarly, we define a corresponding class of measures.
    \begin{Assumption}    \label{asmp:smooth_measure_on_manifold}
       The measure $\mu$ has an $\alpha$-smooth density $p(y)$ w.r.t.\ Hausdorff measure on $M$ satisfying $p_{\min} \le p \le p_{\max}$ and there is a constant $L_\mu$ such that for all $y\in M$
        $
        {\norm{p\circ \Phi_y}_{C^\alpha\del{B_{T_yM}(0,\tau/4)}} \le L_\mu.}
        $ 
    \end{Assumption}

    Besides the dimension $d$, the complexity of the manifold additionally depends on its volume $\Vol M$ and its local smoothness which we captured in terms of the reach $\tau$ and H\"older constant $L_M$. The complexity of the measure $\mu$ primarily depends on its similarity to uniform measure, and thus to $p_{\min}$ and $p_{\max}$. To take it into account we introduce the constant $C_{\log}$ defined below.
    \begin{Assumption}    
    \label{asmp:log_complexity_of_measure}
    We assume that there is a constant $C_{\log} > \max\del{\log 2 d,4}$ such that the following bounds hold: (i) $e^{-dC_{\log}} < p(y) < e^{dC_{\log}}$; (ii) $\log \Vol M < e^{dC_{\log}}$, (iii) $\min(\tau, L_M^{-1}) \ge e^{-C_{\log}}$.   
    \end{Assumption} 
    For convenience we introduce $r_0 := \min\del{1, \tau, L_M^{-1}}/8 > 0$ -- the radius of a neighborhood in which the functions $\Phi_y$ for all $y\in M$ are well behaved, see Appendix~A for details.
    
    \begin{remark}
        Assumption~\ref{asmp:log_complexity_of_measure} imposes different bounds on $r_0$ and $\Vol$. Essentially, this follows from a relation between the radius and volume of the $d$-dimensional sphere, i.e. ${\Vol B_d(0,r) \propto r^d}$. 
    \end{remark}
     Note that in \cite{zhangYinLiangLiu} and \cite{caiLi25}  the authors do not assume a lower bound on the density but they do not treat the case of degenerate distributions (here living on a submanifold of smaller dimension). In our proofs this assumption is used crucially in the step associated to learning the manifold (i.e. construction of $M_i^*$). We do not know if this is a necessary assumption but to our knowledge it appears in all papers controlling the error on estimating a distribution living on an unknown manifold. 
    
    The smoothness of the manifold is defined locally, therefore a common strategy is to find a dense set $\cG$ and then for each point, $G\in \cG$ consider a manifold in the $r_0$-vicinity of this point. 
    The following proposition bounds the size of such a set. 
    \begin{proposition}
    \label{prop:covering_number_of_M}
    Let $M \subset \mathbb{R}^D$ be a $d$-dimensional manifold. For any $\varepsilon \in (0, r_0)$, there exists a finite subset $\mathcal{G} = \{G_1, \ldots, G_N\} \subset M$ satisfying the following two properties:
    \begin{enumerate}
        \item $\varepsilon$-dense: For every $x \in M$, there exists $G_i \in \mathcal{G}$ such that $\|x - G_i\| \le \varepsilon$.
        \item $(\varepsilon/2)$-sparse: For any distinct $G_i, G_j \in \mathcal{G}$, it holds that $\|G_i - G_j\| \ge \varepsilon/2$.
    \end{enumerate}
    Furthermore, the cardinality $N = N(\varepsilon)$ of the set $\mathcal{G}$ is bounded by:
    $$N \le (\varepsilon/2)^{-d} \operatorname{Vol}(M)$$
\end{proposition}
    
    \begin{proof}
        Take $\cG$ as a maximal $\varepsilon$-separated set. By construction, the balls $M\cap B(G_i,\varepsilon/2)$ do not intersect, and ${\Vol \del{M\cap B(y,\varepsilon/2)} \ge (\varepsilon/2)^d}$ for any $y\in M$, so $N(\varepsilon/2)^{d} \lesssim \Vol M$.
    \end{proof}
    
    Let $\varepsilon$ and $G_1,\ldots, G_N$ be as in Proposition~\ref{prop:covering_number_of_M}.
    Then we can build a partition $\del{G_i, M_i, p_i}$ of a measure subordinated to the points $G_1,\ldots, G_N$ as follows. We take $M_i = \Phi_i\del{B_d(0,\varepsilon)}$, and define $p_i(y) = p(y)\phi_i(y)$, where $\phi_1,\ldots, \phi_N$ -is  a smooth partition of unity subordinated to $M_i$ and satisfying $\phi_i\big|_{B_D(G_i,\varepsilon/2)}\equiv 1$. 

    \subsection{Notation}
    \paragraph*{Convention 1} To simplify notation we assume that $\diam M \le 1$ and $0\in M$. The general case can be obtained by a simple rescaling and shift.

    \paragraph*{Notation 1} We use $\lesssim,\gtrsim, \simeq$ in cases when the corresponding inequality(equation) holds up to a multiplicative constant that depends only on $\beta, \alpha$, the quantities that are responsible for manifold and density regularity. We keep track of quantities $C_{\log}, D, d$. Similarly, when we say that $n$ is large enough we mean that $n\ge n_0$ where $n_0$ may depend on $C_{\log}, d, \alpha, \beta$ but not on $D$.   

    \paragraph*{Notation 2} Throughout the paper we use $t, T$ to denote time, $\varepsilon, \delta, \eta, \gamma < 1$ to denote errors; $D$, $d$ denote the ambient and manifold dimensions respectively. Unless otherwise stated we assume that $\log t, \log T, \log \varepsilon, \log \delta, \log \eta, \log \gamma$, and $\log D$ are all of order $\log n$. We also use the notation $\log_+ x = \max\del{\log x, 0}$.

    \paragraph*{Notation 3}
    Generally, we use the letters $x,y,z$ to denote points in $\R^D,M$, and a tangent space respectively. 
    
    \paragraph*{Notation 4}
    We use capital $T_k$ to denote the intervals on which we approximate the score function, while we use regular $t_k$ to denote the discretization timesteps. 

    Following~\cite{oko2023} we give the formal definition of the class of neural networks we will be working with. 
    \begin{definition} 
    \label{def:relu_nn}
        The class $\Psi(L,W,S,B)$ of $\relu$-neural networks with depth $L$, shape $W=\del{W_1,\ldots, W_L}$, sparsity $S$, and norm constant $B$
        is
        \[
        \Psi\del{L, W, S, B} =\Big\{\phi(x) = \del{A^L\relu + b^L} \circ \ldots \circ\del{A^1\relu + b^1}(x)\,\big| A^i \in \R^{W_i\times W_{i+1}},\\ 
        b^i\in \R^{W_{i+1}},
        \sum_{i=1}^L (\norm{A^i}_0 + \norm{b^i}_0) \le S, \sup_i (\norm{A^i}_\infty + \norm{b^i}_\infty) \le B 
        \Big\}.
        \]
    \end{definition}
    
\section{Approximation of the Score Function in High Dimension}
\label{sec:main_results}
    \label{sec:main_results_score_function_approx}

    Throughout this section let $\mu$ be a measure on a $d$-dimensional manifold $M$ satisfying Assumptions~\ref{asmp:smooth_manifold}--\ref{asmp:log_complexity_of_measure}, and $\cY=\curly{y_1,\ldots, y_n}$ be i.i.d. samples from $\mu$. 
    Let $X_t$ be the forward process \eqref{eq:forward_process} with initial condition $X_0\sim \mu$, denote as $p(t, x)$ the density function~\eqref{eq:p(t,x)_intro_diffusion} of $X_t$, and as $s(t,x) = \grad \log p(t,x)$ the score function~\eqref{eq:intro_diffusion_score_function}. 

    Our first contribution is the construction of an estimator $\hat{s}$ that converges to $s(t,x)$ independently of the ambient dimension $D$ with an almost optimal rate. 
    We prove this under the following technical assumption, which we believe can be relaxed with careful handling.
    \begin{Assumption}
    \label{asmp:smooth_measure_on_manifold_in_proofs}
        If $d\ge 3$, the manifold $M$ is  $\beta$-smooth, where $\beta \ge \frac{d}{d-2}(\alpha+1)$, where $\alpha$ is smoothness of measure $\mu$.   
    \end{Assumption}

    \begin{restatable}{theorem}{TheoremScoreApproximation}  
    \label{thm:score_approximation_1}
    Let $\mu$ be an $\alpha$-smooth measure satisfying Assumptions~\ref{asmp:smooth_manifold}--\ref{asmp:smooth_measure_on_manifold_in_proofs} supported on a $d$-dimensional $\beta$-smooth manifold $M$ embedded into $\R^D$. 
    
    Let $X_t$ be the forward process \eqref{eq:forward_process} with initial condition $X_0\sim \mu$, $p(t, x)$ be the density function~\eqref{eq:p(t,x)_intro_diffusion} of $X_t$, and $s(t,x) = \grad \log p(t,x)$ be the score function~\eqref{eq:intro_diffusion_score_function}.
        
     Denote by $\cY =\curly{y_1,\ldots, y_n}$ a set of $n$ i.i.d.\ samples from $\mu$. Fix any positive $\gamma > 1-1/d$. Let $\underline{T} \asymp (\log n)^{\gamma} n^{-\frac{2(\alpha+1)}{2\alpha+d}}$ and $\overline{T} = O(\log n)$. Then, there exists an estimator $\hat{s}(t,x) $ based on $\cY$ and satisfying  %= \hat{s}(t,x,\cY)
        \£\label{eq:thm_score_approximation_loss}
        \expectation_{\cY\sim \mu^{\otimes n}}\int_{\underline{T}}^{\overline{T}}\int_{\R^D} \sigma^2_t\norm{\hat{s}(t,x) - s(t,x)}^2 p(t,x)dxdt \le
        \begin{cases}    
        n^{-\frac{2(\alpha+1)}{2\alpha+d}}\polylog n, & \text{ if } d \ge 3, %n^{2\gamma\alpha} \cdot n^{-\frac{2(\alpha+1)}{2\alpha+d}}(\log n)^{\frac{ 2\alpha+d}{d}}
        \\
        (\log n)^{\gamma} \cdot n^{-1} & \text{ otherwise }
        \end{cases}
        \£
        for all $n\geq C_0$ where $C_0$ is independent of $D$ but may depend on $\alpha, \beta, d, $ and $C_{\log}$. %all other parameters of the problem.
        
        Moreover, for every fixed $A > 0$ there is a constant $C_1$ depending on $\alpha,\beta,d,C_{\log}$ and $A$, such that for all $n\geq C_1$, with probability at least $1-O\del{\polylog n\cdot n^{-\frac{\alpha+1}{2\alpha+d}}}$ over the training sample $\cY$, the constructed estimator satisfies, for every $t\in [\underline{T}, \overline{T}]$,
        \£
        \label{eq:main_result_tail_bound}
        \rP_{X_0,Z_D}\left(
        \norm{\sigma_t\hat{s}(t, X_t) + Z_D}
        >
        244\log n
        \,\middle|\,\cY
        \right)
        &\leq n^{-A}.
        \£
        \end{restatable}
    
        This theorem is a corollary of~Theorem~\ref{thm:main_result} presented in~Section~\ref{sec:score_approximation}, where we also present a detailed description of the estimator $\hat{s}(t,x)$. %Theorem~\ref{thm:main_result} is proved in~Appendix~\ref{apdx:score_apprxoximation_by_neural_network}.  
    
    Similar results were obtained in \cite{pmlr-v238-tang24a, oko2023}, and we generally follow their strategy and build $\hat{s}$ as a $\relu$-neural network (see~Definition~\ref{def:relu_nn}). Our main improvement compared to the literature is that we achieve bounds that do not depend on the ambient dimension $D$, while in \cite{pmlr-v238-tang24a} this bound is multiplied by $D^{d/2+ \alpha}$. 
    
    The key innovation leading to the result are new regularity bounds on $s(t,x)$ presented in~Section~\ref{sec:high_probability_bounds}, and the construction of an efficient manifold approximation algorithm in~Section~\ref{sec:manifold_approximation}.   

    Finally, in Section~\ref{sec:W_main}, we show how to get rid of the constant $\sqrt{D}$ in the Wasserstein convergence rate bound \eqref{thm:from_score_matching_to_wasserstein}. To achieve this we leverage~\eqref{eq:main_result_tail_bound} and enforce consistency of the score function estimates along the trajectory.
    \begin{corollary}
\label{cor:main_W_result}
    Let $\hat{s}$ be the estimator constructed in  Theorem~\ref{thm:score_approximation_1}. If $\overline{T} = \log D + \log n = O(\log n)$, then Theorem~\ref{thm:W_1_convergence_theorem} implies that there is a discretization scheme consisting of $O(n^{\frac{2\alpha+d}{\alpha+1}}\polylog n)$ discretization steps achieving 
    \begin{equation*}    
    \expectation_{\cY\sim\mu^{\otimes n}} W_1(\mu, \hat{\mu})
\lesssim \polylog n\cdot
  \begin{cases}    
         n^{-\frac{(\alpha+1)}{2\alpha+d}}, & \text{ if } d \ge 3,  %n^{-\frac{2(\alpha+1)}{2\alpha+d}}%n^{2\gamma\alpha} \cdot n^{-\frac{2(\alpha+1)}{2\alpha+d}}(\log n)^{\frac{ 2\alpha+d}{d}}
        \\
          n^{-1/2} & \text{ otherwise } %n^{-1}
        \end{cases}
    \end{equation*}
\end{corollary}
\begin{proof}
    The result is obtained by applying Theorem~\ref{thm:W_1_convergence_theorem} to the estimator from Theorem~\ref{thm:score_approximation_1}. See details in Appendix~\ref{apdx:W_new}.
\end{proof}
   
    \subsection{Construction of the estimator}
\label{sec:score_approximation}

In this subsection we describe the estimator used in
Theorem~\ref{thm:score_approximation_1}.  There are two cases.  When
\(d\le2\), the empirical measure already gives the required score-matching
accuracy.  When \(d\ge3\), we use the localized, low-dimensional construction
described below.

\subsubsection{Construction when \(d\le 2\)}
In the case \(d\le2\), the construction is simple, as we can use the empirical
score to get the desired rate.  Let
\[
    \hat\mu_n=\frac1n\sum_{j=1}^n\delta_{y_j}
\]
be the empirical measure.  We define \(\hat s\) as the score function associated to
the Gaussian convolution of \(\hat\mu_n\):
\begin{equation}
\label{est:dsmall}
    \hat{s}(t,x)
    =
    \frac{1}{\sigma_t^2}
    \frac{
        \sum_{y_j\in\cY}
        e^{-\norm{x-c_ty_j}^2/(2\sigma_t^2)}
        \del{c_ty_j-x}
    }{
        \sum_{y_j\in\cY}
        e^{-\norm{x-c_ty_j}^2/(2\sigma_t^2)}
    } .
\end{equation}
Proposition~\ref{prop: wd_vs_sml} bounds the score-matching error by
\(W_2(\mu,\hat\mu_n)\), and the empirical Wasserstein bound
of~\cite{divol2022measure} gives the \(d\le2\) rate in
Theorem~\ref{thm:score_approximation_1}.  The tail bound~\eqref{eq:main_result_tail_bound} for
\(\hat s\) defined above is proved in
 Appendix~D.%Appendix~\ref{apdx:score_apprxoximation_by_neural_network}.

\subsubsection{Construction when \(d\ge3\)}
For \(d\ge3\), the construction is blockwise in time.  We split
\[
    [\underline T,\overline T]
    =
    \bigcup_{k=0}^{K-1}[T_k,T_{k+1}],
    \qquad
    T_{k+1}/T_k=2,
\]
and we construct an estimator \(\hat s_k\) on each block, and then set
\(\hat s(t,\cdot)=\hat s_k(t,\cdot)\) for
\(t\in[T_k,T_{k+1}]\).  Below, we fix a block and suppress the index \(k\)
when this causes no ambiguity. We use Tweedie's formula~\eqref{eq:score_as_expectation} and write the
estimator in the form
\[
    \hat s_k(t,x)=\frac{c_t\hat e_k(t,x)-x}{\sigma_t^2},
    \qquad t\in[T_k,T_{k+1}],
\]
where \(\hat e_k\) is an estimator of
\(\expectation[X_0\mid X_t=x]\),\footnote{where $(X_0, X_t)\sim \mu(d x_0) p_t(d x_t |x_0)$}
whose construction is detailed below.
Let $k\leq K-1$,
define the integer \(N=N(T_k,n)\), up to integer rounding, by
\[
N(T_k,n):=
\begin{cases}
    \max\del{
        (\log n)^{2\gamma d}(\sigma_{T_k}/c_{T_k})^{-d},
        n\cdot n^{-\frac{2(\alpha+1)}{2\alpha+d}}
    },
    & T_k\ge n^{-\frac{2}{2\alpha+d}},
    \\
    (\log n)^{-\gamma d}n^{\frac{d}{2\alpha+d}},
    & T_k<n^{-\frac{2}{2\alpha+d}}.
\end{cases}
\]
Take \(G_1,\ldots,G_N\) from the training sample \(\cY\), set
$
    \varepsilon_N\asymp \del{\frac{\log N}{N}}^{1/d},
$
and define
\[
    \cV_i=\curly{G_j-G_i:\norm{G_j-G_i}\le\varepsilon_N},
    \qquad
    \cH_i=\spn\cV_i,
    \qquad
    d_i=\dim\cH_i.
\]
Let \(P_{\cH_i}:\R^{d_i}\to\cH_i\subset\R^D\) be a linear isometry
identifying \(\cH_i\) with \(\R^{d_i}\).  These objects are fixed once the
training sample is fixed and are not optimized over in the empirical risk
minimization.

For \(x\in\R^D\), define
\[
    i_{\min}(x):=\argmin_{i\le N}\norm{x-c_tG_i},
    \qquad
    G_{\min}(x):=G_{i_{\min}(x)}.
\]
Let \(\rho:\R_+\to[0,1]\) be
\[
\rho(u)=
\begin{cases}
    1, & 0\le u\le 1/2,\\
    2-2u, & 1/2<u<1,\\
    0, & u\ge1.
\end{cases}
\]
We define
\[
    \rho_i(t,x)
    =
    \rho\left(
        \frac{
            \norm{x-c_tG_i}^2-\norm{x-c_tG_{\min}(x)}^2}
            {2C(T_k,n)}
    \right), \quad C(T_k,n)=
\begin{cases}
    \sigma_{T_k}^2\log n\log\log n,
    & T_k\ge n^{-\frac{2}{2\alpha+d}},
    \\
    6\varepsilon_N^2,
    & T_k<n^{-\frac{2}{2\alpha+d}},
\end{cases}
\]
and the local multipliers
\[
h_i(t,x)=
\begin{cases}
    \exp\del{
        \frac{
            \norm{x-c_tG_{\min}(x)}^2-\norm{x-c_tG_i}^2}
            {2\sigma_t^2}
    },
    & t\ge n^{-\frac{2}{2\alpha+d}},
    \\
    1,
    & t<n^{-\frac{2}{2\alpha+d}}.
\end{cases}
\]
Note that by construction $h_i(t,x)\le 1$. Write \(\widetilde\rho_i(t,x):=h_i(t,x)\rho_i(t,x)\) and
\[
    \widetilde x_{i,t}:=P_{\cH_i}^T(x-c_tG_i) \in \mathbb R^{d_i}.
\]

The estimator \(\hat s_k\) is obtained by minimizing
\(\cR_{\cY}(\phi,T_k,T_{k+1})\) over neural networks with the architecture
\begin{equation}
\label{eq:definition of phi}
    \phi(t,x)
    =
    \frac{c_t}{\sigma_t^2}
    \frac{
        \sum_{i=1}^N
        \widetilde\rho_i(t,x)
        \phi_{w_i}(t,\widetilde x_{i,t})
        \del{
            G_i+P_{\cH_i}\phi_{e_i}(t,\widetilde x_{i,t})
        }
    }{
        \sum_{i=1}^N
        \widetilde\rho_i(t,x)
        \phi_{w_i}(t,\widetilde x_{i,t})
    }
    -
    \frac{x}{\sigma_t^2}.
\end{equation}
Here
\[
    \phi_{e_i}:\R_+\times\R^{d_i}\to\R^{d_i},
    \qquad
    \phi_{w_i}:\R_+\times\R^{d_i}\to\R_+
\]
are neural networks.  More formally, we consider the class
\begin{multline}
\label{eq:definition of cS}
    \cS_k=
    \Big\{
    \phi \text{ of the form~\eqref{eq:definition of phi}}:
    \phi_{e_i},\phi_{w_i}\in\Psi(L,W,S,B),\quad
    \norm{\phi_{e_i}}\le C_2,\\
    \phi_{w_i}(t,\widetilde x_{i,t})\ge n^{-2},\quad
    \left\|
        \frac{c_t\phi_{e_i}(t,\widetilde x_{i,t})-\widetilde x_{i,t}}
        {\sigma_t}
    \right\|_\infty\le\log n
    \Big\},
\end{multline}
where the last two constraints are imposed on the relevant local-network
domains and \(C_2\ge2\diam M\).  The network parameters satisfy
\[
    L=O(\polylog n),\qquad
    \norm{W}_\infty=O(\polylog n),\qquad
    S=O(\polylog n),\qquad
    B=e^{O(\polylog n)}.
\]
The important point is that \(P_{\cH_i}\), \(G_i\), and the cutoffs are
explicit functions of the training sample, and are not optimized in the
empirical risk minimization.

For comparison, one may also represent the construction by a single
ambient neural network and optimize over the coordinate maps implicitly.
This gives the larger class
\begin{equation}
\label{eq:definition of cS prime}
    \cS'_k
    =
    \curly{
        \phi'\in\Psi(L',W',S',B'):
        \norm{\phi'}_\infty
        \le C'_2\frac{\sqrt D\log n}{\sigma_t}
    },
\end{equation}
where
\[
    L'=O(\polylog n),\qquad
    S'=O(N\polylog n),\qquad
    B'=e^{O(\polylog n)},
\]
and \(W'_i=O(N\polylog n)\) for hidden layers while the input and output
layers have ambient dimension \(D\).

The following blockwise theorem is the version used to prove
Theorem~\ref{thm:score_approximation_1}.

\begin{restatable}{theorem}{MainResult}
\label{thm:main_result}
Let \(d\ge3\), and consider the assumptions of
Theorem~\ref{thm:score_approximation_1}.  Let $k\leq K-1$
and
\[
    \hat s_k\in
    \argmin_{\phi\in\cS_k}
    \cR_{\cY}(\phi,T_k,T_{k+1})
\]
be an empirical risk minimizer over the class \(\cS_k\) defined
in~\eqref{eq:definition of cS}, with
\[
    L,\norm{W}_\infty,S=O(\polylog n),
    \qquad
    B=e^{O(\polylog n)}.
\]
Then, for all \(n\) large enough,
with constants independent of \(D\),
\[
    \expectation_{\cY\sim\mu^{\otimes n}}
    \int_{T_k}^{T_{k+1}}
    \expectation\norm{\hat s_k(t,X_t)-s(t,X_t)}^2\,dt
    \lesssim
    \polylog n
    \min\left(
        \sigma_{T_k}^{-2}n^{-\frac{2(\alpha+1)}{2\alpha+d}},
        n^{-\frac{2\alpha}{2\alpha+d}}
    \right).
\]
If \(\hat s'_k\) is the empirical risk minimizer over the ambient class
\(\cS'_k\) in~\eqref{eq:definition of cS prime},
\[
    \expectation_{\cY\sim\mu^{\otimes n}}
    \int_{T_k}^{T_{k+1}}
    \expectation\norm{\hat s'_k(t,X_t)-s(t,X_t)}^2\,dt
    \lesssim
    D\polylog n
    \min\left(
        \sigma_{T_k}^{-2}n^{-\frac{2(\alpha+1)}{2\alpha+d}},
        n^{-\frac{2\alpha}{2\alpha+d}}
    \right).
\]
\end{restatable}

Assembling \(\hat s(t,\cdot)=\hat s_k(t,\cdot)\) for
\(t\in[T_k,T_{k+1}]\) gives the bound~\eqref{eq:thm_score_approximation_loss}
corresponding to the case \(d\ge3\) of
Theorem~\ref{thm:score_approximation_1}.  The preliminary construction of
\(\cH_i\) is the dimension-reduction step: it reduces the approximation
problem from \(\R^D\) to \(N\) local problems of dimension
\(d_i=O(\log n)\). 
Note that in Theorem \ref{thm:main_result}, there is a phase transition  corresponding to 
$$\sigma_{T_k}^{-2}n^{-\frac{2(\alpha+1)}{2\alpha+d}}= 
        n^{-\frac{2\alpha}{2\alpha+d}}  \quad \Leftrightarrow T_k = n^{-2/(2\alpha +d)}.$$
 As will be apparent in the sketch of the proof below, this phase transition  also corresponds to $\sigma_{T_k} = \epsilon_N = n^{-1/(2\alpha +d)} $, which is also  the optimal bandwidth in Kernel estimation.      

\subsubsection{Proof sketch for \(d\ge3\)}

The proof starts from a bias-variance decomposition of the generalisation bound given in Theorem D.7, in Appendix D.5 which has the form 
\[
 \int_{T_k}^{T_{k+1}}
    \expectation\norm{\hat s_k(t,X_t)-s(t,X_t)}^2\,dt
    &\lesssim
    \inf_{\phi\in\cS_k}
    \int_{T_k}^{T_{k+1}}
    \expectation\norm{\phi(t,X_t)-s(t,X_t)}^2\,dt\\
    & \quad + \frac{(\log n)^6}{n} \log \mathcal N (\delta_n) + \delta_n,
\]
where the expectation is with respect to both $X_t $ and $\mathcal Y$ and where $\log \mathcal N (\delta_n)$ is an entropy term. 
 The
substantive part of the proof  is then  the control on the bias term: we must exhibit an element of \(\cS_k\)
that approximates \(s\).

By Tweedie's formula, this is a problem about approximating the posterior
mean
\[
    s(t,x)=\frac{c_t e(t,x)-x}{\sigma_t^2},
    \qquad
    e(t,x)=
    \frac{
        \int_M y\,e^{-\norm{x-c_ty}^2/(2\sigma_t^2)}\mu(dy)}
    {
        \int_M e^{-\norm{x-c_ty}^2/(2\sigma_t^2)}\mu(dy)}.
\]

First we approximate $M$ by piecewise polynomial patches. To do so, we use the polynomial approximation $\Phi_i^*$ of the true local parametrization $\Phi_i: B_d(0, r_0) \rightarrow M$ of $M$ around $G_i$ proposed in \cite{aamari2018}. 
We prove in Section \ref{sec:manifold_approximation}, that on an  event of probability greater than \(1-Cn^{-(\alpha+1)/(2\alpha+d)}\polylog n\)
\[
    \norm{\Phi_i^*-\Phi_i}_\infty\lesssim\varepsilon_N^\beta, \quad \text{and } \quad \Phi_i^*(B_d(0,8\varepsilon_N))\subset G_i+\cH_i, \quad d_i = \text{dim}(\cH_i)\lesssim \log n,
\]
and the sets $\{G_i, i\leq N\}$ are $\varepsilon_N$-dense in $M$.
Thus the unknown embedded manifold is replaced by a union of polynomial
pieces living in spaces $\cH_i$ of dimension $d_i \lesssim \log n$. 

The next important step is  the localization of the posterior distribution  of $X_0$ given $X_t$. Theorem~\ref{thm:concentration_around_X0} in Section \ref{sec:high_probability_bounds} implies that with high probability the posterior mass of $X_0$ given $X_t$ is concentrated on 
\[
   B_M(X_0,r_t):=M\cap B_D(X_0,r_t), \quad \text{with} \quad r_t
    \asymp
    \frac{\sigma_t}{c_t}
    \sqrt{d\del{\log_+(c_t/\sigma_t)+C_{\log}}+\log n}.
\]
This implies that the integrals over $M$ in the definition of $e(t,x)$ can be restricted to $B_M(X_0,r_t) = \cup_{i=1}^N B_M(X_0,r_t)\cap \phi_i(B_d( 0, 8\varepsilon_N))$.  The set of indices where the intersection is non empty  depends on the unobserved \(X_0\).
The cutoff  functions \(\rho_i\)  are an observable approximation of $1_{\text{dis}(X_0, \Phi_i(B_d(0, 8\varepsilon_N))) \leq r_t}$. Indeed,  the distance comparison in
Lemma~\ref{lemma:good_set_size}, combined with the localization event, gives
the following high-probability cutoff property:
\[
\begin{array}{ll}
\text{when } \, T_k\ge n^{-2/(2\alpha+d)}:
&
\dist(X_0,M_i)<r_t
\ \Longrightarrow\
\rho_i(t,X_t)=1,
\\[0.4em]
&
\rho_i(t,X_t)>0
\ \Longrightarrow\
\dist(X_0,M_i)\lesssim r_t\sqrt{\log\log n},
\\[0.8em]
\text{when } \, T_k<n^{-2/(2\alpha+d)}:
&
\rho_i(t,X_t)>0
\ \Longrightarrow\
B_M(X_0,r_t)\subset M_i .
\end{array}
\]
 The two  regimes in $T_k$
then use the same architecture \eqref{eq:definition of phi} in different ways.

\emph{Large times: $T_k \geq n^{-2/(2\alpha+d)}$}.  In this regime, $\varepsilon_N = o(r_t)$
so that there are multiple indices $i$ for which $\rho_i \neq 0$. We then define $M_i = \Phi_i(B_d(0, 8\varepsilon_N))$, $M_i^* = \Phi_i^*(B_d(0, 8\varepsilon_N)) \subset G_i+\cH_i$, a decomposition of $\mu = \sum_{i=1}^N\mu_i$ based on a partition of unity with supports $M_i$ and $\mu^* = \sum_{i=1}^N \mu_i^*$ constructed from the push forward measures $\mu_i^*$ of $\mu_i$ by $\Phi_i^*\circ \Phi_i^{-1}$. Then Proposition \ref{prop: wd_vs_sml} in Section \ref{sec:manifold_approximation} shows that with high probability 
$$ \| e(t,X_t)  - e_{tr}^*(t,X_t)\|^2 \lesssim  \sigma_t^2 W_2^2(\mu, \mu^*)\lesssim \sigma_t^2 \varepsilon_N^{2\beta}, \quad e_{tr}^* = \frac{ \sum_{i=1}^N \rho_i(t,X_t) p_i^*(t, X_t) e_i^*(t,X_t) }{ \sum_{i=1}^N \rho_i(t,X_t) p_i^*(t, X_t)}, $$
where $e_i^*(t,X_t)$ is the posterior expectation of $X_0$ given $X_t$ under the prior $\mu_i^*$ and 
$$p_i^*(t, X_t) = \int_{M_i^*}e^{-\frac{ \|X_t-c_t y\|^2}{2\sigma_t^2 } }\mu_i^*(dy)= e^{-\frac{ \|X_t-c_t X_{i,t}\|^2}{2\sigma_t^2 }} \int_{M_i^*}e^{-\frac{ (\|X_t-c_t y\|^2 -\|X_t-c_t X_{i,t}\|^2) }{2\sigma_t^2 } }\mu_i^*(dy),$$
and $X_{i,t} = c_t G_i + pr_{\cH_i}(X_t - c_t G_i) $ is the orthogonal projection of $X_t$ onto $c_t G_i + \cH_i$. Noting that $\|X_t - X_{i,t}\|^2 - \|X_t - c_t G_i\|^2 = -\| X_{i,t}- c_t G_i\|^2$,  we obtain that 
$$p_i^*(t, X_t) = e^{-\frac{ \|X_t - c_t G_i\|^2}{2\sigma_t^2 }} \int_{M_i^*}e^{-\frac{ (\|X_{i,t}-c_t y\|^2 -\|X_{i,t}-c_t G_i\|^2)  }{2\sigma_t^2 } }\mu_i^*(dy).$$
Written in the coordinates of \(G_i+\cH_i\), this 
gives the approximating denoiser
\[
    e_{\mathrm{large}}(t,x)
    =
    \frac{
        \sum_{i=1}^N\widetilde \rho_i(t,x)
        \widetilde p_i^*(t,\widetilde x_{i,t})
        \del{G_i+P_{\cH_i}\widetilde e_i^*(t,\widetilde x_{i,t})}}
    {
       \sum_{i=1}^N\widetilde\rho_i(t,x)
        \widetilde p_i^*(t,\widetilde x_{i,t})},
\]
where \(\widetilde p_i^*\) and \(\widetilde e_i^*\) are the local density and
conditional mean generated by the \(i\)-th polynomial piece, written in the
coordinates of \(\cH_i\).  This has exactly the same structure
as~\eqref{eq:definition of phi}: \(\phi_{w_i}\) approximates
\(\widetilde p_i^*\), while \(\phi_{e_i}\) approximates
\(\widetilde e_i^*\).

\emph{Small times: $T_k\leq n^{-2(2\alpha+d)}$.} Here $\varepsilon_N = o(r_t)$.   The cutoff property says that for every active $i$ (such that $\rho_i >0$),  $\Phi_i(B_d(0,M_i)$ contains \(B_M(X_0,r_t)\) and is too large. To better localize we construct, in Appendix~D.2,  a \textit{better} local chart $M_i$ with radius $r_t$ such that the conditional mean of $X_0$ given $X_t$ under the restriction of $\mu$ on $M_i$, $e_i(t,X_t)$, is very close to $e(t,X_t)$ with high probability. This implies that 
$$e(t, X_t) \approx   \frac{ \sum_{i=1}^N \rho_i(t,X_t) e_i(t,X_t) }{ \sum_{i=1}^N \rho_i(t,X_t) }.$$
We then approximate $e_i$ for each active $i$ by expressing each integral in $e_i$ in terms of the coordinates $z= \Phi_i^{-1}(y)$, and by Taylor expanding the density $p(z)$ of $\mu$ and approximating $y=\Phi_i(z)$ by $\Phi_i^*(z)$.
This leads to low dimensional polynomial approximations of the numerator and denominator defining $e_i$ which leads to the architecture~\eqref{eq:definition of phi}.

Combining the manifold-approximation error, the localization error, and the
low-dimensional neural-network approximation gives an approximant in
\(\cS_k\) with the rate in Theorem~\ref{thm:main_result}.  A bound on the entropy term $\mathcal N(\delta_n)$
then yields the ERM bound.

    \section{High Probability Bounds on the Score Function}
\label{sec:high_probability_bounds}

In this section we present our main tool, the high-probability bounds on the score function $s(t,X_t)$ under the manifold hypothesis w.r.t. $X_t$ -- the forward process~\eqref{eq:forward_process}. These results  play a key role in  the approximation of the score presented  in~Section~\ref{sec:score_approximation}. 

    We start with the study of how $D$-dimensional Gaussian noise $Z_D\sim \cN\del{0,\Id_D}$ interacts with the smooth $d$-dimensional manifold $M$. The following theorem bounds the scalar product between $Z_D$ and vectors connecting points on the manifold.

    \begin{restatable}{proposition}{TheoremBoundsOnScalarProduct}
    \label{prop:maximum_of_Gaussian_Process_on_Manifold}
        Let $M$ be a $\beta\ge 1$-smooth manifold satisfying~Assumption~\ref{asmp:smooth_manifold}. Then for any $\delta > 0$ and $\varepsilon <r_0$ with probability $1-\delta$ for all $y,y'\in M$
        \£
        \label{eq:bound_on_scalar_product_1}
        \abs{\innerproduct{Z_D}{y-y'}} \le 4\varepsilon\sqrt{d} + \del{\norm{y-y'}+6\varepsilon}\sqrt{4d\log 2\varepsilon^{-1} + 4\log_+ \Vol M  + 2\log 2\delta^{-1}}.
        \£

    \end{restatable}
    \begin{proof}
        See  Appendix~B.1 for the complete proof.
        Adapting classical results on the maximum of Gaussian processes we show that for any $L$-Lipschitz function $f:\R^d\mapsto \R^D$ satisfying $f(0) = 0$
        \[
        \rP\del{\sup_{z\in B_d(0,r)} \innerproduct{Z_D}{f(z)} \le Lr\del{\sqrt{d} + \sqrt{2\log \delta^{-1}}}} \ge 1-\delta.
        \]
        Then, for a point $y\in M$ we locally represent the manifold as $\Phi_y\del{B_d(0,\varepsilon)}$, see~Section~\ref{sec:Statistical_Model_for_Manifolds}, and applying the inequality above to $f(z) =\Phi_y(z)-\Phi_y(0)$ we get a local version of inequality~\eqref{eq:bound_on_scalar_product_1}. Finally, taking an $\varepsilon$-dense set and applying the chaining argument we generalize the local result to the whole manifold $M$.
    \end{proof}
    \begin{remark}
        Proposition~\ref{prop:maximum_of_Gaussian_Process_on_Manifold} is an adaptation to the case of smooth regular manifolds of results related to the connection between metric entropy and Gaussian processes, see e.g.
        \cite[Chapter 5]{wainwright2019} or \cite[Chapter 9]{vershynin2018}. This connects our approach with~\cite{li2024} on sampling complexity of diffusion models. 
    \end{remark}
    The above theorem shows that the inner-product between the Gaussian vector $Z_D$ and vectors on the manifold does not depend on the ambient dimension. In particular, using standard bounds on the length of Gaussian vectors, this means that the correlation between $Z_D$ and $y-y'$ is $\lesssim \sqrt{d/D}$ with high probability, implying that the noise added during the forward diffusion process is almost orthogonal to the manifold if $D\gg d$. 

    Assuming that $\beta\ge 2$ and applying the same technique it is also possible to bound the projection of $Z_D$ onto tangent spaces $T_{y}M$, uniformly in $y\in M$.

    \begin{restatable}{proposition}{ProjectionOntoTangentSpace}
    \label{prop:projection_onto_tangent_space}
        Let $M$ be a $\beta \ge 2$-smooth manifold satisfying Assumption~\ref{asmp:smooth_manifold}. Then for any $\delta > 0$ with probability at least $1-\delta$ for all $y\in M$
        \[
         \norm{\pr_{T_yM} Z_D} \le 8\sqrt{d\del{4\log 2d + 2\log r^{-1}_0 + 2\log_+\Vol M + 2\log \delta^{-1}}}.
        \]
    \end{restatable}

    Equipped with the bounds on the scalar product between the noise $Z_D$ and the vectors on the manifold $M$ we are ready to analyze the score function. Let $\mu$ be a measure satisfying Assumptions~\ref{asmp:smooth_manifold}--\ref{asmp:log_complexity_of_measure}. We recall Tweedie's formula~\eqref{eq:score_as_expectation} stating that
    \[
    s(t,x) = \frac{1}{\sigma^2_t}\frac{\int_{M} \del{c_t y-x}e^{-\norm{x-c_t y}^2/2\sigma^2_t}p(y)dy}{\int_{M} e^{-\norm{x-c_t y}^2/2\sigma^2_t}p(y)dy} = \int_{M} \frac{c_t y-x}{\sigma^2_t}p(y|t, x)dy,
    \]
    where $p(y|t, x) \propto e^{-\norm{x-c_t y}^2/2\sigma^2_t}p(y)$ is normalized to be a density function corresponding to the random variable $\del{X_0\big|X_t=x}$. 
    In other words, the score function is the expectation of the function $\del{c_t y-x}/{\sigma^2_t}$ w.r.t.\ the measure on $M$ with density $p(y|t,x)$. 
    The density itself is a product of two components: $p(y)dy \propto dy$ which is proportional to the uniform measure $M$ up to $p_{\min}, p_{\max}$ and 
    $e^{-\norm{x-c_t y}^2/2\sigma^2_t}$ which exponentially penalizes  points far from $x$. 

    As we have shown in Proposition~\ref{prop:maximum_of_Gaussian_Process_on_Manifold} the noise vector $Z_D$ is almost orthogonal to vectors on $M$. Therefore, representing the forward process~\eqref{eq:forward_process} as $X_t = c_tX_0 + \sigma_t Z_D$ we write
    \[
    \norm{X_t-c_ty}^2 = c^2_t\norm{X_0-y}^2 + \sigma^2_t\norm{Z_D}^2 + 2c_t\sigma_t\innerproduct{X_0-y}{Z_D} \approx c^2_t\norm{X_0-y}^2 + \sigma^2_t\norm{Z_D}^2.
    \]
    Since the vector $Z_D$ does not depend on $y$, it disappears during the normalization and therefore, roughly speaking, for small $t < r^2_0$, the log density function $\log p(y|t,X_t) \approx -c^2_t\norm{X_0-y}^2/2\sigma_t^2$ looks like log  the density of $d$-dimensional Gaussian variable on $T_{X_0}M$. The next proposition formalizes this argument. 
    \begin{restatable}{proposition}{RepresentationOfConditionalMeasure}  
    \label{prop:RepresentationOfConditionalMeasure}
    Let measure $\mu$ satisfy Assumptions \ref{asmp:smooth_manifold}--\ref{asmp:log_complexity_of_measure} and be supported on $M$.
    For any $\delta > 0$ with probability at least $1-\delta$ on $X_t$ the following bound holds uniformly in $y\in M$ 
    \£
    \label{eq:RepresentationOfConditionalMeasure_main}
    -20d\del{\log_+ (c_t/\sigma_t) + 4C_{\log}} - 8\log \delta^{-1} - \frac{3}{4}(c_t/\sigma_t)^2\norm{X_0-y}^2 
    \\
    \le \log p(y|t,X_t) 
    \\
    \le  20d\del{\log_+ (c_t/\sigma_t) + 4C_{\log}} + 8\log \delta^{-1} - \frac{1}{4}(c_t/\sigma_t)^2\norm{X_0-y}^2.
    \£
    \end{restatable}    
    Proposition \ref{prop:RepresentationOfConditionalMeasure} is proved in Appendix~B.2. 
    So, the density $p(y|t, X_t)$ exponentially penalizes for being too far away from $X_0$, and, therefore, most of the mass of the measure with density $p(y|t, X_t)$ concentrates in a ball centered at $X_0$.
    \begin{restatable}{theorem}{ConcentrationAroundX}    
    \label{thm:concentration_around_X0} 
    Let measure $\mu$ satisfy Assumptions~\ref{asmp:smooth_manifold}--\ref{asmp:log_complexity_of_measure} and be supported on $M$. Fix $t,\delta, \eta > 0$ and define 
    \[
    r(t,\delta, \eta) = 2(\sigma_t/c_t)\sqrt{20d\del{\log_+ (c_t/\sigma_t) + 4C_{\log}} + 8\log \delta^{-1} + \log \eta^{-1}}.
    \]
    \begin{enumerate}
        \item A.s.in $X_0\sim \mu$ and with probability $1-\delta$ in $Z_D\sim \cN\del{0,\Id_D}$ for a random point $X_t = c_tX_0 + \sigma_tZ_D$ the following holds
        \£
        \label{eq:concentration_around_X0_1}
        \int_{y\in M: \norm{y-X_0} \le r(t,\delta,\eta)} p(y|t,X_t) dy \ge 1-\eta.
        \£
    \item As a corollary, a.s. in $X_0\sim \mu$ and with probability $1-\delta$ in $Z_D\sim \cN\del{0,\Id_D}$ for a random point $X_t = c_tX_0 + \sigma_tZ_D$ the following holds
        \[
        \norm{\sigma_t s(t,X_t) + Z_D} \le 4\sqrt{20d\del{2\log_+ (c_t/\sigma_t) + 4C_{\log}} + 8\log \delta^{-1}}.
        \]
    \item Integrating over $Z_D$, for all $y\in M$
        \£
        \label{eq:concentration:of_ell_s}
        \expectation_{Z_D}\norm{\sigma_ts(t, c_ty + \sigma_tZ_D) + Z_D}^2 \le 8\del{40d\log_+ (c_t/\sigma_t) + 80dC_{\log} + 3}.
        \£
    \end{enumerate}
    
    \end{restatable}
    
    The proof of the above proposition is given in  Appendix~B.2. 
    
    In particular, the result shows that on average the score function computed at point $X_t = c_tX_0 + \sigma_tZ_D$ estimates the added noise $ Z_D/\sigma_t$ with an error scaling as $\sqrt{d}/\sigma_t$ up to $\log$-factor, and thereby the estimate $c^{-1}_t\del{X_t +\sigma^2_t s(t,X_t)}$
    of point $X_0$ has an error $(\sigma_t/c_t)\sqrt{d}$ up to $\log$ factor independent on $D$.

    Finally, let us present the result, that will play an important role in the construction of weight functions $\rho_i(t,x)$ discussed in~Section~\ref{sec:main_results_score_function_approx}.
    \begin{restatable}{lemma}{LemmaWeightsComparedToDistance}
    \label{lemma:good_set_size}     
        Let $\cG = \curly{G_1,\ldots, G_N}$ be an $\varepsilon$-dense set on $M$. Denote the nearest neighbor of $X_t$ in $\cG$ as $G_{\min}(t) = G_{\min}(X_t) := \arg\min_i \norm{X_t-c_t G_i}$. With probability at least $1-\delta$ for all $G_i$
        \begin{multline}  
    \label{eq:radius_of_interest}
        \frac{2}{3}c^{-2}_t\del{\norm{X_t - c_t G_i}^2 - \norm{X_t - c_t G_{\min}(t)}^2} 
        - 128(\sigma_t/c_t)^2\del{d\log_+ (c_t/\sigma_t) + 4d C_{\log}  + \log \delta^{-1}}
        \\
        \le
          \norm{X_0-G_i}^2 \le
        \\
        9\varepsilon^2 + 
        2c^{-2}_t\del{\norm{X_t - c_t G_i}^2 - \norm{X_t - c_t G_{\min}(t)}^2 }
        + 128(\sigma_t/c_t)^2\del{d\log_+ (c_t/\sigma_t) + 4d C_{\log}  + \log \delta^{-1}}.
        \end{multline}
        \end{restatable}
        This lemma is proved in Appendix~B.3. 

\section{Manifold Approximation}
\label{sec:manifold_approximation}
    In this section, we discuss how to efficiently approximate the manifold $M:=\supp \mu$, using a piecewise polynomial surface $M^*$, an important step in the construction of estimator $\hat{s}(t,x)$ to the true score $s(t,x)$. 

    \begin{definition}
        \label{def:smooth_approximation}
        A piecewise polynomial approximation of $M$ of  order $\beta$ is determined through an integer $N$, a positive scale parameter $\varepsilon_N < r_0$, an $\varepsilon_N$ dense set $\cG = \curly{G_1,\ldots, G_N}$ on $M$, and a set of functions $\curly{\Phi_1^*,\ldots, \Phi_N^*}$, where $\Phi_i^*:B_d(0,\varepsilon_N) \rightarrow \R^D$ are polynomials of order $\beta-1$ satisfying $\Phi^*_i(0) = G_i$. 
        The approximating surface $M^*$ is then defined as $M^* = \bigcup \Phi^*_i\del{B_d(0,\varepsilon_N)},$  and each $\Phi^*_i$ plays the role of a polynomial approximation of the local inverse of the projection on the tangent space $T_{G_i}M$.

Finally, we call an approximation optimal if it satisfies the following two conditions: (i) $\log N \lesssim d\log \varepsilon^{-1}_N$; (ii) for each $i\le N, k \le \beta$ and any $z\in B_d(0,\varepsilon_N)$ it holds that $\norm{\grad^k\Phi_i(z)-\grad^k\Phi^*_{i}(z)}_{op} \le L^*\varepsilon_N^{\beta-k}$ for some constant $L^* > 0$.
    \end{definition}
    \begin{example}    
    \label{example:taylor_support_estimation}
    The natural way to construct an optimal approximation is to fix $\varepsilon_N$, take as $\cG$ a minimal $\varepsilon_N$-dense set $\cG =\curly{G_1,\ldots, G_N}$, and choose $N = \abs{\cG}$. We recall that by Proposition~\ref{prop:covering_number_of_M}, $N \le (\varepsilon_N/2)^{-d} \Vol M$, so condition (i) is satisfied. Then, let $\Phi^*_i$ be the Taylor approximation of $\Phi_{G_i}$ of order $\beta-1$; by~Assumption~\ref{asmp:smooth_manifold} we have the uniform bound $\norm{\Phi_{G_i}}_{C^\beta\del{B_d(0,\varepsilon_N)}} \le L_M$ for all $i\le N$, so  assumption (ii) is also satisfied.
    \end{example}
    We use Example~\ref{example:taylor_support_estimation} to construct the neural network in (ii) of Theorem~\ref{thm:main_result} which leads to a suboptimal dependence in $D$. To improve upon it we use a different approximation based on \cite{aamari2018} which leads to the bound in  (i) of Theorem~\ref{thm:main_result}.

    Having $M^*$, we can approximate $\mu$ with a measure $\mu^*$ supported on $M^*$, by first representing $\mu = \sum \mu_i,$ where $\supp \mu_i \subset M_i$ and then taking $\mu^* = \sum \mu^*_i$, where $\mu^*_i = \Phi^*_i \circ \Phi^{-1}_i \mu_i$ -- is the pushforward of $\mu_i$ from $M_i$ to $M^*_i$. Finally, we define $s^*(t,x)$ as the score function corresponding to the forward process~\eqref{eq:forward_process} with initial distribution $\mu^*$.

 In the following section we construct the local polynomial approximations $\Phi_i^*$ of $\Phi_i$ which are keys in the construction of the neural network approximating $s(t,x)$, as explained in Section \ref{sec:score_approximation}. 

\subsection{Support Estimation}
    \label{sec:support_estimation}
    We start the discussion by answering the first question.  
     The approach described in~Example~\ref{example:taylor_support_estimation} has a major flaw: when $M$ is unknown the dimension of the search space consisting of unknown polynomials $\Phi^*_{G_i}:\R^d\rightarrow \R^D$ scales with $D$ and the number of parameters of the resulting estimator blows up.
    
    However, we have access to samples $\cY= \curly{y_1,\ldots, y_n}$ which already contain some information about $M$, and as has been shown in \cite{aamari2018} one can build an estimator of $M$ using $\cY$ with convergence rate independent of the extrinsic dimension $D$. We modify this argument to make the size of the search space also independent of $D$. %We start by recalling the aforementioned estimator.
    
    As a first step, we fix an arbitrary $N \le n$ large enough and take as $\cG$ the subset $\cG = \curly{G_i:= y_i, i \le N}$. Note that $\cG$ is a set of $N$ i.i.d. samples from the measure $\mu$. \cite{aamari17} has shown that with high probability for an appropriate choice of $\varepsilon_N$ the set $\cG$ is $\varepsilon_N$-dense.
    \begin{proposition}\cite[Lemma III.23]{aamari17}
        \label{prop:samples_are_dense}
        Let the measure $\mu$ satisfy Assumption~\ref{asmp:smooth_measure_on_manifold} and $\cG= \curly{G_1,\ldots, G_N}$ be $N$ i.i.d.\ samples from $\mu$. For any $\varepsilon_N \ge \del{\frac{\beta C_d}{p_{\min}}\frac{\log N}{N}}^{1/d}$ where $C_d > 0$ is large enough and depends only on $d$, for $N$ large enough so that $\varepsilon_N < r_0$, with probability $1-N^{-\frac{\beta}{d}}$, the set $\cG$ is $\varepsilon_N/2$-dense on $M$.
    \end{proposition}
    Applying Proposition~\ref{prop:samples_are_dense} we represent the manifold as $M = \bigcup_{i=1}^N \Phi_{G_i}\del{B_d(0,\varepsilon_N)}$ and aim to approximate $\Phi_{G_i}\del{B_d(0,\varepsilon_N)}$. Further choosing $\varepsilon_N = \del{C_{d,\beta}\frac{p^2_{\max}}{p^3_{\min}}\frac{\log N}{N-1}}^{\frac{1}{d}} > \del{\frac{\beta C_d}{p_{\min}}\frac{\log N}{N}}^{1/d}$ where $C_{d,\beta}$ large enough, following \cite{aamari2018} we consider a local polynomial estimator of $\Phi_{G_i}$ of the following form. For $G_i \in \cG$ we define $\cV_i = \curly{G_j-G_i\big| \norm{G_j-G_i} \le \varepsilon_N}$, and consider a solution of the following minimization problem
    \£
    \label{eq:support_estimation_functional}
    P^*_i, \curly{a^*_{i, S}}_{2\le \abs{S} < \beta} \in \argmin_{P,a_S \le \varepsilon_n^{-1}} \sum_{v_j\in \cV_i}  \norm{v_j - PP^Tv_j - \sum_{S} a_S (P^Tv_j)^S}^2.
    \£
    Here $\argmin$ is taken over all linear isometric embeddings $P:\R^d \mapsto \R^D$ and all vectors $a_S \in \R^D$ that satisfy $\norm{a_S} \le \varepsilon_N^{-1}$, where $S$ is a multi-index $S=(S_1, \dots, S_d)$ such that $2\leq \sum S_k \leq \beta$. Then, in the neighborhood of $G_i$ the manifold is approximated as $M^*_i = \Phi_i^*\del{B_d(0,\varepsilon_N)}$ where
    \£
    \label{eq:estimator_Phi_star_based_on_samples}
    \Phi_i^*(z) = G_i + P^*_iz + \sum a^*_{i,S} z^S,
    \£
    where for $z\in \R^d$ and a multi-index $S$, $z^S = \prod_j z_j^{S_j}$. Note that the $\Phi_i^*$ are not used to construct the estimator, which depends merely on the projections $P_{\mathcal H_i}$ and are only used in the theoretical construction of an approximation of $s(t, \cdot)$ for $t\in [T_k, T_{k+1}]$ by an element of $\mathcal S_k$.

    Finally, we construct our support estimator as a union of polynomial surfaces $M^* = \bigcup M^*_i$.
    \begin{remark}
        \cite{aamari2018} optimize a slightly different functional, minimising over all projections $\pi$ on $d$-dimensional subspaces and symmetric tensors $T_j: \Im \pi \mapsto \R^D$ of order less than $\beta$. This is equivalent to  \eqref{eq:support_estimation_functional} since the operators $P_i$ parameterize the choice of the basis in $\Im \pi$ and $\pi = P_iP^T_i$. Finally, the choice of the basis establishes a one-to-one correspondence between tensors of order $k$ and homogeneous polynomials of degree $k$.
    \end{remark}
    
    The affine subspace $G_i+\Im P^*_i$ is a tangent space $T_{G_i}M^*$. Let $\pr_i$ and $\pr^*_i$ denote the projections on $T_{G_i}M-G_i$ and $T_{G_i}M^*-G_i$ respectively. \cite[Theorem 2]{aamari2018} states that the angle between subspaces $\angle \del{T_{G_i}M, T_{G_i}M^*} := \norm{\pr^*_i-\pr_i}_{op} \lesssim \varepsilon_N^{\beta-1}$.

      We restate this result as a bound on $P^*_i:\R^d \rightarrow \del{T_{G_i}M^* -G_i}$ and show that it is close to the linear map $P_i := \pr_i P^*_i \del{\del{P^*_i}^T\pr_i P^*_i}^{-1/2}$ corresponding to the embedding $P_i:\R^d\rightarrow \del{T_{G_i}M-G_i}$.  
      
     \begin{restatable}{proposition}{ApproximationOfTangentSpace}%\cite[Theorem 2]{aamari2018}. 
     \label{prop:tangent_space_approximation}
     With probability $1-N^{-\beta/d}$ for all $i\le N$, the maps $P_i$ are well-defined linear isometric embeddings satisfying  $\norm{P^*_i-P_i}_{op} \lesssim\varepsilon_N^{\beta-1}$. 
    \end{restatable}   
    Proposition \ref{prop:tangent_space_approximation} is proved in Appendix~C.1.
    
    Using the map $P_i$ we identify $\R^d$ with $T_{G_i}M$ and set $\Phi_i= \Phi_{G_i}:B_d\del{0,r_0}\subset\R^d\rightarrow M$ as an inverse of a projection map $P^T_i\circ\pr_i:M \mapsto \R^d$, in particular, this means that $\Phi_{G_i}(z) = G_i + P_iz + O\del{\norm{z}^2}$. 

    The following statement is a generalization of \cite{divol2022measure},~Lemma A.2 and is proved in Appendix~C.1.
    \begin{restatable}{lemma}{BoundOnDerivativesOfApproximation}
    \label{lemma:approximation_of_Phi}
        There is a large enough constant $L^*$, that does not depend on $D$, such that for any  $\varepsilon_ N < r_0/4$ with probability $1-N^{-\frac{\beta}{d}}$ for all $y_i$, all $ 0 \le k < \beta$ and $z\in B_d\del{0,8\varepsilon_N}$
        \[
        \norm{\grad^k\Phi_i(z)-\grad^k\Phi^*_i(z)}_{op} \le L^*\varepsilon_N^{\beta-k}.
        \]
        In particular $\norm{y-y^*} \lesssim \varepsilon_N^{\beta}$ and ${\angle \del{T_{y^*}M^*,T_{y}M} \lesssim \varepsilon_N^{\beta-1}}$ for ${y = \Phi_i(z)\in M, y^* = \Phi^*_i(z)\in M^*}$.
    \end{restatable} 

    As a consequence,  combined with Proposition~\ref{prop:samples_are_dense},  the Hausdorff distance 
    \[
    d_H\del{M,M^*} = \max\del{\sup_{y\in M} \dist(y,M^*), \sup_{y^*\in M^*} \dist(y^*,M)} \lesssim \varepsilon_N^\beta.
    \]

    Note that $\varepsilon_N = \del{C_{d,\beta}\frac{p^2_{\max}}{p^3_{\min}}\frac{\log N}{N-1}}^{\frac{1}{d}}$, implying that $2d\log \varepsilon^{-1}_N \ge \log N$ if $N$ large enough. Together with Lemma~\ref{lemma:approximation_of_Phi} this shows that the constructed approximation satisfies Definition~\ref{def:smooth_approximation}.
    
    Let us take a closer look  at the solution of the optimization problem \eqref{eq:support_estimation_functional}. When $\dim \spn \cV_i \le d$ $P_iP_i^Tv = v$ for all $v \in \cV_i$, while when $\dim \spn \cV_i > d$ we claim that we can reduce the problem to a low-dimensional $\spn \cV_i$. The following propositions justify this.  
    \begin{restatable}{proposition}{SpaceOfSolutionsManifoldOptimization}
\label{prop:support_estimator_solution_properties_1}
            If $\dim \spn \cV_i > d$, then there is a solution $P_i^*, a^*_{i, S}$ of \eqref{eq:support_estimation_functional} such that $\Im P^*_i \subset \spn \cV_i$ and $a^*_{i, S} \in \Ker P_i^* \cap \spn \cV_i$.
        \end{restatable}
    \begin{restatable}{proposition}{NumberOfPointsInVi}
    \label{prop:support_estimator_solution_properties_2}
        There is a large enough constant $C_{\dim}$ that does not depend on $D$ such that with probability $1-N^{-\frac{\beta}{d}}$ for all $y\in M$ the size of the set 
        \[
        \cV_y = \curly{G_i: \norm{G_i-y} \le \varepsilon_N}
        \]
        is bounded by $
        \abs{\cV_y} \le C_{\dim} \log N.$
        In particular for all $G_i\in \cG$ holds
        $
        \abs{\cV_i} \le C_{\dim} \log N.
        $ 
    \end{restatable}
    The above two propositions are proved in Appendix~C.1.

    Together these statements show that with probability $1-N^{-\beta/d}$ there is a piece-wise polynomial surface $M^* = \bigcup_{i=1}^N M^*_i$ approximating $M$ with error $\varepsilon_N^\beta \simeq \del{\log N / N}^{\beta/d}$ such that each polynomial piece $M^*_i$ lies in affine subspace $G_i + \spn \cV_i$ where $\cV_i = \curly{G_j-G_i\big| \norm{G_j-G_i} \le \varepsilon_N}$ which is at most $C_{\dim}\log N$ dimensional.   

    Finally, we recall that by taking a partition $\mu = \sum_{i=1}^N\mu_i$ of the measure $\mu$ subordinated to the cover $M = \bigcup_{i=1}^N M_i$ we construct the measure $\mu^* = \mu^*_i$ where $\mu^*_i$ is  the pushforward of $\mu_i$ under $\Phi^*_i \circ \Phi^{-1}_i$. Moreover,  the functions $\Phi^*_i \circ \Phi^{-1}_i$ provide  transport maps which by Lemma~\ref{lemma:approximation_of_Phi} guarantees that $W_2(\mu, \mu^*) \le L^*\varepsilon^\beta_N$. 
\subsection{Manifold Approximation and the Score Function}
\label{sec:manifold_approximation_and_the_score}
    In this section, we answer the second and third questions raised at the beginning of~Section~\ref{sec:manifold_approximation}. 

    To control the score matching error between $s^*(t,x)$ and $s(t,x)$, we use a general result bounding the score-matching loss between two compactly supported distributions.

        \begin{restatable}{proposition}{wdvssml}
    \label{prop: wd_vs_sml}
    Let $P, Q$ be arbitrary compactly supported measures s.t.\ $W_2(P,Q) < \infty$, where $W_2$ is the quadratic Wasserstein distance. 

    Letting $X\sim P$, $Y\sim Q$, 
    we write $X_t = c_t X +\sigma_t Z_D$, and $Y_t= c_t Y + \sigma_t Z_D$, where $Z_D\sim\cN(0, Id)$ be two standard Ornstein-Uhlenbeck processes initialised from $X,Y$ respectively. These random variables are absolutely continuous w.r.t.\ the Lebesgue measure with densities $p(t,x), q(t,x)$ respectively.  

    If we consider $\grad \log q(t,x)$ as an estimator of $s(t,x) = \grad \log p(t,x)$ on the interval $t\in [t_{\min},t_{\max}]$, then the corresponding score matching loss is bounded by
    \[
    \int_{t_{\min}}^{t_{\max}}\int_{\R^d} \|\grad\log p(t, x)-\grad \log q(t, x)\|^2 p(t, x)dxdt \le W^2_2(P,Q)\frac{c^2_{t_{\min}}}{4\sigma^2_{t_{\min}}}.
    \]
\end{restatable}
\begin{proof}
    We believe that this statement can be traced back at least to~\cite{villani2008optimal}, however, we could not find it in this form. Therefore, we provide a proof in Appendix~C.2.%~\ref{apdx:wd_vs_sml}.
\end{proof}
\begin{remark}
    The bound is tight, to see this take $P = \delta_0$, and $Q = \delta_{x}$, for $\norm{x} = \varepsilon$.
\end{remark}

 In Section~\ref{sec:support_estimation} we constructed a measure $\mu^*$ satisfying $W_2(\mu,\mu^*) \le L^*\varepsilon^\beta_N$, so if we denote as $s^*$ the score function corresponding to $\mu^*$ by~Proposition~\ref{prop: wd_vs_sml}
\[
\int_{t_{\min}}^{t_{\max}}\int_{\R^d} \|s(t,x)-s^*(t,x)\|^2 p(t, x)dxdt \le L^*\varepsilon_N^{2\beta}\frac{c^2_{t_{\min}}}{4\sigma^2_{t_{\min}}}.
\]

Finally, we answer the last question and discuss how the substitution of the true score by the approximation $s^*(t,x)$ affects the regularity bounds on the score function derived in the Section~\ref{sec:high_probability_bounds}. Let us present the bound on the correlation between Gaussian noise $Z_D$ and the correction brought by the substitution of $M$ by $M^*$

\begin{restatable}{proposition}{TaylorGP}    
    \label{prop:TaylorGP}
    Let $M^*$ be a piece-wise polynomial approximation satisfying Definition~\ref{def:smooth_approximation}. Then for all positive $\delta < 1$ with probability at least $1-\delta$ for all $i\le N$ and $y\in M_i$
    \[
    \abs{\innerproduct{Z_D}{y- \Phi_i^*\circ \Phi_i^{-1}(y)}} \le L^*\varepsilon_N^{\beta}\sqrt{2d\log \varepsilon^{-1}_N + 2\log 2\delta^{-1}}.
    \]
\end{restatable}
The proof is provided in Appendix~C.2.

All the bounds in Section~\ref{sec:high_probability_bounds} follow from Proposition~\ref{prop:maximum_of_Gaussian_Process_on_Manifold}, therefore the statements presented in the previous section still hold if $t\ge \varepsilon^{2\beta}_N$ in the case of $M^*$ too, albeit with slightly different constants that are still independent of $D$. We present these results in Appendix~C. 

\section{Wasserstein Convergence}
\label{sec:W_main}

In this section, we analyze the convergence under the $W_1$ metric of the diffusion model that uses the estimator obtained in Theorem~\ref{thm:score_approximation_1} and prove~Theorem~\ref{cor:main_W_result}. We prove a more general statement that does not depend on the specific construction of our estimator and works for any estimator satisfying the assumptions below.

\begin{Assumption}
\label{asmp:score}
There exist constants $\varepsilon_{\mathrm{denoiser}}>0$, $C_{\mathrm{tail}}\ge 1$ such that the
score function estimate $\hat s$ verifies:
\begin{enumerate}
    \item[(i)] $\hat s$ has a continuous denoising loss bounded by $\varepsilon^2_{\mathrm{denoiser}}$, i.e.
    \£
    \label{eq:continuous_denoiser_loss}
    \int_{\underline T}^{\overline T}\sigma_t^2
    \expectation\|\hat s(t,X_t)-s(t,X_t)\|^2\,dt
    \le \varepsilon_{\mathrm{denoiser}}^2,
    \£
    \item[(ii)] for every $t\in[\underline T,\overline T]$ and some $\eta >0$ $\hat s$ satisfies the following high probability bound
    \£
    \label{eq:tail_bound_W1}
    \rP\del{
    \|\sigma_t \hat s(t,X_t) + Z_D\|
    >
    C_{\mathrm{tail}}
    \log \underline T^{-1}
    \sqrt{C_W+\log \eta^{-1}}
    }
    \le \eta,
    \£
    for some $C_{\mathrm{tail}}, C_W > 0$
\end{enumerate}
\end{Assumption}
Under this assumption, we show the following.
\begin{theorem}
    \label{thm:W_1_convergence_theorem}
    Assume that measure $\mu$ satisfies the same assumptions as in Theorem~\ref{thm:score_approximation_1}, let $s$ be the score function corresponding to $\mu$ and $\hat{s}$ be its approximation satisfying Assumption~\ref{asmp:score} with $\eta= O\del{\varepsilon^{-4}_{\mathrm{denoiser}}\underline{T}^{-1}\del{\log \underline{T} + \overline{T}}}$.
    
    Then, there exists a discretization scheme requiring $L= O\del{\varepsilon^{-2}_{\mathrm{denoiser}}\del{\log \underline{T} + \overline{T}}}$ steps to generate one sample from a measure $\hat{\mu}$ which satisfies
\£
\label{eq:thm W_1 convergence}
&W_1(\mu, \hat{\mu}) \lesssim
\sqrt{D}e^{-\overline{T}} + 
    \sqrt{\underline{T}}\log \underline{T}^{-1}\sqrt{C_W + \log (\varepsilon^{-1}_{\mathrm{denoiser}} \cdot \underline{T}^{-1}\cdot \overline{T})}
    \\
    &\quad
    +
    \varepsilon_{\mathrm{denoiser}}\cdot \log (\overline{T} \cdot \underline{T}^{-1})
    \log \underline{T}^{-1}\sqrt{C_W + \log (\varepsilon^{-1}_{\mathrm{denoiser}} \cdot \underline{T}^{-1}\cdot \overline{T})}\sqrt{(d\log \underline{T}^{-1} + C_{W})(\log {\underline T^{-1}} + \overline T) } \nonumber
\£
\end{theorem}
In the rest of the section, we prove Theorem~\ref{thm:W_1_convergence_theorem}. In Section~\ref{sec:randomized_mesh} we construct the discretization mesh, in Section~\ref{sec:modified_discretization_scheme} we introduce the discretization scheme, and in Section~\ref{sec:w1_convergence_proof} we prove Theorem~\ref{thm:W_1_convergence_theorem} for this specific discretization scheme. 

\subsection{Randomized Discretization Mesh}
\label{sec:randomized_mesh}
In this section, we construct a randomized discretization mesh that inherits certain useful properties of those used in~\cite{oko2023}~and~\cite{benton2024nearly}. The randomization is needed to ensure that the denoising loss evaluated at the grid points can be controlled by its continuous version~\eqref{eq:continuous_denoiser_loss}.

Following~\cite{benton2024nearly} we introduce $\kappa < 1/4$, which we fix later. We consider the same grid as in Section \ref{sec:score_approximation} with 
$\overline T := T_K > T_{K-1} >\cdots > T_0 :=\underline T$, a crude dyadic partition
defined by $T_{k+1}=2T_k$ and without loss of generality we assume that $ \log (\overline T/\underline T) \in \log_2\mathbb N$. %for $0\le k\le K-1$ and $T_{K} < 2\underline T$. 

\begin{figure}
    \centering
    \includegraphics[width=1.0\linewidth]{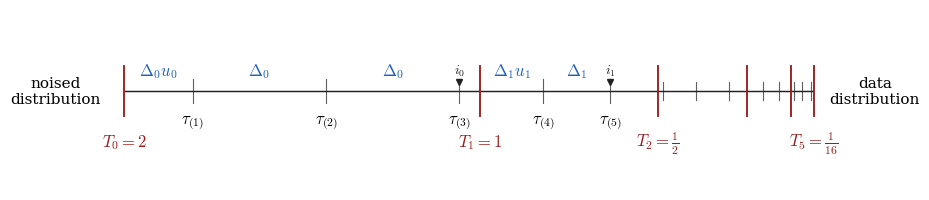}
    \caption{Illustration of constructed partition. The interval $[\underline{T},\overline{T}]$ is first split into dyadic blocks $[T_{k}, T_{k+1}]$ (red), and each block is split into smaller blocks of size $\Delta_k:=(\kappa/4)\min(1,T_{k})$ with a random indent $u_k\Delta_k$, where $u_k \sim U[0,1]$.}
    \label{fig:placeholder}
\end{figure}
For each block $[T_{k},T_{k+1}]$, let
$\Delta_k:=(\kappa/4)\min(1,T_{k})$, sample $u_k\sim U[0,1]$
independently and define
\[
\tau_{k,1}:=T_{k+1}-\Delta_k u_k,
\qquad
\tau_{k,\ell}:=\tau_{k,1}-(\ell-1)\Delta_k,\qquad \ell\ge 2.
\]
Letting $L_k:=\max\curly{\ell\ge 1:\tau_{k,\ell}>T_{k}}$, the points $\curly{\tau_{k,\ell}}_{k\le K, \ell \le L_k}$ will form the discretization grid that we will work with. We write 
\[
\overline T > \tau_{(1)} > \cdots > \tau_{(L)} > \underline T,
\] for
the grid points listed in decreasing order, where 
$L:=\sum_{k=0}^{K-1} L_k$ is their overall 
number. 
We also define $i_k := \curly{i: \tau_{(i)} = \tau_{K-k,1}}, k\le K$, the index of the largest grid point contained in the interval $[T_{K-k}, T_{K-k+1}]$ and  $i_{K} := L$.

We define $\gamma_i = \tau_{(i)} - \tau_{(i+1)}$ and introduce the score matching error corresponding to the constructed grid 
\[
\varepsilon^2_{denoiser}(u):=\sum_{i=1}^{L-1}\gamma_i \sigma^2_{\tau_{(i)}} \expectation\norm{\hat{s}(\tau_{(i)}, X_{\tau_{(i)}}) -s(\tau_{(i)}, X_{\tau_{(i)}})}^2
\]

The following proposition summarizes the properties of the constructed partition:

\begin{proposition}
\label{prop:discretization_setup}
Let $\{\tau_{(1)}, \dots, \tau_{(L)}\}$  be the random grid constructed above. Then we have the following:
\begin{itemize}
    \item[(i)] the condition introduced in~\cite[Page 5]{benton2024nearly} holds, i.e.\
\[
\tau_{(i)} - \tau_{(i+1)} \le \kappa \min(1,\tau_{(i+1)}), \quad 0 \le i \le M-1;
\]
\item[(ii)] 
\(
\tau_{(1)} > \overline{T} - \kappa, \quad \tau_{(L)} < (1+\kappa) \underline{T}, \quad \tau_{(i_k)}/\tau_{(i_{k+1})} < 4;
\)
\item[(iii)]
The total number of elements $L$ does not exceed
\£
\label{eq:M_bound}
L\le K-1+\frac{4}{\kappa}
\left(\overline T+\left\lceil \log_2 \underline T^{-1}\right\rceil \right) = O\del{\kappa^{-1}(\log {\underline T^{-1}} + \overline 
T)}, and 
\£
\item[(iv)] the discretised denoiser loss evaluated at the grid points $\tau_{(i)}$ can be controlled using~\eqref{eq:continuous_denoiser_loss}.  
\£
\label{eq:discrete_denoiser_loss}
\expectation_u \varepsilon^2_{denoiser}(u) = \expectation_u\sum_{i=1}^{L-1}\gamma_i \sigma^2_{\tau_{(i)}} \expectation\norm{\hat{s}(\tau_{(i)}, X_{\tau_{(i)}}) -s(\tau_{(i)}, X_{\tau_{(i)}})}^2
\\
\le
2\int_{\underline T}^{\overline{T}} \sigma^2_t \expectation\norm{\hat{s}(t, X_{t}) -s(t, X_{t})}^2 dt
\le 
2\varepsilon^2_{\mathrm{denoiser}}.
\£
\end{itemize}

\end{proposition}

Finally, we define the backward discretization times $0=t_1 < t_2 <\ldots < t_{L+1} = \tau_{(1)}$ and $t_i = \tau_{(1)} - \tau_{(i)}, i\le L$, and use them to define the discrete backward process in the next section.

\subsection{Modification of Discretization Scheme}
\label{sec:modified_discretization_scheme}
Let us introduce
\£
\label{eq:R_in_W1_convergence}
R_t(\eta) := \min\del{1, C_{\mathrm{tail}} (\sigma_t/c_t)\log \underline{T}^{-1}\sqrt{C_W  + \log \eta^{-1}}}.
\£
We recall, that by Assumption~\ref{asmp:score}  with probability at least $1-\eta$ the estimator $\hat{e}$ satisfies
\£
\label{eq:self-consistency}
\norm{\hat{e}(t, X_t) - X_0} &\le R_t(\eta),
\£
in particular, implying that for any $s, t > 0$, with probability at least $1-\eta$
\£
\label{eq:high_prob_difference}
\norm{\hat{e}(t, X_t) - \hat{e}(s, X_s)} &\le 2R_{\max(s,t)}(\eta/2).
\£

We propose the following modification of the score function which makes it self-consistent. Along the discretized trajectory $(y_{t_1},y_{t_2},\ldots, y_{t_L})$ we collect the observations \\
 $(\hat{e}(\tau_{(1)}, y_{t_1}),  \hat{e}(\tau_{(2)}, y_{t_2}),\ldots, \hat{e}(\tau_{(i)}, y_{t_i}))$ and we construct a modified estimator $\bar{e}(\tau_{(i)},y_{t_1:t_i})$, where $y_{t_1:t_i} = (y_{t_1},\ldots, y_{t_i})$ for all $i \leq L$,  which satisfies~\eqref{eq:high_prob_difference} almost surely instead of with high probability. Choosing $\eta = \underline{T}\delta/L^2$ in \ref{eq:high_prob_difference} we define 
the estimator of interest as:
\begin{equation}
\label{eq:definition_e_bar}
\bar{e}(\tau_{(i)}, y_{t_{1}:t_i}) :=
    \begin{cases}
        \hat{e}(\tau_{(1)}, y_{t_1}) &\text{ if } i = 1, 
        \\
        \hat{e}(\tau_{(i)}, y_{t_{i}}) &\text{ if }
            \forall j < i: \norm{\bar{e}(\tau_{(j)}, y_{t_{1}:t_j})-\hat{e}(\tau_{(i)}, y_{t_i})} 
             \le 2R_{\tau_{(j)}}\left(\frac{\underline T \delta }{2L^2}\right),
        \\
        \bar{e}(\tau_{(i-1)}, y_{t_1:t_{i-1}}) &\text{ otherwise}.
    \end{cases}
\end{equation}

That is, at each step if according to ~\eqref{eq:high_prob_difference} the new observation $\hat{e}(\tau_{(i)}, y_{t_{i}})$  is  an outlier, we discard it and use the previous estimate $\bar{e}(\tau_{(i-1)}, y_{t_1:t_{i-1}})$. 

This construction guarantees that for all $j < i$
\[
\label{ineq:self-consistent}
\norm{\bar{e}(\tau_{(i)}, y_{t_1:t_{i}}) - \bar{e}(\tau_{(j)}, y_{t_1:t_{j}})} 
\le 2R_{\tau_{(j)}}(\eta/2).
\]

Since $\eta = \underline{T}\delta/L^2$, in ~\eqref{eq:high_prob_difference}, 
the event
\£
\label{eq:denoiser_score_bar_s}
\bigg\{\norm{\hat{e}({\tau_{(i)}}, X_{\tau_{(i)}}) - \hat{e}({\tau_{(j)}}, X_{\tau_{(j)}})} &\le 2R_{\tau_{(j)}}\del{\frac{\underline{T}\delta}{2L^2}}, \qquad \text{for all pairs $j<i$} \bigg\}
\£
has probability at least $1-\underline{T}\delta/L^2$ for every $i$.

In particular, with probability at least $1-\underline{T}\delta/L$ for all $i\le L$ simultaneously $\hat{e}(\tau_{(i)}, X_{\tau_{(i)}}) = \bar{e}(\tau_{(i)}, X_{\tau_{(i)}:\tau_{(1)}})$, with $X_{\tau_{(i)}:\tau_{(1)}} = (X_{\tau_{(i)}},\ldots, X_{\tau_{(1)}})$. Note that the above probability is conditional on $\tau_{(1)}, \cdots, \tau_{(L)}$ which are random. 

Substituting the bound above, and using that $\norm{\bar{e}} \le \norm{\hat{e}} \le C_1$ , a.s., we can bound the discretized denoising loss corresponding to 
 $\bar{s}(\tau_{(i)}, x_{\tau_{(i)}:\tau_{(1)}}) = [c_t\bar{e}(\tau_{(i)}, x_{\tau_{(i)}:\tau_{(1)}})-x_{\tau_{(i)}}]/\sigma^2_{\tau(i)}$ 
as
\[
\expectation_u\sum_{i=1}^{L-1}\gamma_i \sigma^2_{\tau_{(i)}} \expectation\norm{\bar{s}(\tau_{(i)}, X_{\tau_{(i)}:\tau_{(1)}}) -s(\tau_{(i)}, X_{\tau_{(i)}})}^2 
\\
\le  \expectation_u \sum_{i=1}^{L}  
\del{\gamma_i \sigma^2_{\tau_{(i)}} \expectation\norm{\hat{s}(\tau_{(i)}, X_{\tau_{(i)}}) -s(\tau_{(i)}, X_{\tau_{(i)}})}^2 } \\
  +  \expectation_u\left( \mathbb P \left[ \exists i\leq L ; \hat{e}(\tau_{(i)}, X_{\tau_{(i)}}) \neq \bar{e}(\tau_{(i)}, X_{\tau_{(i)}:\tau_{(1)}}) |\tau_{(1)}, \cdots , \tau_{(L)}\right]\sum_{i=1}^{L}   \frac{\gamma_i c_{\tau_{(i)}} }{\sigma^2_{\tau_{(i)}} }\right) \\
\le \expectation_u\sum_{i=0}^{L}\del{\gamma_i \sigma^2_{\tau_{(i)}} \expectation\norm{\hat{s}(\tau_{(i)}, X_{\tau_{(i)}}) -s(\tau_{(i)}, X_{\tau_{(i)}})}^2 +  {\sigma^{-2}_{\tau_{(i)}}\frac{\underline{T}\delta}{L}}}
\le 2(\varepsilon^2_{\mathrm{denoiser}} + \delta),
\]
where to get the last inequality we used that $\sigma^2_{\tau_{(i)}} \ge \tau_{(i)} \ge \underline{T}$.

Following~\cite{potaptchik2024linearconvergencediffusionmodels}, we consider the
DDPM discretization over the interval $[0,\tau_{(1)}]$
\begin{equation}
\label{eq:discrete_version_of_final_SDE}
\begin{cases}
    
\bar Y_{t_{i+1}}\mid \bar Y_{t_1:t_i}
=
c_{\gamma_i}^{-1}\bar Y_{t_i}
+ \frac{\sigma_{\gamma_i}^2}{c_{\gamma_i}}\bar s(\tau_{(i)},\bar Y_{t_1:t_i})
+ \sigma_{\gamma_i}\frac{\sigma_{\tau_{(i+1)}}}{\sigma_{\tau_{(i)}}}Z_i,
Z_i\stackrel{i.i.d.}{\sim}\cN(0,\Id_D), \quad \quad i < L,
\\
Y_{t_{L+1}} \mid \bar Y_{t_1:t_L}= \bar{e}(\tau_{(L)}, \bar{Y}_{t_1 : t_L}),
\\
\bar Y_0\sim \cN(0,\Id_D).
\end{cases}
\end{equation}

As it is shown in~\cite{potaptchik2024linearconvergencediffusionmodels}, the discretized process admits a
continuous interpolation on each interval $[t_i,t_{i+1})$, for $i\leq M$
\begin{equation}
\label{eq:bridge_interpolation}
\begin{cases}
d\bar{Y}_t =
\square{
\bar{Y}_t+2\bar s(\tau_{(1)}-t,\bar{Y}_t\mid \tau_{(i)},\bar{Y}_{t_1:t_i})
}\,dt+\sqrt{2}\,dB'_t,
& t\in[t_i,t_{i+1}),
\\
\bar{Y}_0\sim \cN(0,\Id_D), %\quad \textcolor{red}{\bar Y_{\tau_{(1)}} = \bar e( \tau(L), \bar Y_{t_1: t_M})} 
\end{cases}
\end{equation}
where for $t < \tau_{(i)}$,
\[
\bar s(t,x_t\mid \tau_{(i)},x_{\tau_{(1)}:\tau_{(i)}})
=
c_{{\tau_{(i)}}-t}^{-1}\frac{\sigma_{\tau_{(i)}}^2}{\sigma_t^2}\bar s({\tau_{(i)}},x_{\tau_{(1)}:\tau_{(i)}})
- \frac{x_t-c_{{\tau_{(i)}}-t}^{-1}x_{\tau_{(i)}}}{\sigma_t^2} = \frac{c_t \bar{e}({\tau_{(i)}}, x_{\tau_{(1)}:\tau_{(i)}})-x_t}{\sigma^2_t}.
\]
Equivalently, once the process reaches time $t_i$, the next step simulates for time $\gamma_i$ the
Brownian bridge on $[t_i,\tau_{(1)}]$ connecting $\bar Y_{t_i}$ and $\bar e(\tau_{(i)}, \bar Y_{t_1:t_i})$.

\subsection{$W_1$ Convergence}
\label{sec:w1_convergence_proof}
We are now ready to prove Theorem~\ref{thm:W_1_convergence_theorem}.
    We recall that by construction, the indices 
    $i_k, k\le K+1$
    satisfy $\tau_{(i_k)}/\tau_{(i_{k+1})} < 4$. Following~\cite{oko2023} we use a coarser grid $0=t_{i_1} < t_{i_2} <\ldots <t_{i_{K}} = t_{L}=\tau_{(1)}-\tau_{(L)}$ to construct the coupling. More precisely, let us introduce auxiliary processes $\bar{Y}^{(k)}$  as follows: for $k \leq K$
    \begin{equation}
    \label{eq: Y^l definition}
        \begin{cases}
            d\bar{Y}^{(k)}_t = \square{\bar{Y}^{(k)}_t + 2s(\tau_{(1)}-t, \bar{Y}^{(k)}_t)}dt + \sqrt{2}dB_t', 
            & t\in [0, t_{i_k}], \, k>1 %\tau_{(L)}
            \\
            d\bar{Y}^{(k)}_t = \square{\bar{Y}^{(k)}_t + 2\bar{s}(\tau_{(1)}-t, \bar{Y}^{(k)}_t|\tau_{(L)}-t_i,\bar{Y}^{(k)}_{t_1:t_i})}dt + \sqrt{2}dB_t',  %\tau_{(L)}
            & t\in (t_i,t_{i+1}] \text{ for } i \ge i_k, 
            \\
            \bar{Y}^{(k)}_0 \sim p_{\tau_{(1)}}.
        \end{cases}
    \end{equation}
    In other words, the process $Y^{(k)}$ first coincides with the true backward process $Y$ given by~\eqref{eq:exact_backward_dynamic} on the interval $[0,t_{i_k}]$, and for the rest of time $(t_{i_k}, \tau_{(1)}]$ it follows the approximate dynamic~\eqref{eq:bridge_interpolation}.

    We bound the Wasserstein distance between the target distribution $X_0 \stackrel{dist.}{=} Y^{(1)}_{\tau_{(1)}}$ and its approximation $\bar{Y}_{\tau_{(1)}}$ as
    \[
    W_1(\bar{Y}_{\tau_{(1)}}, X_0) \le     
    W_1(\bar{Y}_{\tau_{(1)}}, \bar{Y}^{(1)}_{\tau_{(1)}}) 
     +
    \sum_{k=1}^{K-1} 
    W_1(\bar{Y}^{(k)}_{\tau_{(1)}}, \bar{Y}^{(k+1)}_{\tau_{(1)}}) 
    + 
    W_1(\bar{Y}^{(K)}_{\tau_{(1)}}, X_0).
    \]
    In what follows, we explain the structure of the argument and defer the technical details to Appendix~E. The proof is based on the above interpolation between the ideal reverse process and the discretized process. The interpolation isolates three distinct sources of error.

First, the discretized process is initialized from a standard Gaussian, whereas the ideal reverse process should be initialized from $p_{\tau_{(1)}}$. Since $\tau_{(1)} \ge \overline T - 1$, this error by~\cite{benton2024nearly} is exponentially small in $\overline T$:
\[
\expectation_u W_1\del{\bar Y_{t_{L}},\bar Y^{(1)}_{t_{L}}}
\lesssim \sqrt D e^{-\overline T}.
\]

Second, even if the process follows the exact reverse dynamics until the last grid point, it still has to replace the final transition from $X_{\tau_{(L)}}$ to $X_0$ by the denoiser. Since $\tau_{(L)} \le 2\underline T$, the high-probability control from Assumption~\ref{asmp:score} gives
\[
\expectation_u W_1\del{\bar Y^{(K)}_{t_{L}},X_0}
\lesssim
\sqrt{\underline T}\log \underline T^{-1}
\sqrt{C_W+\log (L \cdot \underline T^{-1} \cdot \delta^{-1})}
\]

Finally, the remaining terms $W_1(\bar{Y}^{(k)}_{\tau_{(1)}}, \bar{Y}^{(k+1)}_{\tau_{(1)}})$ are only  impacted by the discretization over the interval $[T_{K-k},T_{K-k+1}]$. The self-consistency modification of the denoiser ensures that the endpoints of the bridge remain localized, while Girsanov's theorem reduces the comparison of the two path laws to the accumulated score error on that block. Averaging over the randomized mesh then allows the discrete score error to be controlled by the continuous denoising loss. This gives, for every dyadic block $k$,
\[
\expectation_u W_1\del{
\bar Y^{(k)}_{t_{L}},
\bar Y^{(k+1)}_{t_{L}}
} \le 4 R_{\tau_{(i_k)}}\del{\frac{\underline{T}\delta}{2L^2}}
\TV\del{\bar{Y}^{(k)}_{[0, t_L]}, \bar Y^{(k+1)}_{[0, t_L]}}
\\
\lesssim
\log \underline T^{-1}
\sqrt{C_W+\log (L \cdot \underline T^{-1} \cdot \delta^{-1})}
\sqrt{
\varepsilon_{\mathrm{denoiser}}^2
+\delta
+\kappa d(\log \underline T^{-1}+C_{\log})
(\log \underline T^{-1}+\overline T)
}.
\]

Summing the above bounds over the dyadic blocks gives that
\[
\expectation_u W_1(\bar{Y}_{\tau_{(1)}}, X_0) 
\lesssim
\sqrt D e^{-\overline T}
+
\sqrt{\underline T}\log \underline T^{-1}
\sqrt{C_W+\log (L \cdot \underline T^{-1} \cdot \delta^{-1})}
\\
+
K\log \underline T^{-1}
\sqrt{C_W+\log (L \cdot \underline T^{-1} \cdot \delta^{-1})}
\sqrt{
\varepsilon_{\mathrm{denoiser}}^2
+\delta
+\kappa (d\log \underline{T}^{-1} + C_{W})
(\log \underline T^{-1}+\overline T)
}.
\]
By Proposition~\ref{prop:discretization_setup}, $
L=O\del{\kappa^{-1}(\log \underline T^{-1}+\overline T)}
$ and $
K=O\del{\log \underline T^{-1}+\log \overline T}$. 

Let $\hat{\mu}_u = \textrm{Law}(\bar{Y}_{\tau_{(1)}}|u)$, and $\hat{\mu} = \expectation_u \hat{\mu}_u$ be the law of $\bar{Y}_{\tau_{(1)}}$. Then, substituting $
\kappa=\delta=\varepsilon_{\mathrm{denoiser}}^2
$ into the previous display and applying convexity of Wasserstein distance we get
\[
&W_1(\mu, \hat{\mu}) \le \expectation_u W_1(\mu, \hat{\mu}_u) \lesssim
\sqrt{D}e^{-\overline{T}} + 
    \sqrt{\underline{T}}\log \underline{T}^{-1}\sqrt{C_W + \log (\varepsilon^{-1}_{\mathrm{denoiser}} \cdot \underline{T}^{-1}\cdot \overline{T})}
    \\
    &\quad
    +
    \varepsilon_{\mathrm{denoiser}}
    \log (\overline{T} \cdot \underline{T}^{-1})
    \log \underline{T}^{-1}\sqrt{C_W + \log (\varepsilon^{-1}_{\mathrm{denoiser}} \cdot \underline{T}^{-1}\cdot \overline{T})}\sqrt{(d\log \underline{T}^{-1} + C_{W})(\log {\underline T^{-1}} + \overline T) } \nonumber
\]

    \section{Alternative Approach: Kernel-Based Score Function Approximation} \label{sec:alternative:Divol}
    An alternative approach to building estimator $\hat{s}$ develops the idea that we used in case $d\le 2$, and based on the kernel-density estimator of $\mu$ constructed by~\cite{divol2022measure}. 

    We assume usual assumptions, $\mu$ is a measure satisfying Assumptions~\ref{asmp:smooth_manifold}--\ref{asmp:smooth_measure_on_manifold} supported on a $d$-dimensional manifold $M$, and we observed $n$ i.i.d.\ samples $\cY=\curly{y_1,\ldots,y_n}$ from $\mu$. 
    
    In Theorem~3.7, \citet{divol2022measure} constructs the kernel-density estimator $\hat{\mu}$ in two stages. As a first step, applying results of~\cite{aamari2018} that we discussed in~Section~\ref{sec:support_estimation} they solve optimization problem~\eqref{eq:support_estimation_functional} and construct an approximating surface $M^* = \cup M^*_i$. Then, using smooth partitioning functions $\chi_i$, they construct a smooth approximation $\widehat{\Vol}(dy^*)$ to uniform measure $dy$ as $\widehat{\Vol}(dy^*)=\sum \chi_i(y^*)dy^*$, i.e. for $\phi\in C^1_0(\R^D)$
    \[
    \int_{M^*} \phi(y^*)\widehat{\Vol}(dy^*) = \sum \int_{M^*_i} \phi(y^*)\chi_i(y^*)dy^*.
    \]
    As a second step, they construct a smoothing kernel $K_n(x)$ on $\R^D$ as a rescaling of a special kernel $K$, see~\cite[Section 3]{divol2022measure} for precise description. Denote as
    \[
    \nu_n = \nu_n(\cY) := \frac{1}{n}\sum_{y_i\in \cY} \delta_{y_i}
    \]
    the empirical density function. The estimator $\hat{\mu}_n$ is defined as
    \[
    \hat{\mu}_n = K_n * \del{\frac{\nu_n}{K_n * \hat{\Vol}M}}.
    \]
    \begin{theorem}\cite[Theorem 3.7]{divol2022measure}
        For $d \ge 3$, and any $1\le p \le \infty$
        \[
        \expectation_{\cY\sim\mu^{\otimes n}} W_p\del{\hat{\mu}_n, \mu} \lesssim n^{-\frac{\alpha+1}{2\alpha+d}}.
        \]
    \end{theorem}
    We define $\hat{s}_n$ as the score function of $\hat{\mu}_n$, then applying Proposition~\ref{prop: wd_vs_sml} 
    \[
     4\sigma^2_{T_k}\int_{T_k}^{T_{k+1}}\int_{\R^d} \|\hat{s}_n(t,x)-s(t,x)\|^2 p_t(x)dxdt \le W^2_2(\mu,\hat{\mu}_n).
    \]
    Since $T_{k+1} = 2T_{k}$ we have $\sigma_t \le 2\sigma_{T_k}$ for all $t\in [T_k, T_{k+1}]$. So, taking expectation and summing over all $k\le K = \log (n)$
    \[
    \expectation_{\cY\sim \mu^{\otimes n}}\int_{\underline{T}}^{\overline{T}}\int_{\R^D} \sigma^2_t\norm{\hat{s}_n(t,x) - s(t,x)}^2 p(t,x)dxdt  \lesssim n^{-\frac{\alpha+1}{2\alpha+d}}\log n.
    \]
    At first glance, this gives an alternative and very short proof of~Theorem~\ref{thm:score_approximation_1}. However, in this case, diffusion models are merely used to sample from the already known distribution $\hat{\mu}_n$, while in the results we presented, we indeed learn the distribution by solving the empirical risk minimization problem, which is much closer to what happens in practice.

    \section{Conclusion and Future work} \label{sec:conclusion}
    
    We have presented a novel approach for analyzing diffusion models under the manifold hypothesis in a high-dimensional setting. This approach allows us to show that: (i)~neural networks can approximate the normalized score function $\sigma_t s(t,x)$ with the ambient-dimension-free convergence rate $n^{-2(\alpha+1)/(2\alpha+d)}\polylog(n)$ w.r.t.\ the score matching loss; and (ii) diffusion models achieve the optimal (up to $\polylog$) convergence rate $n^{-(\alpha+1)/(2\alpha+d)}$ in the $W_1$ metric.

    Our results can be extended in several directions. Firstly, it is unclear if Assumption~\ref{asmp:smooth_measure_on_manifold_in_proofs} is necessary. We believe that similarly to~\cite{divol2022measure} inequality $\beta \ge \alpha+1$ should be enough. %Secondly, it is not clear that the factor $\sqrt{D}$ in our $W_1$ convergence rate is optimal. This factor comes from the choice of a particular coupling in~\cite{oko2023} that might be suboptimal under the manifold assumption. %in  compared to~\cite{divol2022measure}, our  still  scales sub-optimally as $\sqrt{D}$. 
    
    Another interesting direction is the study of which assumptions might be relaxed. The assumption that density $p(y) > p_{\min}$ is lower bounded is quite strong, however, the bounds we got are polynomial on $C_{\log}$, which itself is only logarithmic in $p_{\min}$, so any smoothing, e.g. by the heat kernel could help.
    Also, it is interesting to study the case when the observed samples are corrupted by a bounded ambient noise. Finally in a recent line of works \cite{dou24,stephanovitch2025generalizationboundsscorebasedgenerative} have obtained generalization without early stopping without the manifold hypothesis, notably because the score has the required smoothness uniformly over $t$. It is unclear if this can be achieved in the context of degenerate distributions over a smooth manifold. 

    Finally note that our approach can be applied to other diffusion models, for instance in the form $ X_t  = X_0 + \sigma_t Z_d$ with $\sigma_t$ going to infinity, since this diffusion model parametrization differs only by rescaling and time reparameterization. Rescaling can be accommodated by changing the fixed first and last layers, while time reparameterization by a small neural network of $\polylog$ size. This argument also covers stochastic interpolants parametrization~\cite{albergo2025stochasticinterpolantsunifyingframework}.

\section*{Acknowledgments}
The authors would like to thank the anonymous referees, an Associate
Editor and the Editor for their constructive comments that improved the
quality of this paper. The authors are also grateful to Leo Zhang for useful discussions. Finally, we are grateful to Denis Belomestny who informed us about the error in the proof of~\cite{oko2023}, Theorem C.4.
\section*{Funding}
IA was supported by the Engineering and Physical Sciences Research Council [grant number EP/T517811/1]. 

GD was supported by the Engineering and Physical Sciences Research Council [grant number EP/Y018273/1]. 

JR received funding from the European Research Council (ERC) under the European Union’s Horizon 2020 research and innovation programme (grant agreement No 834175)
\bibliographystyle{plainnat}
\bibliography{references}

\newpage

\appendix

 In Appendix~\ref{apdx:geometric_lemmas} we provide auxiliary geometric statements that will be used further in the proofs. In Appendix~\ref{apdx:concentration_of_the_score_function} we prove the results of Section~4. Appendix~\ref{apdx:manifold_approximation} contains the proofs of the results in Section~\ref{sec:manifold_approximation}. Appendix~\ref{apdx:score_apprxoximation_by_neural_network} contains omitted details of Section~\ref{sec:main_results_score_function_approx}. Also in Appendix~\ref{sec:ThC4Oko}, we correct \cite{oko2023}, Theorem C.4 that plays an important role in the proof of our main result Theorem~\ref{thm:score_approximation_1}. In~Appendix~\ref{apdx:nn_results} we recall lemmas on neural networks approximation capabilities from~\cite{oko2023}. Finally, in Appendix~\ref{sec:bounds_on_tangent_spaces} we prove the bounds related to the projection of a normal random variable on a tangent space that play an important role in the case $T_k \le n^{-\frac{2}{2\alpha+d}}$.

    \section{Geometric Results}
   \label{apdx:geometric_lemmas}
        \begin{proposition}
        
        \label{prop:geometric_statements}
        Let $M$ be a manifold satisfying %\cref{asmp:smooth_manifold}
        Assumption~A, and $y\in M$. Let $\Phi_y : T_yM \mapsto M$ be a local inverse of projection $\pr_y$ on $T_yM$ and $r_0 = \min\del{\tau, L^{-1}_M}/8$, where $L_M \ge 1$.
        \begin{enumerate}[label=(\roman*)]
            \item The function $\Phi_y$ is well defined on $B_{T_yM}(0,\tau/4)$. For any $z\in B_{T_yM}(0,\tau/4)$ we have
            $
            \norm{z} \le \norm{\Phi_y(z) - y} \le 2\norm{z}. 
            $
            \item $\norm{\grad \Phi_y(z) - \grad \Phi_y(0)}_{op} \le L_{M}\norm{z}.$   
            \item  $\grad \Phi_y(0) \grad^T\Phi_y(0) = \Id_d$, and for $\norm{z} \le r_0$ we have $\norm{\grad \Phi_y(z) \grad^T\Phi_y(z) - \Id_d}_{op} \le 1/2$.   
        \end{enumerate}
        \begin{proof}
        \,
        
        \begin{enumerate}[label=(\roman*)]
            \item
            See~\cite{divol2022measure},~Lemma A.1(iv).
            \item By the mean value theorem
            \[
            \norm{\grad \Phi_y(z) - \grad \Phi_y(0)}_{op} = \sup_{\norm{u} =1}\norm{d_u\del{\Phi_y(z) - \Phi_y(0)}} \le \norm{z}\sup_{z'\in [0,z], \norm{u},\norm{v}=1} \norm{d_u d_v \Phi_y(z')} \le L_{M}\norm{z}.
            \]
            \item
            The first part holds by the definition of $\Phi_y$ as an inverse projection on $T_yM$. Next, we prove the second part.
            \[
            \norm{\grad\Phi_y\grad^T\Phi_y-\Id}_{op} 
            &= \sup_{\norm{v} = 1}\norm{\grad \Phi_y(z) \grad^T\Phi_y(z) v - v} 
            = \sup_{\norm{u},\norm{v} = 1}\innerproduct{\grad \Phi_y(z) (d_v\Phi_y(z)) - v}{u} 
            \\
            &= \sup_{\norm{u},\norm{v} = 1} [d_v\Phi_y(z)]^T d_u\Phi_y(z) -v^Tu\\
            &=
            \sup_{\norm{u},\norm{v} = 1} [d_v\Phi_y(z)]^T d_u\Phi_y(z) -[d_v\Phi_y(0)]^T d_u\Phi_y(0)
            \\
            &= \sup_{\norm{u},\norm{v} = 1} (d_v\Phi_y(z))^T d_u\del{\Phi_y(z) - \Phi_y(0)} + d_v\del{\Phi_y(z) - \Phi_y(0)}^Tu
            \\
            &\le \sup_{\norm{v} = 1}(\norm{d^T_v\Phi_y(z)}+1) \sup_{\norm{u} = 1}\norm{d_u\del{\Phi_y(z) - \Phi_y(0)}}. 
            \]
            Applying part (ii) we get 
            \[
            &\sup_{\norm{v} = 1} \norm{d^T_v\Phi_y(z)} \le \sup_{\norm{v} = 1}(\norm{d^T_v\Phi_y(z)-d^T_v\Phi_y(0)} + \norm{d^T_v\Phi_y(0)} \le L_M\norm{z}+ 1, 
            \\
            &\sup_{\norm{u} = 1}\norm{d_u\del{\Phi_y(z) - \Phi_y(0)}} \le L_M\norm{z}.
            \]
            So since $\norm{z} < L_M^{-1}/8$
            \[ 
            \norm{\grad\Phi\grad^T\Phi-\Id}_{op} \le \sup_{\norm{v} = 1}(\norm{d^T_v\Phi_y(z)}+1) \sup_{\norm{u} = 1}\norm{d_u\del{\Phi_y(z) - \Phi_y(0)}} 
            \\
            \le L_M\norm{z} (L_M\norm{z}+ 2) \le 3/8 \le 1/2. 
            \]
            
        \end{enumerate}
        \end{proof}
        \end{proposition}
    \begin{proposition}
    \label{prop:Manifold_Ball_Volume}
        Let $M$ be a manifold satisfying Assumption~A. Fix $\varepsilon < r_0$ and denote as $C_d$ the volume of the unit $d$-dimensional ball.
        \begin{enumerate}[label=(\roman*)]
            \item Let $M_y(\varepsilon) = \Phi_y\del{B_d\del{0,\varepsilon}}$ then
        \[
        C_d\varepsilon^d \le \Vol M_y(\varepsilon) \le 2^d C_d \varepsilon^d,
        \]
        where $C_d$ is volume of a $d$-dimensional unit ball. 
        \item Recall that for $y\in M$ we defined $B_M(y,\varepsilon) = M\cap B(y,\varepsilon)$ then
        \[
        2^{-d} C_d \varepsilon^d \le \Vol B_M(y,\varepsilon) \le 2^{d}C_d\varepsilon^d.
        \]
        \end{enumerate}
        
    \end{proposition}
    \begin{proof}
    \begin{enumerate}[label=(\roman*)]
        \item Since projection $\pr_{T_yM}$ is contraction, and $\Phi_y$ is the inverse map
        \[
        \Vol M_y(\varepsilon) \ge \Vol \del{\pr_{TyM}M_y(\varepsilon)} = \Vol B_d\del{0,\varepsilon} = \Vol B_d\del{0,1}\varepsilon^d = C_d\varepsilon^d.
        \]
        On the other hand, applying Proposition~\ref{prop:geometric_statements}
        \[
        \Vol M_y(\varepsilon) = \int_{M_y(\varepsilon)} dy = \int_{B_d\del{0,\varepsilon}} \abs{\grad \Phi_y(z)\grad^T\Phi_y(z)}^{-1/2} dz \le 2^d \Vol B_d\del{0,\varepsilon} \le 2^dC_d\varepsilon^d.
        \]
        \item Since $\Phi_y(z)$ is $2$-Lipschitz $\Phi_y\del{B_d\del{0,\varepsilon/2}} \subset B_M(y,\varepsilon)$
        \[
        \Vol B_M(y,\varepsilon) \ge \Vol \Phi_y\del{B_d\del{0,\varepsilon/2}} \ge 2^{-d}C_d\varepsilon^d.
        \]
        Finally, since $\pr_{T_yM}$ is contraction
        \[
        \Vol B_M(y,\varepsilon) \le \Vol M_y(\varepsilon) \le 2^d C_d \varepsilon^d.
        \]
    \end{enumerate}
        
    \end{proof}
    \section{Concentration of the Score Function}
    \label{apdx:concentration_of_the_score_function}
    \subsection{Correlation with Gaussian Noise}
    \label{apdx:CorrelationWithGaussianNoise}
    
    We start the section by establishing a technical proposition that will be used in the proofs. 
    
    \begin{restatable}{proposition}{CorrelationOfFunctionAndNoise}
    \label{prop:maximum_of_Gaussian_Process}
        Let $f: U\subset B_d(0,r)\subset \R^d \rightarrow \R^D$ be $L$-Lipschitz such that $f(0) =0$.
        Let $Z_D\sim \cN(0,\Id_D)$ be a $D$-dimensional standard normal vector and define the centred Gaussian field $F(z) := \innerproduct{Z_D}{f(z)}$ on $U$.
        Then for any $\delta > 0$ 
        \[
        \label{eq:maximum_of_GP_F}
        \rP\del{\sup_{z\in U} F(z) \le Lr\sqrt{d} + Lr\sqrt{2\log \delta^{-1}}} \ge 1-\delta.
        \]
    \end{restatable}
    
    \begin{proof}
    
    The main idea of the proof is to find another process $F'$ with a larger variance of pairwise differences and with easy-to-control supremum. Then, the first condition guarantees a higher spread in values, and, therefore, a larger expected maximum, while the second helps us to control this supremum exactly. The difference between values in two points is controlled by the Lipschitz constant allowing us to choose $F'(z) = L\innerproduct{z}{Z_d}$ -- a process defined intrinsically in $\R^d$.  
    
    Formally, by the Lipschitz property for any $z,z'\in B(y,r)$ 
        \[
        \expectation \abs{F(z)-F(z')}^2 = \expectation \abs{\innerproduct{Z_D}{f(z)-f(z')}}^2 = \norm{f(z)-f(z')}^2 \le L^2\norm{z-z'}^2. 
        \]
        Let $Z_d\sim \cN\del{0,\Id_d}$, consider $G(z) = L\innerproduct{Z_d}{z}$ an auxiliary Gaussian field on $B(y,r)$. Both $F(z)$ and $G(z)$ are continuous zero-mean Gaussian fields, and by construction $\expectation \abs{F(z)-F(z')}^2 \le \expectation \abs{G(z)-G(z')}^2$, so by the Sudakov-Fernique inequality~\cite[Theorem 2.8]{adler1990introduction}
        \[
        \expectation \sup_{z\in U} F(z) \le \expectation \sup_{z\in U} G(z) = \expectation \sup_{z\in B_d(0,r)} L\innerproduct{Z_d}{z} = Lr \expectation \norm{Z_d} \le Lr\sqrt{d}.
        \]
        Moreover,
        \[
        \sup_{z\in U} \expectation \abs{{F(z)}}^2 = \sup_{z\in U} \norm{f(z)}^2 \le (Lr)^2.
        \]
        Finally, we note that $F(z)$ is continuous, so $\sup \abs{F(z)} <\infty$ a.s., and by the Borell-TIS inequality~\cite{adler1990introduction},~Theorem 2.1,
        \[
        \rP\del{\sup_{z\in U} F(z) \le Lr\sqrt{d} + Lr\sqrt{2\log \delta^{-1}}} \ge 1-\delta.
        \]
    \end{proof}
    Now we are ready to prove the main result of this section

    \begin{proof}[Proof of Proposition~\ref{prop:maximum_of_Gaussian_Process_on_Manifold}]
        The main idea is to first prove the inequality locally and then generalize it to arbitrary $y,y' \in M$ using a chaining argument by taking an $\varepsilon$-dense set on $M$. 
        
        Fix a small enough constant $\varepsilon > 0$ and an arbitrary point $G\in M$. By Assumption~\ref{asmp:smooth_manifold} the manifold $M$ can be locally represented as a graph of a function $\Phi_G$. By choice of $\varepsilon$, the function $\Phi_G(z)$ is $2$-Lipschitz by Proposition~\ref{prop:geometric_statements} and by the definition $\Phi_G(0) = G$, so applying Proposition~\ref{prop:maximum_of_Gaussian_Process} we conclude
        \[
        \rP\del{\sup_{z \in B_{T_{G}M}(G, \varepsilon)} \innerproduct{Z_D}{\Phi_G(z)-G} \le 2\varepsilon\del{\sqrt{d} + \sqrt{2\log \delta^{-1}}}} \ge 1-\delta
        \]
        that implies a local version of the desired result   
        \£
        \label{eq:CorrelationWithGaussianNoise_event_1}
        \rP\del{\sup_{y \in M\cap B(G, \varepsilon)} \innerproduct{Z_D}{y-G} \le 2\varepsilon\del{\sqrt{d} + \sqrt{2\log \delta^{-1}}}} \ge 1-\delta.
        \£
        
        To make it global, we take a $\varepsilon$-dense set $G_1,\ldots, G_N$, where by Proposition~\ref{prop:covering_number_of_M} we can guarantee $N \le (\varepsilon/2)^{-d}\Vol M$. For any pair $G_i, G_j$ we note that $\innerproduct{Z_D}{G_i-G_j} \sim \cN\del{0, \norm{G_i-G_j}^2}$ and therefore
        \£
        \label{eq:CorrelationWithGaussianNoise_event_2}  
        \rP\del{\innerproduct{Z_D}{G_i-G_j} \le \norm{G_i-G_j}\sqrt{2\log \delta^{-1}}} \ge 1-\delta.
        \£
        We estimate from above with $\delta N(N+1)/2$ the probability that simultaneously for any $G_i$ event~\eqref{eq:CorrelationWithGaussianNoise_event_1} holds, and for any pair $G_i\neq G_j$ event~\eqref{eq:CorrelationWithGaussianNoise_event_2} holds too.

        Representing $\innerproduct{Z_D}{y-y'} = \innerproduct{Z_D}{y-G} + \innerproduct{Z_D}{G-G'} + \innerproduct{Z_D}{G'-y'}$, where $G, G'$ are the nearest points in the $\varepsilon$-net to $y, y'$ respectively, and applying the bounds above we get that
    \[
    \rP
    \left[\forall y,y'\in M: \innerproduct{Z_D}{y-y'} \le 4\varepsilon\del{\sqrt{d} + \sqrt{2\log \delta^{-1}}} + \del{\norm{y-y'}+2\varepsilon}\sqrt{2\log \delta^{-1}} \right] \\
    \ge 1-\delta N(N+1)/2,
    \]
    where we additionally used the triangle inequality 
    $$\norm{G-G'} \le \norm{y-y'} + \norm{y-G} + \norm{y'-G'}\le \norm{y-y'} +  2\varepsilon.$$
    
    Next, we reorder and perform a change of variable 
    \[
    \rP\del{\forall y,y'\in M: \innerproduct{Z_D}{y-y'} \le 4\varepsilon\sqrt{d} + \del{\norm{y-y'}+6\varepsilon}\sqrt{2\log N(N+1)/2 + 2\log \delta^{-1}}}\\ \ge 1-\delta.
    \]
    Since $
    N(N+1)/2 \le N^2 \le (\varepsilon/2)^{-2d}(\Vol M)^2
    $
    we get
    \[
    \rP\Big[\forall y,y'\in M: \innerproduct{Z_D}{y-y'} \le 4\varepsilon\sqrt{d} + \\ \del{\norm{y-y'}+6\varepsilon}\sqrt{4d\log 2\varepsilon^{-1} + 4\log_+ \Vol M  + 2\log \delta^{-1}}\Big] \ge 1-\delta.
    \]
    Finally, noting that the same bound holds for $-\innerproduct{Z_D}{y-y'}\stackrel{dist.}{=} \innerproduct{Z_D}{y-y'}$ we obtain the bound on $\abs{\innerproduct{Z_D}{y-y'}}$ 
    \£
    \label{eq:corr_with_gaussian_noise_raw}
    \rP\Big[\forall y,y'\in M: \abs{\innerproduct{Z_D}{y-y'}} \le 4\varepsilon\sqrt{d} + \\
    \del{\norm{y-y'}+6\varepsilon}\sqrt{4d\log 2\varepsilon^{-1} + 4\log \Vol M  + 2\log 2\delta^{-1}} \Big] \ge 1-\delta.
    \£
    \end{proof}
  
    \subsection{Bounds on the Score Function}
    \label{apdx:Bound_On_The_Score_Function}
    
    \begin{proof}[Proof of Proposition~\ref*{prop:RepresentationOfConditionalMeasure}]
    Recall that $p\del{y|t,X_t} \propto e^{-\norm{X_t-c_ty}^2/2\sigma^2_t}p(y)$, so 
    \£
    \label{eq:log_pytXt}
    \log p\del{y|t,X_t} = -\log \del{\int_M e^{-\norm{X_t-c_ty}^2/2\sigma^2_t} p(y)dy}-\norm{X_t-c_t y}^2/2\sigma^2_t + \log p(y).
    \£
    We start our proof with an analysis of term $-\norm{X_t-y}^2$. By definition of $X_t$
    \[
    \norm{X_t-c_t y}^2 = \norm{c_t X_0 + \sigma_tZ_D-c_ty}^2 = \sigma_t^2 \norm{Z_D}^2 + 2c_t\sigma_t\innerproduct{Z_D}{X_0-y} + c^2_t\norm{X_0-y}^2.
    \]
    Since $X_0\in M$, by Proposition~\ref{prop:maximum_of_Gaussian_Process_on_Manifold}, for $\varepsilon < r_0$ with probability $1-\delta$ for all $y\in M$ 
    \[
    \abs{\innerproduct{Z_D}{X_0-y}} \le 4\varepsilon\sqrt{d} + \del{\norm{X_0-y}+6\varepsilon}\sqrt{4d\log 2\varepsilon^{-1} + 4\log_+ \Vol M  + 2\log 2\delta^{-1}}.
    \]
    Taking $\varepsilon =\min\del{r_0, \frac{1}{4}(\sigma_t/c_t)}$ we get
    \[
    \abs{\innerproduct{Z_D}{X_0-y}} \le \frac{4}{4}(\sigma_t/c_t)\sqrt{d}
    \\
    + \del{\norm{X_0-y}+\frac{6}{4}(\sigma_t/c_t)}\sqrt{4d\log_+ (\sigma_t/c_t)^{-1} +4d\log 8r^{-1}_0 + 4\log_+ \Vol M  + 2\log 2\delta^{-1}}.
    \]
    Since $2 < \sqrt{4\log 8}$ we have $\frac{4}{4}\sqrt{d} \le \frac{2}{4}\sqrt{4d\log 8}$ and therefore
    \£
    \label{eq:useful_bound_on_scalar_product}
    2\sigma_tc_t\abs{\innerproduct{Z_D}{X_0-y}} \le 2\sigma^2_t\sqrt{4d\log_+ (c_t/\sigma_t) +4d\log 8r^{-1}_0 + 4\log_+ \Vol M  + 2\log 2\delta^{-1}}
    \\
    + 2\sigma_tc_t\norm{X_0-y}\sqrt{4d\log_+ (c_t/\sigma_t) +4d\log 8r^{-1}_0 + 4\log_+ \Vol M  + 2\log 2\delta^{-1}},
    \£ 
    while the last term can be bounded by Young's inequality as
    \[
    2\sigma_t c_t\norm{X_0-y}\sqrt{4d\log_+ (c_t/\sigma_t) +4d\log 8r^{-1}_0 + 4\log_+ \Vol M  + 2\log 2\delta^{-1}}
    \\
    \le
    \frac{1}{2}c^2_t\norm{X_0-y}^2 + 2\sigma^2_t\del{4d\log_+ (c_t/\sigma_t) +4d\log 8r^{-1}_0 + 4\log_+ \Vol M  + 2\log 2\delta^{-1}}.
    \]
    Substituting, since $\sqrt{x} \le x$ for $x\ge 1$, we conclude that with probability at least $1-\delta$ that for all $y\in M$
    \[
    2\sigma_tc_t\abs{\innerproduct{Z_D}{X_0-y}} \le \frac{1}{2}c^2_t\norm{X_0-y}^2 \\+ 4\sigma^2_t\del{4d\log_+ (c_t/\sigma_t) +4d\log 8r^{-1}_0 + 4\log_+ \Vol M  + 2\log 2\delta^{-1}}.
    \]
    Finally, by definition of $C_{\log}$ we have $4d\log 8r^{-1}_0 + 4\log_+ \Vol M + 2\log 2 \le C_{\log}\del{12d+2}$ whence 
    \£
    \label{eq:bound_on_squared_difference_Xt_y_new}
     \frac{1}{2}c^2_t\norm{X_0-y}^2-4\sigma^2_tC_0 =  
     \\
      -4\sigma^2_t\del{4d\log_+ (c_t/\sigma_t) + (12d+2)C_{\log}  + 2\log \delta^{-1}} + \frac{1}{2}c^2_t\norm{X_0-y}
    \\
    \le
    \norm{X_t-c_t y}^2 - \sigma_t^2 \norm{Z_D}^2 \le
    \\
    3c^2_t\norm{X_0-y}^2/2 + 4\sigma^2_t\del{4d\log_+ (c_t/\sigma_t) + (12d+2)C_{\log} + 2\log \delta^{-1}} 
    \\
    = \frac{3}{2}c^2_t\norm{X_0-y}^2 + 4\sigma^2_tC_0,
    \£
    where, for brevity, we introduced 
    $$C_0 := \del{4d\log_+ (c_t/\sigma_t) + (12d+2)C_{\log} + 2\log \delta^{-1}}.$$
    Next, we bound the normalizing constant 
    $$\log \del{\int_M e^{-\norm{X_t-c_ty}^2/2\sigma^2_t} p(y)dy}.$$ 
    Using the inequality \eqref{eq:bound_on_squared_difference_Xt_y_new} and  a local representation of the manifold as $\Phi_{X_0}\del{B_{T_{X_0}M}(X_0,r_0)}$, where $\Phi_{X_0}$ is $2$-Lipschitz, we obtain
    \begin{align}   
    \label{eq:log_p_lower_bound}
    \int_M e^{-\norm{X_t-c_t y}^2/2\sigma_t^2}p(y)dy \nonumber
    &\ge 
    e^{-\norm{Z_D}^2/2}e^{-2C_0}\int_{\Phi_{X_0}\del{B_{T_{X_0}M}(X_0,r_0)}} e^{-3c_t^2\norm{X_0- y}^2/4\sigma_t^2}p(y)dy 
    \\
    &\ge
    2^{-d} p_{\min}e^{-\norm{Z_D}^2/2}e^{-2C_0}\int_{z\in \R^d, \norm{z} \le r_0} e^{-3c_t^2\norm{z}^2/\sigma_t^2}dz 
    \\
    &\ge
    \frac{e^d}{(2d)^{d+1}}e^{-2}  p_{\min}e^{-\norm{Z_D}^2/2}e^{-2C_0}(r_0 \wedge (\sigma_t/c_t))^d. \nonumber
    \end{align}
    The last inequality follows from integration in polar coordinates, standard bounds on $\Gamma$-function \cite{Necdet2008},~Theorem 1.5 and inequality on the incomplete gamma function $\gamma(a,x)$ from \cite{Neuman2013},~Theorem 4.1. More precisely, we applied that for $a,b>0$ and $c := \min(a,b)$ holds
    \[
    \int_{z\in \R^d, \norm{z} \le a} e^{-3\norm{z}^2/b^2}dz \ge \int_{z\in \R^d, \norm{z} \le c} e^{-3\norm{z}^2/c^2}dz 
    = \int_0^{c}e^{-3r^2/c^2}r^{d-1}dr
    \\
    = 
    \frac{2\pi^{d/2}}{\Gamma(d/2)}\frac{c^d}{6^{d/2}}\gamma\del{\frac{d}{2},3}
    \ge
    \frac{2\cdot\pi^{d/2}}{(2\pi)^{1/2}(d/2e)^{d}}\frac{c^d}{6^{d/2}}e^{-\frac{3d}{2d+1}}\frac{3^{d/2}}{(d/2)} \ge c^d\frac{e^d}{d^{d+1}}\frac{e^{-2}}{2}.
    \]
    Taking logarithms in~\eqref{eq:log_p_lower_bound} we conclude 
    \[
    -\log \del{\int_M e^{-\norm{X_t-c_t y}^2/2\sigma_t^2}p(y)dy} 
    \\
    \le 
     \norm{Z_D}^2/2 + 2C_0 + d\log 2d + 2 + \log_+ p^{-1}_{\min} + d\log r^{-1}_0 + d\log_+(\sigma_t/c_t)^{-1} 
    \\
    \le \norm{Z_D}^2/2 + 2C_0 + d\log_+(c_t/\sigma_t) + C_{\log}\del{3d + 1}.
    \] 
    In a similar way, we can get an upper-bound, since $\int_M p(y)dy = 1$
    \[
    \int_M e^{-\norm{X_t-c_t y}^2/2\sigma_t^2}p(y)dy
    \le 
    e^{-\norm{Z_D}^2/2}e^{2C_0}\int_{M}e^{-c_t^2\norm{X_0- y}^2/4\sigma_t^2}p(y)dy 
    \le
    e^{-\norm{Z_D}^2/2}e^{2C_0}. 
    \]
    So,
    \[
    -\log \del{\int_M e^{-\norm{X_t-c_t y}^2/2\sigma_t^2}p(y)dy} 
    \ge -2C_0 + \norm{Z_D}^2/2.
    \]
    Finally, recalling the representation~\eqref{eq:log_pytXt} of $\log p\del{y|t,X_t}$ and that $\abs{\log p(y)} \le d C_{\log}$
    \[
    &-16d\log_+ (\sigma_t/c_t) - (49d+8)C_{\log} - 8\log \delta^{-1} - \frac{3}{4}(c_t/\sigma_t)^2\norm{X_0-y}^2 
    \\
    &= -4C_0 - d C_{\log} - \frac{3}{4}(c_t/\sigma_t)^2\norm{X_0-y}^2 
    \\
    &\le \log p\del{y|t,X_t} 
    = -\log \del{\int_M e^{-\norm{X_t-c_ty}^2/2\sigma^2_t} p(y)dy}-\norm{X_t-c_t y}^2/2\sigma^2_t + \log p(y)  
    \\
    &
    \le 4C_0 + d\log_+(c_t/\sigma_t) + C_{\log}\del{4d + 1} -\frac{1}{4}(c_t/\sigma_t)^2\norm{X_0-y}^2
    \\
    &= 17d\log_+ (c_t/\sigma_t) + (52d+9)C_{\log} + 8\log \delta^{-1} -\frac{1}{4}(c_t/\sigma_t)^2\norm{X_0-y}^2,
    \]
    that implies 
    \£
    \label{eq:more_precise_p_x_y_t_bound}
    -19d\del{\log_+ (c_t/\sigma_t) + 4C_{\log}} - 8\log \delta^{-1} - \frac{3}{4}(c_t/\sigma_t)^2\norm{X_0-y}^2 
    \\
    \le \log p(y|t,X_t) \le 
    \\
    19d\del{\log_+ (c_t/\sigma_t) + 4C_{\log}} + 8\log \delta^{-1} - \frac{1}{4}(c_t/\sigma_t)^2\norm{X_0-y}^2.
    \£
    \end{proof} 
    \begin{proof}[Proof of Theorem~\ref{thm:concentration_around_X0}]
    We prove the 3 items of the Theorem. 
    \begin{enumerate} %[label=(\roman*)]
        \item 
        Fix $r > 0$. Applying \eqref{eq:more_precise_p_x_y_t_bound}
        \begin{align*}
        &\int_{\norm{y-X_0} \ge r} p(y|t,X_t) dy \\
        &\le 
        e^{19d\del{\log_+ (\sigma_t/c_t) + 4C_{\log}} + 8\log \delta^{-1}}
        \int_{\norm{y-X_0} \ge r} e^{- \frac{1}{4}(c_t/\sigma_t)^2\norm{X_0-y}^2} dy
        \\
        &\le
        e^{19d\del{\log_+ (\sigma_t/c_t) + 4C_{\log}} + 8\log \delta^{-1}}[e^{- \frac{1}{4}(c_t/\sigma_t)^2r^2}\Vol M]\\
        &\le 
        e^{20d\del{\log_+ (\sigma_t/c_t) + 4C_{\log}} + 8\log \delta^{-1}}e^{- \frac{1}{4}(c_t/\sigma_t)^2r^2}
         \le \eta,
        \end{align*}
        where we additionally used that $\Vol M < e^{dC_{\log}}$ and
        the last inequality holds if we take 
        \[
        r = r(t,\delta, \eta) := 2(\sigma_t/c_t)\sqrt{20d\del{\log_+ (\sigma_t/c_t) + 4C_{\log}} + 8\log \delta^{-1} + \log \eta^{-1}}.
        \]
        \item First we express the LHS as
        \begin{multline}
        \label{eq:almost_sure_bound_proof_concentration}
        \sigma_ts(t,X_t) + Z_D = \int_M \frac{c_ty-c_tX_0 - \sigma_tZ_D}{\sigma_t}p(y|t,X_t) + Z_D \\= \int_M \frac{c_t}{\sigma_t}\del{y-X_0}p(y|t,X_t) .
        \end{multline}
        Since $\diam M \le 1$, $\norm{X_0-y} \le 1$ a.s., so splitting integral by event $\norm{X_0-y} \le r(t,\delta,\eta)$ and applying \eqref{eq:concentration_around_X0_1} we have that with probability $1-\delta$
        \[
        \norm{\sigma_ts(t,X_t) + Z_D} \le (c_t/\sigma_t)\del{r(t,\delta,\eta) + \eta}.
        \]
        Taking $\eta = \min\del{1, (\sigma_t/c_t)}$ with probability $1-\delta$
        \£
        \label{eq:almost_sure_bound_proof_concentration_2}  
        \norm{\sigma_ts(t,X_t) + Z_D} \le 2\sqrt{20d\del{\log_+ (\sigma_t/c_t) + 4C_{\log}} + 8\log \delta^{-1} + \log_+ (\sigma_t/c_t)} + 1
        \\
        \le
        4\sqrt{20d\del{2\log_+ (\sigma_t/c_t) + 4C_{\log}} + 8\log \delta^{-1}}.
        \£
        \item 
        Similarly, by \eqref{eq:almost_sure_bound_proof_concentration}, we have that a.s.\ $\norm{\sigma_t s(t,X_t) + Z_D} \le c_t/\sigma_t$, so applying the first inequality of~\eqref{eq:almost_sure_bound_proof_concentration_2} and $(a+b)^2 \le 2a^2+2b^2$
        \[
        \expectation \norm{\sigma_ts(t,X_t) + Z_D}^2 
        \\
        \le 8\del{20d\del{\log_+ (\sigma_t/c_t) + 4C_{\log}} + 8\log \delta^{-1} + \log_+ (\sigma_t/c_t)} + 2 + (c_t/\sigma_t)^2\delta.
        \]
        Taking $\delta = \min\del{1, (\sigma_t/c_t)^2}$
        \[
        \expectation \norm{\sigma_ts(t,X_t) + Z_D}^2 
        \\
        \le 8\del{20d\del{\log_+ (\sigma_t/c_t) + 4C_{\log}} + 16\log_+ (\sigma_t/c_t) + 2\log_+ (\sigma_t/c_t)} + 3
        \\
        \le
        8\del{20d\del{2\log_+ (\sigma_t/c_t) + 4C_{\log}}} + 3.
        \]
        \end{enumerate}
    \end{proof}
    \subsection{Point Comparison}
    \label{apdx:auxilarly_results}
    This section proves key statements used to build an efficient approximation of the score function.
    The following lemma is a full version of Lemma~\ref{lemma:good_set_size}.
    \begin{restatable}{lemma}{LemmaWeightsComparedToDistanceFull}
    \label{lemma:good_set_size_full}
        Let $\cP = \curly{m_1,\ldots, m_N} \subset M$ be a set of points in $M$. Denote the nearest neighbor of $X_t$ in $\cP$ with $m_{\min}(t) = m_{\min}(X_t) := \arg\min_i \norm{X_t-c_t m_i}$, then with probability at least $1-\delta$: 
        \begin{enumerate}[label=(\roman*)]
            \item $m_{\min}(t)$ is close to $m_{\min}(0)$
        \£
        \label{eq:bound_on_m_min}
        \norm{X_0-m_{\min}(t)}^2 \le  3\norm{X_0-m_{\min}(0)}^2 + 64(\sigma_t/c_t)^2\del{d\log_+ (\sigma_t/c_t)^{-1} + 4d C_{\log}  + \log \delta^{-1}};
        \£
        \item 
        for any $m_i\in \cP$
        \£
        \label{eq:dist_to_X_t_as_dist_to_X_0}
        \norm{X_t - c_t m_i}^2 - \norm{X_t - c_t m_{\min}(t)}^2 
        \ge
        \frac{c^2_t}{2}\del{\norm{X_0-m_i}^2 - 9\norm{X_0-m_{\min}(0)}^2}
        \\
        -128(\sigma_t/c_t)^2\del{d\log_+ (\sigma_t/c_t)^{-1} + 4d C_{\log}  + \log \delta^{-1}}; 
        \£
        \item 
        if $\cP$ is an $\varepsilon$-dense set, for all $m_i$
        \begin{multline}  
        \frac{2}{3}c^{-2}_t\del{\norm{X_t - c_t m_i}^2 - \norm{X_t - c_t m_{\min}(t)}^2} 
        - 128(\sigma_t/c_t)^2\del{d\log_+ (\sigma_t/c_t)^{-1} + 4d C_{\log}  + \log \delta^{-1}}
        \\
        \le
          \norm{X_0-m_i}^2 \le
        \\
        9\varepsilon^2 + 
        2c^{-2}_t\del{\norm{X_t - c_t m_i}^2 - \norm{X_t - c_t m_{\min}(t)}^2 }
        + 128(\sigma_t/c_t)^2\del{d\log_+ (\sigma_t/c_t)^{-1} + 4d C_{\log}  + \log \delta^{-1}}.
        \end{multline}
        \end{enumerate}
\end{restatable}
    \begin{proof}
    \begin{enumerate}[label=(\roman*)]
        \item 
        The result follows from inequality \eqref{eq:bound_on_squared_difference_Xt_y_new} that expresses $\norm{X_0-y}^2$ in terms of $\norm{X_t-y}^2$; more precisely it states that for $\delta < 1$ with probability $1-\delta$ for all $y\in M$
        \[
         \frac{1}{2}c^2_t\norm{X_0-y}^2-4\sigma^2_tC_0 
        \le
        \norm{X_t-c_t y}^2 - \sigma_t^2 \norm{Z_D}^2 \le
        \frac{3}{2}c^2_t\norm{X_0-y}^2 + 4\sigma^2_tC_0,
        \]
        where $C_0 = \del{4d\log_+ (\sigma_t/c_t) + (12d+2)C_{\log}  + 2\log \delta^{-1}} \le 4\del{d\log_+ (\sigma_t/c_t) + 4dC_{\log}  + \log \delta^{-1}}$.%$ = 4C'_0$.
        
        On the one hand by choice of $m_{\min}(t)$ and $m_{\min}(0)$ and the second part of \eqref{eq:bound_on_squared_difference_Xt_y_new} 
        \[    
        \norm{X_t - c_t m_{\min}(t)}^2 \le \norm{X_t - c_t m_{\min}(0)}^2  \le \sigma_t^2\norm{Z_D}^2 + 4C_0\sigma^2_t + 3c^2_t\norm{X_0-m_{\min}(0)}^2/2.
        \]
        On the other hand, applying the first part of the same inequality \eqref{eq:bound_on_squared_difference_Xt_y_new}
        \[
            \norm{X_t - c_t m_{\min}(t)}^2 \ge \sigma_t^2\norm{Z_D}^2 -4 C_0\sigma^2_t + c^2_t\norm{X_0-m_{\min}(t)}^2/2.
        \]
        Combining, we prove the first part of the statement
        \[
        c^2_t\norm{X_0-m_{\min}(t)}^2 \le 16C_0\sigma^2_t + 3c^2_t\norm{X_0-m_{\min}(0)}^2
        \\ 
        \le 3c^2_t\norm{X_0-m_{\min}(0)}^2 + 64\sigma^2_t\del{d\log_+ (\sigma_t/c_t) + 4dC_{\log}  + \log \delta^{-1}}.
        \]
        \item
        To prove the second part we first express both distances to $X_t$ through distances to $X_0$ using \eqref{eq:bound_on_squared_difference_Xt_y_new} and then substitute \eqref{eq:bound_on_m_min} to get the statement
        \[
        &\norm{X_t - c_t m_{\min}(t)}^2 - \norm{X_t - c_t m_i}^2
        \\
        &=
        \del{\norm{X_t - c_t m_{\min}(t)}^2 - \sigma^2_t\norm{Z_D}^2} - \del{\norm{X_t - c_t m_i}^2-\sigma^2_t\norm{Z_D}^2}         
        \\
        &\le 8C_0\sigma^2_t
             + \frac{c^2_t}{2}\del{3\norm{X_0-m_{\min}(t)}^2 - \norm{X_0-m_i}^2}
        \\
        &\le
        32C_0\sigma^2_t
             + \frac{c^2_t}{2}\del{9\norm{X_0-m_{\min}(0)}^2 - \norm{X_0-m_i}^2}.
        \]
        
        \item Since $\cP$ is a $\varepsilon$-dense set $\norm{X_0-m_{\min}(0)} = \min_{m_i\in \cP} \norm{X_0-m_i}^2 \le \varepsilon^2$ rearranging terms in~\eqref{eq:dist_to_X_t_as_dist_to_X_0} we get RHS
        \[
          \norm{X_0-m_i}^2 \le
        9\varepsilon^2 + 
        2c^{-2}_t\del{\norm{X_t - c_t m_i}^2 - \norm{X_t - c_t m_{\min}(t)}^2} 
         + 32(\sigma_t/c_t)^2C'_0.
        \]
        To get LHS we need to bound $\norm{X_0-m_i}^2$ from below.
        Applying \eqref{eq:bound_on_squared_difference_Xt_y_new} 
        \[
            \norm{X_t - c_t m_i}^2 \le \sigma_t^2\norm{Z_D}^2 + 4C_0\sigma^2_t + 3c^2_t\norm{X_0- m_i}^2/2 , 
        \]
        and
        \[
            \norm{X_t - c_t m_{\min}(t)}^2 &\ge \sigma_t^2\norm{Z_D}^2 - 4C_0\sigma^2_t + c^2_t\norm{X_0-c_t m_{\min}(t)}^2/2 
            \\
            &\ge \sigma_t^2\norm{Z_D}^2 - 4C_0\sigma^2_t + c^2_t\norm{X_0-c_t m_{\min}(0)}^2/2.
        \]
        Subtracting we get 
        \[
        \norm{X_t - c_t m_i}^2 - \norm{X_t - c_t m_{\min}(t)}^2 \le 8C_0\sigma^2_t +\frac{c^2_t}{2}\del{3\norm{X_0-m_i}^2 - \norm{X_0- m_{\min}(0)}^2},
        \]
        rearranging terms and noting that $\norm{X_0-m_{\min}(0)} \ge 0$
        \[
        \frac{2}{3}c^{-2}_t\del{\norm{X_t - c_t m_i}^2 - \norm{X_t - c_t m_{\min}(t)}^2}-16C_0(\sigma_t/c_t)^2 \le \norm{X_0-m_i}^2
        \]
        we get the LHS.
        \end{enumerate}
    
    \end{proof}

    \section{Manifold Approximation}
    \label{apdx:manifold_approximation}
    \subsection{Proofs of Section 5.1 on  the properties of $\Phi_i^*$}
    \label{apdx:properties_of_support_estimator}
     
    This section contains the proofs of the results in Section~\ref{sec:support_estimation}. %Recall that we consider $N$ i.i.d.\ samples $\mathcal G = \{ y_1, \cdots, y_N\}$ and for all $y_i \in \mathcal G$ we consider $P_i^*, \Phi_i^*$ defined by \eqref{eq:support_estimation_functional} and \eqref{eq:estimator_Phi_star_based_on_samples} respectively.
    
    \begin{proof}[Proof of Proposition~\ref{prop:tangent_space_approximation}]
    To simplify notation, we fix index $i$ and drop it for the rest of the proof.
    Recall that $P = \pr P^* ((P^*)^T\pr P^*)^{-1/2}$. 
    
    First, we show that this operator is well-defined, it is enough to show that $(P^*)^T\pr P^*$ is symmetric and positive-definite. Since $\pr^*$ is a projection on $\Im P^*$ the image of embedding $P^*$ we have $\pr^* P^* = P^*$ and $(P^*)^T P^* = \Id_d$, so
    \[
    (P^*)^T\pr P^* = \Id_d + (P^*)^T(\pr-\pr^*)P^*,  
    \]
    note that $(P^*)^T(\pr-\pr^*)P^*$ is symmetric and $\norm{(P^*)^T(\pr-\pr^*)P^*} \lesssim \varepsilon^\beta_N < 1/2$ giving us that $P^*$ is well defined. 
    
    Moreover, applying $(1+x)^{-1/2} = 1 - x/2 + O(x^2)$ we bound
    \[
    \norm{\Id_d -\del{(P^*)^T\pr P^*}^{-1/2}} = \norm{\Id_d -\del{\Id_d + (P^*)^T(\pr-\pr^*) P^*}^{-1/2}} \le \norm{(P^*)^T(\pr-\pr^*)P^*}_{op} \lesssim \varepsilon^{\beta-1}_N
    \]
    
    We are now ready to bound $\norm{P^*_i-P_i}$. Noting that $\pr^* P^* = P^*$, $\pr^2 = \pr = \pr^T$, $(\pr^*)^2 = \pr^* = (\pr^*)^T$, and $\norm{P^*}_{op} = \norm{\pr^*}_{op} = \norm{\pr}_{op} = 1$
    \[
    \norm{P^* - P}_{op} &= \norm{\pr^* P^* - \pr P^* ((P^*)^T\pr P^*)^{-1/2}}_{op} 
    \\
    &\le 
    \norm{(\pr - \pr^*) P^*}_{op} + \norm{\pr P^*\del{\Id_d - ((P^*)^T\pr P^*)^{-1/2}}}_{op}
    \\
    &\le 
    \norm{\pr - \pr^*}_{op} \norm{P^*}_{op} + \norm{\pr P^*}_{op}\norm{\del{\Id_d - ((P^*)^T\pr P^*)^{-1/2}}}_{op} 
    \lesssim
    \varepsilon^{\beta-1}_N.
    \]
    
    Finally, the last statement to prove is that
    $P_i:\R^d \rightarrow T_{G_i}M-G_i$ is an isometric embedding. By construction $\Im P_i \subseteq \Im \pr = T_{G_i}M-G_i$, and for any $v,u\in \R^d$ 
    \[
    \innerproduct{P_iv}{P_i u} = \del{\pr_i P^*_i \del{\del{P^*_i}^T\pr_iP^*_i}^{-1/2}v}^T\del{\pr_i P^*_i \del{\del{P^*_i}^T\pr_iP^*_i}^{-1/2} u} = 
    \\
    = v^T\del{\del{P^*_i}^T\pr_iP^*_i}^{-1/2}(P^*_i)^T\pr_i  P^*_i \del{\del{P^*_i}^T\pr_iP^*_i}^{-1/2} u= v^T u,
    \]
    where we again used that $\pr_i\pr_i = \pr_i$.
    \end{proof}
    
    \begin{proof}[Proof of Lemma~\ref{lemma:approximation_of_Phi}]
    We fix the index $i$ and drop it to simplify notation.
    Let $\overline{\Phi}:\Im P \mapsto M, \overline{\Phi}^*:\Im P^* \mapsto M^*$ be defined as $\overline{\Phi}(Pz) = \Phi(z), \overline{\Phi}^*(P^*z) = \Phi^*(z)$. Introducing $V_j = d^j\overline{\Phi}(0),V^*_j = d^j\overline{\Phi}^*(0)$ where $2\le j < \beta-1$ we represent 
    \[
    &\overline{\Phi}(v) = y_i + v +  \sum V_j\del{v^{\otimes k}} + R(v), & v\in \Im P, & \,
    \\
    &\overline{\Phi}^*(v) = y_i + v + \sum V^*_j\del{v^{\otimes k}}, & v\in \Im P^*, & \,
    \]
    where $R(v)$ is  the residual of order $\beta$. We first control $\grad^k \del{\Phi\circ P^TP^*}(z) - \grad^k \Phi^*(z)$ and then $\grad^k \del{\Phi\circ P^TP^*}(z) - \grad^k \Phi(z)$. 

    Substituting $P^*z$ and $PP^TP^*z = \pr P^*z$
    \[
    &\overline{\Phi}(PP^TP^*z) = y_i + PP^TP^*z +  \sum V_j\del{\del{PP^TP^*z}^{\otimes k}} + R(PP^TP^*z), 
    \\
    &\overline{\Phi}^*(P^*z) = y_i + P^*z + \sum V^*_j\del{\del{P^*z}^{\otimes k}}. 
    \]
    \cite{divol2022measure},~Lemma A.2 (iii) states that $\norm{V^*_j\circ \pr^* -V_j \circ \pr}_{op} \lesssim \varepsilon_N^{\beta-j}$ for $j\ge 2$, note that taking $V_1 = V^*_1 = \Id$ this also holds for $j=1$. So for all  $\norm{z}\lesssim \varepsilon_N$, 
    \[
    \norm{\Phi^*(z) - \Phi(P^TP^*z)} = \norm{\overline{\Phi}^*(P^*z) - \overline{\Phi}(PP^TP^*z)} \\\
    \lesssim \norm{\pr-\pr^*}_{op}\norm{P^*z} + \sum_j \norm{V^*_j\circ \pr^* - V_j\circ \pr}\norm{P^*z}^k + \norm{R(PP^TP^*z)}
    \lesssim \varepsilon_N^\beta,
    \]
    since $\pr^* P^*z = P^*z, \norm{z^*} = \norm{P^*z} \lesssim \varepsilon_N$ and $\norm{PP^TP^*}_{op} \le 1$.
    
    To prove the result for $k\ge 1$ we note that if we take $1\le k < \beta-1$ derivatives along the directions $v_1,\ldots, v_k\in \R^d$ with $\norm{v_j} = 1$, since $\pr^*P^*z = P^*z$
   \[
    &d_{v_1}\ldots d_{v_k}\del{\overline{\Phi_i}\circ \pr \circ P^*}(z) = 
    \sum_{j=k}^{\beta-1} \frac{j!}{(j-k)!}
    V_j\del{\del{\pr P^*z}^{\otimes (j-k)}\otimes \del{\pr P^*v_1}\otimes\ldots \otimes \del{\pr P^*v_k}} 
    + R^k(z),
    \\
    &d_{v_1}\ldots d_{v_k}\del{\overline{\Phi_i}^*\circ P^*}(z) = \sum_{j=k}^{\beta-1} \frac{j!}{(j-k)!}V^*_j\del{\del{\pr^* P^*z}^{\otimes (j-k)}\otimes \del{\pr^* P^*v_1}\otimes\ldots \otimes \del{\pr^* P^*v_k}},
    \]
    where $\norm{R^k(z)} = O\del{\varepsilon_N^{\beta-k}}$. Thus
    \begin{align*}
     &\abs{
     d_{v_1}\ldots d_{v_k}\del{\Phi\circ P^TP^*}(z) 
     -d_{v_1}\ldots d_{v_k}\Phi^*(z)
     }\\ 
     &=
    \abs{d_{v_1}\ldots d_{v_k}\del{\overline{\Phi_i}\circ \pr \circ P^*}(z) -d_{v_1}\ldots d_{v_k}\del{\overline{\Phi_i}^*\circ P^*}(z)} 
    \\
    &\le O\del{\varepsilon_N^{\beta-k}} + \sum_{j=k}^{\beta-1}\norm{V^*_j\circ \pr^* - V_j\circ \pr}\norm{P^*z}^{j-k}\prod \norm{P^* v_k} \\
    &\lesssim \varepsilon_N^{\beta-k} + \sum_{j=k}^{\beta-1} \varepsilon_N^{\beta-j}\varepsilon_N^{j-k} \lesssim \varepsilon_N^{\beta-k}.
    \end{align*}
    This proves that $\norm{\grad^k \del{\Phi\circ P^TP^*}(z) - \grad^k \Phi^*(z)} \lesssim \varepsilon_N^{\beta-k}$. Therefore, to prove the statement it is enough to show that $\norm{\grad^k \del{\Phi\circ P^TP^*}(z) - \grad^k \Phi(z)} \lesssim \varepsilon_N^{\beta-k}$ and $\norm{\del{\Phi\circ P^TP^*}(z) - \Phi(z)} \lesssim \varepsilon_N^{\beta}$. 
    
    Denote by $A = P^TP^*$, then since $P$ is an isometry 
    \[
    \norm{A-\Id_d}_{op} = \norm{P^TP^*-\Id_d}_{op} = \norm{PP^T P^*-P}_{op} = \norm{\pr\del{P^*-P}}_{op} \le \norm{P^*-P}_{op} \lesssim \varepsilon_N^{\beta-1}.
    \] 
    and $\norm{A}_{op} \le \norm{P^T}\norm{P^*} = 1$. We start with the case $k=0$
    \[
    \norm{\del{\Phi\circ P^TP^*}(z) - \Phi(z)} \lesssim L_M \norm{P^TP^*z - z} \le L_M \norm{P^TP^*-\Id_d}\norm{z} \lesssim \varepsilon_N^{\beta}.  
    \]
    
    Let $k\ge 1$. Taking derivatives along the directions $v_1,\ldots, v_k\in \R^d$ with $\norm{v_j} = 1$ and then applying the chain rule and chaining argument 
    \begin{align*}
    &\norm{d_{v_1}\ldots d_{v_k}\del{\Phi\circ A}(z) - d_{v_1}\ldots d_{v_k}\Phi(z)} \\
    &= \norm{d_{v_1}\ldots d_{v_k}\del{\Phi_i\circ A}(z) - d_{v_1}\ldots d_{v_k} \Phi(z)} 
    \\
    &= \norm{d^k\Phi(Az)\del{A{v_1}\otimes \ldots \otimes A v_k}-d^k\Phi(z)\del{v_1\otimes \ldots \otimes v_k}} \\
    &\le
    \norm{d^k\Phi(Az)\del{{v_1}\otimes \ldots \otimes v_k}-d^k\Phi(z)\del{v_1\otimes \ldots \otimes v_k}}
    \\
    &\qquad + \sum_{j=1}^k \norm{d^k\Phi(Az)\del{v_1\otimes\ldots v_{j-1} \otimes (Av_j-v_j) \otimes A v_{j+1} \ldots \otimes A v_k}} 
    \\
    &\lesssim \sup_{z'\in[z,Az]}\norm{d^{k+1}\Phi(z')}_{op}\norm{A}_{op}\norm{A-\Id_d}_{op}\norm{z} + \sum_{j=1}^k \norm{d^{k}\Phi(Ax)}_{op}\norm{A}^{k-j}_{op}\norm{A-\Id_d}_{op} 
    \\
    &\le (k+\varepsilon_N)L_M\del{\norm{A-\Id_d}_{op}} \lesssim \varepsilon_N^{\beta-1},  
    \end{align*}
   by the mean value theorem. So $\norm{\grad^k \del{\Phi\circ P^TP^*}(z) - \grad^k \Phi(z)} \lesssim \varepsilon_N^{\beta-1} \le \varepsilon_N^{\beta-k}$.
    \end{proof}
    \begin{proof}[Proof of Proposition~\ref{prop:support_estimator_solution_properties_1}]
    For simplicity, we drop the index $i$ and the star notation and consider the minimization problem
    \[
    P, \curly{a_{S}}_{2\le \abs{S} < \beta} \in \argmin_{P,a_S \le \varepsilon_N^{-1}} \sum_{v_j\in \cV}  \norm{v_j - PP^Tv_j - \sum_{S} a_S (P^Tv_j)^S}^2,
    \]
    where $\argmin$ is taken over all linear isometric embeddings $P:\R^d \mapsto \R^D$ and all vectors $a_S \in \R^D$ that satisfy $\norm{a_S} \le \varepsilon_N^{-1}$.

    As a preliminary step, let us show that a solution always satisfies $a_S \perp \Im P$. Note that $PP^T$ is a projection on $\Im P$, so $v-P^TPv \perp \Im P$, and therefore
    \[
    \norm{v_j - PP^Tv_j - \sum_{S} a_S (P^Tv_j)^S}^2 \\
    = \norm{(\Id_D-PP^T)v_j - \sum_{S} (\Id_D-PP^T) a_S (P^Tv_j)^S}^2 + \norm{PP^T a_S (P^Tv_j)^S}^2, 
    \]
    so the $\argmin$ should satisfy $P^TPa_S = 0$, or equivalently $a_S\perp \Im P$.
    
    Note that $P^TP$ is a projection on $\Im P$.
        First, we prove that there is $P$ satisfying $\Im P \subseteq \spn \cV$. Assume the opposite and take $P$ such that $\dim (\Im P \cap \spn\cV)$ is the largest. Then, by the assumption $\dim (\Im P \cap \spn\cV) < \dim \Im P$, since $\dim \spn \cV \ge d = \dim \Im P$ there is a vector $u_P \in \Im P$ such that $u_P \perp \spn \cV$, $\norm{u_{P}} = 1$ and similarly a vector $u_{\cV}\in \spn \cV$ such that $u_{\cV} \perp \Im P$, $\norm{u_{\cV}} = 1$.

        Let $e_1,\ldots, e_d\in \R^D$ be column vectors in $P$, then $\innerproduct{e_i}{e_j} = \delta_{ij}$ and
        \[
        P^Tv = \del{\innerproduct{v}{e_1},\ldots, \innerproduct{v}{e_d}}.
        \]
        We introduce vectors $
        f_i = e_i+ \del{u_{\cV} - u_P}\innerproduct{u_P}{e_i} $ and note that $u_{\cV} \perp e_i$ for all $i$, so 
        \[
        \innerproduct{f_i}{f_j} = \innerproduct{e_i+ \del{u_{\cV} - u_P}\innerproduct{u_P}{e_i}}{e_j+ \del{u_{\cV} - u_P}\innerproduct{u_P}{e_j}} = \delta_{ij}.
        \]
        We introduce a new operator 
        \[
        Q = \del{f_1,\ldots, f_d}^T,
        \]
        and note that by construction $u_{P} \perp \Im Q$ and $u_{\cV} \in \Im Q$, since 
        \[
        \sum_i \innerproduct{e_i}{u_{P}}f_i = \innerproduct{e_i}{u_{P}}\del{e_i+ \del{u_{\cV} - u_P}\innerproduct{u_P}{e_i}} = \underbrace{\del{\sum \innerproduct{e_i}{u_{P}}^2}}_{=\norm{u_P} = 1}u_{\cV},
        \]
        so
        \£
        \label{eq:Q_larger_P}
        \dim\del{\Im Q \cap \spn \cV} = \dim\del{\Im P \cap \spn \cV} + 1.
        \£
        Then for any $v_i \in \cV$, since $u_{P} \perp \spn \cV$
        \begin{align*}
        \innerproduct{e_j}{v_i} &= \innerproduct{f_j}{v_i} - \innerproduct{u_P}{e_j}\innerproduct{u_{\cV}}{v_i}  
         = \innerproduct{f_j}{v_i} - \sum_{k=1}^d\innerproduct{f_k}{v_i}\innerproduct{e_k}{u_P}\innerproduct{u_P}{e_j},
        \end{align*}
        so expanding the brackets we represent
        \£
        \label{eq:a_S_vs_b_S_polynomials}
        \sum_{2 \le \abs{S} < \beta} a_S\del{P^Tv_i}^S &= \sum_{\abs{S} < \beta} a_S\del{\innerproduct{e_1}{v_i}, \ldots, \innerproduct{e_d}{v_i}}^S \nonumber \\
        &= \sum_{\abs{S} < \beta} b_S\del{\innerproduct{f_1}{v_i}, \ldots, \innerproduct{f_d}{v_i}}^S = \sum_{2 \le \abs{S} < \beta} b_S\del{Q^Tv_i}^S,
        \£
        for some vector coefficients $b_S \in \R^D$. 
        
        We are finally ready to present an estimator that has a smaller or equal error w.r.t. $l_{\cV}$,
        \[
        \Phi'(z) = Qz + \sum_{2 \le \abs{S} < \beta} \underbrace{\del{b_S-\innerproduct{b_S}{u_P}u_P - \innerproduct{b_S}{u_{\cV}}u_{\cV}}}_{=c_S}z^S
        \]
        and the reconstruction error is equal to 
        \[
        \frac{1}{\abs{\cV}}\sum_{v_i\in \cV} \norm{v_i-QQ^Tv_i -\sum c_S\del{Q^Tv_i}^S}^2.
        \]
        
        We decompose $\R^D$ into two subspaces $L_1 = \spn\curly{u_{\cV},u_P}$ and $L_2 = L_1^{\perp}$ -- the orthogonal complement of $L_1$, and denote as $\pr_1$, $\pr_2$ corresponding projections. By construction for $v_i \in \cV$, since $c_S \perp u_{P}, u_{\cV}$ and $v_i \perp u_{P}$, and $Q^TQ$ is a projector on $\Im Q\ni u_{\cV}$
        \[
        \pr_1\del{v_i-QQ^Tv_i -\sum c_S\del{Q^Tv_i}^S} = \pr_1\del{v_i-Q^Tv_i} = \innerproduct{v_i -QQ^Tv_i}{u_{\cV}}u_{\cV} = 0.
        \]
        At the same time since by construction $c_S = \pr_2 b_S$, applying~\eqref{eq:a_S_vs_b_S_polynomials}
        \[
        \pr_2\del{v_i-QQ^Tv_i -\sum c_S\del{Q^Tv_i}^S} =\pr_2\del{v_i-PP^Tv_i -\sum a_S\del{P^Tv_i}^S}, 
        \]
        where we additionally used that $\pr_2 QQ^Tv_i = \pr_2 PP^Tv_i$ following from
        \[
        QQ^T v_i = \sum \innerproduct{f_j}{v_i}f_j = \sum \innerproduct{e_j + (u_{\cV}-u_{P})\innerproduct{e_j}{u_{P}}}{v_i}\del{e_j + (u_{\cV}-u_{P})\innerproduct{e_j}{u_{P}}} 
        \\
        = \sum_j \innerproduct{e_j}{v_i}e_j + 
        \innerproduct{e_j + (u_{\cV}-u_P)\innerproduct{e_j}{u_{P}}}{v_i}\innerproduct{e_j}{u_{P}}\del{u_{\cV}-u_{P}} + \innerproduct{u_{\cV}-u_{P}}{v_i}\innerproduct{e_j}{u_{P}}e_j
        \\
        =
        P^TP v_i + \square{\sum_j \innerproduct{e_j + (u_{\cV}-u_P)\innerproduct{e_i}{u_{P}}}{v_i}\innerproduct{e_i}{u_{P}}}\del{u_{\cV}-u_{P}} + \innerproduct{u_{\cV}-u_P}{v_i}u_P.
        \]
        And we finally conclude that for any $v_i\in \cV$
        \[
        \norm{v_i-PP^Tv_i -\sum a_S\del{P^Tv_i}^S}^2 
        \\
        = \norm{\pr_1\del{v_i-PP^Tv_i -\sum a_S\del{P^Tv_i}^S}}^2 + \norm{\pr_2\del{v_i-PP^Tv_i -\sum a_S\del{P^Tv_i}^S}}^2
        \\
        \ge \norm{v_i-QQ^Tv_i -\sum b_S\del{Q^Tv_i}^S}^2.
        \]
        Averaging over all $v_i\in \cV$ and using~\eqref{eq:Q_larger_P} we get the contradiction, so there is a $P$ satisfying $\Im P \subseteq \spn \cV$.

        Finally, choosing $\Im P \subset \cV$ for all $v\in \cV$ by construction $v-PP^Tv \in \cV \cap \Ker P$, so 
        \[
        \norm{v_j - PP^Tv_j - \sum_{S} a_S (P^Tv_j)^S}^2 \ge \norm{v_j - PP^Tv_j - \sum_{S} \pr_3 a_S (P^Tv_j)^S}^2,
        \]
        where $\pr_3$ is the projection on $\spn \cV \cap \Ker P$ showing that there is a solution satisfying $a_S \in \spn \cV \cap \Ker P$ for all $S$.
        \end{proof}

    \begin{proof}[Proof of Proposition~\ref{prop:support_estimator_solution_properties_2}]
        Recall that $\varepsilon_N = \del{C_{d,\beta}\frac{p^2_{\max}}{p^3_{\min}}\frac{\log N}{N-1}}^{\frac{1}{d}} \simeq \del{\frac{\log N}{N}}^{\frac{1}{d}}$, where a constant $C_{d,\beta}$ depends only on $d$ and $\beta$. Let us take $\cF=\curly{F_1, \ldots F_K}$ -- a minimal in size $\varepsilon_N$-dense set on $M$. For a point $y\in M$ there is a point $F_j$, such that $\norm{y-F_j} \le \varepsilon_N$, so $B_M(y,\varepsilon_N) \subset B_M(F_j, 2\varepsilon_N)$ and therefore $\abs{\cV_y} \le \sup_{j} \abs{\cG \cap B_M(F_j,2\varepsilon_N)}$. So, it is enough to bound quantity $\sup_{j} \abs{\cG \cap B_M(F_j,2\varepsilon_N)}$.

        Since points $G_i$ are i.i.d. samples from measure $\mu$ with density $p_{\min} \le p(y)\le p_{\max}$ the probability of the event $\rP\del{G_i \in B_M(F_j,2\varepsilon_N)} = \mu\del{B_M(F_j,2\varepsilon_N)} \simeq \varepsilon_N^{d}  \simeq \frac{\log N}{N}$. So by Chernoff's multiplicative bound, for a large enough constant $C$
        \[
        \rP\del{\abs{\cY \cap B_M(F_j,2\varepsilon_N)} \ge C u\log N } \le 2e^{-u^2\log N},
        \]
        and taking $u = \sqrt{\frac{\log 2\delta^{-1}}{\log N}}$
        \[
        \rP\del{\abs{\cG \cap B_M(F_j,2\varepsilon_N)} \ge C\sqrt{\log 2\delta^{-1} \log N} } \le \delta.
        \]
        Taking $\delta = N^{-\beta/d}/K$ and summing over all $F_j$ since $K \le (\varepsilon_N)^{-d}\Vol M \simeq \frac{N}{\log N}$ we conclude that there is a constant $C_{\dim}$ s.t.
        \[
        \rP\del{\abs{\cG \cap B(F_j,2\varepsilon)} \ge C_{\dim}\log N } \le N^{-\frac{\beta}{d}}.
        \]
    \end{proof}
    \subsection{Effect of the Manifold Approximation on the Score Function}
    \label{sec:high_prob_bound_M_star}
    
    Recall that the Kullback-Leibler divergence between two distributions $P$ and $Q$ on $\R^D$ with densities $p$ and $q$ respectively is defined as
    \[
    D_{KL}(P\|Q) = \int_{\R^D} \log \frac{p(x)}{q(x)} p(x)dx.
    \]
    We start with the following auxiliary statement.
    \begin{proposition}[De Bruijn identity]
    \label{prop:kl_div_derivative}
    In the same notation as in Proposition~\ref{prop: wd_vs_sml}, if $P_t$ and $Q_t$ are the laws of $X_t$ and $Y_t$
    \[
    \frac{d}{dt} D_{KL}(P_t\| Q_t) = -2\expectation{\|\grad\log p(t,X_t) - \grad\log q(t,X_t)\|^2}.
    \]
    
    \end{proposition}
    \begin{proof}
        See \cite{hassan2023},~Theorem 7.1.
    \end{proof}

        \begin{proof}[Proof of Proposition~\ref{prop: wd_vs_sml}]
    Since $W_2(P,Q) \le \varepsilon$, there exists a coupling $\Gamma$ between $P$ and $Q$ such that
    \[
    \int|v-w|^{2} \Gamma(dv, dw) \le \varepsilon^2.
    \]
    Letting  $P_t$ and $Q_t$ denote the laws of $X_t$ and $Y_t$ respectively, by \Cref{prop:kl_div_derivative}
    \[
    \frac{d}{dt} D_{KL}(P_t\| Q_t) = -2\expectation{\|\grad\log p(t,X_t) - \grad\log q(t,X_t)\|^2}.
    \]
    Integrating over $t$ we rewrite
    \begin{align*}    
    \int_{t_{\min}}^{t_{\max}}\int_{\R^D}\|\grad\log p(s, x) - \grad\log q(s, x)\|^2p_s(x)dx ds &= \frac{1}{2}\del{D_{KL}(P_{t_{\min}}\|Q_{t_{\min}}) - D_{KL}(P_{t_{\max}}\|Q_{t_{\max}})}  
    \\
    &\le \frac{1}{2}D_{KL}(P_{t_{\min}}\|Q_{t_{\min}}).
    \end{align*}
    It remains to show that 
    \begin{align*}
    D_{KL}(P_{t}\|Q_{t}) 
    &:= C_{\sigma_t}\int \log\frac{\int e^{-|x-c_tv|^2/2\sigma^2_t}P(dv)}{\int e^{-|x-c_tw|^2/2\sigma^2_t}Q(dw)}
    \int e^{-|x-c_tv|^2/2\sigma^2_t}P(dv) dx \\
    &\le W_2^2(P_t,Q_t)\frac{c^2_t}{2\sigma^2_t}.
    \end{align*}
    Using the fact that $\Gamma$ has marginals $P,Q$ we rewrite
    \[
    D_{KL}(P_{t}\|Q_{t}) = C_{\sigma_t}\int \log\frac{\iint e^{-|x-c_t w|^2/2\sigma^2_t}\Gamma(dv, dw)}{\iint e^{-|x-c_t v|^2/2\sigma^2_t}\Gamma(dv, dw)}{\iint} e^{-|x-c_t v|^2/2\sigma^2_t}\Gamma(dv,dw) dx.
    \]
    Next, using the log-sum inequality for a fixed $x$ we get
    \begin{align*}    
     &\log\frac{\iint e^{-|x-c_tv|^2/2\sigma^2_t}\Gamma(dv,dw)}{\iint e^{-|x-c_tw|^2/2\sigma^2_t}\Gamma(dv,dw)}{\textstyle\iint} e^{-|x-c_tv|^2/2\sigma^2_t}\Gamma(dv,dw)  \\
     &\le 
     \iint e^{-|x-c_tv|^2/2\sigma^2_t}\log\frac{e^{-|x-c_t v|^2/2\sigma^2_t}}{e^{-|x-c_tw|^2/2\sigma^2_t}}\Gamma(dv,dw) 
     \\
     &\le \iint \frac{|x-c_t w|^2-|x-c_t v|^2}{2\sigma^2_t}e^{-|x-c_t v|^2/2\sigma_t}\Gamma(dv,dw)\\
    &= \iint \frac{2c_t\innerproduct{v-w}{x} + c_t^2\del{|v|^2-|w|^2}}{2\sigma^2_t}e^{-|x-c_t v|^2/2\sigma_t}\Gamma(dv,dw).
    \end{align*}
    Integrating over $x$ we get
    \begin{multline*}    
    D_{KL}(P_{t}\|Q_{t}) 
    \le \frac{1}{2\sigma_t^2}C_{\sigma_t}\int \int \int \del{2c_t\innerproduct{v-w}{x} + c_t^2\del{|v|^2-|w|^2}}e^{-|x-c_tv|^2/2\sigma^2_t}\Gamma(dv,dw) dx 
    \\
     = \frac{1}{2\sigma^2_t}\iint 2c_t\innerproduct{\del{C_{\sigma_t}\int (x e^{-|x-c_tv|^2/2\sigma^2_t}dx}}{v-w} \Gamma(dv,dw) 
    \\
    + 
    \frac{1}{2\sigma^2_t}\iint \del{C_{\sigma_t}\int e^{-|x-c_t v|^2/2\sigma^2_t}dx} c_t^2\del{|w|^2-|v|^2} \Gamma(dv,dw).
    \end{multline*}
    Since $C_{\sigma_t}e^{-|x-c_t v|^2/2\sigma^2_t}$ is the density of $\cN(c_t v,\sigma^2_t)$ we get
    \begin{align*}    
    D_{KL}(P_{t}\|Q_{t}) 
    &\le
    \frac{1}{2\sigma^2_t}\iint 2c^2_t\innerproduct{v}{v-w} \Gamma(dv,dw) + \frac{1}{2\sigma^2_t} \iint c_t^2\del{|w|^2-|v|^2} \Gamma(dv,dw) 
    \\
    &=\frac{c_t^2}{2\sigma^2_t}\iint |v-w|^2 \Gamma(dv,dw) \le W^2_2(P,Q)\frac{c_t^2}{2\sigma_t^2},
    \end{align*}
    where the last equality follows from the definition of $\Gamma$.
    \end{proof}
    Our next goal is to prove~\Cref{prop:concentration_around_X0_approximation} an analog of Theorem~\ref{thm:concentration_around_X0} for approximation $M^*$. 
    We start with a simple proposition that bounds the length of the projection $Z_D\sim \cN\del{0,\Id_D}$ on subspaces $\cH_i = \spn \cV_i$ constructed in~\ref{sec:support_estimation}.
    \begin{proposition}
    \label{prop:bound_on_projection_on_cH_i}
        With probability at least $1-\delta$ for all $i\le N$ simultaneously
        \[
        \norm{\pr_{\cH_i} Z_D} \le \sqrt{(C_{\dim}+2)\log N + 2\log 2\delta^{-1}},
        \]
        where $C_{\dim}$ is given in Proposition~\ref{prop:support_estimator_solution_properties_2}.
    \end{proposition}
    \begin{proof} First note that 
        $\pr_{\cH_i}Z_D \sim \cN\del{0,\Id_{\dim \cH_i}}$, so using standard bounds~\cite{laurent2000} on the tails of the $\chi$-distribution
        \[
        \rP\del{\norm{\pr_{\cH_i}Z_D} \le \sqrt{\dim \cH_i +2\log\del{2\delta^{-1}N}}} \ge 1-\delta/N.
        \]
        Since $\dim \cH_i \le C_{\dim} \log N$,  we get
        \[
        \rP\del{\norm{\pr_{\cH_i}Z_D} \le \sqrt{C_{\dim}\log N +2\log\del{2\delta^{-1}N}}} \ge 1-\delta/N.
        \]
        Combining for all $i\le N$ we get the statement.

    \end{proof}
    
    \begin{proof}[Proof of Proposition~\ref{prop:TaylorGP}]
        We follow the same strategy as in the proof of Proposition~\ref{prop:maximum_of_Gaussian_Process_on_Manifold}.
        First, we prove that the statement holds on a submanifold $M_i = \Phi_i\del{B_d(0,\varepsilon_N )}$, and then we spread it on the whole manifold.
        
         Take $y\in M_i$ and represent it as $y = \Phi_i(z)$, where $z \in B_d(0,\varepsilon_N)$, and define ${y^*=\Phi^*_i(z)}$. Then $y-y^* = \Phi_i(z)-\Phi_i^*(z)$. By construction $\norm{\grad \Phi_i(z)-\grad \Phi_i^*(z)} \le L^*|z|^{\beta-1}$, so on $B_d(0,\varepsilon_N)$ the function $\Phi_i-\Phi^*_i$ is $L^*\varepsilon_N^{\beta-1}$ Lipschitz, and $\Phi_i(0) = \Phi^*_i(0) = G_i$. Therefore by~\Cref{eq:maximum_of_GP_F} for all $\delta <1$
        \[
        \rP\del{\sup_{z\in B(0,\varepsilon_N)} \innerproduct{Z_D}{\Phi_i(z)-\Phi_i^*(z)} \le L^*\varepsilon_N^\beta\del{\sqrt{d} + \sqrt{2\log {(\delta/N)^{-1}}}}} \ge 1-\delta/N.
        \]
        So, the probability that such an event holds for all $M_i$ is bounded from below as 
        \[
        \rP\del{\forall i: \sup_{z\in B(0,\varepsilon_N)} \innerproduct{Z_D}{\Phi_i(z)-\Phi_i^*(z)} \le L^*\varepsilon_N^\beta\del{\sqrt{d} + \sqrt{2\log N + 2\log \delta^{-1}}}} \ge 1-\delta.
        \]
        Applying the same bound to $-\innerproduct{Z_D}{\Phi_i(z)-\Phi_i^*(z)}\stackrel{dist.}{=}\innerproduct{Z_D}{\Phi_i(z)-\Phi_i^*(z)}$ and using $\sqrt{a} + \sqrt{b} \le \sqrt{2(a+b)}$ we get the statement.
    \end{proof}
    Next, we prove an analog of Proposition~\ref*{prop:RepresentationOfConditionalMeasure} for $M^*$.
    
    \begin{proposition} 
    \label{prop:bound_on_X_t_y_star}
    Let $t\ge \varepsilon^{2\beta}_N$, $X_0\sim \mu, Z_D \sim \cN \del{0, \Id_D}$ and $X_t = c_tX_0 + \sigma_t Z_D$. For any $\delta > 0$ a.s. in $X_0$ and with probability at least $1-\delta$ in $Z_D$ for all $i\le N$, $y \in M_i$ and $y^* = \Phi_i^*\circ \Phi^{-1}_i(y)$
    \£
    \label{eq:bound_on_squared_difference_Xt_y_new_approximation}
     \frac{1}{2}c^2_t\norm{X_0-y}^2-16\sigma^2_t\del{4d\log_+ (\sigma_t/c_t) + 2d\log \varepsilon^{-1}_N + (12d+2)C_{\log} + (L^*)^2 +  2\log \delta^{-1}}   
    \\
    \le
    \norm{X_t-c_t y^*}^2 - \sigma_t^2 \norm{Z_D}^2 \le
  \\
  \frac{3}{2}c^2_t\norm{X_0-y}^2 + 8\sigma^2_t\del{4d\log_+ (\sigma_t/c_t) + 2d\log\varepsilon^{-1}_N + (12d+2)C_{\log} + (L^*)^2 +  2\log \delta^{-1}}.
    \£
    \end{proposition}
    \begin{proof}
    
    First, we represent
    \[
    \norm{X_t-c_ty^*}^2 
    &= \norm{X_t-c_ty}^2 + 2c_t\sigma_t\innerproduct{Z_D}{y-y^*} + 2c^2_t\innerproduct{X_0-y}{y-y^*} + c^2_t\norm{y-y^*}^2
    \\
    &=
    \sigma^2_t\norm{Z_D}^2 + c^2_t\norm{X_0-y}^2 + 2c_t\sigma_t\innerproduct{Z_D}{X_0-y}  + 2c_t\sigma_t\innerproduct{Z_D}{y-y^*} 
    \\
    &\quad \quad + 2c^2_t\innerproduct{X_0-y}{y-y^*} + c^2_t\norm{y-y^*}^2.
    \]
    We recall \eqref{eq:useful_bound_on_scalar_product} stating that if 
    \[
    C^2_0 := 4d\log_+ (c_t/\sigma_t) + 4d\log \varepsilon^{-1}_N + (12d+2)C_{\log}  + 2\log 2\delta^{-1} \ge 4d\log_+ (c_t/\sigma_t) + (12d+2)C_{\log}  + 2\log 2\delta^{-1}
    \]
    then with probability at least $1-\delta$
    \[
    2\sigma_tc_t\abs{\innerproduct{Z_D}{X_0-y}} \le 2\sigma^2_tC_0
    + 2\sigma_tc_t\norm{X_0-y}C_0
    \le 
    c^2_t\norm{X_0-y}^2/4 + 6\sigma^2_t C^2_0.
    \]
    By Proposition~\ref*{prop:TaylorGP} since $\sigma_t \ge \varepsilon^\beta_N$ with probability at least $1-\delta$
    \[
    2c_t\sigma_t \abs{\innerproduct{Z_D}{y-y^*}} \le 
    2c_t\sigma_t L^*\varepsilon_N^{\beta}\del{\sqrt{d} + \sqrt{2d\log \varepsilon_N^{-1} + 2\log 2\delta^{-1}}}
    \le
    2\sigma^2_t\del{(L^*)^2 + C^2_0}.
    \]
    Finally,
    \[
    \abs{2c^2_t\innerproduct{X_0-y}{y-y^*} + c^2_t\norm{y-y^*}^2} 
    &\le 2c^2_t\norm{X_0-y}\norm{y-y^*} + c^2_t\norm{y-y^*}^2 
    \\
    &\le 2c^2_tL^*\varepsilon^\beta_N\norm{X_0-y}  + c^2_t(L^*)^2\varepsilon^{2\beta}_N
    \\
    &\le 2c^2_tL^*\sigma_t\norm{X_0-y} + (L^*)^2c^2_t\sigma^2_t 
    \\
    &\le c^2_t\norm{X_0-y}^2/4 + 5(L^*)^2\sigma_t^2. 
    \]
    Summing up, we get that with probability at least $1-\delta$
    \[
    \abs{2c_t\sigma_t\innerproduct{Z_D}{X_0-y}  + 2c_t\sigma_t\innerproduct{Z_D}{y-y^*} + 2c^2_t\innerproduct{X_0-y}{y-y^*} + c^2_t\norm{y-y^*}^2} 
    \\
    \le
    c^2_t\norm{X_0-y}^2/2 + 8\sigma^2_t C^2_0 +  7(L^*)^2\sigma^2_t,
    \]
    and the desired inequality
    \[
     \frac{1}{2}c^2_t\norm{X_0-y^*}^2- 8\sigma^2_t(C^2_0 + (L^*)^2)  
    &\le
    \norm{X_t-c_t y^*}^2 - \sigma_t^2 \norm{Z_D}^2 
    \\
    &\le
  \frac{3}{2}c^2_t\norm{X_0-y^*}^2 + 8\sigma^2_t(C^2_0 + (L^*)^2).
    \]
    \end{proof}
    Finally, we are ready to present the main result of this section. Consider a measure $\mu$ supported on $M$. Let $\mu^*$ be the measure on $M^*$ constructed in Section~\ref*{sec:manifold_approximation} corresponding to $\varepsilon_N > 0$.
    As in Section~\ref*{sec:high_probability_bounds} we consider a conditional measure $\mu^*(y|t, x) \propto e^{-\norm{x-c_t y^*}^2/2\sigma^2_t}\mu^*(dy^*)$.
  \begin{restatable}{proposition}{ConcentrationAroundXApproximation}    
    \label{prop:concentration_around_X0_approximation}
    Let a measure $\mu$ satisfy Assumption~\ref*{asmp:smooth_measure_on_manifold}, and measure $\mu^*$ be as above. Fix $t \ge \varepsilon_N^{2\beta}$, non-negative $\delta, \eta < 1$, and the radius
    \[
    r(t,\delta, \eta) = 20(\sigma_t/c_t)\sqrt{ d(\log_+ (c_t/\sigma_t) + \log \varepsilon^{-1}_N+ 3 C_{\log})  +(L^*)^2 +\log 2\delta^{-1} + \log \eta^{-1}}
    \]
    then with probability at least $1-\delta$
    in $Z_D\sim \cN\del{0,\Id_D}$ and a.s. in $X_0\sim \mu$ for a point $X_t = c_tX_0+\sigma_tZ_D$ the following bound holds 
        \[
        \int_{y^*\in M: \norm{y^*-X_0} \le r(t,\delta,\eta)} p(y^*|t,X_t) dy^* \ge 1-\eta.
        \]
    \end{restatable}
    \begin{proof}
    As in \Cref{prop:bound_on_X_t_y_star} we introduce $C^2_0 = 4d\log_+ (c_t/\sigma_t) + 4d\log \varepsilon^{-1}_N + (12d+2)C_{\log}  + 2\log 2\delta^{-1}$.
    
    Since $\mu^*(y^*|t,X_t)\propto e^{-\norm{X_t-c_t y^*}^2/2\sigma_t^2}\mu(dy^*)$, first, we control the normalization constant. Applying the change of variable formula, \Cref{prop:bound_on_X_t_y_star}, and then~\eqref{eq:log_p_lower_bound} with probability at least $1-\delta$
    \[
    B &:= \int_{M^*} e^{-\norm{X_t-c_t y^*}^2/2\sigma_t^2}\mu(dy^*)
    =
    \sum_i \int_{M^*_i} e^{-\norm{X_t-c_t y^*}^2/2\sigma_t^2}\mu^*_i(dy^*)
    \\
    &=
    \sum_i \int_{M_i} e^{-\norm{X_t-c_t \Phi^*_i\circ \Phi_i^{-1}(y)}^2/2\sigma_t^2}\mu_i(dy)
    \\
    &\ge
    e^{-\norm{Z_D}^2/2}e^{-4(C^2_0 + (L^*)^2)}\sum_i \int_{M_i} e^{-3\norm{X_0-y}^2/2\sigma_t^2}\mu_i(dy)
    \\
    &= 
    e^{-\norm{Z_D}^2/2}e^{-4(C^2_0 + (L^*)^2)}\int_{M} e^{-3c^2_t\norm{X_0-y}^2/2\sigma_t^2}\mu(dy)
    \\
    &\ge
    \frac{e^d}{(2d)^{d+1}}e^{-2}  p_{\min}e^{-\norm{Z_D}^2/2}e^{-4(C^2_0 + (L^*)^2)}(r_0 \wedge (\sigma_t/c_t))^d 
    \\
    &\ge e^{-\norm{Z_D}^2/2}e^{-C_{\log}(3d+1)-4(C^2_0 + (L^*)^2) -d \log_+ (c_t/\sigma_t)}.
    \]
    By the same change of variables formula and \Cref{prop:bound_on_X_t_y_star} with probability at least $1-\delta$
    \[
    A&:=\int_{y^* \in M^*: \norm{X_0-y^*} \ge r(t,\delta,\eta)} e^{-\norm{X_t-c_t y^*}^2/2\sigma_t^2}\mu^*(dy^*) 
    \\
    &= \sum_i \int_{y^*\in M^*_i: \norm{X_0-y^*} \ge r(t,\delta,\eta)} e^{-\norm{X_t-c_t y^*}^2/2\sigma_t^2}\mu^*_i(dy^*) 
    \\
    &= \sum_i \int_{y\in M_i: \norm{X_0-\Phi^*_i\circ \Phi^{-1}_i(y)} \ge r(t,\delta,\eta)} e^{-\norm{X_t-c_t \Phi^*_i\circ \Phi^{-1}_i(y)}^2/2\sigma_t^2}\mu_i(dy)
    \\
    &\le \sum_i \int_{y \in M_i: \norm{X_0-y} \ge r(t,\delta,\eta)- L^*\varepsilon^\beta_N} e^{-\norm{X_t-c_t \Phi^*_i\circ \Phi^{-1}_i(y)}^2/2\sigma_t^2}\mu_i(dy) 
    \\
    &\le e^{-\norm{Z_D}^2/2}e^{4(C^2_0+(L^*)^2)}\int_{y: \norm{X_0-y} \ge r(t,\delta,\eta)- L^*\varepsilon^\beta_N} e^{-c^2_t\norm{X_0-y}^2/4\sigma_t^2}\mu(dy) 
    \\
    &\le e^{-\norm{Z_D}^2/2}e^{4(C^2_0+(L^*)^2)-c^2_t\del{r(t,\delta,\eta)- L^*\varepsilon^\beta_N}^2/4\sigma_t^2}.
    \]
    Thus with probability at least $1-\delta$
    \[
    \mu^*\del{B_{M^*}(X_0, r(t,\delta, \eta) |t, X_t} = 1-\frac{A}{B}
    \\
    \ge 1 - e^{C_{\log}(3d+1) +  8(C^2_0+(L^*)^2) + d\log_+ (c_t/\sigma_t) -c^2_t\del{r(t,\delta,\eta)- L^*\varepsilon^\beta_N}^2/4\sigma_t^2}.
    \]
    So, for 
    \[
    r(t,\delta, \eta) \ge  L^*\varepsilon^\beta_N + 2(\sigma_t/c_t)\sqrt{\log \eta^{-1} + C_{\log}(3d+1) + d\log_+ (c_t/\sigma_t) +  8(C^2_0+(L^*)^2)},
    \]
    with probability at least $1-\delta$
    \[
    \mu^*\del{B_{M^*}(X_0, r(t,\delta, \eta) |t, X_t} \ge 1-\eta.
    \]
    The statement follows if we let $t\ge \varepsilon^{2\beta}_N$, note that $(\sigma_t/c_t) \ge \varepsilon^\beta_N$, and substitute the definition of $C_0$.   
    \end{proof}

    \section{Score Approximation by Neural Networks}
    \label{apdx:score_apprxoximation_by_neural_network}
    In this section, we prove Theorem~\ref*{thm:main_result} for $d\ge 3$ on the score approximation by a neural network.

    Recall that $\underline{T} \ge \del{\log n}^{\gamma}
    n^{-\frac{2(\alpha+1)}{2\alpha+d}}$ and $\overline{T} \lesssim \log n$.
    Throughout the appendix, we harmlessly enlarge the polylogarithmic exponent
    by replacing $\gamma$ with $4\gamma$ when convenient.
    
    We divide the interval $[\underline{T}, \overline{T}]$ into segments $\underline{T} = T_1 < T_2 < \ldots <T_K = \overline{T}$, where $T_{k+1}/T_k =2$ and build estimator $\hat{s}(t,x)$ as a union of separate estimators $\hat{s}_k(t,x)$, where $k$-th estimator approximates $s(t,x)$ on $[T_k,T_{k+1}]$. We find each $\hat{s}_k(t,x)$ as a solution of $\hat{s}_k(t,x) = \argmin_{\phi\in \cS_k} \cR_{\cY}(\phi)$, where the class $\cS_k$ was defined in~(\ref*{eq:definition of cS}). The first step is to construct $\phi\in \cS_k$ approximating the score; and in  \Cref{sec:main_theorem_t_large} we prove that for $T_k \ge n^{-\frac{2}{2\alpha+d}}$ 
    \[
    \int_{T_k}^{T_{k+1}}\expectation\norm{\phi(t,X_t)-s(t,X_t)}^2 dt \le \sigma^{-2}_{T_k}\del{\log n}^{4\gamma \beta} n^{-\frac{2(\alpha+1)}{2\alpha+d}},
    \]
    while in \Cref{sec:main_theorem_t_small} we prove that for $T_k < n^{-\frac{2}{2\alpha+d}}$
    \[
    \int_{T_k}^{T_{k+1}}\expectation\norm{\phi(t,X_t)-s(t,X_t)}^2 dt \le \sigma^{-2}_{T_k}\del{\log n}^{6\gamma \beta} n^{-\frac{2\beta}{2\alpha+d}} + n^{-\frac{2\alpha}{2\alpha+d}} (\log n)^{2\alpha+1}  
    \]

    In \Cref{apdx:generalization_error} we first bound the covering number of classes $\cS_k$, then we bound the variance of the empirical risk estimator via the covering number of $\cS_k$, thus obtaining Theorem~\ref*{thm:main_result}. 
    Finally, in~\cref{sec:ThC4Oko}, we provide correction of \cite{oko2023}, Theorem C.4 that we use in~\Cref{apdx:generalization_error}.

 In both cases ($T_k\ge n^{-\frac{2}{2\alpha+d}}$ and
 $T_k< n^{-\frac{2}{2\alpha+d}}$), we first fix  $\del{\log n}^{2\gamma d} n^{\frac{d}{2\alpha+d}} \ge N_k \ge n^{\frac{d-2}{2\alpha+d}}$, and by suppressing $k$, write
 $N=N_k$ and $
    \varepsilon_N :=
    \del{C_{d,\beta}\frac{p^2_{\max}}{p^3_{\min}}\frac{\log N}{N}}^{1/d}.
 $ Throughout the proof, we assume that $n$ is large enough, so $\varepsilon_N < r_0$.
 
 The construction of $\phi$ starts from a set
 $\cG := \curly{G_1,\ldots, G_N} \subset \cY= \{y_1, \cdots, y_n\}$,
 for instance $\cG  = \{ y_1, \cdots, y_N\}$ so that
 $\cG \sim \mu^{\otimes N}$.  
Following
 Section~\ref*{sec:support_estimation}, for each $G_i$ we consider
 $\cV_i := \curly{G_j-G_i\big| \norm{G_j-G_i} \le \varepsilon_N}$ and the polynomial functions
\[
    \Phi_i^*(z) = G_i + P^*_iz + \sum a^*_{i,S} z^S,
\]
where $P^*_i, a^*_{i,S}$ solution of~(\ref*{eq:support_estimation_functional}) satisfying $\Im P^*_i \subset \spn \cV_i, a^*_{i,S} \in \spn \cV_i$. The functions $\Phi_i^*$ are polynomial approximations of the local maps $\Phi_i$ defined as the local inverse to projections on $T_{G_i}M$ over the balls $B_d(0, r_0)$.

We introduce an event $\cE_{\mathrm{geom},k}$, depending only on the training sample $\cY$, under which the following hold:
\[
    \cG \text{ is }\varepsilon_N/2\text{-dense in }M,\qquad
    \sup_{i,z\in B_d(0,8\varepsilon_N)}
    \norm{\Phi_i^*(z)-\Phi_i(z)}\lesssim\varepsilon_N^\beta,
\]
\[
    \max_{i\le N}|\cV_i|\le C_{\dim}\log n,\qquad
    \sup_{i,z\in B_d(0,8\varepsilon_N),\,m<\beta}
    \norm{\grad^m\Phi_i^*(z)}\le2L_M,
\]
and
\[
    \sup_{i,z\in B_d(0,8\varepsilon_N)}
    \norm{\grad\Phi_i^*(z)\grad^T\Phi_i^*(z)-\Id_d}\le\frac12.
\]
Note that the last equation guarantees that $\Phi^*$ is Lipschitz on $B_d(0,8\varepsilon_N)$. 

By Proposition~\ref*{prop:samples_are_dense},
Lemma~\ref*{lemma:approximation_of_Phi},
Proposition~\ref*{prop:support_estimator_solution_properties_2}, and
\Cref{prop:geometric_statements}, this event has probability at least $1-O(N^{-\beta/d})\ge1-O(n^{-\frac{\alpha+1}{2\alpha+d}})$.

Throughout the appendix, high-probability statements are proved
conditionally on $\cE_{\mathrm{geom},k}$, with probability then taken only
over the fresh pair $(X_0,Z_D)$.

Also, recall that 
\£\label{eq:partition_function}
    \rho(u) 
    =
    \begin{cases}
    1 & \text{ if } 0\le u \le 1/2,
    \\
    2-2u & \text {if } u \in [1/2, 1],
    \\ 0 & \text{ if } u\geq 1
    \end{cases},
\£
and the localization functions are defined as 
\begin{equation}\label{rhoi}
\rho_i(t, X_t) = \rho\del{\frac{\del{\norm{X_t - c_t G_i}^2 - \norm{X_t - c_t G_{\min}(X_t)}^2}}{2C(T_k,n)}}
\end{equation}
where 
\begin{equation}\label{def:Ctn}
\begin{split}
    C(T_k,n) & = \sigma_{T_k}^2 \log n \log \log n  \quad \text{if } T_k\geq n^{-2/(2\alpha+d)}\\
    C(T_k,n) &= 6 \varepsilon_N^2\quad \text{if } T_k< n^{-2/(2\alpha+d)}
    \end{split}
\end{equation} 
and where
\[
    i_{\min}(x)\in\argmin_{i\leq N}\|x-c_tG_i\|,
    \qquad
    G_{\min}(x):=G_{i_{\min}(x)} .
\]
Thus $\rho_i$ is equal to $0$ when $G_i$ is far from $x/c_t$ compared to
the closest center $G_{\min}(x)$.  Finally recall that $s(t,x) = \frac{c_t}{\sigma^2_t}e(t,x) - x/\sigma_t^2$ where 
\[ 
e(t,x) = %\frac{c_t}{\sigma^2_t}
\frac{\int_{M} y e^{-\frac{\norm{x-c_ty}^2}{2\sigma^2_t}}\mu(dy)}{\int_{M} e^{-\frac{\norm{x-c_ty}^2}{2\sigma^2_t}}\mu(dy)}.
\]
We first consider the (simpler) case  $T_k\ge n^{-\frac{2}{2\alpha+d}}$.

\subsection{Case $T_k\ge n^{-\frac{2}{2\alpha+d}}$}
\label{sec:main_theorem_t_large}

Choose
\[
N=N_k
:=
\left\lceil
\max\Big( (\log n)^{2\gamma d}\,(\sigma_{T_k}/c_{T_k})^{-d},\ \ n\cdot n^{-\frac{2(\alpha+1)}{2\alpha+d}}\Big)
\right\rceil
\]

By construction (for $n$ large) we have $N \le\ 4^d(\log n)^{2\gamma d}\,n^{\frac{d}{2\alpha+d}}$ and
\£
\label{eq:basic_ineq_t_large}
(\log n)^{1/d}\,n^{-\frac{d-2}{d}\frac{1}{2\alpha+d}}\ \ge\ \varepsilon_N\ge\ (\log n)^{-2\gamma}\,n^{-\frac{1}{2\alpha+d}},
\qquad
 (\log n)^{1/d-2\gamma}(\sigma_{T_k}/c_{T_k}) \ge \varepsilon_N
\£
By Assumption~\ref*{asmp:smooth_measure_on_manifold_in_proofs} we have $\varepsilon_N^\beta \lesssim (\log n)^{\beta/d} n^{-\frac{\alpha+1}{2\alpha+d}}$. Let $M_i := \Phi_i\del{B_d(0,\varepsilon_N)}$ and define $M^*_i = \Phi_i^*\del{B_d(0,\varepsilon_N)}$
so that the surface $M^*_i$ is a polynomial approximation of $M_i$ and define $M^* = \cup_{i=1}^N M_i^*$. Since $\diam M \le 1$ by properties of $\Phi_i^*$, $\diam M^* \le 2$.

We build the partition of unity on $M$ corresponding to the set $\cG_i$ as  
    \[
    \chi_i(y) = \frac{\rho(\norm{y-G_i}/\varepsilon_N)}{\sum_{j=1}^N \rho(\norm{y-G_j}/\varepsilon_N)},
    \]
    where $\rho$ is defined in \eqref{eq:partition_function}. 
Since $\cG$ is $\varepsilon_N/2$ dense on $M$, $\sum \chi_i(y) = 1$ for all $y\in M$. We consider the measure $\mu_i$ with density $\chi_i(y)p(y)$ supported on $M\cap B_D(G_i,\varepsilon_N) \subset M_i$ and denote by $\mu^*_i$ its pushforward on $M^*_i$ by $\Phi^*_i\circ \Phi^{-1}_i$. By construction $\mu^* := \sum \mu^*_i$ is a probability measure satisfying $W_2(\mu, \mu^*) \lesssim \varepsilon_N^\beta$  %n^{-\frac{\alpha+1}{2\alpha+d}}(\log n)^{\beta/d}%
Therefore, by Proposition~\ref*{prop: wd_vs_sml} the score function $s^*$ corresponding to measure $\mu^*$ satisfies 
\begin{equation}
\label{eq:D1_W2_vs_SML}
\int_{T_k}^{T_{k+1}}\expectation \|s(t,X_t)-s^*(t,X_t)\|^2 dt \lesssim \sigma^{-2}_{T_k} \varepsilon_N^{2\beta} \le \sigma^{-2}_{T_k} n^{-\frac{2(\alpha+1)}{2\alpha+d}} (\log n)^{\frac{2\beta}{d}}.
\end{equation}
and since $\sigma^2_t \le 2 \sigma_{T_k}$ for $t\in [T_k,T_{k+1}]$
\begin{equation}
\label{eq:D1_W2_vs_SML}
\int_{T_k}^{T_{k+1}} \sigma^2_t\expectation \|s(t,X_t)-s^*(t,X_t)\|^2 dt \lesssim n^{-\frac{2(\alpha+1)}{2\alpha+d}} (\log n)^{\frac{2\beta}{d}}.
\end{equation}

This achieves our first intermediate goal of replacing the original manifold $M$ with its approximation $M^*$. We start by studying $s^*$ more closely.

As we discussed in Section~\ref*{sec:main_results} introducing $\cA_i = G_i +  \spn \cV_i$, $\cH_i = \spn \cV_i$ and $x_{i,t} = \pr_{c_t\cA_i} x$ we can represent the score function $s^*(t,x)$ as
\£
\label{eq;s_star_as_sum}
s^*(t,x) = 
\sum_{i=1}^N p^*(i|t,x) s^*_i(t,x)
=
\frac{c_t}{\sigma^2_t}\sum_{i=1}^N p^*(i|t,x) e^*_i(t,x)-\frac{x}{\sigma^2_t}
\£
where $\frac{c_t}{\sigma^2_t}e^*_i(t,x) := s^*_i(t,x) -\frac{x}{\sigma^2_t}$ and $s^*_i(t,x)$ is the score function corresponding to the measure $\mu^*_i$, and %$p^*(i|t,x) := \rP\del{X^*(0)\in M^*_i| X^*(t) = x}$, \textcolor{red}{Are the $M_i^*$ disjoint? I thought not, in which case $p^*(M^*_i|t,x)\neq \rP\del{X^*(0)\in M^*_i| X^*(t) = x}$. So a bit confusing I think we should directly to the formula below.} where $X^*(t)$ is an OU process with initial condition $\mu^*$, so
\[
p^*(i|t,x) &=\frac{ \int_{M^*_i}e^{-{\norm{x-c_t y^*}^2}/{2\sigma^2_t}}\mu^*_i(dy^*)}{ \sum_{j=1}^N\int_{M^*_j}e^{-{\norm{x-c_t y^*}^2}/{2\sigma^2_t}}\mu^*_j(dy^*)},
\\ 
s^*_i(t,x) &= \frac{1}{\sigma^2_t}\frac{\int_{M^*_i}(c_t y-x)e^{-\norm{x-c_t y^*}^2/2\sigma^2_t}\mu^*_i(dy^*)}{\int_{M^*_i}e^{-\norm{x-c_t y^*}^2/2\sigma^2_t}\mu^*_i(dy^*)} = \frac{c_t}{\sigma^2_t} e_i^*(t,x) -\frac{  x}{\sigma_t^2}.
\]

As a next step, we show that terms corresponding to $\rho_i(t,x) < 1$, where $\rho_i$ is defined by $\eqref{rhoi}$, are negligible in~\eqref{eq;s_star_as_sum}. So we can approximate $s^*$ by the truncated version denoted $s^*_{tr}$ defined by: 
\begin{equation}\label{trunc_sstar}
s^*_{tr}(t,x) := \frac{c_t}{\sigma^2_t}\frac{\sum \rho_i(t,x) p^*(i|t,x) e^*_i(t,x)}{\sum \rho_i(t,x) p^*(i|t,x) } - \frac{x}{\sigma^2_t}, 
\end{equation}
We introduce
\[
    I_1(t)
    :=
    \left\{i\le N:
    \norm{X_0-G_i}
    \le \frac 13 (\sigma_t/c_t)\sqrt{\log n\log\log n}
    \right\},
\]
and
\[
    I_2(t)
    :=
    \left\{i\le N:
    \norm{X_0-G_i}
    \le 3(\sigma_t/c_t)\sqrt{\log n\log\log n}
    \right\}.
\]
and in Proposition~\ref{prop:active_indices_large_tail} show that with probability at least $1-n^{-d}$
\[
    I_1(t)\subseteq\{i:\rho_i(t,X_t)=1\},
    \qquad
    \{i:\rho_i(t,X_t)>0\}\subseteq I_2(t).
\]
Then by Proposition~\ref{prop:concentration_around_X0_approximation} there is an absolute constant $C>0$ such that
\[
    r_t
    \le
    C(\sigma_t/c_t)
    \sqrt{d\log_+(c_t/\sigma_t)+dC_{\log}+d\log n}
\]
with probability at least $1 - n^{-d}$ satisfy
\[
    \int_{B_{M^*}(X_0,r_t)} p(y^*|t,X_t)dy^* \ge 1-n^{-d}.
\]
Since $\diam M^*_i \le 4\varepsilon_N$ by representing representing $M^* = \bigcup M^*_i$ we get
\[
B_{M^*}(X_0,r_t) \subset \bigcup_{i=1}^N \del{B(X_0,r_t)\cap M^*_i}  \subset \bigcup_{\norm{X_0-G_i} \le r_t + 4\varepsilon_N}^N M^*_i
\]
for sufficiently large $n$ we have
$
   r_t +4\varepsilon_N
    \le \frac13(\sigma_t/c_t)\sqrt{\log n\log\log n},
$
so  
\[
    \sum_{i\notin I_1(t)} p^*(i\mid t,X_t)
    \le n^{-d} \le n^{-2}.
\]
Therefore, since $\|e^*_i\del{t,X_t}\|\leq 2 c_t/\sigma_t^2$ a.s., with probability $1-n^{-2}$
\[
\frac{c_t}{\sigma^2_t}\norm{e^*_{tr}\del{t,X_t}-  e^*\del{t,X_t}} \le \frac{2c_tn^{-2}}{(1-n^{-2}) \sigma^2_t}.
\]

At the same time, a.s. 
\[
\norm{e^*_{tr}\del{t,X_t}- e^*\del{t,X_t}} \le \norm{e^*\del{t,X_t}} + \norm{e^*_{tr}\del{t,X_t}} \le 3, %\frac{c_t}{\sigma^2_t}
\]
so
\[
\expectation \norm{s^*_{tr}\del{t,X_t}- s^*\del{t,X_t}}^2 \le \frac{c^2_t}{\sigma^4_t}\del{16n^{-4} + 18n^{-2}} % \le 20n^{\frac{4}{2\alpha+d}}n^{-2},
\]
and since $T_k \le \overline{T} = O(\log n)$ and $d \ge 3$
\begin{equation}
\label{eq:error_of_s_star_tr}
\int_{T_k}^{T_{k+1}}\sigma_t^2 \expectation \norm{s^*_{tr}\del{t,X_t}- s^*\del{t,X_t}}^2 dt \lesssim n^{-2}\int_{T_k}^{T_{k+1}} \frac{ 1 }{t }dt  \lesssim  n^{-2}.
\end{equation}

\subsubsection{Polynomial Approximation of $s^*_{tr}$}
In this section, we introduce the main regularity bounds that allow us to construct an efficient neural network approximation of $e^*_{tr}$. We show that for some functions $h_i, f_i, g_i$ we can represent it as
\£
\label{eq:D12_intor_e^*}
e^*_{tr}(t,x) = \frac{\sum_i \rho_i(t,x) 
    h_i(t,x) f_i(t,x)
    }{\sum_i \rho_i(t,x) 
    h_i(t,x) g_i(t,x)
    }, %\frac{c_t}{\sigma^2_t}    
\£
where $\rho_i(t,x) h_i(t,x)$ are easy-to-calculate weight functions, $f_i(t,x), g_i(t,x)$ are functions that can be efficiently approximated by neural networks. This decomposition corresponds to \eqref*{eq:definition of phi} presented in Section~\ref*{sec:score_approximation}. Recall that
$x_{i,t} = \pr_{c_t \cA_i} x$  and since $M^*_i \subset \cA_i$ we may represent  $\norm{x-c_ty^*}^2 = \norm{x-x_{i,t}}^2 + \norm{x_{i,t}-c_ty^*}^2$ so that 
\[
p^*(i|t,x) = e^{-\frac{\norm{x-x_{i,t}}^2+\norm{x_{i,t}-c_tG_i}^2}{2\sigma_t^2}} \int_{M^*_i} e^{\frac{\norm{x_{i,t}-c_tG_i}^2-\norm{x_{i,t}-c_ty^*}^2}{2\sigma_t^2}}\mu^*_i(dy^*)
\]
Using a similar argument, we write
\[
e^*_i(t,x) = \frac{\int_{M^*_i} y^* e^{(\norm{x_{i,t}-c_tG_i}^2-\norm{x_{i,t}-c_ty^*}^2)/2\sigma^2_t}\mu^*_i(dy^*)}{\int_{M^*_i} e^{(\norm{x_{i,t}-c_tG_i}^2-\norm{x_{i,t}-c_ty^*}^2)/2\sigma^2_t}\mu^*_i(dy^*)}. %\frac{c_t}{\sigma^2_t}
\]

We can thus write $e^*_i = f_i/g_i$ with
\[
{
\begin{aligned}
g_i(t,x) &:= \int_{M^*_i}
\exp\del{\frac{\norm{x_{i,t}-c_tG_i}^2-\norm{x_{i,t}-c_ty^*}^2}{2\sigma_t^2}}
\mu^*_i(dy^*), \\
f_i(t,x) &:= \int_{M^*_i} y^*
\exp\del{\frac{\norm{x_{i,t}-c_tG_i}^2-\norm{x_{i,t}-c_ty^*}^2}{2\sigma_t^2}}
\mu^*_i(dy^*),
\end{aligned}
}
\]
and take
{
$h_i(t,x) := \exp\del{\frac{\norm{x-c_t G_{\min}(x)}^2 -\norm{x-c_tG_i}^2}{2\sigma_t^2}}$
}
to get representation~\eqref{eq:D12_intor_e^*}.

We establish regularity bounds to show that $f_i$ and $g_i$ can be efficiently approximated by neural networks. Noting that $\sqrt{\log n\log \log n} >>
\sqrt{(d + C_{\dim} + 5) \log n + 4d C_{\log}  + (L^*)^2}$,  we have:
\begin{itemize}

    \item
By Proposition~\ref{prop:active_indices_large_tail}, with probability at least $1-n^{-d}$ every active index
$\rho_i(t,X_t) > 0$ belongs to $I_2$, i.e.
\£
\label{eq:bound_on_distance_to_G_i_t_large}
\norm{X_0-G_i}
\le 3\frac{\sigma_t}{c_t}\sqrt{\log n\log\log n}.
\£

\item Applying Proposition~\ref*{prop:TaylorGP} with $\delta = n^{-2}$, since $\varepsilon_N > n^{-1}$, we have that with probability $1-n^{-2}$ for all $y_i\in M_i, y^*_i = \Phi^*_i\circ \Phi^{-1}_i(y_i)$ 
\[
\abs{\innerproduct{Z_D}{y_i-y^*_i}} &\le \abs{\innerproduct{Z_D}{y_i- y_i^*}} \le L^*\varepsilon_N^{\beta}\sqrt{(2d+5)\log n}.
\]
\item 
    Applying Proposition~\ref*{prop:maximum_of_Gaussian_Process_on_Manifold} with $\delta =n^{-2}$ and $\varepsilon = n^{-2}$ for all $y, y'\in M$
    \[
    \abs{\innerproduct{Z_D}{y-y'}} \le 4n^{-2}\sqrt{d} + \del{\norm{y-y'}+6n^{-2}}\sqrt{(4d+2)\log 2n^{2} + 4dC_{\log}}.
    \]
\item 
As a corollary of the two previous statements, since $\beta \ge 3$, for all ${y\in M, y^*\in M^*}$ s.t.\ $\norm{y-y^*} \lesssim 4\varepsilon_N$ we have

\£
\label{eq:bound_on_scalar_product_case_t_large}
\abs{\innerproduct{Z_D}{y-y^*}} \le 5\varepsilon_N \sqrt{(8d+4)\log n + (4d+2)\log 2 + 4dC_{\log} } \le \varepsilon_N \sqrt{\log n\log \log n} .
\£
\end{itemize}

We first bound the term in the exponent $\norm{X_{i,t}-c_tG_i}^2-\norm{X_{i,t}-c_ty^*}^2$ where $X_{i,t} = \pr_{c_t \cA_i} X_t$ and $y^*\in M^*_i$. 
\[
\norm{X_{i,t}-c_tG_i}^2
-\norm{X_{i,t}-c_ty^*}^2 = \norm{X_t-c_tG_i}^2-\norm{X_t-c_ty^*}^2 
\\
= -c^2_t\norm{y^* - G_i}^2 + 2c^2_t\innerproduct{y^*-G_i}{X_0-G_i} + 2c_t\sigma_t\innerproduct{y^*-G_i}{Z_D}.
\]

Combining ~\eqref{eq:bound_on_distance_to_G_i_t_large}~and~\eqref{eq:bound_on_scalar_product_case_t_large}, since for all $y^*\in M^*_i$ $\norm{y^*-G_i} \le 2\varepsilon_N < (\sigma_t/c_t)(\log n)^{-\gamma}$,
with probability at least $1-n^{-2}$ for all $i\in I_2$ and $y^*\in M^*_i$  we have
\[
\abs{\norm{X_{i,t}-c_tG_i}^2-\norm{X_{i,t}-c_ty^*}^2} 
\\
\le
4c^2_t\varepsilon^2_N + 2c^2_t\cdot 2\varepsilon_N\cdot 3(\sigma_t/c_t)\sqrt{\log n\log\log n} 
+ 2c_t\sigma_t \cdot \varepsilon_N\sqrt{\log n\log\log n}
\\
\le \sigma^2_t (\log n)^{-2\gamma} 
+ 7\sigma^2_t\sqrt{\log n\log\log n} (\log n)^{-\gamma} 
\]
Therefore, assuming that $\gamma > 2$ and that $n$ large enough, by using $\abs{1-e^{x}} \le 2\abs{x}$ for $x\le 1$ we obtain 
\£
\label{eq:bound_on_exponent_t_large}
\frac{\abs{\norm{X_{i,t}-c_tG_i}^2-\norm{X_{i,t}-c_ty^*}^2}}{2\sigma^2_t} &\le \frac{1}{2}(\log n)^{-\gamma+2},\nonumber 
\\
\abs{e^{-({\norm{X_{i,t}-c_tG_i}^2-\norm{X_{i,t}-c_ty^*}^2})/2\sigma^2_t} - 1}  &\le (\log n)^{-\gamma+2}.
\£
Now we can bound $g_i(t,X_{i,t})$ from above and below. \Cref{prop:Manifold_Ball_Volume} states that $\Vol M_i \le C_d\varepsilon^d_N$ with probability $1-n^{-2}$ for all $i\in I_2$.  We bound $g_i(t,X_{i,t})$ from above by
\[
g_i(t,X_{i,t}) &= \int_{M^*_{{i}}} e^{{\frac{\norm{X_{i,t}-c_tG_i}^2-\norm{X_{i,t}-c_ty^*}^2}{2\sigma_t^2}}}\mu^*_{i}(dy^*) 
\\
&\le \int_{M^*_{i}} \del{1+(\log n)^{-\gamma+2}}\mu^*_{i}(dy^*) 
\\
&= \del{1+(\log n)^{-\gamma+2}}\int_{M_{i}} \chi_i(y) p(dy) \le 2^{d+1}p_{\max}C_d\varepsilon^d_N.
\]
At the same time, with probability $1-n^{-2}$ for all $i\in I_2$ we bound from below $g_i(t,X_{i,t})$ by
\[
g_i(t,X_{i,t}) &= \int_{M^*_i} e^{{\frac{\norm{X_{i,t}-c_tG_i}^2-\norm{X_{i,t}-c_ty^*}^2}{2\sigma_t^2}}}\mu^*_i(dy^*) 
\\
&\ge \int_{M^*_i} \del{1-(\log n)^{-\gamma+2}}\mu^*_i(dy^*) 
= \del{1-(\log n)^{-\gamma+2}}\int_{M_i} \chi_i(y) p(dy).
\]
By Proposition~\ref*{prop:support_estimator_solution_properties_2}, $\abs{\cY\cap B(G_i,4\varepsilon_N)}\le C_{\dim}\log N \le C_{\dim} \log n$, so for $y\in B_M(G_i, \varepsilon_N/2)$
\[
    \chi_i(y) = \frac{\rho(\norm{y-G_i}/\varepsilon_N)}{\sum_{j=1}^N \rho(\norm{y-G_j}/\varepsilon_N)} \ge \frac{1}{\abs{\cY \cap B(y,\varepsilon_N)}} \ge \frac{1}{C_{\dim}\log n},
\]
substituting and applying \Cref{prop:Manifold_Ball_Volume}
\[
g_i(t,X_{i,t}) \ge \del{1-(\log n)^{-\gamma+2}}\int_{M_i} \chi_i(y) p(dy) 
\\
\ge \frac{p_{\min}}{2C_{\dim}\log n}\text{Vol}(B_M(G_i, \varepsilon_N/2)) \ge \varepsilon^d_N 4^{-d}\frac{p_{\min}}{2C_{\dim}\log n}.
\]
Since $\diam M^* \le 2$, we bound $\norm{f_i}$ by
\[
\norm{f_i(y,X_{i,t})} &= \norm{\int_{M^*_i} y^*e^{{\frac{\norm{X_{i,t}-c_tG_i}^2-\norm{X_{i,t}-c_ty^*}^2}{2\sigma_t^2}}}\mu^*_i(dy^*)} 
\\
&\le 2 g_i(t,X_{i,t}) \le 2^{d+2}p_{\max}C_d\varepsilon^{d}_N.
\]
Since $
(\log n)^{1/d}n^{-\frac{d-2}{d}\frac{1}{2\alpha+d}}\ge \varepsilon_N \ge (\log n)^{-\gamma+1}n^{-\frac{1}{2\alpha+d}}$, 
for $n$ large enough, with probability $1-n^{-2}$ for all $i\in I_2$
\£
\label{eq:lower_bound_g_i}
g_i(t,X_{i,t}) &\ge \varepsilon^d_N 4^{-d}\frac{p_{\min}}{2C_{\dim}\log n} \gtrsim \frac{\varepsilon^d_N}{\log n} \ge (\log n)^{-\gamma d +d-1} n^{-\frac{d}{2\alpha+d}}% \ge n^{-2},%n^{-2} \quad \text{and }
\\
 \norm{f_i(t,X_{i,t})} &\le 2^{d+2}p_{\max}C_d\varepsilon^{d}_N \lesssim \varepsilon^d_N < 1.
\£
Note that the constants depend only on $C_{dim}, C_{\log}$ and $d$.
We show below that if $\hat{g}_i,\hat{f}_i, \hat{e}_i$ are such that with probability at least $1-n^{-2}$ for all $i\in I_2$
\begin{align}\label{bound:hatf_hatg}
\abs{\hat{g}_i(t,X_{i,t}) - g_i(t,X_{i,t})} \le n^{-2}(\log n)^{-\gamma d}, \,\, 
 \norm{\hat{e}_i(t,X_{i,t}) - e_i(t,X_{i,t})} \le n^{-1}, \,\, \norm{\hat{e}_i} \le 1,
\end{align}
where $\hat{e}_i = \hat{f}_i/\hat{g}_i$. Note that $\hat g_i(t, X_{i,t}) \ge g_i(t, X_{i,t})/2$ and $\norm{e_i} \le 1$. Then the approximation
\[
\hat{e}^*(t,x) :=  \frac{\sum_i \rho_i(t,x) 
    h_i(t,x) \hat{g}_i(t,x)\hat{e}_i(t,x)
    }{\sum_i \rho_i(t,x) 
    h_i(t,x) \hat{g}_i(t,x)} 
\]
verifies with probability greater than $1-n^{-2}$
\begin{align*}
    &\norm{e^*_{tr}(t,X_t)-\hat{e}^*(t,X_t)}  %\frac{\sigma^2_t}{c_t} 
    \\
    &\quad 
    = 
    \norm{
    \frac{
        \sum_i \rho_i(t,x) 
            h_i(t,x) \hat{g}_i(t,x)\hat{e}_i(t,x)
    }{
        \sum_i \rho_i(t,x) 
        h_i(t,x) \hat{g}_i(t,x)
}
    -  
    \frac{
        \sum_i \rho_i(t,x) 
        h_i(t,x) g_i(t,x) e_i(t,x)
    }{
        \sum_i \rho_i(t,x) 
        h_i(t,x) g_i(t,x)
    }
    }
    \\
    &\quad 
    \le 
    \norm{\frac{
        \sum \rho_i h_i \del{\hat{g}_i\hat{e}_i - g_i e_i}
    }{
    \sum \rho_i h_i \hat{g_i}
    }}
    +
    \norm{\frac{
        \sum \rho_i h_i g_i e_i}
        {\sum \rho_i h_i g_i}
        \del{\frac{\sum \rho_i h_i g_i}{\sum \rho_i h_i \hat{g}_i}-1}
        }
    \\
    &\quad 
    \le 
    \frac{
        \sum \rho_i h_i \hat{g}_i\norm{\hat{e}_i - e_i} 
    }{
    \sum \rho_i h_i \hat{g_i}
    }
    +
    \frac{
        \sum \rho_i h_i \norm{e_i}\abs{\hat{g}_i - g_i} 
    }{
    \sum \rho_i h_i \hat{g_i}
    }
    +
    \norm{e^*_{tr}}
        \frac{\sum \rho_i h_i \abs{g_i-\hat{g}_i}}{\sum \rho_i h_i \hat{g}_i}
    \\
    &\quad \le 
    \sup_{i\in I_2}\norm{\hat{e}_i - e_i} 
    +
    2\frac{
        \sup_{i\in I_2}\abs{\hat{g}_i - g_i} 
    }{
    \inf_{i\in I_2}\hat{g_i}
    }
     \leq  n^{-1} + 2   \frac{ n^{-2}}{ n} \le 3n^{-1} 
\end{align*}
Finally if we define $\hat{s}^*(t,x) =\frac{c_t}{\sigma^2_t}  \hat{e}^*(t,x) -\frac{x}{\sigma^2_t}$ and use that $\norm{e^*},\norm{\hat{e}^*} \le 1$ a.s.\ we get
\[
\expectation \norm{\hat{s}^*(t,X_t) - s^*_{tr}(t,X_t)}^2 \le \frac{4}{\sigma^4_t} \del{ n^{-2} + n^{-2} } \le \frac{8}{\sigma^4_t}n^{-2}.
\]
After integration, this results in
\£
\label{eq:error_of_s_star_clip}
\int_{T_k}^{T_{k+1}}\expectation \norm{\hat{s}^*(t,X_t) - s^*_{tr}(t,X_t)}^2dt \lesssim   \sigma_{T_k}^{-2}n^{-2} \int_{T_k}^{T_{k+1}}\frac{ 1 }{ \sigma_t^2} dt \lesssim  \sigma_{T_k}^{-2}n^{-2} \log n%
\£
\subsubsection{Neural Networks Approximation} \label{app:NNapprox:larget}

The goal of this section is to construct neural network functions $\hat{e}_i$ and $\hat{g}_i$ that satisfy \eqref{bound:hatf_hatg} for all $i \in I_2$ with probability at least $1-n^{-2}$. Since $\hat{e_i}$ is an approximation of $e_i = f_i/g_i$, instead of directly constructing $\hat{e}_i$ we focus on approximation of $f_i$ with $\hat{f}_i$, by~\eqref{eq:lower_bound_g_i} it is enough to guarantee $\norm{\hat{f}_i-f_i} \le n^{-2}(\log n)^{-\gamma d}$.

We first construct polynomial functions in $X_{i,t}$ and then we use \cite{oko2023} to approximate the polynomial functions by neural networks. 
Since $M^*_i = \Phi^*_i\del{B_d\del{0,\varepsilon_N}}$, using the change of variable $y^* = \Phi^*_i(z) $, we have

\[
\begin{aligned}
 g_i(t,x) &:= \int_{B_d\del{0,\varepsilon_N}} e^{\frac{\norm{x_{i,t}-c_tG_i}^2-\norm{x_{i,t}-c_t\Phi^*_i(z)}^2}{2\sigma_t^2}}p_i(z)dz, \\
 f_i(t,x) &:= \int_{B_d\del{0,\varepsilon_N}} \Phi^*_i(z) e^{\frac{\norm{x_{i,t}-c_tG_i}^2-\norm{x_{i,t}-c_t\Phi^*_i(z)}^2}{2\sigma_t^2}}p_i(z)dz.
\end{aligned}
\]
Recall that $\Phi^*_i(z) = G_i + P^*_iz + R_i(z)$, where $R_i(z)=O(\norm{z})^2$ and $R_i(z) \perp P^*_i z$, so
\[
\norm{x_{i,t}-c_tG_i}^2-\norm{x_{i,t}-c_t\Phi^*_i(z)}^2 = \norm{x_{i,t}-c_tG_i}^2-\norm{x_{i,t}-c_t\del{G_i + P^*_iz + R_i(z)}}^2 
\\
= 2c_t\innerproduct{x_{i,t}-c_tG_i}{P^*_i z} - c^2_t\norm{P^*_i z}^2 -2c_t\innerproduct{x_{i,t}-c_tG_i}{R_i(z)}  - c^2_t\norm{R_i(z)}^2
\\
= \underbrace{2c_t\innerproduct{(P^*_i)^T(x_{i,t}-c_tG_i)}{z} - c^2_t\norm{z}^2}_{=:F^{(1)}_i(x_{i,t},z)} -\underbrace{2c_t\innerproduct{x_{i,t}-c_tG_i}{R_i(z)}  - c^2_t\norm{R_i(z)}^2}_{=:F^{(2)}_i(x_{i,t},z)} = F^{(1)}_i(x_{i,t},z) +  F^{(2)}_i(x_{i,t},z).
\]
where to get the last line we used that $P^*_i$ is an isometric embedding. We bound each of terms $F^{(1)}_i(X_{i,t},z)$ and $F^{(2)}_i(X_{i,t},z)$ separately 
\[
\abs{F^{(1)}_i(X_{i,t},z)} = \abs{2c_t\innerproduct{(P^*_i)^T(X_{i,t}-c_tG_i)}{z} - c^2_t\norm{z}^2} 
\\
\le 2c_t \abs{\innerproduct{\pr_{\cH_i} (c_t X_0 -c_t G_i + \sigma_t Z_D)}{P^*_i z}} 
+ c_t^2\norm{z}^2
\\
\le 2c^2_t \norm{X_0-G_i}\norm{z} + 2c_t\sigma_t\norm{(P^*_i)^T Z_D}\norm{z} 
+ c_t^2\norm{z}^2
\]
Since $(P^*_i)^T Z_D \sim \cN\del{0,\Id_d}$ applying Laurent–Massart inequality~\cite{laurent2000} we have 
\[
\rP\del{\forall i: \norm{(P^*_i)^T Z_D} \le \sqrt{d+ 2\log \delta^{-1} + 2\log 2N}} 
\\
\ge 1- \sum_{i=1}^N \rP\del{\norm{(P^*_i)^T Z_D} \ge \sqrt{d+ 2\log 2\delta^{-1}N}} \ge 1-\delta.
\]
So substituting $\delta = n^{-2}$, combined with \eqref{eq:bound_on_distance_to_G_i_t_large}, and using that $\norm{z} \le \varepsilon_N$ we get that with probability at least $1-n^{-2}$
\[
\abs{F^{(1)}_i(X_{i,t},z)} &\le 
4 \sigma_t c_t\cdot \tilde C(n)\varepsilon_N 
+ 2c_t\sigma_t\sqrt{d+ 2\log \delta^{-1} + 2\log 2N}\norm{z} 
+ c_t^2\norm{z}^2 
\\
&\le 8 \sigma_t c_t\cdot \tilde C(n)\varepsilon_N, % 132 C(n) (\sigma_t c_t)\varepsilon_N + c^2_t\varepsilon^2_N,
\]
and substituting $2\varepsilon_N < (\sigma_t/c_t)(\log n)^{-\gamma}$,  $\tilde C(n) =\sqrt{\log n \log \log n} = o(\log n)$ for $n$ large enough with probability at least $1-n^{-2}$
\[
\abs{F^{(1)}_i(X_{i,t},z)} %\le 132 C(n) (\sigma_t c_t)(\sigma_t/c_t)(\log n)^{-\gamma} + c^2_t(\sigma_t/c_t)^2(\log n)^{-2\gamma}
\le \sigma^2_t (\log n)^{-\gamma + 1}.
\]
Similarly, we bound $F^{(2)}_i(X_{i,t},z)$
\[
\abs{F^{(2)}_i(X_{i,t},z)} = \abs{2c_t\innerproduct{X_{i,t}-c_tG_i}{R_i(z)}  - c^2_t\norm{R_i(z)}^2} 
\\
\le 2c^2_t\norm{\pr_{\cH_i}(X_0-G_i)}\norm{R_i(z)} +\sigma_t c_t\norm{\pr_{\cH_i} Z_D}\norm{R_i(z)}  + c^2_t\norm{R_i(z)}^2 
\]
By construction $\norm{R_i(z)}\lesssim \varepsilon^2_N$, so combining with \eqref{prop:bound_on_projection_on_cH_i} and \eqref{eq:bound_on_distance_to_G_i_t_large} with probability at least $1-n^{-2}$ for all $i$ s.t. $\rho_i(X_t) > 0$ 
\[
\abs{F^{(2)}_i(X_{i,t},z)}
\lesssim  c_t \sigma_t \tilde C(n)\varepsilon_N^2 +\sigma_t c_t \varepsilon^2_N\sqrt{\log n} + c^2_t\varepsilon_N^4 %\sqrt{(C_{\dim}+2)\log N + 2\log 2n^2}
\]
Using that $\varepsilon_N < \min\curly{(\log n)^{1/d}\,n^{-\frac{d-2}{d}\frac{1}{2\alpha+d}},(\sigma_t/c_t)(\log n)^{-\gamma}}$, $(d-2)/d \geq 1/3$ and $ \tilde C(n)= o( \log n)$ we get for $n$ large enough
\[
&\abs{F^{(2)}_i(X_{i,t},z)} = o(\sigma^2_t n^{-\frac{1}{6}\frac{1}{2\alpha+d}})
\]
Summarizing
\[
\frac{\abs{F^{(1)}_i(X_{i,t},z)}}{2\sigma^2_t} \le  (\log n)^{-\gamma +1}, \text{ and } \quad \frac{\abs{F^{(2)}_i(X_{i,t},z)}}{2\sigma^2_t} \le n^{-\frac{1}{6}\frac{1}{2\alpha+d}}
\]
Therefore the Taylor approximation of  $ e^{u}$  near 0 leads to
\[
    \exp\del{(F^{(1)}_i(X_{i,t}) + F^{(2)}_i(X_{i,t}))/2\sigma_t^2} = \\
    \del{\sum_{k=0}^{\left\lceil \frac{5 \log n}{2(\gamma - 1) \log(\log n)} \right\rceil}  \frac{\del{F^{(1)}_i(X_{i,t})}^k}{k! 2^k\sigma^{2k}_t}}\del{\sum_{k=0}^{16\cdot 6 \cdot (2\alpha+d)}  
    \frac{\del{F^{(2)}_i(X_{i,t})}^k}{k! 2^k\sigma^{2k}_t}} + O(n^{-4})
\]
and since $\diam M^* \le 2$
\begin{align}
\label{eq:final_repr_f_g_t_large}
    \hat{g}_i(t,X_{i,t}) 
    &:= 
    \int_{B_d(0,\varepsilon_N)} \del{\sum_{k=0}^{\left\lceil \frac{5 \log n}{2(\gamma - 1) \log(\log n)} \right\rceil}  \frac{\del{F^{(1)}_i(X_{i,t})}^k}{k! 2^k\sigma^{2k}_t}}\del{\sum_{k=0}^{16\cdot 6 \cdot (2\alpha+d)}  
    \frac{\del{F^{(2)}_i(X_{i,t})}^k}{k! 2^k\sigma^{2k}_t}}
    p_i(z)dz \nonumber
    \\
    &= g_i\del{t,X_{i,t}} + O(n^{-4}),
    \\
    \hat{f}_i(t,X_{i,t}) &:= 
    \int_{B_d(0,\varepsilon_N)} \Phi^*_i(z)\del{\sum_{k=0}^{\left\lceil \frac{5 \log n}{2(\gamma - 1) \log(\log n)} \right\rceil}  \frac{\del{F^{(1)}_i(X_{i,t})}^k}{k! 2^k\sigma^{2k}_t}}\del{\sum_{k=0}^{16\cdot 6 \cdot (2\alpha+d)}  
    \frac{\del{F^{(2)}_i(X_{i,t})}^k}{k! 2^k\sigma^{2k}_t}}
    p_i(z)dz
    \\
    &= f_i\del{t,X_{i,t}} + O(n^{-4}). \nonumber
\end{align}

$\Phi^*_i$ is  polynomial
$\Phi^*_i(z) = G_i + P^*_iz + \sum a^*_{i,S}z^{S}$, 
where $z\in \R^d$, $\Im P^*_i \subset \cH_i, a^*_{i,S}\in \cH_i$, therefore introducing a linear isometry $P_{\cH_i}:\R^{\dim \cH_i}\rightarrow \cH_i$ that identifies $\cH_i$ with $\R^{\dim \cH_i}$ we represent 
\[
\Phi^*_i(z) = G_i + P_{\cH_i}\widetilde{\Phi}^*_i(z),
\]
where $\widetilde{\Phi}^*_i(z):\R^{d} \mapsto \R^{\dim \cH_i}$ is a polynomial functions
\[
\widetilde{\Phi}^*_i(z) = P^T_{\cH_i}P^*_iz + {\sum P^T_{\cH_i}a^*_{i,S}z^{S}}
\]
of degree $\beta-1$ in $d$ variables with coefficients belonging to $\R^{\dim \cH_i}$ and $\widetilde{R}_i(z) := \sum P^T_{\cH_i}a^*_{i,S}z^{S}$ is corresponding second order residual.

So denoting $\widetilde X^{(2)}_i(t) = P^T_{\cH_i}\del{X_{i,t}-c_tG_i}$ we get 
\[
F^{(2)}_i(X_{i,t}, z) = 2c_t\innerproduct{\widetilde{X}^{(2)}_i(t)}{\widetilde{R}_i(z)}  - \norm{\widetilde{R}_i(z)}^2 = \norm{\widetilde{X}^{(2)}_i(t)}^2-\norm{\widetilde{X}^{(2)}_i(t) - c_t \widetilde{R}_i(z)}^2 := \widetilde{F}^{(2)}_i(\widetilde{X}^{(2)}_i(t), z)
\]
Similarly, introducing $\widetilde X^{(1)}_i(t) = (P^*_i)^T(X_{i,t}-c_tG_i) \in \R^d$
\[
F^{(1)}_i(X_{i,t},z) = {2c_t\innerproduct{\widetilde{X}^{(1)}_i(t)}{z} - c^2_t\norm{z}^2} = \norm{\widetilde{X}^{(1)}_i(t)} - \norm{\widetilde{X}^{(1)}_i(t)-c_tz} := \widetilde{F}^{(1)}_i(\widetilde{X}^{(2)}_i(t),z).
\]

Expanding brackets of~\eqref{eq:final_repr_f_g_t_large} we represent $\hat{f}$ and $\hat{g}$ as polynomials in $c_t, \sigma^{-2}_t, \widetilde{X}^{(1)}_i, \widetilde{X}^{(2)}_i$ with coefficients being integrals depending on unknown coefficients of $\widetilde{\Phi}^*_i$ and density $p_i$ that itself depends on samples $\cG\subset \cY$ used to construct the measure, and $\mu$ -- the density. More presisely
\[
    \tilde{g}_i(t,X_{i,t}) &= \sum_{\abs{S^{(1)}} <s_n }\sum_{\abs{S^{(2)}} <16\cdot 24 \cdot (2\alpha+d)}\del{\widetilde X^{(1)}_i(t), \sigma^{-1}_t, c_t}^{S^{(1)}}\del{\widetilde X^{(2)}_i(t), \sigma^{-1}_t, c_t}^{S^{(2)}} C_{S^{(1)},S^{(2)},g_i}\del{\cY,\mu}  
    \\
    \tilde{f}_i(t,X_{i,t}) &= 
    \sum_{\abs{S^{(1)}} <s_n}\sum_{\abs{S^{(2)}} <16\cdot 24 \cdot (2\alpha+d)}\del{\widetilde X^{(1)}_i(t), \sigma^{-1}_t, c_t}^{S^{(1)}}\del{\widetilde X^{(2)}_i(t), \sigma^{-1}_t, c_t}^{S^{(2)}} C_{S^{(1)},S^{(2)},f_i}\del{\cY,\mu}
\]
where $s_n = \left\lceil \frac{5 \log n}{2(\gamma - 1) \log(\log n)} \right\rceil$,  $C_{S^{(1)},S^{(2)},g_i}\del{\cY,\mu}$ are scalars and $C_{S^{(1)},S^{(2)},f_i}\del{\cY,\mu}\in \R^{\dim \cH_i}$, while $S^{(1)}\in \mathbb N^{d+2}$ and $S^{(2)}\in \mathbb N^{\dim \cH_i + 2}$

We treat the coefficients $C_{S^{(1)},S^{(2)},g_i}\del{\cY,\mu}, C_{S^{(1)},S^{(2)},f_i}\del{\cY,\mu}$ as unknown weights of the neural network that require optimization. Since $\dim \cH_i \le C_{\dim}\log n$ and $s_n\le \log n$, the number of monomials is bounded by 
\[
\binom{d+\log n}{\log n}\binom{C_{\dim}\log n + 16\cdot 24\cdot (2\alpha+d)}{\log n} \le \del{(2C_{\dim}+1)\log n + d}^{16\cdot 24 \cdot (2\alpha+d) + d} = \polylog n.
\]
Therefore, the number of weights required to parameterize all the coefficients $C_{S,g_i}\del{\cY,\mu}, C_{S,f_i}\del{\cY,\mu}$ is bounded by $\polylog n$.

Summarizing, $\tilde{g}_i$ and $\tilde{f}_i$ are both represented as polynomials in $O(\log n)$ variables, with $\polylog n$ non-zero coefficients, and coefficients belonging to $\polylog n$ subspace. Therefore, using Lemmas~\ref{lemma:parallelization_of_neural_networks}-- \ref{lemma:oko_nn_appproximating_reciprocal_function} both $\tilde{g}_i(t,x_{i,t}), \tilde{f}_i(t,x)$ can be approximated by a neural networks $\phi_{g_i}, \phi_{f_i}\in \Psi(L,W,S,B)$, where $L, W,S = \polylog n$ and $B = e^{\polylog n}$, with an error less than $n^{-4}$ (actually less than $n^{- C} $ for any $C>0$).
Finally, applying~\Cref{lemma:oko_nn_appproximating_reciprocal_function} we approximate $\tilde{e}_i := \tilde{f}_i(t,x)/\tilde{g}_i(t,x_{i,t})$ with a neural network $\phi_{e_i} \in \Psi(L,W,S,B)$ of the same type,  with an error less than $2n^{-2}$.

Applying argument for all $i\le N$ and combining~\eqref{eq:error_of_s_star_clip},~\eqref{eq:error_of_s_star_tr}~and~\eqref{eq:D1_W2_vs_SML} we obtain that for $T_k\ge n^{-\frac{2}{2\alpha+d}}$ there is a network $\phi\in\cS_k$ s.t. 
    \[
    \int_{T_k}^{T_{k+1}}\sigma^2_t\expectation \norm{\phi(t,X_t) - s(t,X_t)}^2 dt\lesssim \sigma^{-2}_{T_k} n^{-\frac{2(\alpha+1)}{2\alpha+d}} (\log n)^{\frac{2\beta}{d}}.
    \]

\subsection{Case $\underline T \le T_{k} \le n^{-\frac{2}{2\alpha+d}}$}
\label{sec:main_theorem_t_small}
\subsubsection{Summary} \label{sec:summary:D2}
Recall that 
$\underline T\ge {(\log n)^{\beta\gamma + \frac\beta d }}n^{-\frac{2(\alpha+1)}{2\alpha+d}} $;
we choose $\gamma >3 - 1/d$ and 
$N = (\log n)^{-\gamma d} n^{\frac{d}{2\alpha+d}}$.
Then 
$N^{-\frac{\beta}{d}} = n^{-\frac{\beta}{2\alpha+d}} (\log n)^{\beta\gamma} \leq r_0$, $ \varepsilon_N \lesssim n^{-\frac1{2\alpha+d}} (\log n)^{(\gamma+ 1/d)}$, and  
$$
\forall t \in [T_k, T_{k+1}],\quad  \sigma_t/c_t \le 2\sqrt{t} \le \varepsilon_N  (\log n)^{-(\gamma+ 1/d)} \le \varepsilon_N  (\log n)^{-3},\quad \varepsilon_N^\beta \lesssim \underline{T}.
$$
Similarly to the case $T_k \ge n^{-\frac{2}{2\alpha+d}}$, as a first step, we localize the score function and construct its truncated version $s_{tr}$.
As before, we introduce $ \rho_i(t, X_t)$ defined in \eqref{rhoi}, with this time $C(T_k ,n)= 6\varepsilon^2_N$.
In this section, we define \[
C(n) = \sqrt{d\del{\log n + 4C_{\log}} + 4\log n} \asymp \sqrt{ \log n}
\]
We show that with probability at least $1-n^{-2}$ if $\rho_i(t,X_t) > 0$ then $i\in I_2$ where
\£ \label{I2:D2}
I_2 = \curly{i\le N: \norm{X_0-G_i} \le 4 \varepsilon_N}. %3\varepsilon_N
\£

In contrast to the case $T_k \ge n^{-\frac{2}{2\alpha+d}}$, here we have $(\sigma_t/c_t) =o(\varepsilon_N)$, while Theorem~\ref*{thm:concentration_around_X0} implies that the support of conditional measure $\mu(y|t,X_t)$ of scale $O(\sigma_t/c_t)$.  concentrated on $\Phi_i(G_i, 5\varepsilon_N)$. On the other hand, this means that the $O(\varepsilon_N)$ scale is too coarse, and we need to better localize the support.    

We solve this problem, by considering the projection of $X_t$ onto $T_{G_i}\Phi^*_i(0, \varepsilon_N)$, namely
\£ \label{zistar:D2}
z^*_i = z^*_i(X_t) = (P^*_i)^T \del{X_t-c_t G_i}/c_t,
\£
and by defining the submanifolds $M_i$ and their counterparts $M^*_i$ as
\begin{align}   \label{Mi:Mistar} 
&M_i := \Phi_i\del{B_d\del{z^*_i,4(\sigma_t/c_t)\sqrt{20}C(n)}}, & M^*_i := \Phi^*_i\del{B_d\del{z^*_i,4(\sigma_t/c_t)\sqrt{20}C(n)}}
\end{align}
that centered around $m_i := \Phi_i(z^*_i)$ and $m^*_i := \Phi^*_i(z^*_i)$.

As we show in~\Cref{apdx:score_function_localization_t_small} with probability at least $1-O(n^{-2})$
for all $i\in I_2$ the local score $s_i(t,X_t)$ verifies
\[
\norm{s_i(t,X_t) - s(t,X_t)} \le 2\frac{c_t}{\sigma^2_t}n^{-2}, \quad \text{where }s_i(t,X_t) := \frac{1}{\sigma^2_t}\frac{\int_{M_i}(c_ty-X_t) p(y|t,X_t)}{p\del{M_i|t,X_t}}.
\]
So, if we define 
\£ \label{str:D2}
    s_{tr}(t,X_t) = \frac{\sum_i \rho_i(t,X_t) s_i(t,X_t)}{\sum_i \rho_i(t,X_t)},
\£
then we can prove that 
\£ \label{str-s:D2}
\int_{T_k}^{T_{k+1}}\expectation \norm{s_{tr}(t,X_t) - s(t,X_t)}^2dt \lesssim n^{-2} \sigma_{T_k}^{-2}.     
\£

The next step is the replacement of $s_{tr}$ with its polynomial counterpart $s^*_{tr}$, which we define by constructing $s^*_i$. Denoting by $p_i(z) = \abs{\grad\Phi_i\grad^T\Phi_i}^{-1}p\del{\Phi_i(z)}$ the density of pushforward of $\mu$ restricted on $M_i$ on the tangent space $T_{G_i}M_i$ we represent
\£ \label{si:D2}
    s_i(t,X_t)
    =
    \frac{c_t}{\sigma^2_t}\frac{\int_{B_d(z^*_i, 4(\sigma_t/c_t)\sqrt{20}C(n))}\Phi_i(z) e^{-\norm{c_t\Phi_i(z)-X_t}^2/2\sigma^2_t}p_i(z)dz}{\int_{B_d(z^*_i, 4(\sigma_t/c_t)\sqrt{20}C(n))} e^{-\norm{c_t\Phi_i(z)-X_t}^2/2\sigma^2_t}p_i(z)dz} - \frac{X_t}{\sigma^2_t},
\£
and approximate it by substituting $\Phi_i$ and $p_i$ with their polynomial approximations $\Phi^*_i, p^*_i$ giving us
\£
\label{eq:s^*_i_definition}
s^*_i(t,X_t) := 
    \frac{c_t}{\sigma^2_t}\frac{\int_{B_d(z^*_i, 4(\sigma_t/c_t)\sqrt{20}C(n))}\Phi^*_i(z) e^{-\norm{c_t\Phi^*_i(z)-X_t}^2/2\sigma^2_t}p^*_i(z)dz}{\int_{B_d(z^*_i, 4(\sigma_t/c_t)\sqrt{20}C(n))} e^{-\norm{c_t\Phi^*_i(z)-X_t}^2/2\sigma^2_t}p^*_i(z)dz} - \frac{X_t}{\sigma^2_t}.
\£
Defining $s^*_{tr}$ as a mixture of $s^*_i$ and correspondingly $e^*_{tr}$, 
\[
    s^*_{tr}(t,X_t) := \frac{\sum_i \rho_i(t,X_t) s^*_i(t,X_t)}{\sum_i \rho_i(t,X_t)}, e^*_{tr}(t,X_t) := \frac{\sum_i \rho_i(t,X_t) e^*_i(t,X_t)}{\sum_i \rho_i(t,X_t)},
\]
with $\frac{c_t}{\sigma^2_t}e^*_i(x) = s^*_i(x) + \frac{x}{\sigma^2_t}$,  
the direct calculations performed in \Cref{apdx:polynomial_approximation_t_small} show
\£
\label{eq:sml_s_star_t_small}
\int^{T_{k+1}}_{T_k}\expectation \norm{s^*_{tr}(t,X_t) - s_{tr}(t,X_t)}^2dt =
O\del{ \sigma_{T_k}^{-2}\varepsilon_N^{2\beta} (\log n)^2 +  (\log n)^{2\alpha+1} n^{-\frac{2\alpha}{2\alpha+d}}  } %n^{2\gamma\alpha}(\log n)^{2\alpha/d+1}
\\
=
O\del{T_k^{-1}(\log n)^{2\gamma\beta+2}n^{-\frac{2\beta}{2\alpha+d}} +   (\log n)^{2\alpha+1} n^{-\frac{2\alpha}{2\alpha+d}}}
.
\£
Next, we represent $e^*_i = f_i/g_i$, where $f_i,g_i$ are regularized such that they can be efficiently approximated using polynomials.

As in the case $T_k\ge n^{-\frac{2}{2\alpha+d}}$ the main challenge is the term $\norm{X_t-c_t\Phi^*_i(z)}$ in the exponent, and which  scales as $O(D)$.  To tackle this first note that by definition $M^*_i$ is centered around $m^*_i = \Phi^*_i(z^*_i)$. We approximate $M^*_i$ by the tangent space $T_{m^*_i}M^*_i$ and linearize $\Phi^*_i$ as
\[
        \Phi^*_i(z) = \Phi^*_i(z^*_i) + \grad \Phi^*_i(z^*_i)(z-z^*_i) +R_i(z,z^*_i),
\]
 and we finally define the projection of $x$ onto $T_{m^*_i}M^*_i$
 \£
    \label{eq:x^pr_i}
    x^{\pr}_i :=  \pr_{z^*_i}\del{x- c_t m_i^*}.
\£
Note that $x^{\pr}_i$ is not the same as $x_{i,t} = \pr_{c_t \cA_i} x$ .

Using that $(x - c_tm^*_i) - x^{\pr}_i \perp T^*_{m^*_i} M^*_i$ and $x-c_t G_i-x_{i,t} \perp \cH_i \supseteq T^*_{m^*_i} M^*_i$ we extract from $\norm{X_t-c_t m^*_i - c_t\Phi^*_i(z)}^2$ a term $\norm{X_t-X_{t,i}^{\pr}}^2$ depending only $x$ and not on $z$. This term carries the ambient $O(D)$ contribution and cancels between numerator and denominator in the posterior ratio. More precisely, we write
\£
\label{eq:definition of g, f small time}
        g_i(t,x) :=
    &\int_{B_d(z^*_i, 4(\sigma_t/c_t)\sqrt{20}C(n))} 
        \exp\del{-\frac{{\norm{x^{\pr}_i - c_t\grad\Phi^*_i(z^*_i)(z-z^*_i)}^2}}{2\sigma^2_t}} \nonumber
        \\
        &\qquad\times \exp\del{\frac{c^2_t\norm{R_i(z,z^*_i)}^2
        + 2c_t\innerproduct{x_{i,t}- \Phi^*_i(z^*_i)}{R_i(z,z^*_i)}}{2\sigma^2_t}}
        p^*_i(z)dz \nonumber
        ,
        \\
        f_i(t,x) := 
    &\int_{B_d(z^*_i, 4(\sigma_t/c_t)\sqrt{20}C(n))} \Phi^*_i(z)
        \exp\del{-\frac{\norm{x^{\pr}_i - c_t\grad\Phi^*_i(z^*_i)(z-z^*_i)}^2}{2\sigma^2_t}} \nonumber
        \\ 
        &\qquad \times \exp\del{\frac{c^2_t\norm{R_i(z,z^*_i)}^2
        + 2c_t\innerproduct{x_{i,t}- \Phi^*_i(z^*_i)}{R_i(z,z^*_i)}}{2\sigma^2_t}}
        p^*_i(z)dz.   % e^{-{\norm{x^{\pr}_i- c_t\Phi^*_i(z^*_i) - c_t\grad\Phi^*_i(z^*_i)(z-z^*_i)}^2}/2\sigma^2_t}   
\£
In \Cref{apdx:polynomial_approximation_of_e_star_i_t_small} we verify that $e^*_i = f_i/g_i$ and show that with probability at least $1-n^{-2}$ for all $i\in I_2$
\£
\label{eq:regularity of exponents in g, f small time}
&\norm{X^{\pr}_i(t) - c_t\grad\Phi^*_i(z^*_i)(z-z^*_i)}^2 \le {676}\sigma^2_t C^2(n) \nonumber
\\
&
\abs{c^2_t\norm{R_i(z,z^*_i)}^2 -2 c_t\innerproduct{X_t- c_t\Phi^*_i(z^*_i)}{R_i(z,z^*_i)}} = O(\sigma_t^3 (d\log n)^2)
\£
implying $
g_i(t,X_t) \ge n^{-2688d}. $ So if we construct $\hat{f}_i, \hat{g}_i$ satisfying with probability at least $1-n^{-2}$ for all $i\in I_2$
$
\abs{g_i(t,X_t) - \hat{g}_i(t,X_t)}, \norm{f_i(t,X_t) - \hat{f}_i(t,X_t)}\le n^{-2370d}, $
 then since $\norm{e_i} \le 2$ by construction 
\[ %\frac{c_t}{\sigma^2_t}
\hat{e}_i(t,x) &= \frac{\hat{f}_i(t,x)}{\hat{g}_i(t,x)}
, \quad
\hat{e}_{tr}(t,x) = \frac{\sum_{i=1}^N \rho_i(t,x)\ind_{\norm{\hat{e}_i(t,x)}\le 2\sigma^{-2}_t}\hat{e}_i(t,x)}{\sum_{i=1}^N \rho_i(t,x)}, \quad \hat{s}(t,x) = \frac{c_t}{\sigma^2_t}\hat{e}_{tr}(t,x) - \frac{x}{\sigma^2_t},
\]
satisfy
\[
\int_{T_k}^{T_{k+1}}\expectation \norm{\hat{s}(t,X_t) - s^*_{tr}(t,X_t)}^2 dt \le n^{-2}.
\]
Finally, in \Cref{apdx:neural_netwok_approximation_t_small} using~\eqref{eq:regularity of exponents in g, f small time} we perform the Taylor expansion of the exponents in~\eqref{eq:definition of g, f small time} and show that both $\hat{f}_i$ and $\hat{g}_i$ satisfying~\eqref{eq:f_hat_definition_t_small} can be obtained as a composition of polynomials of $\polylog$ size, which can be approximated by neural networks also of $\polylog$ size.

  \subsubsection{Score Function Localization}
    \label{apdx:score_function_localization_t_small}
    In this section we show that $s(t, X_t)$ is close to $ s_{tr}(t,X_t)$ which is defined in \eqref{str:D2}.

Applying it with $A=n^{-2}$ we conclude that with probability at least $1-n^{-2}$ if $\rho_i(t,X_t) > 0$ then $i\in I_2$ and 
\[
B_M\del{X_0,2(\sigma_t/c_t)\sqrt{20}C(n,A)} \subset M_i \subset \Phi_i\del{B_d\del{0,7\varepsilon_N}}.
\]
Applying Theorem~\ref*{thm:concentration_around_X0} with the same fixed tail
probability gives, for all $i\in I_2$,
\[
    \int_{M_i}p(y|t,X_t)dy
    \ge
    \int_{B_M\del{X_0,2(\sigma_t/c_t)\sqrt{20}C(n)}}p(y|t,X_t)dy
    \ge 1-n^{-2}.
\]
     We obtain that with probability greater than $1-n^{-2}$, for all $i\in I_2$
    $p(M_i| t , X_t) \geq 1 - n^{-2}$ and using the definition \eqref{si:D2} of $s_i(t,X_t)$ together with $e_i(t,X_t) = s_i(t,X_t) + X_t/\sigma^2_t$ and $\diam M\le 1$.
    we have that   with probability at least $1-n^{-2}$, for all $i\in I_2$
     \[
    \norm{e_i(t,X_t)p\del{M_i|t,X_t} - e(t,X_t)} = \norm{\int_{M\backslash M_i} y\, p(y|t,X_t)} \le n^{-2}, %\frac{c_t}{\sigma^2_t}
    \]
    Also, since $\norm{s_i(t,X_t) - s(t,X_t)} = \frac{c_t}{\sigma^2_t}\norm{e_i(t,X_t) - e(t,X_t)}$, we bound
    \begin{align*}
    & \quad \norm{s_i(t,X_t) - s(t,X_t)} \\
    & \le \frac{c_t}{\sigma^2_t}\norm{e_i(t,X_t)\del{1-p\del{M_i|t,X_t}}} \nonumber  + \frac{c_t}{\sigma^2_t}\norm{e_i(t,X_t)p\del{M_i|t,X_t} - e(t,X_t)} 
    \le  2\frac{c_t}{\sigma^2_t}n^{-2}   
    \end{align*}
    with probability at least $1-n^{-2}$, for all $i\in I_2$.

    Finally, recall that 
    $ s_{tr}(t,X_t) = \sum_i \rho_i(t,X_t) s_i(t,X_t)/(\sum_i \rho_i(t,X_t))$, 
     is a  mixture of $s_i$ with weights  $\rho_i(t,X_t)\neq 0$ only if $i\in I_2$. Therefore  with probability at least $1-n^{-2}$
    \[
    \norm{s_{tr}(t,X_t) - s(t,X_t)} \le 2\frac{c_t}{\sigma^2_t}n^{-2}.
    \]
    Almost surely $\norm{e_i(t,X_t) - e(t,X_t)} \le \norm{e_i(t,X_t)} + \norm{e(t,X_t)}\le 2$, so taking the expectation
    \[
    \expectation \norm{s_{tr}(t,X_t) - s(t,X_t)}^2  &\le 8\frac{c^2_t}{\sigma^4_t}n^{-4}   \quad  \text{and} \\ 
     \int_{T_k}^{T_{k+1}}\expectation \norm{s_{tr}(t,X_t) - s(t,X_t)}^2dt &\le 8 \frac{n^{-2}}{ \sigma_{T_k}^2} \int_{T_k}^{T_{k+1}} \frac{1}{t} dt \le 8 \frac{n^{-2}}{ \sigma_{T_k}^2} \log 2, \]
     which proves \eqref{str-s:D2}.
\subsubsection{Polynomial Approximation of $f_i$ and $g_i$}
\label{apdx:polynomial_approximation_t_small}
    The main goal of this section is to substitute $s_{tr}(t,x)$ with the score corresponding to the piecewise  polynomial approximations of the manifold $M$ by $M^* = \cup_i M_i^*$ with $M_i^* = \Phi_i^* (\Phi_i^{-1}( M_i))= \Phi_i^*\del{B_d\del{z^*_i, 4(\sigma_t/c_t)\sqrt{20}C(n)}}$ and for each $i$, $p_i=\abs{\grad\Phi_i\grad^T\Phi_i}^{-1}p\circ \Phi_i$ by a polynomial function $p_i^*$.  
   
    Recall that from \eqref{eq:balls_inclusion_t_small}, we have with probability at least $1-n^{-2}$  
    $
   M_i \subset B_d(0,7\varepsilon_N)$  so that we can express
    \[
    e_i(t,X_t) = \frac{\int_{M_i}y\, p(y|t,X_t)}{p\del{M_i|t,X_t}}
    \\ %\frac{1}{\sigma^2_t}
    =
    \frac{\int_{B_d(z^*_i, 4(\sigma_t/c_t)\sqrt{20}C(n))}(\Phi_i(z)-\Phi_i(z^*_i)) e^{-\norm{\Phi_i(z)-X_t}^2/2\sigma^2_t}p_i(z)dz}{\int_{B_d(z^*_i, 4(\sigma_t/c_t)\sqrt{20}C(n))} e^{-\norm{\Phi_i(z)-X_t}^2/2\sigma^2_t}p_i(z)dz} + \Phi_i(z^*_i) %\frac{}{\sigma^2_t}.
    \]

   A natural way to approximate $e_i(t,x)$ is to substitute $\Phi_i$ by $\Phi^*_i$, and $p_i(z)$ by its Taylor approximation of order $\alpha-1$ at $0$ which we denote $p^*_i(z)$. This leads to 
    \[
    e^*_i(t,X_t) := 
    \frac{\int_{B_d(z^*_i, 4(\sigma_t/c_t)\sqrt{20}C(n))}\del{\Phi^*_i(z)-\Phi_i^*(z^*_i)} e^{-\norm{\Phi^*_i(z)-X_t}^2/2\sigma^2_t}p^*_i(z)dz}{\int_{B_d(z^*_i, 4(\sigma_t/c_t)\sqrt{20}C(n))} e^{-\norm{\Phi^*_i(z)-X_t}^2/2\sigma^2_t}p^*_i(z)dz} + \Phi^*(z^*_i) %\frac{}{\sigma^2_t}.
    \]
   We now bound $e_i^* - e_i$. First, 
   since $p_i(z)$ is $\alpha$-smooth then 
   $$\abs{p_i(z)-p^*_i(z)} \lesssim  (\sigma_t C(n))^\alpha \lesssim \varepsilon^\alpha_N$$ and by~\Cref{prop:geometric_statements} $p_i(z) = \abs{\grad\Phi_i\grad^T\Phi_i}^{-1}p\del{\Phi_i(z)} \ge 2^{-d}p_{\min}$, so 
    \[
    p_i(z) = p^*_i(z) \del{1+O\del{(\sigma_t C(n))^\alpha }} =  p^*_i(z)\del{1+O\del{(\log n)^{\alpha/2}n^{-\frac{\alpha}{2\alpha+d}}}}.
    \]
    
       By Lemma~\ref*{lemma:approximation_of_Phi} for all $z\in B_d(0,7\varepsilon_N)$
       \[
       \norm{\Phi^*_i(z)-\Phi_i(z)} \lesssim \varepsilon^\beta_N \le (\log n)^{2\gamma \beta}n^{-\frac{\beta}{2\alpha+d}}.
       \]
     
    Let $z\in B_d(z^*_i,4(\sigma_t/c_t)\sqrt{20}C(n))$, denoting ${y:= \Phi_i(z)}$ and $y^* := \Phi^*_i(z)$ we first bound
    \[
    \norm{X_t - c_t y^*}^2 - \norm{X_t- c_ty}^2 =  c^2_t\norm{y-y^*}^2 + 2c^2_t\innerproduct{X_0-y}{y-y^*} + 2c_t\sigma_t\innerproduct{Z_D}{y-y^*}.   
    \]
    By Lemma~\ref*{lemma:approximation_of_Phi} the first term is bounded as
    \[
    c^2_t\norm{y-y^*}^2 \le c^2_t\norm{\Phi_i^*(z)-\Phi_i(z)}^2 \lesssim c^2_t\cdot \varepsilon^{2\beta}_N  = o( \varepsilon^{\beta}_N\sigma_t c_t)
    \]
    By Proposition~\ref*{prop:TaylorGP} for $n$ large enough with probability $1-n^{-2}$ the third term is bounded by 
    \[
    2c_t\sigma_t\cdot\abs{\innerproduct{Z_D}{y-y^*}} \lesssim c_t\sigma_t \cdot \varepsilon_N^\beta \cdot C(n) \le c_t\sigma_t \cdot (\log n)^{(\gamma+1/d)\beta +1/2}n^{-\frac{\beta}{2\alpha+d}}/2.
    \]
    At the same time by~\Cref{prop:points_z_as_projections_on_tangent_space}, 
    $
    X_0 \in B_M(m_i, 2(\sigma_t/c_t)\sqrt{20}C(n))$, and 
    since $\Phi_i$ is $2$-Lipschitz
    \£
    \label{eq:dist_X_0_y_in_cube_t_small}
    \norm{X_0-y} \le \norm{X_0-m_i}+\norm{m_i-y} \le (8 +2)(\sigma_t/c_t)\sqrt{20}C(n),\£
    the second term is bounded by, for $n$ large enough,
    \[
        c^2_t\abs{\innerproduct{X_0-y}{y-y^*}} \lesssim c_t\sigma_t C(n) \varepsilon^\beta_N =o( c_t\sigma_t \cdot  (\log n)^{(\gamma+1/d)\beta +1/2}n^{-\frac{\beta}{2\alpha+d}}),
    \]
    where the constants depend only on $d, C_{log}, \beta$. 
    Summing up we conclude that with probability at least $1-n^{-2}$ for all $i\in I_2$ and $y\in B_d(z^*_i,4(\sigma_t/c_t)\sqrt{20}C(n))$
    \[
    \frac{\abs{\norm{X_t - c_t y^*}^2 - \norm{X_t- c_ty}^2}}{2\sigma^2_t} \lesssim \frac{\varepsilon_N^\beta c_t C(n)}{\sigma_t} = o(1), % n^{2\gamma\beta}n^{-\frac{\beta}{2\alpha+d}}(\sigma_t/c_t)^{-1} \le n^{-2\gamma\beta} < 1/2,
    \]
    where the latter is due to $(\sigma_t/c_t) \ge (\log n)^{\gamma/2}n^{-\frac{\alpha+1}{2\alpha+d}} \ge n^{-\frac{d\beta}{(d-2)(2\alpha+d)}}(\log n)^{\gamma/2} >> \sqrt{\log n} \varepsilon_N^\beta$, for all $\beta >0$. 
    As a result
    \[
    e^{-\frac{\norm{X_t - c_t y^*}^2}{2\sigma_t^2}} = e^{-\frac{\norm{X_t - c_t y}^2}{2\sigma_t^2}}\del{1+O\del{\frac{\varepsilon_N^\beta c_t C(n)}{\sigma_t} }}.
    \]
    Combining the inequalities above, with probability $1-n^{-2}$ for all $i\in I_2$
    \[
    &e^*_i(t,X_t)
    =
    \frac{\int_{B_d(z^*_i, 4(\sigma_t/c_t)\sqrt{20}C(n))}\del{\Phi^*_i(z)-\Phi^*_i(z^*_i)} e^{-\norm{\Phi^*_i(z)-X_t}^2/2\sigma^2_t}p^*_i(z)dz}{\int_{B_d(z^*_i, 4(\sigma_t/c_t)\sqrt{20}C(n))} e^{-\norm{\Phi^*_i(z)-X_t}^2/2\sigma^2_t}p^*_i(z)dz} + \Phi_i(z^*_i) %\frac{}{\sigma^2_t}
    \\
    &=
    \frac{\int_{B_d(z^*_i, 4(\sigma_t/c_t)\sqrt{20}C(n))}\del{\Phi_i(z)-\Phi_i(z^*_i)} e^{-\norm{\Phi^*_i(z)-X_t}^2/2\sigma
    ^2_t}p^*_i(z)dz}
    {\int_{B_d(z^*_i, 4(\sigma_t/c_t)\sqrt{20}C(n))} e^{-\norm{\Phi^*_i(z)-X_t}^2/2\sigma^2_t}p^*_i(z)dz} + \Phi_i(z^*_i)+ O\del{\varepsilon_N^\beta}
    \\
    &=e_i(t,X_t) + O\del{\varepsilon_N^\beta }
    \\
    & + \sup_{z \in B_d(z^*_i, 4(\sigma_t/c_t)\sqrt{20}C(n))}\norm{\Phi_i(z)-\Phi_i(z^*_i)}
    O\del{(\sigma_t \sqrt{\log n})^\alpha + 
    \frac{\varepsilon_N^\beta C(n)}{ \sigma_t}}.
    \]
    Since $\Phi_i$ is $2$-Lipschitz $\norm{\Phi_i(z)-\Phi_i(z^*_i)} \le 2\norm{z-z^*_i}\le 8\sqrt{d}(\sigma_t/c_t)\sqrt{20}C(n)$, applying that $\sigma^2_t \ge t/2 \ge T_k/2, \sigma_t/c_t \le 4t < 4T_k$ and $C(n) \lesssim \log n$ we conclude that with probability at least $1-n^{-2}$
    \[
    e^*_i(t,X_t) = e_i(t,X_t) +  O\del{(\log n)^{\gamma\beta+2}n^{-\frac{\beta}{2\alpha+d}}} + O\del{T^{1/2}_k n^{-\frac{\alpha}{2\alpha+d}}(\log n)^{\alpha+1/2}} .
    \]
    Taking expectations, and noting that $\norm{e^*_i(t,X_t)} \le 4T^{-1}_k, \norm{e_i(t,X_t)} \le 2T^{-1}_k$ a.s., we obtain
    \[
    \expectation \norm{s^*_i(t,X_t) - s_i(t,X_t)}^2 
    &= \frac{c_t}{\sigma_t^2}\expectation \norm{e^*_i(t,X_t) - e_i(t,X_t)}^2 
    \\
    &= O\del{T^{-2}_k \varepsilon_N^{2\beta} \log n + T^{-1}_k n^{2\alpha/(2\alpha+d)} (\log n)^{2\alpha+1}}. 
    \]
    Integrating over $[T_k,T_{k+1}]$, since $T_{k+1} = 2T_k\leq 2n^{-2/(2\alpha+d)}$
    \[
    \int^{T_{k+1}}_{T_k}\expectation \norm{s^*_i(t,X_t) - s_i(t,X_t)}^2dt&=  \frac{ 1 }{ \sigma_{T_k}^{2}}O\del{   \varepsilon_N^{2\beta}\log n^2 + \sigma_{T_k}^{2}n^{2\alpha/(2\alpha+d)} (\log n)^{2\alpha+1} }
    \]
Finally,  using that $\rho_i(t,x) \neq 0$ implies $i\in I_2$ and by  definition of $s^*_{tr}$ we have
\[
\int^{T_{k+1}}_{T_k}\expectation \norm{s^*_{tr}(t,X_t) - s_{tr}(t,X_t)}^2dt
=O\del{
    \sigma_{T_k}^{-2}
    {\varepsilon_N^{2\beta}}
    (\log n)^2
    + n^{2\gamma\alpha}n^{-\frac{2\alpha}{2\alpha+d}}(\log n)^{2\alpha+1}
}.
\]
which proves \eqref{eq:sml_s_star_t_small}. 

\subsubsection{Polynomial Approximation of $e^*_i$}
\label{apdx:polynomial_approximation_of_e_star_i_t_small}
To construct a neural network approximation of $e^*_{tr}$ we first approximate $e_i^*$ by a low degree polynomial of at most $O(\polylog n)$ variables. 
A key point is to treat $e^{-\norm{x-c_t\Phi^*_i(z)}^2/2\sigma^2_t}$  in the definition of $e_i^* $ since $\norm{x-c_t\Phi^*_i(z)}^2 \asymp D$. As explained in~\Cref{sec:summary:D2}, we introduce the projection operator onto $T_{m^*_i}M^*_i$ for $m_i^*=\Phi^*_i(z^*_i)$, 
\£\label{eq:x^pr_i}
        \pr_{z^*_i} = \grad \Phi^*_i(z^*_i)\del{\grad^T\Phi^*_i(z^*_i)\grad\Phi^*_i(z^*_i)}^{-1}\grad^T \Phi^*_i(z^*_i),
        \quad x^{\pr}_i :=  \pr_{z^*_i}\del{x- c_t m_i^*} 
        \£
        We then write 
        \[
        \Phi^*_i(z) = \Phi^*_i(z^*_i) + \grad \Phi^*_i(z^*_i)(z-z^*_i) +R_i(z,z^*_i),
        \]
        where polynomial $R_i(z,z^*_i)$ is the second order polynomial residual and
        note that $\Im \grad \Phi^*_i(z^*_i) = T_{m^*_i}M^*_i$.  Since
        $x^{\pr}_i$ is the orthogonal projection of
        $x-c_t\Phi_i^*(z_i^*)$ onto this tangent space,
        \[
        \begin{aligned}
        \norm{x-c_t \Phi^*_i(z)}^2
        &=
        \norm{x-c_t\Phi_i^*(z_i^*)-x^{\pr}_i}^2
        +\norm{x^{\pr}_i-c_t\grad\Phi_i^*(z_i^*)(z-z_i^*)}^2 \\
        &\quad
        -c_t^2\norm{R_i(z,z_i^*)}^2
        -2c_t\innerproduct{x_{i,t}-c_t\Phi_i^*(z)}{R_i(z,z_i^*)}.
        \end{aligned}
        \]
        Indeed, the first two terms come from the orthogonal decomposition
        into the tangent space and its orthogonal complement.  For the
        residual term, we use $R_i(z,z_i^*)\in\cH_i$ and
        $x-x_{i,t}\perp\cH_i$, which gives
        \[
            \innerproduct{x-c_t\Phi_i^*(z_i^*)
            -c_t\grad\Phi_i^*(z_i^*)(z-z_i^*)}{R_i(z,z_i^*)}
            =
            \innerproduct{x_{i,t}-c_t\Phi_i^*(z)}{R_i(z,z_i^*)}
            +c_t\norm{R_i(z,z_i^*)}^2 .
        \]
        Therefore, since the first term in the displayed decomposition does
        not depend on $z$, we can write $e_i^* = f_i/g_i$ with $f_i$ and
        $g_i$ defined in \eqref{eq:definition of g, f small time}.
        
        We now bound from below $g_i$. 
        
        First we bound the residual exponential term in~\eqref{eq:definition of g, f small time}, using~\Cref{lemma:BoundsOnScalarProductSmallTime}
    in~\Cref{apdx:Correlation Between Tangent Vectors and Gaussian Noise}. 
    It states that with probability at least $1-n^{-2}$, 
 $\forall i\in I_2$ and $ \forall z\in B_d\del{z^*_i, 4(\sigma_t/c_t)\sqrt{20}C(n)}$ 
\[
           \left|c^2_t\norm{R_i(z,z^*_i)}^2 +2 c_t\innerproduct{X_{i,t}- c_t\Phi^*_i(z)}{R_i(z,z^*_i)} \right|\lesssim \sigma_t^3 (d\log n)^2.
\]
        Then, since $T_k < n^{-\frac{2}{2\alpha+d}}$
        we get that for $n$ large enough and with probability greater than $1-n^{-2}$
\begin{align}
        \label{eq:small_time_residual_in_g}
        \frac{\abs{c^2_t\norm{R_i(z,z^*_i)}^2
        +2 c_t\innerproduct{X_{i,t}- c_t\Phi_i^*(z)}{R_i(z,z^*_i)}}}{2\sigma_t^2}
        &\le \frac{1}{2}n^{-\frac{1}{2\alpha+d}}(\log n)^2,\nonumber
        \\
        \abs{\exp\del{\frac{c^2_t\norm{R_i(z,z^*_i)}^2
        +2c_t\innerproduct{X_{i,t}-c_t\Phi_i^*(z)}{R_i(z,z^*_i)}}{2\sigma_t^2}}-1}
        &\le n^{-\frac{1}{2\alpha+d}}(\log n)^2.
\end{align}
        Next, we analyze the main term ${\norm{x^{\pr}_i- c_t\grad\Phi^*_i(z^*_i)(z-z^*_i)}^2}$. Substituting $X_t$ as $x$ in~\eqref{eq:x^pr_i} leads to
\begin{align}
    \label{eq:decomposition_of_main_term_t_small}
        X^{\pr}_i(t) - c_t\grad\Phi^*_i(z^*_i)(z-z^*_i) &= \pr_{z^*_i}\del{(X_t)-c_t\Phi^*_i(z^*_i)} - c_t\grad\Phi^*_i(z^*_i)(z-z^*_i) \nonumber \\
        &=
        \sigma_t\pr_{z^*_i}Z_D +  c_t\pr_{z^*_i}\del{X_0-\Phi^*_i(z_i^*)} - c_t\grad\Phi^*_i(z^*_i)(z-z^*_i).
\end{align}
        Recall that $m^*_i = \Phi^*_i(z^*_i)$ then 
\[
\norm{m^*_i-X_0} \le \norm{m^*_i-m_i} + \norm{m_i-X_0} \le O\del{\varepsilon^\beta_N} + 2(\sigma_t/c_t)\sqrt{20}C(n) \le 4(\sigma_t/c_t)\sqrt{20}C(n),
\]
where the last inequality follows from ~\Cref{prop:points_z_as_projections_on_tangent_space}.
We bound the projection $\pr_{z^*_i}Z_D$  using~\Cref{lemma:global_bound_on_projection_on_tangent_space} in~\Cref{apdx:Correlation Between Tangent Vectors and Gaussian Noise} with $\delta = n^{-2}$, so that with probability greater than $1 - n^{-2}$,
    \begin{align*}
       \norm{\pr_{z_i^*} Z_D}
       &\le  56dL_M\varepsilon_N + 8\sqrt{d}(1+7L_M\varepsilon_N)\sqrt{2\log 2d + 6\log n} \le 16 C(n)
    \end{align*}
for $n$ large enough  since $\varepsilon_N =o(1)$. Finally, since $\|\nabla \Phi_i^*(z_i^*)\|_{op}\leq \sup_{z\in B_d(0, 7 \varepsilon_N)} \|\nabla \Phi_i^*(z)\|_{op}\leq 2  $ by Lemma~\ref{lemma:approximation_of_Phi} the last term of  \eqref{eq:decomposition_of_main_term_t_small} is also bound by $ 8\sqrt{20}\sigma_t C(n) $ leading to
\begin{align}
\label{eq:step_7_bound_on_projected_term}
&\norm{X^{\pr}_i(t) - c_t\grad\Phi^*_i(z^*_i)(z-z^*_i)}^2
\nonumber \\
&\le 4\sigma^2_t\norm{\pr_{z^*_i}Z_D}^2 +  4c^2_t\norm{X_0-\Phi_i(z_i^*)}^2 + 4c^2_t\norm{\grad\Phi^*_i(z^*_i)(z-z^*_i)}^2 %+ 4c^2_t\norm{ R_i(z,z^*_i)}^2
\nonumber \\
&\le 4\cdot \sigma^2_t C^2(n)  [256+ 16 + 64] = 1344\sigma^2_t C^2(n).
\end{align}

Moreover, for $n$ large enough $p^*(z) > \frac{1}{2}\inf_{B_d(0,8\varepsilon_N)} p(z) \ge 2^{-d-1}p_{\min}$, so combining with \eqref{eq:small_time_residual_in_g} we get
\begin{align*}
g_i(t,x) 
        &\ge 
        \frac{1}{4}p_{\min}
        e^{-676 C^2(n)}\int_{B_d(z_i^*, 2(\sigma_t/c_t)\sqrt{20}C(n))} 
         dz 
=  \frac{1}{4}p_{\min}
        e^{-676 C^2(n)}C_d \del{2(\sigma_t/c_t)\sqrt{20}C(n)}^d,
\end{align*}
where $C_d$ is the volume of  the $d$-dimensional unit ball. 
Since $C(n) = \sqrt{d\del{\log n + 4C_{\log}} + 4\log n} $, then for $n$ large enough $C^2(n) < 4d\log n$. So, since $\sigma_t/c_t > n^{-\frac{\alpha+1}{2\alpha+d}} > n^{-1}$ for $n$ large enough with probability at least $1-n^{-2}$, for all $i \in I_2$,
\[
g_i(t,X_t) \ge n^{-2688 d}.
\]
Suppose $\hat{f}_i$ and $\hat{g}_i$ are approximations of $f_i, g_i$ such that with probability $1-n^{-2}$ for all $i\in I_2$ 
\begin{align}
\label{eq:f_hat_definition_t_small}
&\abs{g_i(t,X_t) - \hat{g}_i(t,X_t)} \le n^{-2690d}, \quad 
\norm{f_i(t,X_t) - \hat{f}_i(t,X_t)}  \le n^{-2690d}.
\end{align}
then  with probability at least $1-n^{-2}$, for all $i\in I_2$
\[
\norm{\hat{e}_i(t,X_t) - e^*_i(t,X_t)} \le \norm{e^*_i(t,X_t)} n^{-2} \le 4n^{-2}, \quad \text{with } \, \hat{e}_i(t,x) = \frac{\hat{f}_i(t,x)}{\hat{g}_i(t,x)}.
\]
This implies that 
$\hat{e}_{tr}(t,x) = \sum_{i=1}^N \rho_i(t,x)\hat{e}_i(t,x)/(\sum_{i=1}^N \rho_i(t,x))$  verifies 
$$\norm{\hat{e}_{tr}(t,X_t) - e^*_{tr}(t,X_t)} \le 4n^{-2}.$$
 So, defining $\hat{e}_{\clip}(t,x) = \hat{e}(t,x)\ind_{\norm{\hat{e}(t,x)}\le 1}$ and using that $\norm{e^*(t,x)} \le 1$ a.s.
\[
\expectation \norm{\hat{e}_{\clip}(t,X_t) - e^*_{tr}(t,X_t)}^2 \lesssim \del{n^{-2} + n^{-4}} \lesssim n^{-2}. 
\]
Integrating we obtain 
\[
\int_{T_k}^{T_{k+1}}\expectation \norm{\hat{e}_{\clip}(t,X_t) - e^*_{tr}(t,X_t)}^2 dt \lesssim n^{-2} \le T_kn^{-1},
\]
where the constant depends only on $d, \beta, C_{log}$.
So, defining $\hat{s}_{\clip}(t,x) := \frac{c_t\hat{e}_{\clip}(t,x)-x}{\sigma^2_t}$ and combining with~\eqref{eq:sml_s_star_t_small}
\[
\int^{t_{k+1}}_{t_k}\expectation \norm{\hat{s}_{\clip}(t,X_t) - s_{tr}(t,X_t)}^2 dt=O\del{\sigma^{-2}_{T_k}(\log n)^{2\gamma\beta}n^{-\frac{2\beta}{2\alpha+d}} +  (\log n)^{2\alpha +1}n^{-\frac{2\alpha}{2\alpha+d}}}.
\]

\subsubsection{Neural Network Approximation}
\label{apdx:neural_netwok_approximation_t_small}
    In this Section we show that we can construct neural networks for $\hat{f}_i, \hat{g}_i$ satisfying~\eqref{eq:f_hat_definition_t_small}. To do that we first approximate them by polynomials whose coefficients are themselves polynomials of $\widetilde x_{i,t}, c_t, c_t^{-1}, \sigma_t, \sigma_t^{-1}$.
    
    First,  as in the case $T_k \geq n^{-\frac{2}{2\alpha+d}}$, using the linear isometry $P_{\cH_i}:\R^{\dim \cH_i}\rightarrow \cH_i$ that identifies $\cH_i$ with $\R^{\dim \cH_i}$, we can represent 
     $   \Phi^*_i(z) = G_i + P_{\cH_i}\widetilde{\Phi}^*_i(z),$ where $\widetilde{\Phi}^*_i$ is a polynomial of degree $\beta-1$ with coefficients $\widetilde{P}^*_i := P^T_{\cH_i}P^*_i:\R^d\rightarrow \R^{\dim \cH_i}$, $\widetilde{a}^*_{i,S}\in \R^{\dim \cH_i}$ belonging to $\R^{\dim \cH_i}$. Define $\widetilde{x}_{i} = P^T_{\cH_i}\del{x-c_tG_i} \in \R^{\dim \cH_i}$,   
    and $\widetilde{x}^{\pr}_i = P_{\cH_i}x^{\pr}_i$ so that
    \[
    \widetilde{x}^{\pr}_i= \widetilde{\pr}_{z^*_i}\del{\widetilde{x}_{i,t}- c_t \widetilde{\Phi}^*_i(z^*_i)}
    = \grad \widetilde{\Phi}^*_i(z^*_i)\del{\grad^T\widetilde{\Phi}^*_i(z^*_i)\grad\widetilde{\Phi}^*_i(z^*_i)}^{-1}\grad^T \widetilde{\Phi}^*_i(z_i^*)(\widetilde{x}_{i,t}- c_t \widetilde{\Phi}^*_i(z^*_i)).
    \]
    And since $P_{\cH_i}$ is an isometry
    \begin{align*}
    \norm{x^{\pr}_i - c_t\grad\Phi^*_i(z^*_i)(z-z^*_i)}^2 &= {\norm{\widetilde{x}^{\pr}_i - c_t\grad\widetilde{\Phi}^*_i(z^*_i)(z-z^*_i)}^2}, 
    \\
    \norm{R(z,z^*_i)}^2 &= \norm{\widetilde{R}(z,z^*_i)}^2
    \\
      \innerproduct{X_{i,t}- c_t\grad\Phi^*_i(z^*_i)(z-z^*_i)}{R(z,z^*_i)} &= 
     \innerproduct{\widetilde{X}_{i,t}(t)- c_t\grad\widetilde{\Phi}^*_i(z^*_i)(z-z^*_i)}{\widetilde{R}(z,z^*_i)}.
    \end{align*}
    Therefore $f_i$ and $g_i$ can be re-expressed in terms of $\widetilde{x}_{i,t}, \widetilde{\Phi}_i^*, \widetilde{R}$ etc. 
Then using a Taylor approximation of $e^u$ near $0$ together with 
    \eqref{eq:step_7_bound_on_projected_term} and  \eqref{eq:small_time_residual_in_g}, we have that for $n$ large enough with probability at least $1-n^{-2}$
  \[
    & \abs{e^{-{\norm{\widetilde{x}^{\pr}_i - c_t\grad\widetilde{\Phi}^*_i(z^*_i)(z-z^*_i)}^2}/2\sigma^2_t} - \sum_{k=1}^{K_1} \frac{1}{k!}\del{-{\norm{\widetilde{x}^{\pr}_i- c_t\grad\widetilde{\Phi}^*_i(z^*_i)(z-z^*_i)}^2}/2\sigma^2_t}^k} 
    \\
    & \quad \le \del{\frac{e\cdot 2368d\log n}{K_1}}^{K_1} \leq {e^{-e^2 \cdot 2368d\log n} < n^{-2370d}},
    \]
    by choosing $K_1 = e^2\cdot 2368d\log n = O(d \log n)$. And similarly 
    \[
    &\bigg|e^{
    \del{c^2_t\norm{\widetilde{R}(z,z^*_i)}^2 - c_t\innerproduct{\widetilde{X}_{i,t}(t)- c_t\grad\Phi^*_i(z^*_i)(z-z^*_i)}{\widetilde{R}(z,z^*_i)}}/2\sigma^2_t
    } 
    \\
    &- \sum_{k=1}^{K_2} \frac{1}{k!}
    \del{
    \del{c^2_t\norm{\widetilde{R}(z,z^*_i)}^2 - c_t\innerproduct{\widetilde{X}_{i,t}(t)- c_t\grad\widetilde{\Phi}^*_i(z^*_i)(z-z^*_i)}{\widetilde{R}(z,z^*_i)}}/2\sigma^2_t
    }^k
    \bigg|
    \\
    &\le 2\del{\frac{en^{-\frac{1}{2(2\alpha+d)}}}{K_2}}^{K_2}
     \quad \leq n^{-2370d}
    \]
    by choosing $K_2=2(2\alpha+d)2370d = O(d\alpha+d)) $. We write $Q_1(z-z_i^*)$ the polynomial approximation of the first term and $Q_2(z-z_i^*)$ the polynomial approximation of the second (as functions of $z$). $Q_1$ has degree $2K_1$ and $Q_2 $ $2K_2$. Moreover $p_i^*$ is a polynomial function of degree $\alpha$; therefore  
    \begin{align*}
    \bar g_i(t,x) := \int_{\|z-z_i^*\|\leq 4\sigma_tC(n)/c_t} Q_1(z-z_i^*)Q_2(z-z_i^*)p_i^*(z)dz = \int_{\|z\|\leq 4\sigma_tC(n)/c_t} Q_1(z)Q_2(z)Q_3(z)dz
    \end{align*}
    where $Q_3$  is a polynomial function of $z$ of degree  $\alpha$ and we have $p_i^* \leq 2p_{max}$
    $$|\bar g_i(t,x) - g_i(t,x)| \leq 4n^{-2370d} \Vol B_d(0, 4\sigma_tC(n)/c_t) \lesssim T_k^{d}(\log n)^dn^{-2370d}= o(n^{-2370d})$$
     Similarly, since $\widetilde{\Phi}_i^*$ is also polynomial with degree $\beta$ and $\|\widetilde{\Phi}_i^*(z)\|\leq 2$ over the ball of integration, $|\bar f_i(t,x) - f_i(t,x)|  o(n^{-2370d})$  with 
     \begin{align*}
    \bar f_i(t,x) :=  \int_{\|z\|\leq 4\sigma_tC(n)/c_t} Q_4(z)Q_1(z)Q_2(z)Q_3(z)dz, \quad Q_4(z) = \Phi_i^*(z+z_i^*).
    \end{align*}
     We write 
     $ Q_j(z) = \sum_{|S|\leq 2K_j} a_{S,Q_j} z^{S} $ with $ \beta\geq 2K_3 \geq \beta-1$ and $\alpha-1 \leq 2K_4 \leq \alpha$ and note that $Q_j: \mathbb R^d \rightarrow \mathbb R$ for $j= 1, 2, 3$ and $Q_4: \mathbb R^d \rightarrow \mathbb R^{\dim \cH_i}$. Then by construction $ a_{S,Q_1}$  are polynomial functions  of degree less than $O(K_1)$ of $\widetilde{x}_{i,t}^{pr}, \sigma_t^{-2}, c_t, z_i^*$ since $\widetilde{\Phi}_i^* =  \widetilde{P}_i^* + \sum_{S'} \widetilde{a}_{i,S}^* z^{S'}$ . Also $a_{S,Q_2}$ are polynomial functions of  degree less than $O(K_2)$  of the first $\beta$ derivatives of $\widetilde{\Phi}_i^*$ at $z_i^*$, of $c_t, \sigma_t^{-1},\widetilde{x}_{i,t}$ so that they are polynomial functions  of degree less than $O(K_2)$  of $c_t, \sigma_t^{-1},\widetilde{x}_{i,t}, z_i^*$. Similarly for $a_{S, Q_3}$ which as a degree of order $O(\alpha d)$. Finally $a_{S, Q_4}\in \mathbb R^{\dim \cH_i}$, whose components are polynomial functions of degree $\beta$ of $\widetilde{x}_{i,t}, z_i^* $.
     Since for all vector of indices $S$
     $$ \int_{\norm{z} \le 4(\sigma_t/c_t)\sqrt{20}C(n)} z^S dz 
    = (\sigma_t/c_t)^S\del{4\sqrt{20}C(n))}^S \int_{\norm{z}\le 1} z^S = c_{S,n} (\sigma_t/c_t)^S$$ we can write 
    \begin{align*}
        \bar g_i(t, x) &= \sum_{S_1, \cdots, S_3}\prod_{l=1}^3a_{S_l,Q_l} (\sigma_t/c_t)^{\sum_lS_l}C_{\sum_lS_l,n}
         \\
         \bar f_i(t, x) &= \sum_{S_1, \cdots, S_4}\prod_{l=1}^4a_{S_l,Q_l} (\sigma_t/c_t)^{\sum_lS_l}C_{\sum_lS_l,n}.
    \end{align*}

We now show that $\widetilde{x}_{i,t}^{\pr}$ and $z_i^*$ are polynomial functions of $\widetilde{x}_{i,t}$ and $c_t^{-1}$. Recall that   $x_{i,t} = \pr_{\cH_i}(x-ct G_i) + c_t G_i$ and since $\Im P^*_i \subset \cH_i$ then %$x_{i,t} := \pr_{\cH_}x = \pr_{\cH_i}(x-c_tG_i) + c_tG_i$ 
    $
    z^*_i := \del{P^*_i}^T(x-c_tG_i)/c_t = \del{P^*_i}^T(x_{i,t}-c_tG_i)/c_t
    $
 and
    \[
    \del{\widetilde{P}^*_i}^T\widetilde{x}_{i,t}/c_t = \del{P^*_i}^T P_{\cH_i}P^T_{\cH_i}\del{x_{i,t}-c_tG_i} = z^*_i.
    \]
    Therefore $z^*_i$ is a polynomial of order $2$ of $\widetilde{x}_{i,t},c_t^{-1}$. 
    Now recall that 
    \[
    \widetilde{x}^{\pr}_i = \grad \widetilde{\Phi}^*_i(z^*_i)\del{\grad^T\widetilde{\Phi}^*_i(z^*_i)\grad\widetilde{\Phi}^*_i(z^*_i)}^{-1}\grad^T \widetilde{\Phi}^*_i(z^*_i)(\widetilde{x}_{i,t}-c_t\widetilde{\Phi}^*_i(z^*_i)) 
    \] 
    and that $
    \widetilde{\Phi}^*_i(z) = \widetilde{P}^*_iz + \sum \widetilde{a}^*_{i,S}z^{S}$; with $\widetilde{a}^*_{i,S}\in \R^{\dim \cH_i}$ and $|S|_1\leq \beta -1$. Therefore  $
    \grad\widetilde{\Phi}^*_i(z_i^*)$ is a polynomial of degree smaller than $\beta-2$ of $z_i^*$, hence it is a polynomial of degree $2(\beta-2)$ of $\widetilde{x}_{i,t},c_t^{-1}$. 
    As a consequence we obtain that 
     \begin{itemize}
        \item the entries of the $d\times d$ matrix $\grad^T \Phi^*_i(z^*_i)\grad \Phi^*_i(z^*_i)$ are polynomials in $c^{-1}_t, \widetilde{x}_{i,t}$ of order at most $4(\beta - 2)$. In other words $\grad^T \Phi^*_i(z^*_i)\grad \Phi^*_i(z^*_i) = Q_{d,d}(c^{-1}_t, \widetilde{x}_{i,t})$, for some polynomial $Q_1$ with coefficients in $\R^{d\times d}$;
        \item
        the entries of  the $d\times 1$ vector $\grad^T\widetilde{\Phi}^*_i(z^*_i)(\widetilde{x}_{i,t}-c_t\widetilde{\Phi}^*_i(z^*_i))$ are polynomials in $c^{-1}_t, x_{i,t}$ of order at most $2\beta$. In other words $\grad^T \Phi^*_i(z^*_i)(\widetilde{x}_{i,t}-c_t\widetilde{\Phi}^*_i(z^*_i))= Q_{d}(c^{-1}_t, \widetilde{x}_{i,t})$, for some polynomial $Q_{d}$ with coefficients in $\R^{d}$; 
        \item 
        Cramer's rule states that if $A\in \R^{d\times d}$ is invertible then 
        \[
        A^{-1} = \frac{1}{\det A}((-1)^{i+j}\det M_{ji})_{i,j\le d},
        \]
        where the matrix $M_{ij}$ is the $ij$-minor -- matrix. Also the 
        determinant of  a $d\times d$ matrix is a polynomial of order $d$ in its entries
        Therefore, the inverse matrix can be represented as
        \[
\del{\grad^T\Phi^*_i(z^*_i)\grad\Phi^*_i(z^*_i)}^{-1} = \frac{Q_{\operatorname{adj}}(c^{-1}_t, \widetilde{x}_{i,t})}{Q_{\det}(c^{-1}_t, \widetilde{x}_{i,t})},
        \]
        where $Q_{\operatorname{adj}}$ is a polynomial in $c^{-1}_t, \widetilde{x}_{i,t}$ of order $2(d-1)(\beta-2)$ with coefficients in $\R^{d\times d}$, and $Q_{\det}$ is a polynomial in $c^{-1}_t, \widetilde{x}_{i,t}$ of order $2d(\beta-2)$ with scalar coefficients.
    \end{itemize}
    Combining the above we obtain that $\widetilde{x}_{i,t}^{\pr}$ is a polynomial of degree  4 of $Q_d, Q_{d,d}, Q_{\det}^{-1}, Q_{\operatorname{adj}}$.
        
         Also, applying~\Cref{lemma:global_bound_on_projection_on_tangent_space} we note that with probability $1-n^{-2}$ for all $i\in I_2$
    \[
    \norm{\widetilde{X}_{i,t}(t)} = \norm{X_{i,t}} = \norm{\pr_{\cH_i}\del{c_tX_0+\sigma_tZ_D-c_tG_i}} \le c_t\norm{X_0-G_i}+ \norm{\pr_{\cH_i}Z_D} \lesssim O\del{\log n},
    \]
    and using that $1/2 \le \norm{\grad \Phi^*_i(z)\grad^T\Phi^*_i(z)}\le 2$ we conclude that all variables $\widetilde{x}_{i,t}, \sigma^{-1}_t, \sigma_t, c_t, c^{-1}_t, Q_d, Q_{d\times d},Q_{\operatorname{adj}}$ are bounded by $n$.
    
    We can now construct the neural network approximations of $\bar f_i$ and $\bar g_i$.
    Using Lemmas~\ref{lemma:parallelization_of_neural_networks}-- \ref{lemma:oko_nn_appproximating_reciprocal_function}, neural networks can be chosen $\phi_{\bar g}, \phi_{\bar f}, \phi_{Q_d}, \phi_{Q_{d\times d}}, \phi_{Q_{\operatorname{adj}}}, \phi_{Q^{-1}_{\det}}\in \Psi(L_1,W_1,S_1,B_1)$, where $L_1, W_1,S_1 = \polylog n$ and $B_1 = e^{\polylog n}$.
    
    Finally, applying~\Cref{lemma:concatentation_of_neural_networks} and then~\Cref{lemma:oko_nn_appproximating_reciprocal_function} we concatenate these neural networks and get two neural networks $\phi_{g_i}, \phi_{e_i} \in \Psi(L,W,S,B)$, where $L, W,S = \polylog n$ and $B = e^{\polylog n}$ that approximate $g_i$ and $e_i = f_i/g_i$ respectively  with an error $n^{-2370d}$.

    Constructing these networks for all $i\le N$ we show that for $T_k \le n^{-\frac{2}{2\alpha+d}}$ there is a network $\phi\in\cS_k$ s.t.
    \[
    \int_{T_k}^{T_{k+1}}\expectation \norm{\phi(t,X_t) - s(t,X_t)}^2 dt\le O\del{\sigma^{-2}_{T_k}(\log n)^{4\gamma\beta}n^{-\frac{2\beta}{2\alpha+d}} +  (\log n)^{2\alpha+1}n^{-\frac{2\alpha}{2\alpha+d}}}.
    \]
\subsection{Tail Bounds}
\label{sec:proof_tail_bound_D}
In this section, we first prove propositions describing the sets of indices for which $\rho_i(t,X_t) > 0$. Then we prove the tail bound statement of Theorem~\ref*{thm:main_result}. All the statements in the section hold for an arbitrary chosen initial point $X_0\in M$ and with high probability in $Z_D$.
    
\begin{proposition}
\label{prop:active_indices_large_tail}
Assume $T_k\ge n^{-2/(2\alpha+d)}$, and fix $t\ge T_k$ and $A > 0$. For
$t\in[T_k,T_{k+1}]$ define
\[
    I_1(t)
    :=
    \left\{i\le N:
    \norm{X_0-G_i}
    \le \frac 13 (\sigma_t/c_t)\sqrt{\log n\log\log n}
    \right\},
\]
and
\[
    I_2(t)
    :=
    \left\{i\le N:
    \norm{X_0-G_i}
    \le 3(\sigma_t/c_t)\sqrt{\log n\log\log n}
    \right\}.
\]
Then for all sufficiently large $n$ and $A \le \frac{1}{1200}\log \log n$ with probability at least
$1-n^{-A}$,
\[
    I_1(t)\subseteq\{i:\rho_i(t,X_t)=1\},
    \qquad
    \{i:\rho_i(t,X_t)>0\}\subseteq I_2(t).
\]
\end{proposition}

    \begin{proof}
Both statements are corollaries of Lemma~\ref*{lemma:good_set_size} obtained
by substituting the $\varepsilon_N$-dense set $\cG$.  The lemma states that
with probability at least $1-n^{-A}$, for all $i\le N$ simultaneously
\begin{align}
\label{eq:good_set_size_Gi_large_time}
    \frac23 c_t^{-2}\Delta_i(X_t)
    -128(\sigma_t/c_t)^2B_t
    \le \norm{X_0-G_i}^2
    \le
    9\varepsilon_N^2+2c_t^{-2}\Delta_i(X_t)
    +128(\sigma_t/c_t)^2B_t,
\end{align}
where
\[
 \Delta_i(X_t)
    :=
    \norm{X_t-c_tG_i}^2-\norm{X_t-c_tG_{\min}(X_t)}^2,\qquad  
    B_t
    :=
    d\log_+(c_t/\sigma_t)+4dC_{\log}+A\log n.
\]
Note that $\rho_i(t,X_t) = \rho(\Delta_i(X_t)/2C(T_k,n))$.  Since
$t\ge n^{-\frac{2}{2\alpha+d}}$,  for large enough $n$ we have $192B_t\le \frac{1}{6}\log n\log\log n$.

We start by proving the first inclusion
$I_1(t)\subseteq\{i:\rho_i(t,X_t)=1\}$.  By definition of $\rho_i$, it is
enough to prove that $\Delta_i(X_t)\le C(T_k,n)$.

Let $i\in I_1(t)$, the first inequality in~\eqref{eq:good_set_size_Gi_large_time} combined with definition of $I_1(t)$ implies
\[
\Delta_i(X_t) \le \frac 32 c^2_t\norm{X_0-G_i}^2 + 192\sigma^2_t B_t \le \frac{1}{6}\sigma_t^2 \log n\log\log n + 192\sigma^2_t B_t \le \frac{1}{3} \sigma^2_t\log n \log \log n.
\]

Finally, using that $\sigma^2_t \le 2\sigma^2_{T_k}$ we conclude that for $i\in I_1(t)$ 
\[
\Delta_i(X_t) \le \frac 23 \sigma^2_{T_k} \log n \log \log n = \frac{2}{3} C(T_k,n).
\]
This proves the first inclusion.

Next, we prove that $\{i:\rho_i(t,X_t)>0\}\subseteq I_2(t)$. Let $i$ satisfy $\rho_i(t,X_t) >0$, by definition of $\rho_i$ we have 
\[
\Delta_i \le 2C(T_k, n) = 2\sigma^2_{T_k} \log n\log \log n \le 2\sigma^2_{t} \log n\log \log n
\]

Substituting this into~\eqref{eq:good_set_size_Gi_large_time} we get
\[
\norm{X_0-G_i}^2
    \le
    9\varepsilon_N^2+4(\sigma_t/c_t)^2 \log n\log \log n
    +128(\sigma_t/c_t)^2B_t
\]
Since $128 B_t\le \frac{1}{9} \log n \log \log n$ and by~\eqref{eq:basic_ineq_t_large}
$\varepsilon_N=o(\sigma_t/c_t)$ for $n$ large enough,
\[
\norm{X_0-G_i}^2
    \le 5(\sigma_t/c_t)^2 \log n\log \log n
    \le 9(\sigma_t/c_t)^2\log n\log \log n.
\]
That proves the second inclusion.
\end{proof}
The next proposition is a small time counterpart of the proposition above.
 \begin{proposition}
  \label{prop:small_time_score_localization_geometry}
  Fix $t\in[\underline{T},n^{-\frac{2}{2\alpha+d}}]$ and $A>0$. Then, for all sufficiently large $n$ with $\log n>A$, with probability at least $1-n^{-A}$, the following holds.

  \begin{itemize}
      \item[(i)] Every active index satisfies
      \[
          \rho_i(t,X_t)>0
          \quad\Longrightarrow\quad
          i\in I_2,
          \qquad
          I_2:=\curly{i\le N:\norm{X_0-G_i}\le4\varepsilon_N}.
      \]
      In particular, for every such index $i$,
      \[
          X_0 \in \Phi_i(B_d(0,4\varepsilon_N)),
          \qquad
          \operatorname{dist}(X_0,G_i+\cH_i)\lesssim \varepsilon_N^\beta .
      \]

      \item[(ii)]
      Let 
      \[
      C(n,A) = \sqrt{d\del{\log n + 4C_{\log}} + 4A\log n}.
      \]
      Then, for all sufficiently large $n$, every active index ($\rho_i(t,X_t) > 0$) satisfies
      \begin{equation}
      \label{eq:balls_inclusion_t_small}
          B_M\del{X_0,2(\sigma_t/c_t)\sqrt{20}C(n,A)}
          \subset
          \Phi_i\del{B_d\del{z^*_i,4(\sigma_t/c_t)\sqrt{20}C(n,A)}}
          \subset
          \Phi_i\del{B_d\del{0,7\varepsilon_N}} .
      \end{equation}
  \end{itemize}
  \end{proposition}

\begin{proof}
\,

\begin{itemize}
    \item[(i)]
Since $\cG$ is $\varepsilon_N/2$-dense in $M$,
Lemma~\ref{lemma:good_set_size}, applied with failure probability
$n^{-A}/2$ and $\varepsilon=\varepsilon_N/2$, gives simultaneously for all
$i\le N$
\begin{align*}
    \norm{X_0-G_i}^2
    &\le
    \frac94\varepsilon_N^2
    +2c_t^{-2}\del{\norm{X_t-c_tG_i}^2-\norm{X_t-c_tG_{\min}(X_t)}^2}
    \\
    &\quad
    +128(\sigma_t/c_t)^2
    \del{d\log_+(c_t/\sigma_t)+4dC_{\log}+A\log n + \log 2}.
\end{align*}
If $\rho_i(t,X_t)>0$, then
\[
    \norm{X_t-c_tG_i}^2-\norm{X_t-c_tG_{\min}(X_t)}^2
    <2C(T_k,n)=6\varepsilon_N^2.
\]
Since $c_t = e^{-t}$, for large enough
$n$ for any $t <n^{-2/(2\alpha+d)}$ we have $c_t^{-2}=1+O(t) \le25/24$.  Since $A<\log n$ and $\log_+(c_t/\sigma_t)\lesssim\log n$, the
logarithmic factor in the last term is bounded by a fixed multiple of
$(\log n)^2$, at the same time by the choice of $\gamma$
\[
    \sigma_t/c_t\le\varepsilon_N(\log n)^{-3},
\]
so for $n$ large enough, the last term is bounded by $\frac14 \varepsilon_N^2$. Therefore
\[
    \norm{X_0-G_i}^2
    \le
    \del{\frac94+\frac{25}{2}+\frac{1}{4}}\varepsilon_N^2
    <16\varepsilon_N^2,
\]
which gives that $i\in I_2$.

By construction, $\Phi_i$ is the inverse of the projection on the tangent space at $G_i$, so if $z_i$ is the projection of
$X_0-G_i$ on $T_{G_i}M$ then $X_0=\Phi_i(z_i)$ and $\|z_i\|\le\|X_0-G_i\|\le4\varepsilon_N$. Finally, by construction, the approximation $\Phi^*_i$ satisfy 
\[
    \operatorname{dist}(X_0,G_i+\cH_i)
    \le
    \norm{\Phi_i(z_i)-\Phi_i^*(z_i)}
    \le C\varepsilon_N^\beta.
\]
\item[(ii)] 
Let $m_i = \Phi_i(z^*_i)$. Proposition~\ref{prop:points_z_as_projections_on_tangent_space},
with the union bound over $i\le N$ and failure probability $n^{-A}/2$, gives that for all $i\in I_2$
\[
    \norm{m_i-X_0}
    \le
    C\varepsilon_N^\beta
    +2(\sigma_t/c_t)\sqrt{d+\log n+2A\log n} \le 2(\sigma_t/c_t)\sqrt{20}C(n,A),
\]
where the last inequality holds for all sufficiently large $n$, as $(\sigma_t/c_t) \ge \varepsilon^{\beta}_N \log n$. That shows
\[
 B_M\del{X_0,2(\sigma_t/c_t)\sqrt{20}C(n,A)}
          &\subset
          B_M\del{m_i,4(\sigma_t/c_t)\sqrt{20}C(n,A)} 
          \\
          &\subset \Phi_i\del{B_d\del{z^*_i,4(\sigma_t/c_t)\sqrt{20}C(n,A)}}.
\]
Finally, since $z^*_i$ is a projection on $T_{G_i}M$ of $m_i$
\[
\norm{z^*_i} \le \norm{m_i-G_i} \le \norm{m_i-X_0} + \norm{X_0-G_i} \le 2(\sigma_t/c_t)\sqrt{20}C(n,A) + 4\varepsilon_N.
\]
So, since for $n$ large enough ${2(\sigma_t/c_t)\sqrt{20}C(n,A) \le \varepsilon_N}$ we conclude
\[
B_d\del{z^*_i,4(\sigma_t/c_t)\sqrt{20}C(n,A)} \subset B_d\del{0,4\varepsilon_N + 6(\sigma_t/c_t)\sqrt{20}C(n,A)} \subset B_d\del{0,7\varepsilon_N},
\]
By applying $\Phi_i$, we finish the proof.
\end{itemize}
\end{proof}
By applying these propositions, we prove 
{
\begin{lemma}
\label{lem:large_time_constructed_tube}
Assume $T_k\ge n^{-2/(2\alpha+d)}$.  For every fixed
$t\in[T_k,T_{k+1}]$ and $A$ satisfying $A \le \frac{1}{1200}\log\log n$, with probability at least $1-2n^{-A}$, every active
index $\rho_i(t,X_t)>0$ satisfies
\[
    \norm{c_te_i^*(t,X_t)-X_0}\le 4\sigma_t\sqrt{\log n \log \log n}.
\]
and
\[
    \norm{c_te_i^*(t,X_t)-X_{i,t}}\le 8\sigma_t\sqrt{\log n \log \log n}.
\]
for all $n\ge C_0(d,C_{\log})$.  
\end{lemma}
}
{
\begin{proof}
By construction
$M_i^*=\Phi_i^*(B_d(0,\varepsilon_N))\subset G_i+\cH_i$ and
$\Phi_i^*(0)=G_i$.  Since $\Phi_i^*$ is $2+O(\varepsilon_N^{\beta-1})$-Lipschitz, for $n$ large enough
every $y^*\in M_i^*$ satisfies $\|y^*-G_i\|\le 3\varepsilon_N$.  The point
$e_i^*(t,X_t)$ is a weighted average over $M_i^*$ and $\varepsilon_N = o(\sigma_t/c_t)$, hence
\[
    \norm{e_i^*(t,X_t)-G_i}\le 3\varepsilon_N \le \sigma_t/c_t.
\]
Combining Proposition~\ref{prop:active_indices_large_tail} and
Proposition~\ref{prop:bound_on_projection_on_cH_i} both applied
with $\delta=n^{-A}$ we get that for $n$ large enough with probability at least $1-2n^{-A}$
\[
    \norm{X_0-G_i}
    \le
    3(\sigma_t/c_t)
    \sqrt{\log n\log\log n}
\]
and
\[
    \max_{i\le N}\norm{\pr_{\cH_i}Z_D}\le 4\sqrt{\log n \log \log n}.
\]

Using $X_t=c_tX_0+\sigma_tZ_D$ and $e_i^*(t,X_t)\in G_i+\cH_i$,
\begin{align*}
    X_{i,t}-c_te_i^*(t,X_t)
    &=
    c_tP_{\cH_i}P_{\cH_i}^T(X_0-G_i)
    +\sigma_tP_{\cH_i}P_{\cH_i}^TZ_D
    -c_t(e_i^*(t,X_t)-G_i).
\end{align*}
Combining with the inequalities above, we conclude
\[
\norm{X_{i,t}-c_te_i^*(t,X_t)} \le 8\sigma_t\sqrt{\log n \log \log n}.
\]
While, by dropping $Z_D$ term we get
\[
    \norm{c_te_i^*(t,X_t)-X_0}\le 4\sigma_t\sqrt{\log n \log \log n}.
\]
\end{proof}
}
The following lemma is a small-time counterpart of the lemma above.
{
\begin{lemma}
\label{lem:small_time_constructed_tube}
Assume $\underline T\le T_k< n^{-2/(2\alpha+d)}$.
For every fixed $t\in[T_k,T_{k+1}]$ and $A\le \log n$, with probability at least
$1-O(n^{-A})$, every active index $\rho_i(t,X_t)>0$ satisfies
\[
    \norm{c_te_i^*(t,X_t)-X_{0}}
    \le
    241 \sigma_t \log n
\]
and
\[
    \norm{c_te_i^*(t,X_t)-X_{i,t}}
    \le
    243 \sigma_t \log n
\]
for all $n\ge C_0(d,C_{\log})$.  
\end{lemma}
}

\begin{proof}
{
By construction $e_i^*(t,X_t)$ is the integral over $M^*_i=\Phi^*_i\del{B_d\del{z^*_i,4(\sigma_t/c_t)\sqrt{20}C(n,A)}}$. So it is enough to bound $\norm{c_t\Phi^*_i(t,X_t)-X_{i,t}}$ for all
$z\in B_d\del{z^*_i,4(\sigma_t/c_t)\sqrt{20}C(n,A)}$.

Since $\Phi_i^*$ is $2+O(\varepsilon_N^{\beta-1})$-Lipschitz, for $n$ large enough 
\[
\norm{\Phi^*_i(z^*_i)-\Phi^*_i(z)}\le 12(\sigma_t/c_t)\sqrt{20}C(n,A) \le 120(\sigma_t/c_t)\log n.
\]
By Proposition~\ref{prop:small_time_score_localization_geometry} with probability at least $1-n^{-A}$ it holds $X_0\in M_i$, since $\norm{\Phi_i-\Phi_i^*}_\infty \le C\varepsilon_N^\beta$
\[
\norm{X_0 -\Phi^*_i(z^*_i)} \le C\varepsilon_N^\beta + 12(\sigma_t/c_t)\sqrt{20}C(n,A) \le C\varepsilon_N^\beta + 120(\sigma_t/c_t)\log n.
\]
Finally, applying Proposition~\ref{prop:bound_on_projection_on_cH_i}, for $n$ large enough, with probability at least $1-n^{-A}$
\[
\norm{Z_D} \le \sqrt{(C_{\dim} +2) \log n + 2 \log 2n^A } \le 2 \log n.
\]
Therefore, representing
\begin{align*}
    X_{i,t}-c_te_i^*(t,X_t)
    &=
    c_tP_{\cH_i}P_{\cH_i}^T(X_0-G_i)
    +\sigma_tP_{\cH_i}P_{\cH_i}^TZ_D
    -c_t(e_i^*(t,X_t)-G_i).
\end{align*}
and using that $c_t \le 1$ and $\varepsilon_N^{\beta}\lesssim \underline{T}\lesssim \sigma_t$ we conclude
\[
\norm{X_{i,t}-c_te_i^*(t,X_t)} \le C\varepsilon_N^\beta + 242 \sigma_t \log n \le 243 \sigma_t \log n.
\]
Similarly, by dropping the term with $Z_D$
\[
\norm{X_{0}-c_te_i^*(t,X_t)} \le C\varepsilon_N^\beta + 240 \sigma_t \log n \le 241 \sigma_t \log n.
\]
}
\end{proof}
By combining these lemmas, we get the tail bound in Theorem~\ref*{thm:main_result} as an immediate consequence.
\begin{corollary}
\label{cor:class_level_tail_bound}
For every fixed $t\in[T_k,T_{k+1}]$ and $A \le \frac{1}{1200}\log n\log\log n$, on event
$\cE_{\mathrm{geom},k}$, with probability at least $1-n^{-A}$ every
$\phi_s\in\cS_k$ satisfies
\[
{
    \norm{\sigma_t\phi_s(t,X_t)+Z_D}
    \le
    244\log n.
}
\]
{
Consequently the ERM $\hat s_k\in\cS_k$ satisfies
the tail bound in Theorem~\ref*{thm:main_result}.
}
\end{corollary}

\begin{proof}
{
For $\phi_s\in\cS_k$, define
}
{
\[
    \phi_e(t,x):=\frac{\sigma_t^2}{c_t}\phi_s(t,x)+\frac{x}{c_t},
\]
}
{
We need to bound
\[
\norm{\sigma_t\phi_s(t,X_t)+Z_D} = \frac{\norm{c_t\phi_e(t,X_t)-X_0}}{\sigma_t}. 
\]
By construction, $\phi_e(t,X_t)$ is a convex combination of functions
$e_i^\phi(t,X_t):=G_i + P_{\cH_i}\phi_{i}(t,X_t)$ with indices satisfying $\rho_i(t,X_t) > 0$, so it is enough to bound $\norm{c_te^{\phi}_i(t,X_t)-X_0}.$

Conditioned on $\cE_{\mathrm{geom},k}$ we have $\dim\cH_i\le C_{\dim}\log n$, so combined with the definition of class $\cS_k$ we have
}
\[
{
    \norm{c_te_i^\phi(t,X_t)-X_{i,t}}
    \le
    \sqrt{\dim\cH_i}\,\sigma_t\log n
    \lesssim
    \sigma_t\log n\sqrt{\log n}.
}
\]

{
Applying Propositions~\ref{lem:large_time_constructed_tube}~and~\ref{lem:small_time_constructed_tube} we obtain that
\[
\norm{c_t X_{i,t}-X_0} \le 243\sigma_t\log n.
\]
Therefore
\[
\frac{\norm{c_te^{\phi}_i(t,X_t)-X_0}}{\sigma_t} \le 244\sigma_t\log n
\]
}
\end{proof}

    \subsection{Generalization error \& Proof of Theorem~\ref*{thm:main_result} }
    \label{apdx:generalization_error}
        Recall that we approximate the score $s(t,x)$ by neural networks of the form 
       \[
{
        \phi_s(t,x) =
        \frac{c_t}{\sigma^2_t}
        \frac{\sum_i^N \tilde \rho_i(t,x)\phi_{w_i}(t,\widetilde{x}_{i,t})
        \del{G_i +P_{\cH_i}\phi_{e_i}(t,\widetilde{x}_{i,t})}}
        {\sum_i^N \tilde \rho_i(t,x)\phi_{w_i}(t,\widetilde{x}_{i,t})}
        - \frac{x}{\sigma^2_t},
        \quad
        \widetilde{x}_{i,t}=P_{\cH_i}^T(x-c_tG_i),
}
        \]
{
        where $\phi_{e_i}:\R_+\times\R^{\dim \cH_i}\rightarrow \R^{\dim \cH_i}$ and $\phi_{w_i}:\R_+\times\R^{\dim \cH_i}\rightarrow \R_+$ are neural networks of $\polylog$ size belonging to the class $\cS_k$ defined in~\eqref*{eq:definition of cS}.  In particular, on the relevant domain,
}
    \[
        \cS_k = \Big\{\phi_s: \phi_{e_i}, \phi_{w_i} \in \Psi\del{L,W,B,S},\quad \norm{\phi_{e_i}} \le C_2,
        \\
        \left\|\frac{c_t\phi_{e_i}(t,\widetilde{x}_{i,t}) - \widetilde{x}_{i,t}}{\sigma_t}\right\|_{\infty} \le \log n,\quad \norm{\phi_{w_i}}_{\infty} \ge n^{-2}\Big\}, %C_e\sqrt{\log n} n^{-C_w}
    \]
    The network parameters satisfy
    \[
    L = O\del{\polylog n}, \norm{W}_\infty = O\del{\polylog n}, S = O\del{\polylog n}, B = e^{O\del{\polylog n}}. 
    \]
      The estimator $\hat{s}$ of the score $s(t,x)$ on interval $[T_k,T_{k+1}]$  is defined as
    \[
    {\hat{s}_k \in} \argmin_{\phi\in {\cS_k}} \cR_{\cY}(\phi, T_k, T_{k+1}), \quad
    \cR_{\cY}(\phi, T_k, T_{k+1}) = \frac{1}{n} \sum_{i=1}^n \ell_{y_i}(\phi, T_k, T_{k+1}).
    \]
    and the loss function $\ell_y$ for $y\in M$ is defined in~\eqref*{eq:intro_diffusion_ell_y} as
    \[
    \ell_y(\hat{s}, T_k, T_{k+1}) := \int_{T_k}^{T_{k+1}} \expectation_{Z_D}\|\hat{s}\del{t,c_ty+\sigma_t Z_D}+Z_D/\sigma_t\|^2 dt.
    \]

       Theorem C.4 in \cite{oko2023}, states (Note that the original proof has an error which we fix in~\Cref{sec:ThC4Oko})  
        \£
        \label{eq:generalization_bound_c4}
        &\int_{T_k}^{T_{k+1}} \expectation \norm{\hat{s}_k(t,X_t)-s(t,X_t)}^2 dt\lesssim \inf_{\phi\in \cS_k} \int_{T_k}^{T_{k+1}} \expectation \norm{\phi(t,X_t)-s(t,X_t)}^2dt\nonumber
        \\
        & \quad + \frac{\sup_{y\in M}\sup_{\phi} \ell_y(\phi,T_k,T_{k+1})}{n}\del{ \log \mathcal{N}(\cL, \| \cdot \|_{L_\infty(M)}, \delta) + 1} + \delta.
        \£
       We have shown in~\Cref{sec:main_theorem_t_small} that the first term is of order $O( T_k^{-1} n^{-2\beta/({2\alpha+d})} (\log n)^{4\gamma \beta}+ n^{-2\alpha/(2\alpha+d)}(\log n)^{2\alpha+1})$. We now study the second term. 
      We follow \cite{oko2023}, but taking into account the dependency on ambient dimension $D$ and the special form~\eqref*{eq:definition of cS} of our estimators.  
     
        We first bound $\sup_{y\in M}\ell_y\del{\phi, T_k, T_{k+1}}$ uniformly over $\phi\in \cS_k$.
    To do this let us first define 
    \[
{
    \phi_e(t,x) := \frac{\sigma^2_t}{c_t}\phi_s(t,x) + \frac{x}{c_t}
    =
    \frac{\sum_i^N \tilde\rho_i(t,x)\phi_{w_i}(t,\widetilde{x}_{i,t}) \del{G_i +P_{\cH_i}\phi_{e_i}(t,\widetilde{x}_{i,t})}}{\sum_i^N \tilde\rho_i(t,x)\phi_{w_i}(t,\widetilde{x}_{i,t})},
}
    \] 
{
    the estimator of the conditional expectation $e(t,x)$ defined in~\eqref*{eq:intro_diffusion_e(t,x)}.  For fixed $y\in M$, set $X_t(y):=c_ty+\sigma_tZ_D$.  Then
}
    \begin{align}
    \label{eq:ell_y_gen_error}
    \ell_y(\phi, T_k, T_{k+1})
    &=
{
    \int_{T_k}^{T_{k+1}}\frac{c^2_t}{\sigma^4_t}\expectation_{Z_D}\|\phi_e(t,X_t(y)) - y\|^2 dt
}.
    \end{align}

    Applying Corollary~\ref{cor:class_level_tail_bound} with $A=4$ we get that with probability at least $1-2n^{-4}$
    \£
    \label{eq:nn_error_phi_e}
    \norm{\phi_e(t,X_t(y)) - y}
    \le
    244(\sigma_t/c_t)\log n.
    \£
{
    Since $\norm{\phi_{e_i}}\le C_2$ and by the choice of $C_2\ge 2\diam M$ we have that almost surely $\norm{\phi_{e}-y} \le C_2 + \diam M \le 2C_2$.
    Substituting into~\eqref{eq:ell_y_gen_error} gives, for all $y\in M$ and all $\phi \in \cS_k$,
}
    \£
    \label{eq:uniform_bound_on_ell_y_phi}
    {
    \ell_y(\phi, T_k, T_{k+1})
    \le    
    \int_{T_k}^{T_{k+1}}\del{244\log n \cdot \sigma_t^{-2} +C_2n^{-4}\cdot\sigma_t^{-4}}dt 
    \le C_{\ell}\log^3 n,
    }
    \£
    where we used that $\sigma_t^{-2} \le \frac{1}{2t} + 1$ and $T_k \lesssim \log n$. This bounds the first factor in the variance term in~\eqref{eq:generalization_bound_c4}.

{
    To complete the proof we thus only now need to control the covering number of $\cL = \curly{y\mapsto \ell_y(\phi,T_k,T_{k+1}); \phi \in \cS_k}$ and show that for $\delta \ge 1/n$
        \£ \label{covering:L}
        \log \mathcal{N}(\cL, \| \cdot \|_{L_\infty(M)}, \delta) \lesssim N(\polylog n) (\log \delta^{-1}).
        \£ 
    To do this we follow the proof of \cite{oko2023}, Lemma C.2.
    First note that using Lemma 3 of \cite{suzuki2018adaptivity}, 
    the covering number of $\Psi(L, W, S, B)$ is bounded by
    \begin{equation} \label{lemma:suzuki_lemma_3}
    \log \mathcal{N}(\delta, \Psi(L, W, S, B), \|\cdot\|_\infty) \le
    2SL \log (\delta^{-1} L(B \vee 1)(W + 1)).
    \end{equation}
We now explain why this allows us to bound the number of $\cS_k$ by $\norm{\norm{\phi}_2}_{L_\infty([-C,C]^{D+1})}$- balls, where $\norm{\norm{\phi}_2}_{L_\infty([-C,C]^{D+1})} = \sup_{(x,t)\in [-C,C]^{D+1}}\norm{\phi(t,x)}_2$ required to cover $\cS_k$.
 Indeed for all $k$ and $\phi^1, \phi^2\in \cS_k$, 
by definition of $\cS_k$ we have for $j =1,2$,
$$\norm{\phi^j_{e_i}} \le C_2, \quad
\norm{\frac{c_t\phi^j_{e_i}(t,\widetilde{x}_{i,t}) - \widetilde{x}_{i,t}}{\sigma_t}}_{\infty} \le \log n, \quad \text{and} \quad
\phi^j_{w_i}(t,\widetilde{x}_{i,t}) \ge n^{-2}.$$
Moreover, $\sigma^2_t > n^{-1}$, so that 
        \[  
        &\norm{\phi^1 - \phi^2} 
        \\
        &\quad = \frac{c_t}{\sigma^2_t}\norm{
        \frac{\sum_i^N \tilde\rho_i(t,x)\phi^1_{w_i}(t,\widetilde{x}_{i,t}) \del{G_i +P_{\cH_i}\phi^1_{e_i}(t,\widetilde{x}_{i,t})}}{\sum_i^N \tilde\rho_i(t,x)\phi^1_{w_i}(t,\widetilde{x}_{i,t})} 
        -
        \frac{\sum_i^N \tilde\rho_i(t,x)\phi^2_{w_i}(t,\widetilde{x}_{i,t}) \del{G_i +P_{\cH_i}\phi^2_{e_i}(t,\widetilde{x}_{i,t})}}{\sum_i^N \tilde\rho_i(t,x)\phi^2_{w_i}(t,\widetilde{x}_{i,t})}
}
        \\
        & \quad \le
        \sum_{i\in I_2} \sigma^{-2}_t \cdot 2C_2n^{2}\norm{\phi^1_{w_i} - \phi^2_{w_i}} + \norm{\phi^1_{e_i}(t,x) - \phi^2_{e_i}(t,x)} \\
        & \quad \le n^{3}C_2\del{\max_{i,t} \norm{\phi^1_{e_i} - \phi^2_{e_i}} + \max_{i,t} \norm{\phi^1_{w_i} - \phi^2_{w_i}}},
        \]
        which implies that for all $\delta >0$ and all $C>0$
       \begin{align*} %\label{eq:generalization_of_cS}
                \log \cN(\cS_k, \norm{\norm{\phi}_2}_{L_\infty([-C,C]^{D+1})}, \delta) &\leq \log \mathcal{N}(\delta n^{-3}/2C_2, \Psi(L, W, S, B), \|\cdot\|_\infty)  \nonumber \\
                &\lesssim 2n^{d/(2\alpha +d)}(\polylog n) \log( C\delta^{-1}).
                \end{align*} 
    To prove \eqref{covering:L}, we then bound for    $\norm{\norm{\phi^1-\phi^2}_2}_{L_\infty([-C_5\sqrt{\log n},C_5\sqrt{\log n}]^{D+1})} \le \delta'$ with $\phi^1,\phi^2 \in \cS_k$,
      \begin{align*}
        \abs{\ell_{y}(\phi^1) - \ell_{y}(\phi^2)} &= \abs{\int_{T_k}^{T_{k+1}} \del{\expectation\norm{\phi^1\del{t, c_ty+\sigma_t Z_D}+Z_D/\sigma_t}^2 - \expectation\norm{\phi^2\del{t, c_ty+\sigma_t Z_D}+Z_D/\sigma_t}^2}dt}\\
        & = \int_{T_k}^{T_{k+1}}\frac{c_t^2 }{\sigma_t^4}\expectation_{X_y(t)\sim \mathcal N(y, \sigma_t^2 Id_D)}[\|\phi_e^1(t, X_y(t)) - y\|^2 - \|\phi_e^2(t, X_y(t)) - y\|^2]dt
        \end{align*}
         where the last  equality comes from \eqref{eq:ell_y_gen_error} with  $\phi_e^j (t,x)=\sigma_t^2\phi^j(t,x)/c_t + x/c_t $. We bound, using \eqref{eq:nn_error_phi_e}
        \begin{align*}    
        &\|\phi_e^1(t, X_y(t)) - y\|^2 - \|\phi_e^2(t, X_y(t)) - y\|^2 
        \\
        &\quad \quad =\del{\|\phi_e^1(t, X_y(t)) - y\| - \|\phi_e^2(t, X_y(t)) - y\|}\del{\|\phi_e^1(t, X_y(t)) - y\| + \|\phi_e^2(t, X_y(t)) - y\|}
        \\
        &\quad \quad  \lesssim \|\phi_e^1(t, X_y(t)) - \phi_e^2(t, X_y(t))\| (\sigma_t/c_t) (\log n)^{3/2}
        \lesssim \frac{\sigma_t^3}{c_t^2} \|\phi^1(t, X_y(t)) - \phi^2(t, X_y(t))\|(\log n)^{3/2}.
        \end{align*}
        When 
        $X_y(t)\in [-C_5\sqrt{\log n},C_5\sqrt{\log n}]^{D+1} $ then  $$\|\phi^1(t, X_y(t)) - \phi^2(t, X_y(t))\| \leq \norm{\norm{\phi^1-\phi^2}_2}_{L_\infty([-C_5\sqrt{\log n},C_5\sqrt{\log n}]^{D+1})} \le \delta'.$$
        Also
         \begin{align*} 
         \expectation_{X_y(t)\sim \mathcal N(y, \sigma_t^2 \Id_D)}&\left[\ind_{X_y(t) \notin [-C_5\sqrt{\log n},C_5\sqrt{\log n}]^{D}}\|\phi_e^j(t, X_y(t)) - y\|^2 \right] 
         \\
         &\leq \sigma_t^2(\log n)^3\rP[\cN(y, \sigma_t^2 \Id_D) \notin [-C_5\sqrt{\log n},C_5\sqrt{\log n}]^{D}]
        \end{align*}
    Since for all $y \in M$, $\|y\|_\infty\leq C_{\log }$, for all $C>0$
    $$\rP[\mathcal N(y, \sigma_t^2 Id_D) \notin [-C_5\sqrt{\log n},C_5\sqrt{\log n}]^{D}]\leq  \rP\del{\norm{Z_D}_\infty \ge C_5\sqrt{\log n}/2 }\le n^{-C}$$
   since $\log D \lesssim \log n$ and by choosing $C_5 $ large enough.
   We finally have that 
    \begin{align*}
        \abs{\ell_{y}(\phi^1) - \ell_{y}(\phi^2)} &\lesssim 
        \polylog(n)\int_{T_k}^{T_{k+1}}\sigma_t^{-1}dt[\delta' +
        n^{-C}\polylog(n)] \lesssim \delta' \polylog(n)
        \end{align*}
        and choosing $\delta' \asymp \delta/\polylog(n)$ leads to $\abs{\ell_{y}(\phi^1) - \ell_{y}(\phi^2)} \leq \delta$, which in turns implies
    \eqref{covering:L}.    
}

\subsection{Correction of \cite{oko2023}, Theorem C.4}
\label{sec:ThC4Oko}
\begin{theorem} \label{Th:thC4Oko}
    For all $k \le K$ let $\hat{\phi}_k$ be a minimizer of the empirical risk, i.e.\
    $\hat{\phi}_k := \argmin_{\phi\in {\cS_k}} \sum_{y_i \in \cY} \ell_{y_i}(\phi, T_k,T_{k+1})$.
    Then for any $\delta_n > 0$ there is a constant $C_b > 0$ such that
        \£
        \label{ThC4Oko}
        \expectation_{\cY} 
            \int_{T_k}^{T_{k+1}} \expectation \norm{\hat{\phi}_k(t,X_t)-s(t,X_t)}^2 dt
         \\
         \leq  
         3 \inf_{\phi\in {\cS_k}}\int_{T_k}^{T_{k+1}} \expectation \norm{\phi(t,X_t)-s(t,X_t)}^2 dt
         \\
         +
         \frac{C_b\log^6 n}{n}\del{1\vee \log \cN\del{\cL, \norm{ \cdot }_{L_\infty(M)}, \delta_n}}
         +
         12\delta_n
        \£
    
\end{theorem}
\begin{remark}
    The constant $3$ in~\eqref{ThC4Oko} can be replaced with $1+\delta$ for any $\delta > 0$ at the expense of constants in front of the other two terms. 
\end{remark}
\begin{proof}

    We start by introducing some notation. We fix an index $k$ and suppress it in what follows.
    We recall that the de-noising score matching loss $\ell$ at $y\in M$ is defined as
    \[
    \ell_y(\phi, T_k,T_{k+1}) := \int_{T_k}^{T_{k+1}} \expectation \norm{\phi(t, c_t y + \sigma_t Z_D) + \sigma^{-1}_tZ_D}^2 dt.
    \]
    We write $\ell_\phi(y) := \ell_{y}(\phi, T_k, T_{k+1})$ to emphasize the dependence on $y\in M$.
    The score-matching risk $R(\phi,s)$~\eqref*{eq:the_score_matching_loss} is
    \[
    R(\phi, s) := \int_{T_k}^{T_{k+1}}\expectation \norm{\phi(t,X_t) - s(t,X_t)}^2dt.
    \]
    Note that by~\eqref*{eq:SML_to_DSML} 
    \[
    0\le R(\phi,s) :=   \expectation_{\mu} l_y(\phi) -  \expectation_{\mu} l_y(s).
    \]
    Finally, we introduce the empirical risk and its deviation from the true risk
    \[
    R_n(\phi) = \frac{ 1 }{n}\sum_{i=1}^n \ell_\phi(y_i) , \quad \Delta_n(\phi) = \mathbb E_{\mu} \ell_\phi(y) - R_n( \phi).
    \]
    Fix $\delta_n > 0$ and let $\curly{\phi_j}_{j\le \cN_n}$ be a minimal $\delta_n$-cover of ${\cL} = \{ \ell_\phi, \phi \in {\cS_k}\}$ with respect to $L_{\infty}(M)$,
    in other words $\cN_n := \cN\del{\cL, \norm{ \cdot }_{L_\infty(M)}, \delta_n}$.
    The case $\log \cN_n < 1$ is trivial, so we assume that $\log \cN_n\geq 1$.
    Slightly abusing notation for all $j\le \cN_n$ we write index $j$ instead of $\phi_j$, namely $f(j) = f_j := f(\phi_j)$, e.g. $\ell_j = \ell_{\phi_j}$.
    
    Since $\hat{\phi}$ is the minimizer of the empirical risk for any $\phi\in {\cS_k}$
    \begin{align}
    \label{eq:risk_decomposition}
    R(\hat{\phi}, s) 
    &= R_n(\hat \phi) + \Delta_n( \hat\phi) - \expectation_{\mu} \ell_{s}(y)
    \nonumber
    \\
    &\le R_n(\phi) - \expectation_\mu \ell_{s}(y) + \Delta_n( \hat\phi)
    \nonumber
    \\
    &= R(\phi, s) - [\Delta_n( \phi) - \Delta_n( s)] + [\Delta_n( \hat{\phi}) - \Delta_n(s)].
    \end{align}
    Let $\phi^* := \argmin_{\phi\in{\cS_k}} R(\phi, s)$,
     since $\curly{\phi_j}_{j\le \cN_n}$ form $\delta_n$ dense net there are $\hat{j}$ and $j^*$ 
     such that 
      \[\| \ell_{\hat \phi} - \ell_{\hat j}\|_\infty \leq \delta_n, \quad \norm{\ell_{ \phi^*} - \ell_{j^*}}_\infty \leq \delta_n. 
      \]
    Substituting $\phi^*$ into~\eqref{eq:risk_decomposition} we conclude
    \begin{align}
        R(\hat{\phi}, s) &\leq \inf_{\phi\in {\cS_k}} R(\phi, s)
        - [\Delta_n( j^*) - \Delta_n( s)]
     +[\Delta_n( \hat{j}) - \Delta_n( s)] + 5\delta_n.
     \label{eq:riskbound_deltas}
    \end{align}    
{
    So, it is enough to control $\Delta_n( \hat{j}) - \Delta_n(s)$ and
    $\Delta_n( j^*) - \Delta_n(s)$.  We condition on the training sample
    $\cY$ and introduce an independent ghost sample
    $\cX=\{x_1,\ldots,x_n\}\sim\mu^{\otimes n}$.  For each
    $j\le\cN_n$,
}
    \[
    {
    \Delta_n(j) - \Delta_n(s)
    = \expectation_{\cX}\frac{1}{n}\sum_{i=1}^n
    [(\ell_j(x_i) - \ell_j(y_i)) - (\ell_{s}(x_i) - \ell_{s}(y_i))].
    }
    \]
    Set
    $g_j(x,y) := (\ell_j(x) - \ell_j(y)) - (\ell_{s}(x) - \ell_{s}(y))$.
    Then $\expectation_{x,y\sim \mu} g_j(x,y)=0$, and the variance can be bounded by Jensen's inequality as
     \begin{equation*}
         \begin{split}
        &\Var_{x,y\sim \mu} g_j({x, y})
        = 2 \Var_{y\sim \mu}\square{\ell_j(y)-\ell_s(y)}
        \\
        &\quad \leq
        2 \mathbb E_{y\sim \mu} \Bigg[ \bigg( \int_{T_k}^{T_{k+1}} \expectation_{Z_D}
        \norm{\phi_j(t,c_ty+\sigma_tZ_D) +\sigma^{-1}_tZ_D}^2
        -\norm{s(t,c_ty+\sigma_tZ_D) + \sigma^{-1}_t Z_D}^2dt \bigg)^2 \Bigg]
        \\
         & \quad \leq 2 (T_{k+1}-T_k)
         \int_{T_k}^{T_{k+1}} \expectation_{y, Z_D}\square{
         \norm{\phi_j(t,c_ty+\sigma_tZ_D) +\sigma^{-1}_tZ_D}^2
         -\norm{s(t,c_ty+\sigma_tZ_D) + \sigma^{-1}_t Z_D}^2}^2dt.
         \end{split}
     \end{equation*}
    Writing
    \begin{align*}
    &\norm{\phi_j(t,c_ty+\sigma_tZ_D) +\sigma^{-1}_tZ_D}^2
    - \norm{s(t,c_ty+\sigma_tZ_D) + \sigma^{-1}_t Z_D}^2
     \\
     &=\norm{\phi_j(t,c_ty+\sigma_tZ_D) - s(t,c_ty+\sigma_tZ_D)}^2
    \\
     &\quad + 2\innerproduct{\phi_j(t,c_ty+\sigma_tZ_D) - s(t,c_ty+\sigma_tZ_D)}{s(t,c_ty+\sigma_tZ_D) + \sigma^{-1}_t Z_D},
    \end{align*}
    and recalling that $T_{k+1}=2T_k$ we obtain the following bound
    \begin{align*}
    &\Var_{x,y\sim \mu} g_j(x,y)
    \leq 8 T_k
         \int_{T_k}^{T_{k+1}} \expectation_{y, Z_D} \norm{\phi_j(t,c_ty+\sigma_tZ_D) - s(t,c_ty+\sigma_tZ_D)}^4dt
    \\
    &\quad +
    32T_k
         \int_{T_k}^{T_{k+1}} \expectation_{y, Z_D}
         \norm{\phi_j(t,c_ty+\sigma_tZ_D) - s(t,c_ty+\sigma_tZ_D)}^2
         \norm{s(t,c_ty+\sigma_tZ_D) + \sigma^{-1}_t Z_D}^2dt.
    \end{align*}
    \begin{itemize}
        \item 
    By Tweedie's formula $s(t,x) =  \frac{c_t\expectation [X_0|X_t=x]-X_t}{\sigma^2_t}$, and since $\sigma^2_t \ge \sigma^2_{T_{k}} \ge \min(1, T_{k})/2$, a.s.
    \[
    \norm{s(t,x)+\sigma^{-1}_t Z_D}_\infty = \frac{c_t}{\sigma^2_t}\norm{\expectation[X_0|X_t]-X_0}_\infty \le \frac{c_t}{\sigma^2_t} \le 2\max\del{1, T^{-1}_{k}},
    \]
    while by Theorem~\ref*{thm:concentration_around_X0}, since $T_k\ge n^{-2}$ with probability at least $1-n^{-4}$ for a constant $C_s$ that depends only on $C_{\log}$ and $d$
     \[
     \norm{s(t,c_t y + \sigma_t Z_D) + \sigma^{-1}_t Z_D} \le \sigma^{-1}_t\cdot 20\sqrt{80d\del{C_{\log} +  \log n}} \le C_{s}\max\del{1, T^{-1/2}_{k}}\sqrt{\log n}.
     \]
    \item 
    The construction~\eqref*{eq:definition of cS} of $\phi\in \cS_k$ mimics Tweedie's representation of the score. Any $\phi\in \cS_k$ can be represented as $\phi(t,x) = \frac{c_t}{\sigma^2_t}\phi_e(t,x) - \frac{x}{\sigma^2_t}$ where $\norm{\phi_e(t,x)}_\infty \le 1$. Moreover, 
     as we show in~\eqref{eq:nn_error_phi_e}, {for each fixed $y\in M$ this holds with probability at least $1-n^{-4}$ w.r.t.\ $Z_D$:}
     \begin{align*}
     &\norm{\phi(t,c_t y + \sigma_t Z_D) + \sigma^{-1}_t Z_D}
     = \frac{c_t}{\sigma^2_t}\norm{\phi_e(t,c_t y + \sigma_t Z_D) - y}
     \le 2\cdot244\sigma^{-1}_t\log n^{3/2}
     \\
     &\le 1000\max\del{1, T^{-1/2}_{k}}\log n^{3/2}.
     \end{align*}
     \item 
    Summarizing, with probability at least $1-2n^{-4}$
    \[
    \norm{s(t,c_t y + \sigma_t Z_D) - \phi(t,c_t y + \sigma_t Z_D)} \le (1000 + C_{s})\max\del{1, T^{-1/2}_{k}}\log^{3/2} n.
    \]
    This implies that for $\log n \gtrsim T_k \gtrsim n^{-1}$
    \begin{align*}
    &8T_k
         \int_{T_k}^{T_{k+1}} \expectation_{y, Z_D} \norm{\phi_j(t,c_ty+\sigma_tZ_D) - s(t,c_ty+\sigma_tZ_D)}^4dt
    \\
    &\lesssim
    n^{-4}T_k^2\max\del{1, T^{-4}_k}
    +
    T_k\max\del{1, T^{-1}_{k}} \log^{3} n \int_{T_k}^{T_{k+1}}\expectation_{y, Z_D}\norm{s(t,c_t y + \sigma_t Z_D) - \phi_j(t,c_t y + \sigma_t Z_D)}^2dt 
    \\
    &\lesssim
    n^{-2}
    +
    \log^{4} n \int_{T_k}^{T_{k+1}}\expectation_{y, Z_D}\norm{s(t,c_t y + \sigma_t Z_D) - \phi_j(t,c_t y + \sigma_t Z_D)}^2dt = n^{-2} +  \log^4 n R(\phi_j,s).
    \end{align*}
    and similarly
    \begin{align*}
    &16 (T_{k+1}-T_k)
         \int_{T_k}^{T_{k+1}} \expectation_{y, Z_D}
         \norm{\phi_j(t,c_ty+\sigma_tZ_D) - s(t,c_ty+\sigma_tZ_D)}^2
         \norm{s(t,c_ty+\sigma_tZ_D) + \sigma^{-1}_t Z_D}^2 dt
    \\
    &\lesssim 
    n^{-2}
    +
    \log^{4} n \int_{T_k}^{T_{k+1}}\expectation_{y, Z_D}\norm{s(t,c_t y + \sigma_t Z_D) - \phi_j(t,c_t y + \sigma_t Z_D)}^2dt = n^{-2} +  \log^4 n R(\phi_j,s).
    \end{align*}
    \end{itemize}
    Therefore, there is a constant $C_{g}$ that depends only on $C_{\log}$ and $d$ such that
     \[
     \Var_{x,y\sim \mu} g_j(x,y) \le C_{g}\del{n^{-2} +  \log^4 n \cdot R(\phi_j,s)}.
     \]
    Finally, we compute an a.s.\ bound on $g_j(x,y)$, combining~\eqref{eq:uniform_bound_on_ell_y_phi}~and~\eqref*{eq:concentration:of_ell_s}
    there is a constant $C_{\sup}$ that depends only on $C_{\log}$ and $d$
    \[
    \abs{g_j(x,y)}  \le 2\sup_{y\in M} \abs{{\ell_s(y)}}  +  2\sup_{\phi\in \cS_k}\sup_{y\in M}\abs{\ell_j(y)}
    \lesssim \log^3 n + \log^2 n \le C_{\sup} \log^3 n.
    \]
    \begin{remark}
        This argument also works in the original setting~\cite[Theorem C.4]{oko2023}, in this case the inequalities have form 
        \begin{align*}
        &\norm{s(t,c_t y + \sigma_t Z_D) + \sigma^{-1}_t Z_D}
        \lesssim 2\max\del{1,T^{-1/2}_k} \sqrt{D\log n},
        \\
        &\norm{\phi(t,c_t y + \sigma_t Z_D) + \sigma^{-1}_t Z_D}
        \lesssim 2\max\del{1,T^{-1/2}_k} \sqrt{D\log n}.
        \end{align*}
    \end{remark}
    {Under $\rP_{\cY,\cX}$ the pairs $(x_i,y_i)$ are independent with common law $\mu\otimes\mu$.  }Then, by Bernstein's inequality, there is a positive constant $C_b < 1/5$ depending only on $d, C_{\log}$
    \[
    {\rP_{\cY,\cX}\del{ \abs{\sum_i g_j(x_i,y_i)} > t}}
    \le  2\exp\del{ - C_b\frac{t^2}{n^{-1} +  n \log^4 n \cdot R(\phi_j,s) + t \log^3 n}}.
    \]
    So, for $r_j = \max\del{\sqrt{n^{-1}\log \cN_n},\sqrt{R(\phi_j,s)}}$%, since $r^{-2} = \min\del{n/(\log n\log \cN_n),R^{-1}(\phi_j,s)}$ 
    \begin{align*}
    &
    {\rP_{\cY,\cX}\del{ \abs{\frac{\sum_i g_j(x_i,y_i)}{r_j}} > t}}
    \\
    &\le  2\exp\del{ - C_b\frac{t^2}{r_j^{-2}n^{-1} +  r_j^{-2}n \log^4 n \cdot R(\phi_j,s) + r_j^{-1}t\cdot \log^3 n}}
    \\
    &\le  2\exp\del{ - C_b\frac{t^2}{2n \log^4 n + (t/r_j)\cdot \log^3 n}} \le 2\exp\del{ - C_b\frac{t^2}{\log^4 n\cdot \del{2n + t/r_j}}}.
    \end{align*}
    For $t \ge 2\log n\sqrt{n\log \cN_n}$, since $r_j \ge \sqrt{n^{-1}\log \cN_n}$
    we have
    \[
    -t = -\frac{t}{2} -\frac{t}{2} \le -\log n\sqrt{n\log \cN_n} - \frac{t}{2}\frac{\sqrt{n^{-1}\log\cN_n}}{r_j} = -\frac{1}{2}\log n\sqrt{\frac{\log \cN_n}{n}}\del{2n + t/r_j}
    \]
    implying that
    \[
    -\frac{t^2}{ (2n+t/r_j)\log^4 n} \le -\frac{t\sqrt{\log\cN_n}}{2\sqrt{n}\log^3 n}.
    \]
    Therefore,  if we introduce
    \[
    T_n(j) = {\frac{\sum_i g_j(x_i,y_i)}{r_j}},
    \]
    substituting this bound we get that for $t\ge 2\log n\sqrt{n\log \cN_n}$
    \begin{align*}
    &
    {\rP_{\cY,\cX}\del{\abs{\sup_j T_n(j)} \ge t}}
    \\
    &\le 2\cN_n\sup_j\exp\del{ - C_b\frac{t^2}{2n \log^4 n + (t/r_j)\cdot \log^3 n}}
    \le
    2\cN_n \exp\del{ - \frac{C_b}{2}\frac{t \sqrt{\log \cN_n}}{\sqrt{n}\cdot \log^3 n}}.
    \end{align*}
    We are now ready to bound the expectation. Taking $A = 2C^{-1}_b\log^3 n \sqrt{n \log \cN_n} \ge 2\log n\sqrt{n\log \cN_n}$ for $n$ large enough
    \[
    \expectation \abs{\sup_j T_n(j)} = \int_0^{\infty} \mathbb P\del{\abs{\sup_j T_n(j)} \ge t}dt
    \le A + \int_{A}^\infty 2\cN_n \exp\del{ - \frac{C}{2}\frac{t \sqrt{\log \cN_n}}{\sqrt{n}\cdot \log^3 n}}dt 
    \\
    = A +2\cN_n\exp\del{ - \frac{C_b}{2}\frac{A \sqrt{\log \cN_n}}{\sqrt{n}\cdot \log^3 n}}\frac{2}{C_b}\frac{\sqrt{n}\cdot \log^3 n}{\sqrt{\log \cN_n}}
    \le 4C^{-1}_b\log^3 n \sqrt{n \log \cN_n}.
    \]
    Next we bound $\mathbb{E}|\sup_j T_n^2(j)|$ using a similar approach. We have
    \[
    \expectation \abs{\sup_j T^2_n(j)} = \int_0^{\infty} \mathbb P\del{\abs{\sup_j T_n(j)} \ge \sqrt{t}}dt \le A + \int_{A}^\infty  2\cN_n\exp\del{ - \frac{C_b}{2}\frac{\sqrt{t} \sqrt{\log \cN_n}}{\sqrt{n}\cdot \log^3 n}}dt 
    \\
    = A +8 \cN_n\exp\del{-\frac{C_b}{2} \frac{\sqrt{A} \sqrt{\log \mathcal{N}_n}}{\sqrt{n} \log^3 n}} \frac{n (\log^6 n)}{C_b^2\log \mathcal{N}_n} \del{ \frac{C_b}{2} \frac{\sqrt{A} \sqrt{\log \mathcal{N}_n}}{\sqrt{n} \log^3 n} + 1}
    \]
    and choosing $A = 4C^{-2}_b \del{n \log \cN_n}\log^6 n \ge \del{2\log n\sqrt{n\log \cN_n}}^2$
    we get 
    \[
    \expectation \abs{\sup_j T^2_n(j)} 
    \lesssim 4C^{-2}_b \del{n \log \cN_n}\log^6 n + 
    8 \frac{n (\log^6 n)}{C_b^2\log \mathcal{N}_n} \del{\log \cN_n + 1} \le 20 C^{-2}_b n\log \cN_n \log^6 n.
    \]
    
    We are finally, ready to bound $\Delta_n( j) - \Delta_n( s)$ after noting that
    $
    r_j \le \sqrt{(\log \cN_n)/n} + \sqrt{R(\phi_j,s)} 
    $.
    Let $J = J(\cY)$ be a random index potentially depending on the sample $\cY$. 
    Then notice that     
\begin{align*} 
     \mathbb E[\Delta_n( J) - \Delta_n( s) ] &= \frac{1}{n} \mathbb E[r_{J}T_n(J)]\leq  \sqrt{\frac{\log \cN_n}{n}}\frac{\mathbb E[\max_j|T_n( j)|]}{n}  +  \frac{\expectation\left(
        \sqrt{R(\phi_J,s)} \max_j|T_n( j)|
        \right) }{n}\nonumber 
     \\ 
     &
     \leq \frac{\sqrt{\log \cN_n}}{n^{3/2}} \mathbb E[ \max_j|T_n( j)|] +  \frac{1}{n}\square{\expectation R(\phi_J,s)}^{1/2}\square{\expectation \max_j|T_n( j)|^2}^{1/2}
     \\
     & 
     \leq
     5C^{-1}_b\frac{\sqrt{(\log \cN_n)/n}}{n}\sqrt{n \log \cN_n}\log^3 n + \square{\expectation R(\phi_J,s)}^{1/2} 5C^{-1}_b\sqrt{\frac{\log \cN_n}{n}} \log^3 n.
     \\
     &
     \le
     5C^{-1}_b\frac{\log \cN_n}{n}\log^3 n + 5C^{-1}_b\square{\expectation R(\phi_J,s)}^{1/2} \sqrt{\frac{\log \cN_n}{n}} \log^3 n
     \intertext{and using Young's inequality}
     &
     \le
     5C^{-1}_b\frac{\log \cN_n}{n}\log^3 n + \frac{1}{2}\expectation R(\phi_J,s) + \frac{25C^{-2}_b}{2} \frac{\log \cN_n}{n} \log^6 n 
     \\
     & 
     \le \frac{1}{2}\expectation R(\phi_J,s) + 50C^{-2}_b \frac{\log \cN_n}{n} \log^6 n.
     \end{align*}

Substituting $J = \hat{j}$ and recalling that by construction $\norm{\ell_{\phi_{\hat{j}}}-\ell_{\hat{\phi}}}_\infty \le \delta_n$,
\begin{align*}
\mathbb E[\Delta_n( \hat{j}) - \Delta_n( s) ]
&\le
\frac{1}{2}\expectation R(\phi_{\hat{j}},s) + 50C^{-2}_b \frac{\log \cN_n}{n} \log^6 n
\\
&\le
\frac{1}{2}\expectation R(\hat{\phi},s) + \frac{1}{2}\delta_n + 50C^{-2}_b \frac{\log \cN_n}{n} \log^6 n.
\end{align*}
While substituting $J = j^*$ and recalling that $\norm{\ell_{\phi_{j^*}}-\ell_{\phi^*}}_\infty \le \delta_n$ and $\phi^* = \argmin_{\phi\in \cS_k} R(\phi,s)$,
\begin{align*}
\mathbb E[\Delta_n( j^*) - \Delta_n( s) ]
&\le
\frac{1}{2}\expectation R(\phi_{j^*},s) + 50C^{-2}_b \frac{\log \cN_n}{n} \log^6 n
\\
&\le
\frac{1}{2}\min_{\phi\in \cS_k} R(\phi,s) + \frac{1}{2}\delta_n + 50C^{-2}_b \frac{\log \cN_n}{n} \log^6 n.
\end{align*}
Finally, the above combined with \eqref{eq:riskbound_deltas} gives us the desired bound
\begin{align*}
\expectation R(\hat{\phi}, s)
&\leq \inf_{\phi\in \cS_k} R(\phi, s)
        - \expectation[\Delta_n( j^*) - \Delta_n( s)]
     +\expectation[\Delta_n( \hat{j}) - \Delta_n( s)] + 5\delta_n
\\
&\le
\frac{1}{2}\expectation R(\hat{\phi}, s) + \frac{3}{2}\min_{\phi\in \cS_k} R(\phi,s) + 6\delta_n + 50C^{-2}_b \frac{\log \cN_n}{n} \log^6 n.
\end{align*}
\end{proof}

{\color{red}

}

\section{Proofs for Section 6} %\ref{sec:W_main}
\label{apdx:W_new}

This appendix contains the details of the  proofs of Proposition 6.2, %~\ref*{prop:discretization_setup} 
Theorem 6.1 and Corollary~\ref*{cor:main_W_result}
omitted in Section 6.%\ref*{sec:W_main}.
Recall that in this part $L$ denotes the number of splits in the discretization and not the number of layers in the  neural network. 

\begin{proof}[Proof of Proposition 6.2] %~\ref*{prop:discretization_setup}
\,\newline
\begin{itemize}
    \item[(i)]
    We need to show that 
    \[
\tau_{(i)}-\tau_{(i+1)}
\le \kappa \min(1,\tau_{(i+1)}),
\qquad
1\le j\le M.
\]
If $\tau_{(i)}$ and $\tau_{(i+1)}$ belong to the same block
$[T_{K-k},T_{K-k+1}]$, then
\[
\tau_{(i)}-\tau_{(i+1)}=\Delta_k
=\frac{\kappa}{4}\min(1,T_{k+1})
\le \kappa \min(1,\tau_{(i+1)}).
\]
Otherwise, for some $k\le K$ we have that $\tau_{(i)} = T_k-(L_k-1) \Delta_k - \Delta_k u_k$ is the last point of a subpartition of $[T_{k+1},T_{k}]$, while $\tau_{(i+1)} = T_{k+1}-\Delta_{k+1} u_{k+1}$ is the first element of subpartition of the next block $[T_{k+2}, T_{k+1}]$. So, by the definition of $L_k$
\[
\norm{\tau_{(i)} - \tau_{(i+1)}} \le \norm{\tau_{(i)} - T_{k+1}} + \norm{\tau_{(i+1)} - T_{k+1}} \le \Delta_k + \Delta_{k+1} 
\\=(\kappa/4)\del{\min(1,T_{k+1}) + \min(1,T_{k+2})} \le (\kappa/4)\min(1+1,T_{k+1} + T_{k+2}) 
\\
= (\kappa/4)\min(2,3T_{k+2}) \le  \kappa \min(1,T_{k+2}) \le \kappa \min (1,\tau_{i+1}).
\]
\item[(ii)]
We need to show
\[
\tau_{(1)} > \overline{T} - \kappa, \quad \tau_{(L)} < (1+\kappa) T, \quad \tau_{i_k}/\tau_{i_{k+1}} < 4;
\]
First inequality: since $T_0 := \overline{T}$
\[
 \overline{T} -\tau_{(1)} = u_0 \Delta_0 \le \Delta_0 = (\kappa/4)\min(1,\overline{T}) \le \kappa.
\]
Second inequality: Since $T_{K+1} := \underline{T}$
\[
\tau_{(L)}-\underline{T} \le \Delta_K \le (\kappa/4)\min(1,T_{K+1}) \le \kappa \underline{T}.
\]
Third inequality: by the definition of $i_k$ we have $\tau_{i_k} \in [T_{k+1}, T_{k+1}+\Delta_k]$, so using that $\kappa < 1$
\[
\tau_{i_k} \le T_{k+1}+\Delta_k = 2T_{k+2} + (\kappa/4) \min(1, 2T_{k+2}) \le 4T_{k+2} \le 4\tau_{i_{k+1}}.
\]
\item[(iii)] We need to show
\[
L\le (K+1)+\frac{4}{\kappa}
\left(\overline T+\left\lceil \log_2 \underline T^{-1}\right\rceil+1\right) = O\del{\kappa^{-1}(\log {\underline T^{-1}} + \overline 
T)}
\]
Consider an interval $[T_{k+1},T_k]$. Number of points on the interval $[T_k,T_{k+1}]$ is $L_k$ and satisfy
\[
T_k - \Delta_k u_k - (L_k-1) \Delta_k \ge T_{k+1}, 
\]
so
\[
L_k \le \frac{T_{k+1} - \Delta_k u_k + \Delta_k}{\Delta_k} \le 1 + \frac{4}{\kappa}\frac{T_{k+1}}{\min(1,T_{k+1})}.
\]
Let $k_1$ be the smallest index satisfying $T_{k_1} > 1$.
Then
\[
 L= \sum_{k=1}^K L_k = \sum_{k=1}^{k_1} L_k + \sum_{k=k_1+1}^K L_k \le K + \sum_{k=1}^{k_1} \frac{4}{\kappa} T_{k+1} + \sum_{k=k_1+1}^{K} \frac{4}{\kappa} = K + \frac{4}{\kappa}\del{K-k_1 + \sum_{k=1}^{k_1} T_{k+1}},
\]
since $T_{k}/T_{k+1} = 2$ and $T_{k_1+1} \le 1$ we have $K-k_1 \le \lceil\log_2 T^{-1}_{K+1}\rceil = \lceil\log_2 \underline{T}^{-1}\rceil$, at the same time $\sum_{k=1}^{k_1} T_{k+1}= \sum 2^{-(k+1)}\overline{T} \le \overline{T}$. Substituting we conclude
\[
L \le K + \frac{4}{\kappa}\del{\lceil\log_2 \underline{T}^{-1}\rceil + \overline{T}}.
\]
Noting that $K = \log_2(\overline{T}-\underline{T})$ we conclude
\[
L = O\del{\kappa^{-1}(\log {\underline T^{-1}} + \overline 
T)}
\]
\item[(iv)]
Let $g(t) = \sigma^2_i\expectation \norm{\hat{s}(t, X_t) -s(t, X_t)}^2$
Then, recalling that $\gamma_i = \tau_{(i)}-\tau_{(i+1)}$ we need to show that
\[
\expectation_{\tau} \sum_{i=1}^{M-1} \gamma_i g(\tau_{(i)}) \le \int_{\underline{T}}^{\overline{T}} g(t)\, dt.
\]
If $\tau_{(i+1)} = \tau_{k, \ell}$ for $\ell > 1$ then $\gamma_i = \Delta_k$, while if $\tau_{(i+1)} = \tau_{k, 1}$ then ${\gamma_i \le \Delta_k + \Delta_{k-1} \le 2\Delta_{k-1}}$, so
\[
\sum \gamma_i g(\tau_{(i)}) \le 2\sum_{k=1}^K \sum_{\ell=1}^{L_k} \Delta_k g(\tau_{k,\ell}) = 2\sum_{k=1}^K \sum_{\ell=1}^{L_k} \Delta_k g(T_k - u_k \Delta_k - (\ell-1) \Delta_k) 
\]

Taking expectation with respect to the shifts and using Tonelli's theorem,
\begin{align*}
\expectation_{u_k} &\sum_{\ell=1}^{L_k}\Delta_k g(\tau_{k, \ell})
=
\sum_{\ell\ge 1}\int_0^1
\Delta_k g(T_k - u_k \Delta_k - (\ell-1) \Delta_k)
\ind\{T_k - u_k \Delta_k - (\ell-1) \Delta_k \ge T_{k+1}\}\,du
\\
&=
\sum_{\ell\ge 1}\int_{(\ell-1)\Delta_k}^{\ell\Delta_k}
g(T_k+s)\ind\{T_k - s\ge T_{k+1}\}\,ds =
\int_0^{T_{k}-T_{k+1}} g(T_k-s)\,ds
=
\int_{T_{k+1}}^{T_{k}} g(t)\,dt.
\end{align*}
By summing over $k$ we get the desired result.
\end{itemize}
\end{proof}
\subsection{Added details in the proof  of Theorem 6.1} %~\ref*{thm: W_1 convergence theorem}
\,\newline
    \textbf{Initialization Error.} We note that the processes $\bar{Y}$ and $\bar{Y}^{(1)}$ differ only due to their initialization, so by applying~\cite[{Proposition 9}]{benton2024nearly}, Pinsker's and the data-processing inequalities
    \begin{multline*}
    W_1\del{\bar{Y}_{\tau_{(1)}}, \bar{Y}^{(1)}_{\tau_{(1)}} }\le 
    2C_1\TV\del{\bar{Y}_{\tau_{(1)}}, \bar{Y}^{(1)}_{\tau_{(1)}} }
    \le 2C_1\sqrt{\KL\del{\bar{Y}_{\tau_{(1)}} \,||\, \bar{Y}^{(1)}_{\tau_{(1)}}}}
    \\
    \le 2C_1\sqrt{\KL\del{\bar{Y}_{0} \,||\, \bar{Y}^{(1)}_{0}}} \le 2C_1\sqrt{\KL(p_{\tau_{(1)}} \,||\,\cN\del{0,\Id})}\le 2C_1\sqrt{D}e^{-\tau_{(1)}} \le ,
    \end{multline*}
    where we used that % the last step always has form 
    $\bar{Y}_{t_{L+1}} = \bar{e}(\tau_{(L)},\bar{Y}_{t_1:t_L} )$, $\bar{Y}^{(1)}_{t_{L+1}} = \bar{e}(\tau_{(L)},\bar{Y}^{(1)}_{t_1:t_L} )$ and $\norm{\bar{e}(t, \cdot )}_\infty \le C_1$ for all $t$.

    Since $\tau_{(1)} \ge \overline{T}-1$ a.s. in $u$
    \£
    \label{eq:W_1_result_1}
    \expectation_u W_1\del{\bar{Y}_{\tau_{(1)}}, \bar{Y}^{(1)}_{\tau_{(1)}} } \lesssim \sqrt{D}e^{-\overline{T}}. 
    \£
    \textbf{Error in the last step.} Next we bound $W_1(\bar{Y}^{(K+1)}_{\tau_{(1)}}, X_0)$. We note that $\bar{Y}^{(K+1)}_{t_{L+1}} \stackrel{dist.}{=} X_{\tau_{(L)}}$ as it follows the backward dynamic~\eqref*{eq:exact_backward_dynamic} on the interval $[0,\tau_{(1)}-\tau(L)]$ and the difference is only in the last step. In other words, while the true process sends $X_{\tau_{(L)}}\rightarrow X_0$, the discretized one sends $X_{\tau_{(L)}}\rightarrow \bar{e}(\tau_{(L)}, X_{\tau_{(L)}}) = \bar Y_{\tau_{(1)}}^{(K+1)}$. 
    
    Using the fact that with probability at least $1-\frac{\underline{T}\delta}{L^2}$ we have that  ${\bar{e}(\tau_{(L)}, X_{\tau_{(L)}})  = \hat{e}(\tau_{(L)}, X_{\tau_{(L)}})}$ and $\norm{\hat{e}(\tau_{(L)}, X_{\tau_{(L)}}) - X_0} \le R_{\tau_{(L)}}\del{\frac{\underline{T}\delta}{L^2}}$, we conclude that with probability at least $1-2\frac{\underline{T}\delta}{M^2}$
    \[
    \norm{\bar{e}(\tau_{(L)}, X_{\tau_{(L)}}) - X_0} \le R_{\tau_{(L)}}\del{\frac{\underline{T}\delta}{2L^2}}.
    \]
    Finally, using that $\norm{\bar{e}} \le C_1$ almost surely we bound
    \[
    W_1(\bar{Y}^{(K+1)}_{\tau_{(1)}}, X_0) 
    &
    \le \expectation \norm{X_0-\bar{e}(\tau_{(L)}, X_{\tau_{(L)}})} 
    \le R_{\tau_{(L)}}\del{\frac{\underline{T}\delta}{L^2}} + 2C_1\underline{T}\delta/L^2
    \\
    &
    \le (\sigma_{\tau_{(L)}}/c_{\tau_{(L)}})\log \underline{T}^{-1}\sqrt{C_W  + \log \frac{2L^2}{\underline{T}\delta}} + 2\frac{C_1\underline{T}\delta}{L^2}.
    \]
    Since $\tau_{(L)} \le 2\underline{T}$ a.s. in $u$
    \£
    \label{eq:W1_result_2}
       \expectation_u W_1(\bar{Y}^{(K+1)}_{\tau_{(1)}}, X_0) \le \sqrt{\underline{T}}\log \underline{T}^{-1}\sqrt{C_W  + \log \frac{2L^2}{\underline{T}\delta}} + 2\frac{\underline{T}\delta}{L^2} 
    \£
   \textbf{Discretization Error}
    Finally, we bound $W_1(\bar{Y}^{(k)}_{\tau_{(1)}}, \bar{Y}^{(k+1)}_{\tau_{(1)}})$. 

By taking the same Brownian motion on the interval $[0,t_{i_k}]$,  by definition (6.12) %~\eqref{eq: Y^l definition} 
the processes $\bar Y^{(k)}_{t}$ and $\bar Y^{(k+1)}_{t}$ coincide on this interval and in particular $ \bar Y^{(k)}_{t_{i_k}} = \bar Y^{(k+1)}_{t_{i_{k}}} \stackrel{dist.}{=}X_{\tau_{(i_k)}}$. 

Moreover, by construction (see definition (6.12)) %~\eqref{eq: Y^l definition} 
$\bar{Y}^{(k)}_{\tau_{(1)}} = \bar{e}(\tau_{(L)},\bar{Y}^{(k)}_{t_1:t_L})$ and $\bar{Y}^{(k+1)}_{\tau_{(1)}} = \bar{e}(\tau_{(L)}, \bar{Y}^{(k+1)}_{t_1:t_L})$, so
\[
\norm{\bar{Y}^{(k)}_{\tau_{(1)}} - \bar{e}(\tau_{(i_k)}, \bar Y^{(k)}_{t_{i_k}})} 
= \norm{\bar{e}(\tau_{(L)}, \bar{Y}^{(k)}_{t_1:t_L}) - \bar{e}(\tau_{(i_k)}, \bar Y^{(k)}_{t_{i_k}})}
\leq  
2R_{\tau_{(i_k)}}\del{\frac{\underline{T}\delta}{2L^2}}
\\
\norm{\bar{Y}^{(k+1)}_{\tau_{(1)}} - \bar{e}(\tau_{(i_k)}, \bar Y^{(k+1)}_{t_{i_k}})} 
= \norm{\bar{e}(\tau_{(L)}, \bar{Y}^{(k+1)}_{t_1:t_L}) - \bar{e}(\tau_{(i_k)}, \bar Y^{(k+1)}_{t_{i_k}})}
\leq  
2R_{\tau_{(i_k)}}\del{\frac{\underline{T}\delta}{2L^2}}
\]
Therefore the conditional distributions of $\bar Y^{(k)}_{\tau_{(1)}}$ and $\bar Y^{(k+1)}_{\tau_{(1)}}$ given $Y_{[0,t_{i_k}]}$ have a support with diameter bounded by $2R_{\tau_{(i_k)}}\del{\frac{\underline{T}\delta}{2L^2}}$. 

Then writing $W_1(\bar Y^{(k)}_{\tau_{(1)}}, \bar{Y}^{(k+1)}_{\tau_{(1)}}|Y_{[0,t_{i_k}]})$ to denote the Wasserstein distance between the conditional distributions of $\bar Y^{(k)}_{\tau_{(1)}}$ and  $ \bar{Y}^{(k+1)}_{\tau_{(1)}}$ given the past trajectory $Y_{[0,t_{i_k}]}$, we have 
 \[
W_1(\bar Y^{(k)}_{\tau_{(1)}}, \bar{Y}^{(k+1)}_{\tau_{(1)}}) &\le 
\mathbb E\left(W_1(\bar Y^{(k)}_{\tau_{(1)}}, \bar{Y}^{(k+1)}_{\tau_{(1)}}|Y_{[0,t_{i_k}]})\right) \\
& \le 4 \left(R_{\tau_{(i_k)}}\del{\frac{\underline{T}\delta}{2L^2}}
\expectation\TV\del{\bar{Y}^{(k)}_{\tau_{(1)}}, \bar Y^{(k+1)}_{\tau_{(1)}}\big | Y_{[0,t_{i_k}]}}\right)
\\
&\le 
4\left(R_{\tau_{(i_k)}}\del{\frac{\underline{T}\delta}{2L^2}}
\TV\del{\bar{Y}^{(k)}_{[0, t_L]}, \bar Y^{(k+1)}_{[0, t_L]}}\right),
\]
where we have used that $W_1(\mu_1,\mu_2) \leq diam(\mathcal X)\TV\del{\mu_1, \mu_2}$ for  any distributions $\mu_j, j\in\curly{1,2}$ supported on $\mathcal X$, the fact that 
for random variables $X, Y, Z$ we have
$$\expectation \TV(X, Y|Z) = \expectation \TV ( (X,Z), (Y,Z)),$$
as well as the fact that $y_{t_1:t_L} \rightarrow \bar e (\tau(L), y_{t_1:t_L})$ is measurable.

    Next, we bound the $\TV$ distance using first Pinsker's inequality and then Girsanov's theorem on the interval $[0,\tau_{(L)}]$. By construction of the processes $\bar Y^{(k)}$ %~\eqref{eq: Y^l definition} 
    we obtain,
    \[
    \TV^2\del{\bar{Y}^{(k)}_{[0,t_L]}, \bar Y^{(k+1)}_{[0,t_L]}}
    &\le 
    {\KL\del{\bar{Y}^{(k+1)}_{[0, t_L]}\,||\, \bar{Y}^{(k)}_{[0, t_L]}}}
    \\
    &\stackrel{\text{(Girsanov)}}{\le} \frac{1}{2} \mathbb E_u
    \int_{t_{i_k}}^{t_{i_{k+1}}}\expectation\norm{s(\tau_{(1)}-t, Y_{t})-\bar{s}(\tau_{(1)}-t, Y_{t}|\tau_{(i_k)},Y_{t_1:t_{i_k}})}^2\,dt. 
    \]
    
    Note that, unlike \cite{oko2023}, we do not have to introduce a truncated auxiliary process to ensure~(Girsanov), since Novikov's condition holds a.s.\ by construction of the process. 

    Finally, using \cite[Theorem 10]{potaptchik2024linearconvergencediffusionmodels}, for all $u_1, \cdots, u_K$, since expectation is with respect to $P_t$ 
    \begin{align*}
    &\int_{t_{i_k}}^{t_{i_{k+1}}}\expectation\norm{s(\tau_{(1)}-t, Y_{t})-\bar{s}(\tau_{(1)}-t, Y_{t}|\tau_{(i_k)},Y_{t_1:t_{i_k}})}^2\,dt \\
    &\qquad =        \sum_{i=i_k}^{i_{k+1}}\int_{t_{i}}^{t_{i+1}}\expectation\norm{s(\tau_{(1)}-t, Y_{t})-\bar{s}(\tau_{(1)}-t, Y_{t}|\tau_{(i_k)},Y_{t_1:t_{i_k}})}^2\,dt
        \\
    &\qquad \lesssim 
    \sum_{i=i_k}^{i_{k+1}}\expectation\norm{s(\tau_{(i)}, Y_{t_i})-\bar{s}(\tau_{(i)}, Y_{t_i})}^2 + \kappa + d\kappa^2 L (\log \underline{T}^{-1} + C_{\log})
    \\
    &\qquad \lesssim \sigma^{-2}_{\tau_{(i_k)}}(\varepsilon^2_{denoiser}(u) +\delta) + \kappa d(\log \underline{T}^{-1} + C_{\log})(\log {\underline T^{-1}} + \overline T)
    \end{align*}
    where to get the last inequality we used that $\tau_{(i)}/\tau_{(i_{k+1})} < 4$ for all $i\ge i_k$, the bound (6.3) in Proposition 6.2 %\eqref{eq:M_bound} 
    as well as (6.9).%\eqref{eq:denoiser_score_bar_s}.

    Substituting we get
    \[
    W_1(\bar Y^{(k)}_{\tau_{(M)}}, \bar{Y}^{(k+1)}_{\tau_{(M)}}) \le 
    4R_{\tau_{(i_k)}}\del{\frac{\underline{T}\delta}{2M^2}}\sqrt{\sigma^{-2}_{\tau_{(i_k)}}(\varepsilon^2_{denoiser}(u) + \delta) + \kappa d(\log \underline{T}^{-1} + C_{\log})(\log {\underline T^{-1}} + \overline T) }
    \\
    \lesssim
    C_{\mathrm{tail}} \log \underline{T}^{-1}\sqrt{C_W  + \log \frac{2M^2}{\underline{T}\delta}}\sqrt{\varepsilon^2_{denoiser}(u) + \delta + \kappa d(\log \underline{T}^{-1} + C_{\log})(\log {\underline T^{-1}} + \overline T) },
    \]
    where to get the last inequality we used (6.5) and that $\sigma_{t} \le 1$ for all $t$. %~\eqref{eq:R_in_W1_convergence}

    By taking expectation in $u$ and applying Cauchy-Schwarz inequality
    \£
    \label{eq:W1_result_3}
    \expectation_u W_1(\bar Y^{(k)}_{\tau_{(M)}}, \bar{Y}^{(k+1)}_{\tau_{(M)}})
    \lesssim
    \log \underline{T}^{-1}\sqrt{C_W  + \log \frac{2L^2}{\underline{T}\delta}}\sqrt{\varepsilon^2_{denoiser} + \delta + \kappa d(\log \underline{T}^{-1} + C_{\log})(\log {\underline T^{-1}} + \overline T) }
    \£
    \textbf{Combining the terms.}
    Overall combining~\eqref{eq:W_1_result_1},~\eqref{eq:W1_result_2} and~\eqref{eq:W1_result_3}
    \[
    \expectation_u W_1&(Y_{\tau_{(L)}}, \bar{Y}_{\tau_{(L)}}) \le     
    \expectation_u W_1({Y}_{\tau_{(L)}}, \bar{Y}^{(K+1)}_{\tau_{(L)}}) 
     +
    \sum_{k=1}^{K} 
    \expectation_u W_1(\bar{Y}^{(k)}_{\tau_{(L)}}, \bar{Y}^{(k+1)}_{\tau_{(L)}}) 
    + 
    \expectation_u W_1(\bar{Y}_{\tau_{(L)}}, \bar{Y}^{(1)}_{\tau_{(L)}})
    \\
    &\quad\lesssim
    +\sqrt{D}e^{-\overline{T}} + 
    \sqrt{\underline{T}}\log \underline{T}^{-1}\sqrt{C_W  + \log \frac{2L^2}{\underline{T}\delta}} + 2\frac{\underline{T}\delta}{L^2} 
    \\
    &\quad\quad
    +
    K
    \log \underline{T}^{-1}\sqrt{C_W  + \log \frac{2L^2}{\underline{T}\delta}}\sqrt{\varepsilon^2_{denoiser} + \delta + \kappa d(\log \underline{T}^{-1} + C_{\log})(\log {\underline T^{-1}} + \overline T) }
    \]
    Now we do the last substitutions, we note that $K = O(\log (\overline{T}-\underline{T}))$ and $L=O\del{\kappa^{-1}(\log {\underline T^{-1}} + \overline 
T)}$, so substituting $\kappa =\delta =  \varepsilon^2_{\mathrm{denoiser}}
$ we get the desired bound
\[
\expectation_u & W_1(Y_{\tau_{(1)}}, \bar{Y}_{\tau_{(1)}}) \lesssim
    \sqrt{D}e^{-\overline{T}} + 
    \sqrt{\underline{T}}\log \underline{T}^{-1}\sqrt{C_W + \log (\varepsilon^{-1}_{\mathrm{denoiser}} \cdot \underline{T}^{-1}\cdot \overline{T})}
    \\
    &\quad
    +
    \varepsilon_{\mathrm{denoiser}}
    \log (\overline{T} \cdot \underline{T}^{-1})
    \log \underline{T}^{-1}\sqrt{C_W + \log (\varepsilon^{-1}_{\mathrm{denoiser}} \cdot \underline{T}^{-1}\cdot \overline{T})}\sqrt{d(\log \underline{T}^{-1} + C_{\log})(\log {\underline T^{-1}} + \overline T) }.
\]

\subsection{Proof of Corollary~\ref*{cor:main_W_result}}
    Conditioned on event $\cE_{\mathrm{geom},k}$ that depends only on $\cY$, Theorem~\ref*{thm:score_approximation_1} gives Assumption~\ref*{asmp:score} with $\varepsilon_{\mathrm{denoiser}} = O\del{n^{-\frac{\alpha+1}{2\alpha+d}}\polylog n}$, $C_W\asymp \log n$ and $\eta=O\del{\varepsilon_{\mathrm{denoiser}}^{-4}\underline{T}^{-1}(\log \underline{T} + \overline{T}} = O\del{n^{\frac{6(\alpha+1)}{2\alpha+d}}\polylog n}$. Therefore, applying Theorem~\ref{thm:W_1_convergence_theorem} we get that given $\cE_{\mathrm{geom},k}$
    \[
    W_1(\mu,\hat{\mu}) \lesssim n^{-\frac{\alpha+1}{2\alpha+d}}\polylog n.
    \]
    When event $\cE_{\mathrm{geom},k}$ does not hold, which happens with probability at most $O\del{\polylog n \cdot n^{-\frac{\alpha+1}{2\alpha+d}}}$, we note that the last step of the sampling scheme has form 
    \[
    y_{t_{L}}\mapsto \bar{e}(\tau_{(L)}, y_{1:t_L}) = \hat{e}(\tau_{(i)}, y_i) \quad \text{for some } i\le L \text{ by~(6.9)}.
    \] 
    By construction, see~\eqref*{eq:definition of cS}, we have $\dist\del{M, \hat{e}(\tau_{(i)}, y_i)} \le C_2$, therefore for all $y\in M$
    \[
    \norm{y-\hat{e}(\tau_{(i)}, y_i)} \le C_2 +\diam M = O(1).
    \]
    So a.s. $ W_1(\mu,\hat{\mu}) = O(1)$, and as result, since $\rP(\cE_{\mathrm{geom},k}) = 1- O\del{\polylog n\cdot n^{-\frac{\alpha+1}{2\alpha+d}}}$
    \[
    \expectation_{\cY\sim \mu^{\otimes n}}W_1(\mu,\hat{\mu}) 
    &= \rP(\cE_{\mathrm{geom},k})\cdot O\del{\polylog n \cdot n^{-\frac{\alpha+1}{2\alpha+d}}} + (1-\rP(\cE_{\mathrm{geom},k}))\cdot O(1) 
    \\
    &=O\del{\polylog n \cdot n^{-\frac{\alpha+1}{2\alpha+d}}}.
    \]
    That finishes the proof of Corollary~\ref*{cor:main_W_result}.

\section{Auxiliary Results on Neural Networks}
     \label{apdx:nn_results}
     Below we list the results on the neural networks from \cite{oko2023} that were used to build the approximation in the Euclidean case.
\begin{lemma}[\cite{oko2023}, Lemma F.1, Concatenation of neural networks.]
\label{lemma:concatentation_of_neural_networks}
For any neural networks 
\[
\phi^1: \mathbb{R}^{d_1} \rightarrow \mathbb{R}^{d_2}, \phi^2: \mathbb{R}^{d_2} \rightarrow \mathbb{R}^{d_3}, \dots, \phi^k: \mathbb{R}^{d_k} \rightarrow \mathbb{R}^{d_{k+1}}
\]
with $\phi^i \in \Psi(L^i, W^i, S^i, B^i)$ $(i = 1, 2, \dots, d)$, there exists a neural network $\phi \in \Phi(L, W, S, B)$ satisfying $\phi(x) = \phi^k \circ \phi^{k-1} \circ \dots \circ \phi^1(x)$ for all $x \in \mathbb{R}^{d_1}$, with
\[
L = \sum_{i=1}^{k} L^i, \quad W \leq 2 \sum_{i=1}^{k} W^i, \quad S \leq \sum_{i=1}^{k} S^i + \sum_{i=1}^{k-1} (\|A^i_{L_i}\|_0 + \|b^i_{L_i}\|_0 + \|A^{i+1}_1\|_0) \leq 2 \sum_{i=1}^{k} S^i, 
\]
and $B \leq \max_{1 \leq i \leq k} B^i$.

Here $A_j^i$ is the parameter matrix and $b_j^i$ is the bias vector at the $j$th layer of the $i$th neural network $\phi^i$.
\end{lemma}

\begin{lemma}[\cite{oko2023},~Lemma F.3, Parallelization of neural networks.] 
\label{lemma:parallelization_of_neural_networks}
    
    For any neural networks $\phi^1, \phi^2, \dots, \phi^k$ with $\phi^i: \mathbb{R}^{d_i} \rightarrow \mathbb{R}^{d'_i}$ and 
\[
\phi^i \in \Psi(L^i, W^i, S^i, B^i) \quad (i = 1, 2, \dots, d),
\]
there exists a neural network $\phi \in \Psi(L, W, S, B)$ satisfying
\[
\phi(x) = \left[\phi^1(x_1)^\top \ \phi^2(x_2)^\top \ \cdots \ \phi^k(x_k)^\top\right]^\top: \mathbb{R}^{d_1 + d_2 + \dots + d_k} \rightarrow \mathbb{R}^{d'_1 + d'_2 + \dots + d'_k}
\]
for all $x = (x_1^\top \ x_2^\top \ \cdots \ x_k^\top)^\top \in \mathbb{R}^{d_1 + d_2 + \dots + d_k}$ (here $x_i$ can be shared), with
\[
L = L_i, \quad \|W\|_\infty \leq \sum_{i=1}^k \|W^i\|_\infty, \quad S \leq \sum_{i=1}^k S^i, \quad \text{and} \quad B \leq \max_{1 \leq i \leq k} B^i \quad (\text{when } L = L_i \text{ holds for all } i),
\]
\[
L = \max_{1 \leq i \leq k} L^i, \quad \|W\|_\infty \leq 2 \sum_{i=1}^k \|W^i\|_\infty, \quad S \leq 2 \sum_{i=1}^k (S^i + LW^i_L), \quad \text{and} \quad B \leq \max\{\max_{1 \leq i \leq k} B^i, 1\} \quad (\text{otherwise}).
\]

Moreover, there exists a network $\phi_{\operatorname{sum}}(x) \in \Phi(L, W, S, B)$ that realizes $\sum_{i=1}^k \phi^i(x)$, with
\[
L = \max_{1 \leq i \leq k} L^i + 1, \quad \|W\|_\infty \leq 4 \sum_{i=1}^k \|W^i\|_\infty, \quad S \leq 4 \sum_{i=1}^k (S^i + LW^i_L) + 2WL, \quad \text{and} \quad B \leq \max\{\max_{1 \leq i \leq k} B^i, 1\}.
\]
\end{lemma}

\begin{lemma}[\cite{oko2023},~Lemma F.6, Approximation of monomials]
\label{lemma:oko_nn_approximating_polynomial}
Let $d \geq 2$, $C \geq 1$, $0 < \epsilon_{\text{error}} \leq 1$. For any $\epsilon > 0$, there exists a neural network $\phi_{\operatorname{mult}}(x_1, x_2, \dots, x_d) \in \Psi(L, W, S, B)$ with 
\[
L = \mathcal{O}(\log d (\log \epsilon^{-1} + d \log C)), \quad \|W\|_\infty = 48d, \quad S = \mathcal{O}(d \log \epsilon^{-1} + d \log C), \quad B = C^d
\]
such that
\[
\left| \phi_{\text{mult}}(x^1, x^2, \dots, x^d) - \prod_{d' = 1}^{d} x_{d'} \right| \leq \epsilon + d C^{d-1} \epsilon_{\text{error}}, \quad \text{for all } x \in [-C, C]^d \text{ and } x' \in \mathbb{R} \text{ with } \|x - x'\|_\infty \leq \epsilon_{\text{error}}.
\]

\end{lemma}

\begin{lemma}[
\cite{oko2023},~Lemma F.7, Approximating the reciprocal function]
\label{lemma:oko_nn_appproximating_reciprocal_function}
For any $0 < \epsilon < 1$, there exists $\phi_{\text{rec}} \in \Psi(L, W, S, B)$ with $L \leq O(\log^2 \epsilon^{-1})$, $\|W\|_{\infty} = O(\log^3 \epsilon^{-1})$, $S = O(\log^4 \epsilon^{-1})$, and $B = O(\epsilon^{-2})$ such that
\[
\left| \phi_{\text{rec}}(x') - \frac{1}{x} \right| \leq \epsilon + \frac{|x' - x|}{\epsilon^2}, \quad \text{for all } x \in [\epsilon, \epsilon^{-1}] \text{ and } x' \in \mathbb{R}.
\]
\end{lemma}

\begin{lemma}[\cite{oko2023},~Section B.1]
    Constants $c_t$, and $\sigma_t$ with an error less than $\varepsilon$ can be computed by neural networks $\phi_c(t), \phi_\sigma(t)$ of $\polylog \varepsilon^{-1}$ size.
\end{lemma}

\begin{proposition}
    \label{prop: max as nn}
    Let $x_1,x_2,\ldots, x_n \in R$, then there are neural networks $\phi_{\max}, \phi_{\min}\in \Psi(L,B,S,W)$ where $L = O\del{\log n}$, $\norm{W}_\infty, S = O(n)$ and $B = O(1)$ such that
    \[
    \phi_{\max}(x_1,\ldots, x_n) &= \max\del{x_1,\ldots, x_n},
    \\
    \phi_{\min}(x_1,\ldots, x_n) &= \min\del{x_1,\ldots, x_n}.
    \]
\end{proposition}
\begin{proof}
Since $
    \min\del{x_1,\ldots, x_n} = - \max\del{-x_1,\ldots, -x_n}
    $
    it is enough to build $\phi_{\max}$.
    Let $a,b \in \R$, then
    \[
    \max\del{a,b} = a + \max\del{0, b-a}  = \relu\del{a} + \relu\del{-a} + \relu\del{a-b}.
    \]
    So, $\max$ can be implemented as $1$-layer network of constant size. Representing
    \[
    \max\del{x_1,\ldots,x_n} = \max\del{\max\del{x_1,\ldots,x_{n/2}}, \max\del{x_{n/2+1},\ldots, x_n}} 
    \]
    we get the required statement by induction.
\end{proof}

    \section{Bounds on Tangent Spaces}
    \label{sec:bounds_on_tangent_spaces}
    \subsection{Projection on Tangent Space as a Denoiser}
    In this section, we prove the auxiliary results used to build approximation for $T_k < n^{-\frac{1}{2\alpha+d}}$.
   \begin{proposition}\label{prop:points_z_as_projections_on_tangent_space}
    Work on the geometric event where the local chart approximation bounds
    from Section~\ref{sec:support_estimation} hold.  Let
    $X_t=c_tX_0+\sigma_tZ_D$ and, for a fixed index $i$, assume
    $X_0\in \Phi_i(B_d(0,5\varepsilon_N))$.  Define
    \[
        z_i^*(X_t):=c_t^{-1}(P_i^*)^T(X_t-c_tG_i),
        \qquad
        m_i(X_t):=\Phi_i(z_i^*(X_t)).
    \]
    Then for every $\delta\in(0,1)$,
    \[
        \rP_{Z_D}\del{
        \norm{m_i(X_t)-X_0}
        \le
        C\varepsilon_N^\beta
        +2(\sigma_t/c_t)\sqrt{d+2\log \delta^{-1}}
        }
        \ge 1-\delta.
    \]
    Consequently, with probability at least $1-\delta$, the same bound
    holds simultaneously for all $i\le N$ after replacing the square root
    by $\sqrt{d+\log n+2\log \delta^{-1}}$.  In the small-time regime
    used in Appendix~\ref{apdx:score_apprxoximation_by_neural_network},
    this displacement is $o(\varepsilon_N)$, so for $n$ large enough
    $z_i^*(X_t)$ remains in the local chart domain.
    \end{proposition}
        \begin{proof}
        Fix $i$ and write $X_0=\Phi_i(z_0)$ with
        $z_0\in B_d(0,5\varepsilon_N)$.  Since
        $\Phi_i^*(z)-G_i-P_i^*z\perp \Im P_i^*$ and $P_i^*$ is an
        isometric embedding,
        \[
        (P_i^*)^T(c_t\Phi_i(z_0)+\sigma_tZ_D-c_tG_i)
        =
        c_t(P_i^*)^T(\Phi_i(z_0)-\Phi_i^*(z_0))
        +c_tz_0+\sigma_t(P_i^*)^TZ_D .
        \]
        The chart approximation gives
        $\norm{\Phi_i(z_0)-\Phi_i^*(z_0)}\le C\varepsilon_N^\beta$.
        Therefore
        \[
        \norm{z_i^*(X_t)-z_0}
        \le
        C\varepsilon_N^\beta
        +(\sigma_t/c_t)\norm{(P_i^*)^TZ_D}.
        \]
        Since $(P_i^*)^TZ_D\sim\cN(0,\Id_d)$, the standard Gaussian norm
        tail bound implies
        \[
        \norm{z_i^*(X_t)-z_0}
        \le
        C\varepsilon_N^\beta
        +(\sigma_t/c_t)\sqrt{d+2\log\delta^{-1}}
        \]
        with probability at least $1-\delta$.  The map $\Phi_i$ is
        $2$-Lipschitz on the relevant ball, so
        \[
        \norm{m_i(X_t)-\Phi_i(z_0)}
        \le
        C\varepsilon_N^\beta
        +2(\sigma_t/c_t)\sqrt{d+2\log\delta^{-1}}.
        \]
        Since $\Phi_i(z_0)=X_0$, this proves the fixed-index claim on the
        event where $z_i^*(X_t)$ is in the local chart domain.  The final
        sentence of the proposition verifies this condition in the
        small-time application.
        Applying the same argument with $\delta/N$ and using $N\le n$
        gives the simultaneous version.
    \end{proof}

    \subsection{Correlation Between Tangent Vectors and Gaussian Noise}
    \label{apdx:Correlation Between Tangent Vectors and Gaussian Noise}
    The following lemma is the second order counterpart of \cref{prop:maximum_of_Gaussian_Process}.
    \begin{lemma}
    \label{lemma:correlation_of_noise_and_tangent_space_local}
        Let $f(z):B_d(0,\varepsilon)\mapsto 
\R^D$ satisfy $\norm{f}_{C^\beta\del{B_d(0,\varepsilon)}} \le L$ for $\beta \ge 2$ and let $Z_D \simeq \cN\del{0,\Id_D}$. Let $h(z_1,z_2) = f(z_1) - f(z_2) - \grad f(z_2)(z_1-z_2)$ -- an error at $z_1$ of the linear approximation of $f$ at point $z_2$, then 
\[
\rP\del{\sup_{z_1,z_2\in B(0,\varepsilon)}\abs{\innerproduct{Z_D}{h(z_1,z_2)}} \le L\varepsilon^2\del{4\sqrt{d} + \sqrt{2\log \delta}}} \ge 1-\delta, 
\]
    \end{lemma}
\begin{proof}
    The function $h(z_1,z_2):B_d(0,\varepsilon)\times B_d(0,\varepsilon) \subset B_{2d}(0, 2\varepsilon) \mapsto \R^D$. Since $h(0,0) = 0$ to apply \cref{prop:maximum_of_Gaussian_Process} it is enough to estimate the Lipschitz constant of $h(z_1,z_2)$.

    First, let us prove an auxiliary bound on $\norm{h}_{\infty}$. Since $\norm{f}_{C^\beta\del{B_d(0,\varepsilon)}} \le L$ integrating along the segment connecting $z_1$ and $z_2$
    \[
    h(z_1,z_2) = f(z_1) - f(z_2) - \grad f(z_2)(z_1-z_2) = \int_0^1 \innerproduct{\grad f(az_1+(1-a)z_2) - \grad f(z_2)}{z_1-z_2}da, 
    \]
    note $B(0,\varepsilon)$ is a convex set, so
    \£
    \label{eq:norm_of_hz1z2}
    \norm{h(z_1,z_2)}\le  \sup_{a\in [0,1]} \norm{\grad f(az_1+(1-a)z_2) - \grad f(z_2)}_{op}\norm{z_1-z_2} \le L\varepsilon^2. 
    \£
    
    To bound the Lipschitz constant we introduce another pair of points $z'_1, z'_2 \in B(0,\varepsilon)$ and estimate
    \[
        \norm{h(z_1,z_2) - h(z'_1,z'_2)} = \norm{h(z_1,z_2) - h(z'_1,z_2) + h(z'_1,z_2) - h(z'_1, z'_2)}
    \\
        \le 
        \norm{h(z_1,z_2) - h(z'_1,z_2)} + \norm{h(z'_1,z_2) - h(z'_1, z'_2)}.
    \]
    We deal with each term separately. We represent the first term as
    \[
        h(z_1,z_2) - h(z'_1,z_2) &= f(z_1)-f(z'_1) - \grad f(z_2)(z_1-z'_1) 
        \\
        &= \int_0^1 \innerproduct{\grad f\del{a z_1 + (1-a)z'_1}-\grad f(z_2)}{z_1-z'_1}da,
    \]
    so
    \£
    \label{eq:bound_grad_1}
    \norm{h(z_1,z_2) - h(z'_1,z_2)} \le \norm{z_1-z'_1}\int_0^1 \norm{\grad f\del{a z_1 + (1-a)z'_1}-\grad f(z_2)}_{op}da \le L\varepsilon\norm{z_1-z'_1}.
    \£
    The second term $h(z'_1,z_2)-h(z'_1,z'_2)$ is represented in a similar way
    \[
        h(z'_1,z_2) - h(z'_1,z'_2) = \int_0^1 \innerproduct{\grad_{z_a} h\del{z'_1, z_a}}{z_2-z'_2}da,
    \]
    where $z_a = a z_2 + (1-a)z'_2$. The gradient can be computed directly
    \[
    \grad_{z_2}h(z_1,z_2) = \grad_{z_2} \del{f(z_2) - \grad_{z_2} f(z_2)(z_2-z_1)} = -\del{\grad_{z_2}^2 f(z_2)} (z_2-z_1),
    \]
    so we get a bound
    \£
    \label{eq:bound_grad_2}
    \norm{h(z'_1,z_2) - h(z'_1,z'_2)} \le \norm{z_2-z'_2}\int_0^1 \norm{\grad^2 f(z_a)}_{op} \norm{z_a-z'_1}da \le L\varepsilon \norm{z_2-z'_2}.
    \£
    Summing \eqref{eq:bound_grad_1} and \eqref{eq:bound_grad_2} up we conclude
    \[
    \norm{h(z_1,z_2) - h(z'_1,z'_2)} \le L\varepsilon\del{\norm{z_2-z_2'} + \norm{z_1-z'_1}} \le L\sqrt{2}\norm{(z_1,z_2) - (z'_1,z'_2)}.
    \]
    Applying \cref{prop:maximum_of_Gaussian_Process} we derive the lemma.
\end{proof}
\begin{corollary}
\label{corr:correlation_of_noise_and_tangent_space}
    Let $\Phi_1^*,\ldots, \Phi^*_N :B_d(0,\varepsilon)\mapsto \R^D$ be $L^*$-lipschitz functions then for any positive $\delta < 1$ with probability at least $1-\delta$ for all $i\le N$
    \[
    \sup_{z_1,z_2\in B(0,\varepsilon)}\abs{\innerproduct{Z_D}{\Phi_i^*(z_1)-\Phi_i^*(z_2) - \grad\Phi^*_i(z_1)(z_1-z_2)}} \le 2L^*\varepsilon^2\del{4\sqrt{d} + \sqrt{2\log N + 2\log \delta^{-1}}}.
    \]
\end{corollary}
\begin{proof}
    Apply \cref{lemma:correlation_of_noise_and_tangent_space_local} for each $\Phi^*_i$ with $\delta/N$ and combine the inequalities.
\end{proof}
The next lemma is a local bound of the projection on the tangent space of $Z_D \sim \cN\del{0,\Id}$. 
\begin{lemma}
\label{lemma:bound_on_length_of_proj_on_tangent_space}
    For function $f:\R^d\mapsto \R^D$ s.t.\ $f(0)=0$ and $\grad^T f(0) \grad f(0) =\Id_d$ and point $z$ we define 
\[
    \pr_z = \grad f(z)\del{\grad f^T(z)\grad f(z)}^{-1}\grad^T f(z),
\]
the projection map onto tangent to the image of $f$ at $f(z)$. Let $\norm{f}_{C^\beta(B(0,\varepsilon)} \le L$, where $\beta \ge 3$ then for any $\varepsilon < L^{-1}/8$
\[
    \rP\del{\sup_{z\in B(0,\varepsilon)}\norm{\pr_z Z_D} \le 4dL\varepsilon + 4\sqrt{d}(1+L\varepsilon)\sqrt{2\log 4d + 2\log \delta^{-1}}} \ge 1-\delta
\]
\end{lemma}
\begin{proof}
    We note that (ii) and (iii) in \cref{prop:geometric_statements} holds for any $f$ satisfying $\norm{f}_{C^\beta(B(0,\varepsilon)} \le L$ and $\grad^T f(0) \grad f(0) =\Id_d$, so we have: (i) $\norm{\grad f(z)}_{op} < 2$; (ii) $
    \norm{\grad^T f(z) \grad f(z) - \Id_d}_{op} \le 1/2$. The latter implies $\norm{\del{\grad^T f(z) \grad f(z)}^{-1}}_{op} \le 2$ and we conclude
    \[
    \norm{\pr_z Z_D} &= \norm{\grad f(z)\del{\grad f^T(z)\grad f(z)}^{-1}\grad^T f(z)Z_D} 
    \\
    &\le 2\norm{\grad f(z)}_{op}\norm{\grad^T f(z)Z_D} \le 4\norm{\grad^T f(z)Z_D}.
    \]
    So, it is enough to prove that
    \[
    \rP\del{\sup_{z\in B(0,\varepsilon)}\norm{\grad^T f(z) Z_D} \le dL\varepsilon + \sqrt{d}(1+L\varepsilon)\sqrt{2\log 2d + 2\log \delta^{-1}}} \ge 1-\delta.
    \]
    The vector $\grad^T f(z)Z_D \in \R^d$, and we bound it coordinate-wise. Let $v_1,\ldots, v_d$ be an orthonormal basis in $\R^d$ and $f_v = \frac{d}{dv} f(z)$ then $\grad^T f(z) = \del{f^T_{v_1}(z),\ldots, f^T_{v_d}(z)}^T \in \R^{D\times d}$. 
     For $z,z'\in B(0,\varepsilon)$  using that $\norm{f}_{C^\beta(0,\varepsilon)} \le L$ we bound 
    $
    \norm{f_{v_i}(z) - f_{v_i}(z')} \le L\norm{z-z'}.
    $
    Applying \cref{prop:maximum_of_Gaussian_Process} we get
    \[
    \rP\del{\sup_{z\in B(0,\varepsilon)}\abs{\innerproduct{Z_D}{f_{v_i}(z)-f_{v_i}(0)}} \le L\varepsilon\del{\sqrt{d} +\sqrt{2\log (4d\delta^{-1})}}} \ge 1-\delta/2d.
    \]
    Since $\grad^T f(0)\grad f(0) = \Id_d$, we note that $\norm{f_{v_i}(0)}=1$. So $\innerproduct{Z_D}{v_i}\sim \cN\del{0,1}$ and
    \[
    \rP\del{\abs{\innerproduct{Z_D}{f_{v_i}(0)}} \le \sqrt{2\log (4d\delta^{-1})}} \ge 1-\delta/2d.
    \]
    Summing up
    \[
    \rP\del{\sup_{z\in B(0,\varepsilon)}\abs{\innerproduct{Z_D}{f_{v_i}(z)}} \le L\varepsilon\sqrt{d} +(1+L\varepsilon)\sqrt{2\log (4d\delta^{-1})}} \ge 1-\delta/d.
    \]
    Combining inequalities for all $v_i$ 
    \[
    \rP\del{\forall i: \sup_{z\in B(0,\varepsilon)}\abs{\innerproduct{Z_D}{f_{v_i}(z)}} \le L\varepsilon\sqrt{d} + (1+L\varepsilon)\sqrt{2\log 4d + 2\log \delta^{-1}}} \ge 1-\delta.
    \]
    Finally the inequality
    \[
    \norm{\grad^T f(z) Z_D}^2 = \sum_{i=1}^d \abs{\innerproduct{Z_D}{f_{v_i}(z)}}^2 \le d\cdot \sup_{i} \abs{\innerproduct{Z_D}{f_{v_i}(z)}}^2
    \]
    gives 
    \[
    \rP\del{\sup_{z\in B(0,\varepsilon)}\norm{\grad^T f(z) Z_D} \le dL\varepsilon + \sqrt{d}(1+L\varepsilon)\sqrt{2\log 4d + 2\log \delta^{-1}}} \ge 1-\delta.
    \]
\end{proof}
\begin{proof}[Proof of Proposition~\ref*{prop:projection_onto_tangent_space}]
    Let us take $\varepsilon<r_0$-dense set $\cG=\curly{G_1,\ldots, G_N}$ by Proposition~\ref*{prop:covering_number_of_M} we can guarantee $N \le (\varepsilon/2)^{-d}\Vol M$.
    Applying Lemma~\ref{lemma:bound_on_length_of_proj_on_tangent_space} with $\delta/N$ and $f(z) = \Phi_{G_i}(z)$
    \[
    \rP\del{\sup_{z\in B(0,\varepsilon)}\norm{\pr_z Z_D} \le 4dL_M\varepsilon + 4\sqrt{d}(1+L_M\varepsilon)\sqrt{2\log 2d + 2\log N + 2\log \delta^{-1}}} \ge 1-\delta/N.
    \]
    We note that since locally $M$ is a graph of $\Phi_{G_i}$ we have $\pr_{z} \equiv \pr_{\Phi_{G_i}(z)}M$, so combining over all $i\le N$, since $M = \bigcup_i \Phi_{G_i}\del{B_d(0,\varepsilon)}$ we conclude
    \[
    \rP\del{\sup_{y\in M}\norm{\pr_{T_yM} Z_D} \le 4dL_M\varepsilon + 4\sqrt{d}(1+L_M\varepsilon)\sqrt{2\log 2d + 2\log \varepsilon^{-1} + 2\log_+\Vol M + 2\log \delta^{-1}}} 
    \\
    \ge 1-\delta.
    \]
    Taking $\varepsilon = (2d)^{-1} r_0$
    \[
    \rP\del{\sup_{y\in M}\norm{\pr_{T_yM} Z_D} \le 8\sqrt{d}\sqrt{4\log 2d + 2\log r^{-1}_0 + 2\log_+\Vol M + 2\log \delta^{-1}}} \ge 1-\delta.
    \]
    we finish the proof.
\end{proof}
The following result is a counterpart of Lemma above for $M^*$. 
\begin{lemma} %{restatable}
\label{lemma:global_bound_on_projection_on_tangent_space} For all $\delta > 0$ with probability at least $1-\delta$ for all $i \le I_2$ 
    \begin{align*}
       \norm{\pr_{z_i^*} Z_D}
    \le 8dL_M\varepsilon_N + 8\sqrt{d}(1+7L_M\varepsilon_N)\sqrt{2\log 2d + 2\log n + 2\log( 1/\delta)}
    \end{align*}
  \end{lemma}  
\begin{proof}
    Note that by Lemma~\ref*{lemma:approximation_of_Phi} for $n$ large enough $\norm{\Phi^*_i}_{C^2\del{B_d(0,7\varepsilon_N)}} \le L_M + L^*(7\varepsilon)^{\beta-2}_N\le 2L_M$ for all $i\le N$. %\textcolor{red}{Also for each $i $ denote $pr_z = \nabla \Phi_i^*(z) (\nabla \Phi_i^*(z)^T\nabla \Phi_i^*(z))^{-1}\nabla \Phi_i^*(z)^T$. }
    Then 
    $\norm{\pr_{z_i^*} Z_D} \le \sup_{z\in B_d(0, 7 \varepsilon_N)}\norm{\pr_{z} Z_D}$
    Apply \cref{lemma:bound_on_length_of_proj_on_tangent_space} for each $\Phi^*_1, \ldots, \Phi^*_N$ with $\delta/N$ and combine the resulting inequalities using that $N \le n$. 
\end{proof}

  \begin{lemma}
       \label{lemma:BoundsOnScalarProductSmallTime}
           For all $i\in I_2$ and $z\in B_d\del{z^*_i, 4(\sigma_t/c_t)\sqrt{20}C(n)}$ with probability at least $1-n^{-2}$
           \[
           c^2_t\norm{R_i(z,z^*_i)}^2 + c_t\innerproduct{X(t)- c_t\Phi^*_i(z^*_i)}{R_i(z,z^*_i)} \lesssim \sigma_t^3 (d\log n)^2.
           \]
       \end{lemma}
   
 \begin{proof}
    By~\cref{prop:points_z_as_projections_on_tangent_space} and choice of $z$ with probability at least $1-n^{-2}$
    \[
    &\norm{X(0)-\Phi(z^*_i)} \le 2(\sigma_t/c_t)\sqrt{20}C(n) \lesssim (\sigma_t/c_t)\sqrt{d\log n},
    \\
    &\norm{z-z^*_i} \le 4(\sigma_t/c_t)\sqrt{20}C(n) \lesssim (\sigma_t/c_t)\sqrt{d\log n},
    \]
    since $C(n) \lesssim \sqrt{d\log n}$.
     We estimate the terms separately. Recall that 
     \[
     R_i(z,z^*_i) = \Phi^*_i(z^*_i) - \Phi^*_i(z) - \grad \Phi^*_i(z^*_i)(z^*_i-z),
     \]
     so by~\eqref{eq:norm_of_hz1z2} for $n$ large enough
     \£
     \label{eq:bound_on_R_i(z,z^*_i)}
     c^2_t\norm{R_i(z,z^*_i)}^2 \le \del{2L^* \norm{z-z^*_i}^2}^2 \lesssim (\sigma_t/c_t)^4 d^2 \log^2 n < \sigma_t^2.
     \£
     We move the second term
     \[
     \innerproduct{X(0)- c_t\Phi^*_i(z^*_i)}{R_i(z,z^*_i)} = c_t\innerproduct{X(0)- c_t\Phi^*_i(z^*_i)}{R_i(z,z^*_i)} + \sigma_t c_t\innerproduct{Z_D}{R_i(z,z^*_i)}.
     \]
     The first term is bounded similarly using~\eqref{eq:norm_of_hz1z2}
     \[
     c_t\abs{\innerproduct{X(0)- c_t\Phi^*_i(z^*_i)}{R_i(z,z^*_i)}} \le c_t\norm{X(0)- c_t\Phi^*_i(z^*_i)}\norm{R_i(z,z^*_i)} \lesssim c_t\del{(\sigma_t/c_t)\sqrt{d\log n}}^3. 
     \]
     While the second term by~\cref{corr:correlation_of_noise_and_tangent_space} with probability at least $1-n^{-2}$ is bounded by
     \[
     \abs{\sigma_t c_t\innerproduct{Z_D}{R_i(z,z^*_i)}} \lesssim \sigma_t c_t \del{(\sigma_t/c_t) \sqrt{d\log n}}^2 \lesssim c^{-1}_t\sigma^3_t d\log n .
     \]
     Summing the inequalities up and noting that $1 > c_t \ge 1/2$ we finish the proof.
 \end{proof}

\end{document}